\documentclass[a4paper]{nusthesis}

\dsp 

\usepackage[utf8]{inputenc}
\usepackage[english]{babel}
\usepackage{csquotes}

\usepackage{indentfirst} 

\usepackage{silence} 
\usepackage{bbding}
\usepackage{float}
\usepackage{algpseudocode}
\usepackage{multirow}
\usepackage{booktabs}
\usepackage{placeins}
\usepackage{subcaption}
\usepackage{bm}
\usepackage{mathtools}
\usepackage{pgfplots}
\pgfplotsset{compat=1.7}
\usepackage{numprint}
\npthousandsep{,} 
\usepackage[framemethod=TikZ,nobreak=true]{mdframed}

\definecolor{Lavender}{RGB}{248, 233, 236}

\usepackage{bookmark}
\usepackage[
defernumbers=true,
backref=true,
  backend=biber,
  citestyle=numeric,
  giveninits=false,
  sorting=none,
  maxbibnames=99,
style=numeric-comp,
  doi=false]{biblatex}
\DeclareFieldFormat*{title}{``#1''\newunitpunct} 
\usepackage{microtype} 
\usepackage{tabulary} 
\usepackage{hhline}

\usepackage{enumitem}

\usepackage{hyperref}

\usepackage{listings}

\usepackage{amsthm}
\theoremstyle{plain}

\newtheorem{lemma}{Lemma}

\usepackage{mathtools}
\usepackage{bbm}
\DeclarePairedDelimiter{\ceil}{\lceil}{\rceil}

\usepackage{interval}
\intervalconfig{soft open fences}

\usepackage[linesnumbered,ruled,vlined,algochapter]{algorithm2e}
\SetAlgorithmName{Algorithm}{Algorithm}{Algorithm}
\IncMargin{0.5em}
\SetCommentSty{textnormal}
\SetNlSty{}{}{:}
\SetAlgoNlRelativeSize{0}
\SetKwInput{KwGlobal}{Global}
\SetKwInput{KwPrecondition}{Precondition}
\SetKwProg{Proc}{Procedure}{:}{}
\SetKwProg{Func}{Function}{:}{}
\SetKw{And}{and}
\SetKw{Or}{or}
\SetKw{To}{to}
\SetKw{DownTo}{downto}
\SetKw{Break}{break}
\SetKw{Continue}{continue}
\SetKw{SuchThat}{\textit{s.t.}}
\SetKw{WithRespectTo}{\textit{wrt}}
\SetKw{Iff}{\textit{iff.}}
\SetKw{MaxOf}{\textit{max of}}
\SetKw{MinOf}{\textit{min of}}
\SetKwBlock{Match}{match}{}{}

\usepackage{array}
\newcolumntype{L}[1]{>{\raggedright\let\newline\\\arraybackslash\hspace{0pt}}m{#1}}
\newcolumntype{C}[1]{>{\centering\let\newline\\\arraybackslash\hspace{0pt}}m{#1}}
\newcolumntype{R}[1]{>{\raggedleft\let\newline\\\arraybackslash\hspace{0pt}}m{#1}}

\usepackage{color}
\usepackage{xcolor}

\hypersetup{
    colorlinks,
    linkcolor={red!50!black},
    citecolor={blue!50!black},
    urlcolor={blue!80!black}
}

\usepackage{soul}
\soulregister\cite7
\soulregister\ref7
\soulregister\pageref7
\soulregister\autoref7
\soulregister\eqref7

\renewcommand{\eqref}{Equation~\ref}

\newcommand*{\RootPicDir}{pic}
\newcommand*{\PicDir}{\RootPicDir}

\newcommand*{\SetPicSubDir}[1]{\renewcommand*{\PicDir}{\RootPicDir /#1}}

\newcommand*{\RootExpDir}{exp}
\newcommand*{\ExpDir}{\RootExpDir}

\newcommand*{\SetExpSubDir}[1]{\renewcommand*{\ExpDir}{\RootExpDir /#1}}

\usepackage{lipsum} 

\newcounter{theo}[chapter] 
\renewcommand{\thetheo}{\arabic{chapter}.\arabic{theo}}
\newenvironment{theo}[1][]{~\\%
\interlinepenalty=10000
\refstepcounter{theo}%
\ifstrempty{#1}%
{\mdfsetup{%
frametitle={%
\tikz[baseline=(current bounding box.east),outer sep=0pt]
\node[anchor=east,rectangle,fill=green!20]
{\strut Theorem~\thetheo};}}
}%
{\mdfsetup{%
frametitle={%
\tikz[baseline=(current bounding box.east),outer sep=0pt]
\node[anchor=east,rectangle,fill=green!20]
{\strut Theorem~\thetheo:~#1};}}%
}%
\mdfsetup{innertopmargin=10pt,linecolor=green!20,%
linewidth=2pt,topline=true,%
frametitleaboveskip=\dimexpr-\ht\strutbox\relax
}
\begin{mdframed}[]\relax%
}{\end{mdframed}}
\newcounter{lem}[chapter] 
\renewcommand{\thelem}{\arabic{chapter}.\arabic{lem}}

\newcounter{prob}[chapter]
\renewcommand{\theprob}{\arabic{chapter}.\arabic{prob}}

\newcounter{obs}[chapter]
\renewcommand{\theobs}{\arabic{chapter}.\arabic{obs}}

\newcounter{defn}[chapter]
\renewcommand{\thedefn}{\arabic{chapter}.\arabic{defn}}
\newenvironment{defn}[1][]{~\\%
\interlinepenalty=10000
\refstepcounter{defn}%
\ifstrempty{#1}%
{\mdfsetup{%
frametitle={%
\tikz[baseline=(current bounding box.east),outer sep=0pt]
\node[anchor=east,rectangle,fill=magenta!20]
{\strut Definition~\thedefn};}}
}%
{\mdfsetup{%
frametitle={%
\tikz[baseline=(current bounding box.east),outer sep=0pt]
\node[anchor=east,rectangle,fill=magenta!20]
{\strut Definition~\thedefn:~#1};}}%
}%
\mdfsetup{innertopmargin=10pt,linecolor=magenta!20,%
linewidth=2pt,topline=true,%
frametitleaboveskip=\dimexpr-\ht\strutbox\relax
}
\begin{mdframed}[]\relax%
}{\end{mdframed}}
\newcounter{intu}[chapter]

\newcounter{prop}[chapter] 
\renewcommand{\theprop}{\arabic{chapter}.\arabic{prop}}

\begin{document}

\makeatletter
\let\blx@rerun@biber\relax
\makeatother



\newcommand{\midas}{\textsc{Midas}}
\newcommand{\sess}{\textsc{SESS}}
\newcommand{\tlp}{Temporal Label Propagation}
\newcommand{\asteq}{\mathrel{*}=}
\newcommand{\mstream}{\textsc{MStream}}
\newcommand{\memstream}{\textsc{MemStream}}
\newcommand{\sedanspot}{\textsc{SedanSpot}}
\newcommand{\methodedge}{\textsc{AnoEdge}}
\newcommand{\methodgraph}{\textsc{AnoGraph}}
\newcommand{\densestream}{\textsc{DenseStream}}
\newcommand{\densealert}{\textsc{DenseAlert}}
\newcommand{\spotlight}{\textsc{SpotLight}}
\newcommand{\anomrank}{\textsc{AnomRank}}
\newcommand{\augsplicing}{\textsc{AugSplicing}}
\newcommand{\rcf}{Random Cut Forest}
\newcommand{\iso}{Isolation Forest}
\newcommand{\elliptic}{Elliptic Envelope}
\newcommand{\lof}{Local Outlier Factor}
\newcommand{\system}{{GraphAnoGAN}}
\newcommand{\anomalous}{\textsc{Anomalous}}
\newcommand{\dominant}{\textsc{Dominant}}
\newcommand{\sskn}{\textsc{SSKN}}
\newcommand{\amen}{\textsc{AMEN}}

\newcommand{\acm}{\textsc{ACM}}
\newcommand{\generator}{\mathbb{G}}
\newcommand{\D}{\mathbb{D}}
\newcommand{\blogc}{\textsc{BlogC}}
\newcommand{\darpa}{\textsc{DARPA}}
\newcommand{\enron}{\textsc{Enron}}

\newcommand{\EE}{\mathbb{E}} 
\newcommand{\one}{\mathbbm{1}}  

\newtheorem{observation}{Observation}
\newtheorem{iproblem}{Informal Problem}[chapter]
\newtheorem{problem}{Problem}[chapter]
\newtheorem{definition}{Definition}[chapter]
\newtheorem{theorem}{Theorem}[chapter]
\newtheorem{proposition}{Proposition}[chapter]
\newenvironment{conditions}[1][where:]
  {#1 \begin{tabular}[t]{>{}l<{} @{{}={}} l}}
  {\end{tabular}\\[\belowdisplayskip]}

\newcommand*\BitAnd{\mathbin{\&}}
\newcommand*\BitOr{\mathbin{|}}
\newcommand*\ShiftLeft{\ll}
\newcommand*\ShiftRight{\gg}
\newcommand*\BitNeg{\ensuremath{\mathord{\sim}}}

\newcommand\mycommfont[1]{\footnotesize\ttfamily\textcolor{blue}{#1}}
\SetCommentSty{mycommfont}

\renewcommand{\vec}[1]{\mathbf{#1}}

\newcommand{\hx}{{\hat x}}
\newcommand{\hy}{{\hat y}}
\newcommand{\hdelta}{{\hat \delta}}
\newcommand{\cx}{{\check x}}
\newcommand{\hLcal}{{\hat \Lcal}}
\newcommand{\tf}{{\tilde f}}
\newcommand{\tphi}{{\tilde \phi}}
\newcommand{\tC}{{\tilde C}}
\newcommand{\bCcal}{{\bar \Ccal}}
\newcommand{\tpsi}{{\tilde \psi}}
\newcommand{\td}{{\tilde d}}
\newcommand{\tLcal}{{\tilde \Lcal}}
\newcommand{\bbeta}{{\bar \beta}}
\newcommand{\bd}{{\bar d}}

\newcommand{\RR}{\mathbb{R}} 
\newcommand{\NN}{\mathbb{N}} 
\newcommand{\Bcal}{\mathcal{B}}

\renewbibmacro{in:}{%
  \ifentrytype{article}
    {}
    {\bibstring{in}%
     \printunit{\intitlepunct}}}

\begin{frontmatter}
\title{\bfseries Streaming Anomaly Detection}

\author{Siddharth Bhatia}
\degree{Doctor of Philosophy}
\field{Computer Science}
\degreeyear{2022}
\supervisor{Assistant Professor Bryan Hooi}

\examiners{%
  Professor Ng See-Kiong \\
  Associate Professor Stephane Bressan }

\maketitle

\declaredate{06 December, 2022}
\declarationpage

  \dedicate{Dedicated to my teachers}
  \begin{acknowledgments}

First and foremost, I want to thank my family and the Supreme Lord. Without their love and blessings, none of this would have been possible.

I cannot overstate how thankful I am to my advisor, Bryan Hooi. Throughout my time in graduate school, I really appreciated his kindness and patience, and how he genuinely cares about his students. During our meetings, he is always incredibly enthusiastic and energetic, even before a conference deadline. Bryan has always helped greatly in preparing me for an independent academic career by involving me in student mentorship and giving me numerous very helpful suggestions on paper writing and research presentations. I could surely not have gotten a better advisor.

I also want to thank my other thesis committee members: Ng See-Kiong and Stephane Bressan. Their guidance, questions, and comments throughout the process were invaluable to me in shaping the direction of the thesis.

I especially thank Sudipto Guha and Christos Faloutsos for being wonderful collaborators and mentors in research - I certainly learned a lot from our research discussions and meetings, particularly from your insights. I am also very thankful to Rajiv Kumar and Shan Sundar Balasubramaniam, who were my undergraduate research advisors.

I greatly thank the Outlier Detection and Description (ODD) workshop co-organizers: Leman Akoglu, Manish Gupta, Sourav Chatterjee, Xiaodong Jiang, and Bryan. I certainly enjoyed and learned a lot from your experience. I also thank Charu Aggarwal, Danai Koutra, Deepak Padmanabhan, Hanghang Tong, Ian Davidson, James Verbus, Jing Gao, Neil Shah, Rajmonda Caceres, Solon Barocas, and Sudipto for their insightful talks and panel discussion, and making the workshop a success.

I have learned a lot about research from my collaborators: Arjit Jain, Mohit Wadhwa, Rui Liu, Shivin Srivastava, Kenji Kawaguchi, Koki Kawabata, Ritesh Kumar, Shenghua Liu, Pan Li, Neil Shah, Yiwei Wang, Vaibhav Rajan, Nannan Wu, Ying Sun, Philip S. Yu,  Kijung Shin, Minji Yoon, Tanmoy Chakraborty; thanks for being such enthusiastic and helpful collaborators.

I also greatly thank Lei Cao, Samuel Madden, Mihai Cucuringu, Zak Jost, Elena Sizikova, Anton Strezhnev, Swarnima Sircar, Kai Xin Thia, François Scharffe, Kacy Zurkus, Matt Alderman, and Paul Asadoorian for inviting me to speak about my research and for being excellent hosts.

I am thankful to the community developers who extended our open-source projects: Joshua Tokle, Andrew Kane, Scott Steele, Steve Tan, Wong Mun Hou, Tobias Heidler, and Ashrya Agrawal.

I also appreciate Gregory Piatetsky, Limarc Ambalina, Matthew Mayo, Lucy Smith, Josh Miramant, John Desmond, Rahul Agarwal, and Nimish Mishra for covering our work in the press.

I had a lot of interesting conversations with Thijs Laarhoven, Hongfu Liu, Yue Zhao, Tim Januschowski, George Karypis, Milind Tambe, Aparna Taneja, Aude Hofleitner, Huan Liu, Jundong Li, Jiliang Tang, Evangelos Papalexakis, Srijan Kumar, Daniel Ting, Lee Rhodes, Jon Malkin, Graham Cormode, Arif Merchant, Jure	Leskovec, John Palowitch, Sean Taylor, Dhivya Eswaran, Yonatan Naamad, Eamonn Keogh, Yedid Hoshen, Guansong Pang, Jason Robinson, Xinyi Zheng, Acar Tamersoy, Dima Karamshuk, Yikun Ban, Susik Yoon, Kaize Ding, Antonia Saravanou, Derek Young, Anh Dinh, Raj Joshi, Ananta Narayanan Balaji, Qinbin Li, Prateek Saxena, Jonathan Scarlett, Kuldeep Meel, Gim Hee Lee, Abhik Roychoudhury, Kian Lee Tan, Lee Mong Li, Whynee Hsu, Bingsheng He, David Rosenblum, Damith Chatura Rajapakse, Wai Kay Leong, and Wenjie Feng; thank you for sharing your insights.

I greatly appreciate the wonderful support from Wei Ngan Chin, Li-Shiuan Peh, Beng Chin Ooi, Xiaokui Xiao, Line Fong, Agnes Ang, Aminah Ayu, Thiba Ahwahday, Irene Chuan, Catharine Tan, Sarada A, Aerin Oon, Goh Lee Kheng, and others: thanks for always being amazingly helpful with your advice, administrative and technical support and even going the extra mile in so many ways.

Last but certainly not the least, I am very grateful to my friends: Yash Sinha, Shivin Srivastava, and Pankaj Kumar. Graduate school has been a much more enriching experience for me thanks to the chance to be with you all.

\end{acknowledgments}
  \tableofcontents 
  \begin{abstract}

Anomaly detection is critical for finding suspicious behavior in innumerable systems, such as intrusion detection, fake ratings, and financial fraud. We need to detect anomalies in real-time or near real-time, i.e. determine if an incoming entity is anomalous or not, as soon as we receive it, to minimize the effects of malicious activities and start recovery as soon as possible. Therefore, online algorithms that can detect anomalies in a streaming manner are essential. Also, since the data increases as the stream is processed, we can only afford constant memory which makes the problem of streaming anomaly detection more challenging.

We first propose \textsc{Midas} which detects anomalous edges in dynamic graphs in an online manner, using constant time and memory. \textsc{Midas} focuses on detecting \emph{microcluster anomalies}, or suddenly arriving groups of suspiciously similar edges such as denial of service attacks in network traffic data. In addition, by using a principled hypothesis testing framework, \textsc{Midas} provides theoretical bounds on the false positive probability, which previous methods do not provide. We then propose two variants, \midas-R which incorporates temporal and spatial relations, and \midas-F which aims to filter away anomalous edges to prevent them from negatively affecting the algorithm's internal data structures. Our experimental results show that \textsc{Midas} outperforms baselines in accuracy by up to $62\%$ while processing the data orders of magnitude faster.

We then extend the count-min sketch data structure to a Higher-Order Sketch to capture complex relations in graph data, and to reduce detecting suspicious dense subgraph problem to finding a dense submatrix in constant time. Using this sketch, we propose four streaming methods to detect edge and subgraph anomalies in constant time and memory. Furthermore, our approach is the first streaming work that incorporates dense subgraph search to detect graph anomalies in constant memory and constant update time per newly arriving edge. We also provide theoretical guarantees on the higher-order sketch estimate and the submatrix density measure. Experimental results on real-world datasets demonstrate our effectiveness as opposed to popular state-of-the-art streaming edge and graph baselines.

Next, we broaden the graph setting to multi-aspect data. We propose \textsc{MStream} which detects anomalies in multi-aspect data streams including both categorical and numeric attributes and is online, thus processing each record in constant time and constant memory. Moreover, the anomalies detected by \mstream\ are explainable. We further propose \mstream-PCA, \mstream-IB, and \mstream-AE to incorporate correlation between features.

Finally, we consider multi-dimensional data streams with concept drift and propose \memstream, a streaming anomaly detection framework, allowing us to detect unusual events as they occur while being resilient to concept drift. \memstream\ leverages the power of a denoising autoencoder to learn representations and a memory module to learn the dynamically changing trend in data without the need for labels. We prove a theoretical bound on the size of memory for effective drift handling. In addition, we allow quick retraining when the arriving stream becomes sufficiently different from the training data. Furthermore, \memstream\ makes use of two architecture design choices to be robust to memory poisoning. Experimental results show the effectiveness of our approach compared to state-of-the-art streaming baselines.

\end{abstract}
  \newcommand*{\boldname}[3]{%
  \def\lastname{#1}%
  \def\firstname{#2}%
  \def\firstinit{#3}}
\boldname{}{}{}

\boldname{Bhatia^{*}}{Siddharth}{S.}
\boldname{Bhatia}{Siddharth}{S.}

\renewcommand{\mkbibnamegiven}[1]{%
  \ifboolexpr{ ( test {\ifdefequal{\firstname}{\namepartgiven}} or test {\ifdefequal{\firstinit}{\namepartgiven}} ) }
  {\mkbibbold{#1}}{#1}%
}

\renewcommand{\mkbibnamefamily}[1]{%
  \ifboolexpr{ ( test {\ifdefequal{\firstname}{\namepartgiven}} or test {\ifdefequal{\firstinit}{\namepartgiven}} ) }
  {\mkbibbold{#1}}{#1}%
}

\begin{refsection}
\nocite{%
  bhatia2020midas,
  bhatia2022midas,
  Bhatia2021MSTREAM,
  bhatia2022memstream,
  bhatia2021exgan,
  bhatia2022anograph,
  bhatia2022sess,
}

\defbibnote{PubListPrenote}{%
This dissertation is primarily related to the following peer-reviewed articles:
}
\defbibnote{PubListPostnote}{%
\newpage
}

\bookmarksetup{startatroot}
\newrefcontext[sorting=none]
\printbibliography[
  heading=bibintoc,
  title={Publications},
  prenote=PubListPrenote,
  postnote=PubListPostnote
]
\end{refsection}

\begin{refsection}
\nocite{%
  bhatia2021graphanogan,
  Kawabata2021SSMFSS,
  Wang2021AdaptiveDA,
  zhang2020augsplicing,
  Sun2022MonLADML,
  Sun2022AAANAA,
}

\defbibnote{PubListPrenote}{%
The following articles have also been completed over the course of the PhD but are not discussed in the dissertation:
}
\defbibnote{PubListPostnote}{%
}

\bookmarksetup{startatroot}
\newrefcontext[sorting=none]
\printbibliography[
  heading=none,
  title={Publications},
  prenote=PubListPrenote,
  postnote=PubListPostnote,
  resetnumbers=8
]
\end{refsection} 
  \listoffigures
  \listoftables
\end{frontmatter}

\sloppy
\makeatletter
\renewcommand\part{%
  \if@openright
    \cleardoublepage
  \else
    \clearpage
  \fi
  \thispagestyle{empty}%
  \if@twocolumn
    \onecolumn
    \@tempswatrue
  \else
    \@tempswafalse
  \fi
  \null\vfil
  \secdef\@part\@spart}
\makeatother

\SetPicSubDir{Introduction}
\SetExpSubDir{Introduction}

\chapter[Introduction][Introduction]{Introduction}
\label{ch:intro}

The need to detect anomalies in real-time or near real-time is driven by the need to respond quickly to potential security threats or other forms of abnormal behavior. By detecting anomalies as soon as they occur, organizations can take action to prevent or mitigate the impact of such threats, and reduce the likelihood of damage or loss. Moreover, the faster an organization can detect and respond to anomalies, the better able it will be to start recovery as soon as possible.

Consider an intrusion detection system (IDS), which is an important part of an organization's overall security strategy, providing protection against potential threats, valuable information about network security, and a layer of defense against cyber attacks. Anomalous behavior in this scenario can be described as a group of attackers making a large number of connections to some set of targeted machines to restrict accessibility or look for potential vulnerabilities. By continuously monitoring network traffic and alerting on potential threats, IDS allows organizations to respond quickly and help prevent attacks from succeeding or minimize their impact.

We can model an intrusion detection system as a dynamic graph, where nodes correspond to machines, and each edge represents a timestamped connection from one machine to another. In this graph, anomalous behavior often takes the form of a dense subgraph, as shown in several real-world datasets \cite{shin2017densealert,eswaran2018spotlight}.

Several approaches \cite{akoglu2010oddball,chakrabarti2004autopart,hooi2017graph,jiang2016catching,kleinberg1999authoritative,shin2018patterns,tong2011non} aim to detect anomalies in graph settings. However, these approaches focus on static graphs, whereas many real-world graphs are dynamic in nature, and methods based on static connections may miss temporal characteristics of the graphs and anomalies.

Among the methods focusing on dynamic graphs, most of them have edges aggregated into graph snapshots \cite{eswaran2018spotlight,sun2006beyond,sun2007graphscope,koutra2013deltacon,Sricharan,Gupta}. However, in order to minimize the effect of malicious activities and start recovery as soon as possible, we need to detect anomalies in real-time or near real-time i.e., to identify whether an incoming edge is anomalous or not, as soon as we receive it. This requires that we process the data as an edge stream rather than an aggregated graph snapshot. In addition, since the number of vertices can increase as we process the stream of edges, we need an algorithm that uses constant memory in graph size.

Moreover, fraudulent or anomalous events in many applications occur in microclusters or suddenly arriving groups of suspiciously similar edges e.g., denial of service attacks in network traffic data and lockstep behavior. However, existing methods that process edge streams in an online manner, including \cite{eswaran2018sedanspot,ranshous2016scalable}, aim to detect individually surprising edges, not microclusters, and can thus miss large amounts of suspicious activity.

Thus, we ask the question: Given a stream of graph edges from a dynamic graph, how can we detect anomalies, using constant memory and constant update time?

We first propose \midas\ (Chapter \ref{ch:midas}), which detects \emph{microcluster anomalies}, or suddenly arriving groups of suspiciously similar edges, in edge streams, using constant time and memory. By using a principled hypothesis testing framework, \textsc{Midas} provides theoretical bounds on the false positive probability, which previous methods do not provide.

Next, we extend the count-min sketch data structure to a higher-order sketch (Chapter \ref{ch:anograph}). Unlike traditional sketches, higher-order sketches can capture not just the frequency of data points in the data stream, but also the correlations and other higher-order statistics of the data. This higher-order sketch has the useful property of preserving the dense subgraph structure (dense subgraphs in the input turn into dense submatrices in the data structure). We then propose four online algorithms that utilize this enhanced data structure to detect both edge and graph anomalies in constant memory and constant update time. Existing work in streaming graph scenarios seeks to detect the presence of either anomalous edges \cite{eswaran2018sedanspot,bhatia2020midas,belth2020mining,chang2021f} or anomalous subgraphs \cite{shin2017densealert,eswaran2018spotlight,yoon2019fast}, but not both. Moreover, our approach is the only streaming method that makes use of dense subgraph search to detect graph anomalies while only requiring constant memory and time. We also provide theoretical guarantees on the higher-order sketch estimate and the submatrix density measure.

\paragraph{}
Recent intrusion detection datasets typically report tens of features for each individual flow, such as its source and destination IP, port, protocol, average packet size, etc. This makes it important to design approaches that can handle \emph{multi-aspect data}. Developing effective methods for handling \textit{multi-aspect data} (i.e., data having multiple features or dimensions) still remains a challenge, especially in an unsupervised setting, where traditional anomaly detection algorithms, such as One-Class SVM, tend to perform poorly because of the curse of dimensionality.

Some existing approaches for this problem aim to detect \emph{point anomalies}, or individually unusual connections. However, since this ignores the relationships between records, it does not effectively detect large and suddenly appearing \emph{groups} of connections, as is the case in denial of service and other attacks. For detecting such groups, there are also existing methods based on dense subgraph detection~\cite{bhatia2020midas} as well as dense subtensor detection~\cite{shin2017densealert,sun2006beyond}. However, these approaches are generally designed for datasets with a smaller number of dimensions, thus facing significant difficulties scaling to our dataset sizes. Moreover, they treat all variables of the dataset as categorical variables, whereas our approach can handle arbitrary mixtures of categorical variables (e.g., source IP address) and numerical variables (e.g., average packet size).

We propose \mstream\ (Chapter \ref{ch:mstream}), a method for processing a stream of multi-aspect data that detects \emph{group anomalies}, i.e., the sudden appearance of large amounts of suspiciously similar activity. Our approach naturally allows for similarity both in terms of categorical variables (e.g., a small group of repeated IP addresses creating a large number of connections), as well as in numerical variables (e.g., numerically similar values for average packet size). \mstream\ is a streaming approach that performs each update in constant memory and time. This is constant both with respect to the stream length as well as in the number of attribute values for each attribute. We also demonstrate that \mstream\ incorporates correlation between features and that the anomalies detected by \mstream\ are explainable.

Finally, the problem of anomaly detection becomes even more challenging when multi-aspect data streams contain concept drift (drift in the distribution over time). Existing approaches \cite{Hariri2021ExtendedIF,Bhatia2021MSTREAM,Manzoor2018xStreamOD,Na2018DILOFEA,Mirsky2018KitsuneAE,guha2016robust} are unable to fully handle such streams with concept drift. We propose \memstream\ (Chapter \ref{ch:memstream}), which uses a denoising autoencoder \cite{denoisingae} to extract features, and a memory module to learn the dynamically changing trend. Our streaming framework is resilient to concept drift and robust to memory poisoning, and we prove a theoretical bound on the size of memory for effective drift handling. Moreover, we allow quick retraining when the arriving stream becomes sufficiently different from the training data.

\section{Overview}
This thesis is organized into two main parts: (1) Graphs, and (2) Multi-Aspect Data. Related work in both graph and multi-aspect data settings is discussed in Chapter \ref{ch:related}. In Chapter \ref{ch:midas}, we study how to detect anomalous edges in a dynamic graph using the count-min sketch data structure. In Chapter \ref{ch:anograph}, we extend the count-min sketch to a higher-order sketch data structure to detect both anomalous edges and subgraphs. In Chapter \ref{ch:mstream}, we broaden the graph setting to a multi-aspect data stream and detect anomalous records in an online manner. Finally, in Chapter \ref{ch:memstream}, we consider multi-aspect data streams with concept drift. Two complementary directions are discussed in the Appendix: \ref{ch:exgan}: Adversarial generation of extreme/anomalous data; and \ref{ch:sess}: Incorporating semi-supervision in streaming anomaly detection. Table \ref{tab:overview} provides an overview of this thesis.

\begin{table*}[!htb]
\centering
\caption{Overview of the thesis.}
\label{tab:overview}
\resizebox{\columnwidth}{!}{
\begin{tabular}{@{}l|l|l|l|l}
\toprule
{\bfseries Chapter} & {\bfseries Setting} & {\bfseries Anomaly Type} & {\bfseries Data Structure} & {\bfseries Method}\\ 
\midrule
Ch. \ref{ch:midas} & Graph & Edges & Count-Min Sketch & {\midas} \href{https://arxiv.org/pdf/1911.04464}{[PDF]} \\ \midrule
Ch. \ref{ch:anograph} & Graph & Edges + Subgraphs & Higher-Order Sketch & {\textsc{AnoEdge/AnoGraph}} \href{https://arxiv.org/pdf/2106.04486.pdf}{[PDF]}\\ \midrule
Ch. \ref{ch:mstream} & Multi-Aspect Data & Records & Count-Min Sketch & {\mstream} \href{https://dl.acm.org/doi/pdf/10.1145/3442381.3450023}{[PDF]}\\ \midrule
Ch. \ref{ch:memstream} & Multi-Aspect Data & Records & Autoencoder + Memory   & {\memstream} \href{https://dl.acm.org/doi/pdf/10.1145/3485447.3512221}{[PDF]} \\ \bottomrule
\end{tabular}
}
\end{table*}

\paragraph{Reproducibility:} Our code and datasets are open-sourced and publicly available at \href{https://github.com/Stream-AD/}{\bfseries https://github.com/Stream-AD/}.

\paragraph{Summary of Impact}

\begin{itemize}
    \item {\bfseries Open Source Traction:} Our projects received 900+ stars on GitHub. \midas\ was implemented in C++, Python, Golang, Ruby, Rust, R, Java, and Julia.

    \item {\bfseries Awards:} \mstream\ was the WWW 2021 Best Paper Finalist. \midas\ won the popular choice award at Microsoft Azure Hackathon 2020.
    
    \item {\bfseries Invited Talks:} We were invited by the MIT Data Systems Group, Alan Turing Institute, New York University Center for Data Science, Security Weekly, DataScience SG, and Data Science Congress to share our research.

    \item {\bfseries Media Coverage:} Our research was covered by ACM TechNews, AIhub, Hacker News, Hacker Noon, insideBIGDATA, KDnuggets, and Towards Data Science.
    
\end{itemize}

Next, we summarize the goals and contributions of each of our proposed methods.

\section{Chapter Summaries}

\subsection{Chapter \ref{ch:midas}: MIDAS}
    
Given a stream of graph edges from a dynamic graph, how can we assign anomaly scores to edges in an online manner, for the purpose of detecting unusual behavior, using constant time and memory?
		
\paragraph{Contributions:}
\begin{enumerate}
		\item {\bfseries Streaming Microcluster Detection:} We propose a novel streaming approach combining statistical (chi-squared test) and algorithmic (count-min sketch) ideas to detect microcluster anomalies, requiring constant time and memory.
		\item {\bfseries Theoretical Guarantees:} We show guarantees on the false positive probability of \midas.
		\item {\bfseries Effectiveness}: Our experimental results show that \midas{} outperforms baseline approaches by up to $62$\% higher ROC-AUC, and processes the data orders-of-magnitude faster than baseline approaches.
		\item {\bfseries Relations and Filtering}: We propose two variants, \midas-R that incorporates temporal and spatial relations, and \midas-F that aims to filter away anomalous edges to prevent them from negatively affecting the algorithm's internal data structures.
\end{enumerate}

\subsection{Chapter \ref{ch:anograph}: AnoEdge/AnoGraph}
    
Given a stream of graph edges from a dynamic graph, how can we assign anomaly scores to edges and subgraphs in an online manner, for the purpose of detecting unusual behavior, using constant time and memory?

\paragraph{Contributions:}
\begin{enumerate}
    \item {\bfseries Higher-Order Sketch:} We transform the dense subgraph detection problem into finding a dense submatrix (which can be achieved in constant time) by extending the count-min sketch data structure to a higher-order sketch.
    \item {\bfseries Streaming Anomaly Detection:} We propose four novel online approaches to detect anomalous edges and graphs in real-time, with constant memory and update time. Moreover, this is the first streaming work that incorporates dense subgraph search to detect graph anomalies in constant memory/time.
    \item {\bfseries Effectiveness:} We outperform state-of-the-art streaming edge and graph anomaly detection methods on four real-world datasets.
\end{enumerate}
    
\subsection{Chapter \ref{ch:mstream}: MSTREAM}
Given a stream of entries (i.e., records) in \emph{multi-aspect data} (i.e., data having multiple features or dimensions), how can we detect anomalous behavior, including group anomalies involving the sudden appearance of large groups of suspicious activity, in an unsupervised manner? 
    
\paragraph{Contributions:}
\begin{enumerate}
    \item {\bfseries Multi-Aspect Group Anomaly Detection:} We propose a novel approach for detecting group anomalies in multi-aspect data, including both categorical and numeric attributes. Moreover, the anomalies detected by \mstream\ are explainable.
    \item {\bfseries Streaming Approach:} Our approach processes the data in a fast and streaming fashion, performing each update in constant time and memory.
    \item {\bfseries Effectiveness:} Our experimental results using \emph{KDDCUP99}, \emph{CICIDS-DoS}, \emph{UNSW-NB 15} and \emph{CICIDS-DDoS} datasets show that \mstream\ outperforms baseline approaches.
    \item {\bfseries Incorporating Correlation:} We propose \mstream-PCA, \mstream-IB and \mstream-AE to incorporate correlation between features.
\end{enumerate}
    
    \subsection{Chapter \ref{ch:memstream}: MemStream}
    
Given a stream of entries over time in a multi-dimensional data setting where concept drift is present, how can we detect anomalous activities?

    \paragraph{Contributions:}
    \begin{enumerate}
    \item {\bfseries Streaming Anomaly Detection:} We propose a novel streaming approach using a denoising autoencoder and a memory module, for detecting anomalies. \memstream\ is resilient to concept drift and allows quick retraining.
    \item {\bfseries Theoretical Guarantees:} We discuss both the optimum memory size for effective concept drift handling and the motivation behind our architecture design.
    \item {\bfseries Robustness to Memory Poisoning:} \memstream\ prevents anomalies from entering the memory and can self-correct and recover from bad memory states.
    \item {\bfseries Effectiveness:} Our experimental results show that \memstream\ convincingly outperforms $11$ state-of-the-art baselines using $2$ synthetic datasets (that we release as open-source) and $11$ popular real-world datasets.
\end{enumerate}
    
\subsection{Appendix \ref{ch:exgan}: ExGAN} To manage the risk arising from anomalous and extreme events like natural disasters, financial crashes, and epidemics, a vital step is to be able to generate and understand a wide range of extreme scenarios. Existing approaches based on Generative Adversarial Networks (GANs) excel at generating realistic samples but seek to generate \emph{typical} samples, rather than extreme samples.
    
In this chapter, we propose ExGAN which allows the user to specify both the desired extremeness measure, as well as the desired extremeness probability to sample at. Our work draws from Extreme Value Theory, a probabilistic approach for modelling the extreme tails of distributions. Experiments on real US Precipitation data show that ExGAN generates realistic samples efficiently, based on visual inspection and quantitative measures. Moreover, generating increasingly extreme examples can now be done in constant time, as opposed to the $\mathcal{O}(\frac{1}{\tau})$ time required by the baseline.

\subsection{Appendix \ref{ch:sess}: SESS} In this chapter, we discuss semi-supervision for streaming anomaly detection algorithms that use sketches. Using a two-state conceptual system that draws on partially observable markov decision processes, we show that off-the-shelf semi-supervision ideas can lead to undesirable algorithms. We also show that unbalanced classification, as is the case in anomaly detection, provides a significantly greater opportunity for well-designed algorithms. We introduce \sess, which incorporates semi-supervision to improve the performance of \midas\ significantly while retaining the online, low memory characteristics of streaming algorithms. Next, we propose \sess-3D which can directly incorporate node feedback, and further improves the performance by being cache-aware and using higher-order sketches. Finally, we show how the performance of \spotlight\ can be improved in a weakly semi-supervised setting.
\SetPicSubDir{Related Work}
\SetExpSubDir{Related Work}

\chapter[Related Work][Related Work]{Related Work}
\label{ch:related}

This thesis is closely related to areas such as graph streams \cite{luo2020dynamic,boniol2020series2graph,liakos2020rapid,6544842,DBLP:conf/pakdd/ZhangLYFC19}, sketches \cite{Bahri2018ASN,Mu2017StreamingCW,khan2018composite,rusu2009sketching,Shi2020HigherOrderCS,Zhao2011gSketchOQ,Menon2007AnID}, dense subgraph discovery \cite{Ma2020EfficientAF,Epasto2015EfficientDS,Sawlani2020NearoptimalFD,Mcgregor2015DensestSI,Esfandiari2018MetricSA}, concept drift in streams \cite{Pasricha2018IdentifyingAA,Benczr2019ReinforcementLU,Chi2018HashingFA,Shao2014PrototypebasedLO,Bai2016AnOM}, anomaly detection in graphs \cite{DBLP:conf/pakdd/ZhangLYFC19, DBLP:conf/sdm/BogdanovFMPRS13, 7836684, 10.1145/3139241, DBLP:journals/wias/BonchiBGS19, 7817049,Bojchevski2018BayesianRA,Yu2018NetWalkAF,Kumagai2021SemisupervisedAD,Liu2021AnomalyDI,Shao2018AnEF,noble2003graph,saebi2020efficient,kulkarni2017network,Malliaros2012FastRE,perozzi2016scalable, tong2011non,yoon2019fast} and anomaly detection in streams \cite{Jankov2017RealtimeHP,Zou2017NonparametricDO,Moshtaghi2015EvolvingFR, Siffer2017AnomalyDI,Togbe2020AnomalyDF,Zhang2020AugSplicingSB,Wang2008ProcessingOM,10.1145/3178876.3186056,sun2019fast,7837870,DBLP:journals/corr/abs-1906-02524}. Anomaly detection is a vast topic by itself and cannot be fully covered in this thesis. In this chapter, we mainly focus on methods that detect anomalies in graph and multi-aspect data settings. \cite{chandola2009anomaly} discusses traditional anomaly detection methods, \cite{akoglu2015graph} surveys graph-based anomaly detection, \cite{gupta2013outlier} reviews outlier detection in temporal data and \cite{lu2018learning} is a literature survey on concept drift.


\section{Graphs}
\label{rel:graph}

	\noindent \textbf{Anomaly detection in static graphs} can be classified by which anomalous entities (nodes, edges, subgraph, etc.) are spotted.

	\begin{itemize}
		\item Anomalous node detection:
		\textsc{OddBall} \cite{akoglu2010oddball} extracts egonet-based features and finds empirical patterns with respect to the features.
		Then, it identifies nodes whose egonets deviate from the patterns, including the count of triangles, total weight, and principal eigenvalues.
		\textsc{CatchSync} \cite{jiang2016catching} computes node features, including degree and authoritativeness~\cite{kleinberg1999authoritative}, then spots nodes whose neighbors are notably close in the feature space.
		\item Anomalous subgraph detection:
		FRAUDAR \cite{hooi2017graph} and k-cores \cite{shin2018patterns} measure the anomalousness of nodes and edges, detecting a dense subgraph consisting of many anomalous nodes and edges.
		\item Anomalous edge detection:
		AutoPart \cite{chakrabarti2004autopart} encodes an input graph based on similar connectivity among nodes, then spots edges whose removal reduces the total encoding cost significantly.
		NrMF \cite{tong2011non} factorize the adjacency matrix and flag edges with high reconstruction error as outliers.
	\end{itemize}

	\noindent \textbf{Anomaly detection in graph streams} use as input a series of graph snapshots over time.
	We categorize them similarly according to the type of anomaly detected:

	\begin{itemize}
		\item Anomalous node detection:
		DTA/STA \cite{sun2006beyond} approximates the adjacency matrix of the current snapshot based on incremental matrix factorization, then spots nodes corresponding to rows with high reconstruction error. \cite{aggarwal2011outlier} dynamically partitions the network graph to construct a structural connectivity model and detect outliers in graph streams.
		\item Anomalous subgraph detection:
		Given a graph with timestamps on edges, \textsc{CopyCatch} \cite{beutel2013copycatch} spots near-bipartite cores where each node is connected to others in the same core densely within a short time.
		SPOT/DSPOT \cite{Siffer2017AnomalyDI} use extreme value theory to automatically set thresholds for anomalies. IncGM+ \cite{abdelhamid2017incremental} utilizes an incremental method to process graph updates.
		\item Anomalous edge detection:
		\textsc{SpotLight} \cite{eswaran2018spotlight} discovers anomalous graphs with dense bi-cliques, but uses a randomized approach without any search for dense subgraphs, and \textsc{AnomRank} \cite{yoon2019fast} iteratively updates two score vectors and computes anomaly scores.
	\end{itemize}

	\noindent \textbf{Anomaly detection in edge streams} use as input a stream of edges over time.
	Categorizing them according to the type of anomaly detected:

	\begin{itemize}
		\item Anomalous node detection:
		Given an edge stream, \textsc{HotSpot} \cite{yu2013anomalous} detects nodes whose egonets suddenly and significantly change.
		\item Anomalous subgraph detection:
		Given an edge stream, \textsc{DenseAlert} \cite{shin2017densealert} identifies dense subtensors created within a short time and utilizes incremental method to process graph updates or subgraphs more efficiently.
		\item Anomalous edge detection: Only the methods in this category are applicable to our task, as they operate on edge streams and output a score per edge.
		\cite{ye2001an} proposes a method that utilizes the chi-squared test to give a score to the individual events from a stream. CAD \cite{Sricharan} localizes anomalous changes using commute time distance measurement. \densestream\ \cite{shin2017densealert} maintains and updates a dense subtensor in a tensor stream. RHSS \cite{ranshous2016scalable} focuses on sparsely-connected parts of a graph. Sedanspot \cite{eswaran2018sedanspot} uses personalized PageRank to detect edge anomalies based on edge occurrence, preferential attachment, and mutual neighbors in sublinear space and constant time per edge. PENminer \cite{belth2020mining} explores the persistence of activity snippets, i.e., the length and regularity of edge-update sequences' reoccurrences. F-FADE \cite{chang2021f} aims to detect anomalous interaction patterns by factorizing the frequency of those patterns. These methods can effectively detect anomalies, but they require a considerable amount of time.
		
		
	\end{itemize}

\section{Multi-Aspect Data}
\label{rel:multi}

\begin{itemize}

\item Deep Learning: See \cite{chalapathy2019deep,pang2020deep} for extensive surveys. Several deep learning based methods have been proposed for anomaly detection such as GAN-based approaches \cite{yang2020memgan,Bashar2020TAnoGANTS,ngo2019,zenati2018adversarially,deecke2018image,akcay2018ganomaly,schlegl2017unsupervised}, Energy-based \cite{kumar2019maximum,zhai2016deep}, Autoencoder-based \cite{gong2019memorizing,Su2019RobustAD,zong2018deep,xu2018unsupervised,zhou2017anomaly,Slch2016VariationalIF,an2015variational,Goodge2020RobustnessOA,Goodge2022ARESLA}, and RNN-based \cite{saurav2018rnn_online_anomaly}. For example, DAGMM \cite{zong2018deep} learns a Gaussian Mixture density model (GMM) over a low-dimensional latent space produced by a deep autoencoder, \cite{JU2020167} uses metric learning for anomaly detection and DSEBM \cite{zhai2016deep} trains deep energy models such as Convolutional and Recurrent EBMs using denoising score matching instead of maximum likelihood. However, deep learning based approaches do not process the data in a streaming manner and typically require a large amount of training data in an offline setting, whereas we process the data in an online manner.

\item Tensor decomposition:
See \cite{fanaee2016tensor} for an extensive survey on tensor-based anomaly detection. Tensor decomposition methods such as \cite{kolda2009tensor,zhou2016accelerating} can be used to find anomalies. Score Plots obtained from tensor decomposition can also be analyzed manually or automatically for anomaly detection. These score plots can be one-dimensional: \cite{papalexakis2014spotting}, multi-dimensional: \textsc{MalSpot} \cite{mao2014malspot} or time-series  \cite{papalexakis2012parcube}. STenSr \cite{shi2015stensr} models the tensor stream as a single incremental tensor for representing the entire network, instead of dealing with each tensor in the stream separately. \cite{li2011robust} uses subspace learning in tensors to find anomalies. MASTA \cite{fanaee2015multi} uses histogram approximation to analyze tensors. It vectorizes the whole tensor and simultaneously segments it into slices in each mode. The distribution of each slice is compared against the vectorized tensor to identify anomalous slices. STA \cite{sun2006beyond} monitors the streaming decomposition reconstruction error for each tensor at each time instant and anomalies occur when this error goes beyond a pre-defined threshold. However \cite{shin2017densealert} shows limited accuracy for dense-subtensor detection based on tensor decomposition.

\item Dense subtensor detection:
Dense-subtensor detection has been used to detect anomalies in \textsc{M-Zoom} \cite{shin2016m}, \textsc{D-Cube} \cite{shin2017d}, \cite{maruhashi2011multiaspectforensics} and \textsc{CrossSpot} \cite{jiang2015general} but these approaches consider the data as a static tensor. \textsc{DenseAlert} \cite{shin2017densealert} is a streaming algorithm to identify dense subtensors created within a short time and utilizes an incremental method to process graph updates or subgraphs more efficiently.

\item Density-based:  Local Outlier Factor (LOF) \cite{breunig2000lof} estimates the local density at each point, then identifies anomalies as points with much lower local density than their neighbors. Elliptic Envelope \cite{rousseeuw1999fast} fits an ellipse to the normal data points by fitting a robust covariance estimate to the data. DILOF \cite{Na2018DILOFEA} improves upon LOF and LOF variants \cite{Salehi2016FastME,Pokrajac2007IncrementalLO} by adopting a novel density-based sampling scheme to summarize the data, without prior assumptions on the data distribution. LUNAR \cite{goodge2021lunar} is a hybrid approach combining deep learning and LOF. However, these approaches are suitable only for lower-dimensional data due to the curse of dimensionality.

\item Tree-based: Isolation Forest (IF) \cite{liu2008isolation} constructs trees by randomly selecting features and splitting them at random split points, and then defines anomalies as points that are separated from the rest of the data at low depth values. HS-Tree \cite{Tan2011FastAD} uses an ensemble of randomly constructed half-space trees with a sliding window to detect anomalies in evolving streaming data. iForestASD \cite{Ding2013AnAD} uses a sliding window frame scheme to handle abnormal data. Random Cut Forest (RCF) \cite{guha2016robust} tries to further improve upon IF  by creating multiple random cuts (trees) of data and constructing a forest of such trees to determine whether a point is anomalous or not. Recently, \cite{Hariri2021ExtendedIF} shows that splitting by only one variable at a time introduces some biases in IF which can be overcome by using hyperplane cuts instead. They propose Extended Isolation Forest (Ex. IF) \cite{Hariri2021ExtendedIF} where the split criterion is based on a threshold set on a linear combination of randomly chosen variables instead of a threshold on a single variable's value at a time. However, these approaches compute an anomaly score by traversing a tree structure that is bounded by the maximum depth parameter and the size of the sliding window, therefore they do not capture long-range dependence.

\item Popular streaming approaches include STORM \cite{Angiulli2007DetectingDO}, which uses a sliding window to detect global distance-based outliers in data streams with respect to the current window. RS-Hash \cite{Sathe2016SubspaceOD} uses subspace grids and randomized hashing in an ensemble to detect anomalies. For each model in the ensemble, a grid is constructed using subsets of features and data, random hashing is used to record data counts in grid cells, and the anomaly score of a data point is the log of the frequency in its hashed bins. LODA \cite{Pevn2015LodaLO} generates several weak anomaly detectors by producing many random projections of the data and then computing a density estimation histogram for each projection. The outlier scores produced are the mean negative log-likelihood according to each histogram for each point. \textsc{xStream} \cite{Manzoor2018xStreamOD} detects anomalies in feature-evolving data streams through the use of a streaming random projection scheme and ensemble of half-space chains. Kitsune \cite{Mirsky2018KitsuneAE} is an ensemble of light-weight autoencoders for real-time anomaly detection.

\end{itemize}



\part{Graphs}
\SetPicSubDir{MIDAS}
\SetExpSubDir{MIDAS}

\chapter[MIDAS: Microcluster-Based Detector of Anomalies in Edge Streams][MIDAS]{MIDAS: Microcluster-Based Detector of Anomalies in Edge Streams}
\label{ch:midas}

\begin{mdframed}[backgroundcolor=magenta!20] 
Chapter based on work that appeared at AAAI'20 \cite{bhatia2020midas} \href{https://arxiv.org/pdf/1911.04464}{[PDF]} and TKDD'22 \cite{bhatia2022midas} \href{https://arxiv.org/pdf/2009.08452.pdf}{[PDF]}.
\end{mdframed}

\section{Introduction}

Given a stream of graph edges from a dynamic graph, how can we assign anomaly scores to edges in an online manner, for the purpose of detecting unusual behavior, using constant time and memory?

Fraudulent or anomalous events in many applications occur in microclusters or suddenly arriving groups of suspiciously similar edges e.g. denial of service attacks in network traffic data. However, existing methods which process edge streams in an online manner aim to detect individually surprising edges, not microclusters, and can thus miss large amounts of suspicious activity.

In this chapter, we propose \midas{}, which detects \emph{microcluster anomalies}, or suddenly arriving groups of suspiciously similar edges, in edge streams. It is worth noting that in other literature, microcluster may have different meanings~\cite{aggarwal2010on2,kranen2011the,bah2019an}, while we specifically refer to a group of sudden arriving edges. The \midas{} algorithm uses count-min sketches (CMS) \cite{cormode2005improved} to count the number of occurrences in each timestamp, then use the chi-squared test to evaluate the degree of deviation and produce a score representing the anomalousness. The higher the score, the more anomalous the edge is. The proposed method uses constant memory and has a constant time complexity processing each edge. Additionally, by using a principled hypothesis testing framework, \midas{} provides theoretical bounds on the false positive probability, which those methods do not provide.

We then propose a relational variant \midas-R, which incorporates temporal and spatial relations. In the base version of the \midas{} algorithm, the CMS is cleared after every timestamp change. However, some anomalies persist for multiple timestamps. Maintaining partial counts of previous timestamps to the next allows the algorithm to quickly produce a high score when the edge occurs again. This variant also considers the source and destination nodes as additional information that helps determine anomalous edges.

Finally, we propose \midas-F, to solve the problem by which anomalies are incorporated into the algorithm's internal states, creating a `poisoning' effect that can allow future anomalies to slip through undetected. \midas-F introduces two modifications: 1) We modify the anomaly scoring function, aiming to reduce the `poisoning' effect of newly arriving edges; 2) We introduce a conditional merge step, which updates the algorithm's data structures after each time tick, but only if the anomaly score is below a threshold value, also to reduce the `poisoning' effect.

	Our main contributions are as follows:
	\begin{enumerate}
		\item {\bfseries Streaming Microcluster Detection:} We propose a novel streaming approach combining statistical (chi-squared test) and algorithmic (count-min sketch) ideas to detect microcluster anomalies, requiring constant time and memory.
		\item {\bfseries Theoretical Guarantees:} We show guarantees on the false positive probability of \midas.
		\item {\bfseries Effectiveness}: Our experimental results show that \midas{} outperforms baseline approaches by up to $62$\% higher ROC-AUC, and processes the data orders-of-magnitude faster than baseline approaches.
		\item {\bfseries Relations and Filtering}: We propose two variants, \midas-R that incorporates temporal and spatial relations, and \midas-F that aims to filter away anomalous edges to prevent them from negatively affecting the algorithm's internal data structures.
	\end{enumerate}
	{\bfseries Reproducibility}: Our code and datasets are publicly available at \href{https://github.com/Stream-AD/MIDAS}{https://github.com/Stream-AD/MIDAS}.

	\section{Problem}

	Let $\mathcal{E} = \{e_1, e_2, \cdots\}$ be a stream of edges from a time-evolving graph $\mathcal{G}$. Each arriving edge is a tuple $e_i = (u_i, v_i, t_i)$ consisting of a source node $u_i \in \mathcal{V}$, a destination node $v_i \in \mathcal{V}$, and a time of occurrence $t_i$, which is the time at which the edge was added to the graph. For example, in a network traffic stream, an edge $e_i$ could represent a connection made from a source IP address $u_i$ to a destination IP address $v_i$ at time $t_i$. We do not assume that the set of vertices $\mathcal{V}$ is known a priori: for example, new IP addresses or user IDs may be created over the course of the stream.

	We model $\mathcal{G}$ as a directed graph. Undirected graphs can simply be handled by treating an incoming undirected $e_i = (u_i, v_i, t_i)$ as two simultaneous directed edges, one in either direction.

	We also allow $\mathcal{G}$ to be a multigraph: edges can be created multiple times between the same pair of nodes. Edges are allowed to arrive simultaneously: i.e. $t_{i+1} \ge t_i$, since in many applications $t_i$ are given in the form of discrete time ticks.

	The desired properties of our algorithm are as follows:

	\begin{itemize}
		\item {\bfseries Microcluster Detection:} It should detect suddenly appearing bursts of activity that share many repeated nodes or edges, which we refer to as microclusters.
		\item {\bfseries Guarantees on False Positive Probability:} Given any user-specified probability level $\epsilon$ (e.g. $1\%$), the algorithm should be adjustable so as to provide a false positive probability of at most $\epsilon$ (e.g. by adjusting a threshold that depends on $\epsilon$). Moreover, while guarantees on the false positive probability rely on assumptions about the data distribution, we aim to make our assumptions as weak as possible.
		\item {\bfseries Constant Memory and Update Time:} For scalability in the streaming setting, the algorithm should run in constant memory and constant update time per newly arriving edge. Thus, its memory usage and update time should not grow with the length of the stream or the number of nodes in the graph.
	\end{itemize}

	\section{MIDAS and MIDAS-R Algorithms}

	\subsection{Overview}

	Next, we describe our \midas{} and \midas-R approaches. The following provides an overview:

	\begin{enumerate}
		\item {\bfseries Streaming Hypothesis Testing Approach:} We describe our \midas{} algorithm, which uses streaming data structures within a hypothesis testing-based framework, allowing us to obtain guarantees on false positive probability.
		\item {\bfseries Detection and Guarantees:} We describe our decision procedure for determining whether a point is anomalous, and our guarantees on false positive probability.
		\item {\bfseries Incorporating Relations:} We extend our approach to the \midas-R algorithm, which incorporates relationships between edges temporally and spatially\footnote{We use `spatially' in a graph sense, i.e. connecting nearby nodes, not to refer to any other continuous spatial dimension.}.
	\end{enumerate}

	\subsection{MIDAS: Streaming Hypothesis Testing Approach}

	Consider the example in Figure \ref{fig:intro} of a single source-destination pair $(u,v)$, which shows a large burst of activity at time $10$. This burst is the simplest example of a microcluster, as it consists of a large group of edges that are very similar to one another (in fact identical), both {\bfseries spatially} (i.e. in terms of the nodes they connect) and {\bfseries temporally}.
	
	\begin{figure}[!htb]
		\center{\includegraphics[width=0.7\columnwidth]{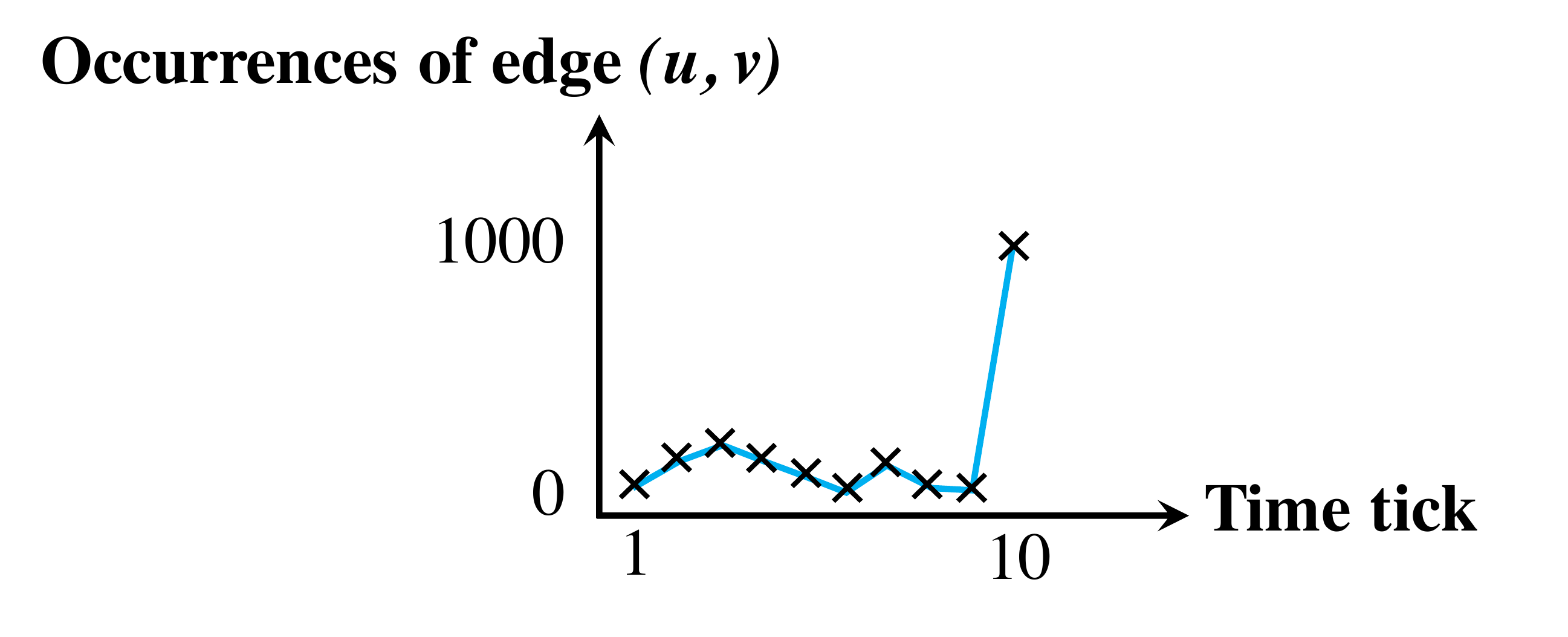}}
		\caption{\label{fig:intro} Time series of a single source-destination pair $(u,v)$, with a large burst of activity at time tick $10$.}
	\end{figure}

	\subsubsection{Streaming Data Structures}

	In an offline setting, there are many time-series methods that could detect such bursts of activity. However, in an online setting, recall that we want memory usage to be bounded, so we cannot keep track of even a single such time series. Moreover, there are many such source-destination pairs, and the set of sources and destinations is not fixed a priori.

	To circumvent these problems, we maintain two types of Count-min sketch (CMS)~\cite{cormode2005improved} data structures. Assume we are at a particular fixed time tick $t$ in the stream; we treat time as a discrete variable for simplicity. Let $s_{uv}$ be the total number of edges from $u$ to $v$ up to the current time. Then, we use a single CMS data structure to approximately maintain all such counts $s_{uv}$ (for all edges $uv$) in constant memory: at any time, we can query the data structure to obtain an approximate count $\hat{s}_{uv}$.

	Secondly, let $a_{uv}$ be the number of edges from $u$ to $v$ in the current time tick (but not including past time ticks). We keep track of $a_{uv}$ using a similar CMS data structure, the only difference being that we reset this CMS data structure every time we transition to the next time tick. Hence, this CMS data structure provides approximate counts $\hat{a}_{uv}$ for the number of edges from $u$ to $v$ in the current time tick $t$.

	\subsubsection{Hypothesis Testing Framework}

	Given approximate counts $\hat{s}_{uv}$ and $\hat{a}_{uv}$, how can we detect microclusters? Moreover, how can we do this in a principled framework that allows for theoretical guarantees?

	Fix a particular source and destination pair of nodes, $(u,v)$, as in Figure \ref{fig:intro}. One approach would be to assume that the time series in Figure \ref{fig:intro} follows a particular generative model: for example, a Gaussian distribution. We could then find the mean and standard deviation of this Gaussian distribution. Then, at time $t$, we could compute the Gaussian likelihood of the number of edge occurrences in the current time tick, and declare an anomaly if this likelihood is below a specified threshold.

	However, this requires a restrictive Gaussian assumption, which can lead to excessive false positives or negatives if the data follows a very different distribution. Instead, we use a weaker assumption: that the mean level (i.e. the average rate at which edges appear) in the current time tick (e.g. $t=10$) is the same as the mean level before the current time tick $(t<10)$. Note that this avoids assuming any particular distribution for each time tick, and also avoids a strict assumption of stationarity over time.

	Hence, we can divide the past edges into two classes: the current time tick $(t=10)$ and all past time ticks $(t<10)$. Recalling our previous notation, the number of events at $(t=10)$ is $a_{uv}$, while the number of edges in past time ticks $(t<10)$ is $s_{uv} - a_{uv}$.

	Under the chi-squared goodness-of-fit test, the chi-squared statistic is defined as the sum over categories of $\frac{(\text{observed} - \text{expected})^2}{\text{expected}}$. In this case, our categories are $t=10$ and $t<10$. Under our mean level assumption, since we have $s_{uv}$ total edges (for this source-destination pair), the expected number at $t=10$ is $\frac{s_{uv}}{t}$, and the expected number for $t<10$ is the remaining, i.e. $\frac{t-1}{t} s_{uv}$. Thus the chi-squared statistic is:

	\begin{align*}
		X^2 &= \frac{(\text{observed}_{(t=10)} - \text{expected}_{(t=10)})^2}{\text{expected}_{(t=10)}} \\
		&+ \frac{(\text{observed}_{(t<10)} - \text{expected}_{(t<10)})^2}{\text{expected}_{(t<10)}}\\
		&= \frac{(a_{uv} - \frac{s_{uv}}{t})^2}{\frac{s_{uv}}{t}} + \frac{((s_{uv} - a_{uv}) - \frac{t-1}{t} s_{uv})^2}{\frac{t-1}{t} s_{uv}}\\
		&= \frac{(a_{uv} - \frac{s_{uv}}{t})^2}{\frac{s_{uv}}{t}} + \frac{(a_{uv} - \frac{s_{uv}}{t})^2}{\frac{t-1}{t} s_{uv}}\\
		&= (a_{uv} - \frac{s_{uv}}{t})^2 \frac{t^2}{s_{uv}(t-1)}
	\end{align*}
	Note that both $a_{uv}$ and $s_{uv}$ can be estimated by our CMS data structures, obtaining approximations $\hat{a}_{uv}$ and $\hat{s}_{uv}$ respectively. This leads to our following anomaly score, using which we can evaluate a newly arriving edge with source-destination pair $(u,v)$:

	\begin{defn}[Anomaly Score]
		Given a newly arriving edge $(u,v,t)$, our anomaly score is computed as:
		\begin{align}
			\text{score}(u,v,t) = (\hat{a}_{uv} - \frac{\hat{s}_{uv}}{t})^2 \frac{t^2}{\hat{s}_{uv}(t-1)}
		\end{align}
	\end{defn}

	Algorithm \ref{alg:midas} summarizes our \midas{} algorithm.

	\begin{algorithm}
		\caption{\midas:\ Streaming Anomaly Scoring \label{alg:midas}}
		\KwIn{Stream of graph edges over time}
		\KwOut{Anomaly scores per edge}
		{\bfseries $\triangleright$ Initialize CMS data structures:} \\
		Initialize CMS for total count $s_{uv}$ and current count $a_{uv}$ \\
		\While{new edge $e=(u,v,t)$ is received:}{
			{\bfseries $\triangleright$ Update Counts:} \\
			Update CMS data structures for the new edge $uv$\\
			{\bfseries $\triangleright$ Query Counts:} \\
			Retrieve updated counts $\hat{s}_{uv}$ and $\hat{a}_{uv}$\\
			{\bfseries $\triangleright$ Anomaly Score:}\\
			{\bfseries output} $\text{score}((u,v,t)) = (\hat{a}_{uv} - \frac{\hat{s}_{uv}}{t})^2 \frac{t^2}{\hat{s}_{uv}(t-1)}$\\
		}
	\end{algorithm}

	\subsection{Detection and Guarantees}

	While Algorithm~\ref{alg:midas} computes an anomaly score for each edge, it does not provide a binary decision for whether an edge is anomalous or not.
	We want a decision procedure that provides binary decisions and a guarantee on the false positive probability: i.e. given a user-defined threshold $\epsilon$, the probability of a false positive should be at most $\epsilon$.
	Intuitively, the key idea is to combine the approximation guarantees of CMS data structures with the properties of a chi-squared random variable.

	The key property of CMS data structures we use is that given any $\epsilon$ and $\nu$, for appropriately chosen CMS data structure sizes ($w = \lceil ln \frac{2}{\epsilon} \rceil, b = \lceil \frac{e}{\nu} \rceil$) \cite{cormode2005improved}, with probability at least $1-\frac{\epsilon}{2}$, the estimates $\hat{a}_{uv}$ satisfy:

	\begin{align}
		\hat{a}_{uv} \le a_{uv} + \nu \cdot N_t
	\end{align}

	where $N_t$ is the total number of edges in the CMS for $a_{uv}$ at time tick $t$.
	Since CMS data structures can only overestimate the true counts, we additionally have

	\begin{align}
		s_{uv} \le \hat{s}_{uv}
	\end{align}

	Define an adjusted version of our earlier score:

	\begin{align}
		\tilde{a}_{uv} = \hat{a}_{uv} - \nu N_t
	\end{align}

	To obtain its probabilistic guarantee, our decision procedure computes $\tilde{a_{uv}}$, and uses it to compute an adjusted version of our earlier statistic:

	\begin{align}
		\tilde{X^2} = (\tilde{a}_{uv} - \frac{\hat{s}_{uv}}{t})^2 \frac{t^2}{\hat{s}_{uv}(t-1)}
	\end{align}

	Note that the usage of $X^2$ and $\tilde{X^2}$ are different.
	$X^2$ is used as the score of individual edges while $\tilde{X^2}$ facilitates making binary decisions.

	Then our main guarantee is as follows:

	\begin{theo}[False Positive Probability Bound]
		\label{thm:bound}
		Let $\chi_{1-\epsilon/2}^2(1)$ be the $1-\epsilon/2$ quantile of a chi-squared random variable with 1 degree of freedom.
		Then:
		\begin{align}
			P(\tilde{X^2} > \chi_{1-\epsilon/2}^2(1)) < \epsilon
		\end{align}
		In other words, using $\tilde{X^2}$ as our test statistic and threshold $\chi_{1-\epsilon/2}^2(1)$ results in a false positive probability of at most $\epsilon$.
	\end{theo}

	\begin{proof}
		Recall that
		\begin{align}
			X^2 = (a_{uv} - \frac{s_{uv}}{t})^2 \frac{t^2}{s_{uv}(t-1)}
		\end{align}
		was defined so that it has a chi-squared distribution. Thus:
		\begin{align}
			\label{eq:cond1}
			P(X^2 \le \chi_{1-\epsilon/2}^2(1)) = 1-\epsilon/2
		\end{align}
		At the same time, by the CMS guarantees we have:
		\begin{align}
			\label{eq:cond2}
			P(\hat{a}_{uv} \le a_{uv} + \nu \cdot N_t) \ge 1-\epsilon/2
		\end{align}

		By union bound, with probability at least $1-\epsilon$, both these events \eqref{eq:cond1} and \eqref{eq:cond2} hold, in which case:

		\begin{align*}
			\tilde{X^2} &= (\tilde{a}_{uv} - \frac{\hat{s}_{uv}}{t})^2 \frac{t^2}{\hat{s}_{uv}(t-1)}\\
			& = (\hat{a}_{uv} - \nu \cdot N_t - \frac{\hat{s}_{uv}}{t})^2 \frac{t^2}{\hat{s}_{uv}(t-1)}\\
			& \le (a_{uv} - \frac{s_{uv}}{t})^2 \frac{t^2}{s_{uv}(t-1)}\\
			& = X^2 \le \chi_{1-\epsilon/2}^2(1)
		\end{align*}
		Finally, we conclude that
		\begin{align}
			P(\tilde{X^2} > \chi_{1-\epsilon/2}^2(1)) < \epsilon.
		\end{align}
	\end{proof}

	\subsection{Incorporating Relations}

	In this section, we describe our \midas-R approach, which considers edges in a {\bfseries relational} manner: that is, it aims to group together edges that are nearby, either temporally or spatially.

	\textbf{Temporal Relations:} Rather than just counting edges in the same time tick (as we do in \midas), we want to allow for some temporal flexibility: i.e. edges in the recent past should also count toward the current time tick, but modified by reduced weight. A simple and efficient way to do this using our CMS data structures is as follows: at the end of every time tick, rather than resetting our CMS data structures for $a_{uv}$, we scale all its counts by a fixed fraction $\alpha \in (0, 1)$. This allows past edges to count toward the current time tick, with a diminishing weight. Note that we do not consider $0$ or $1$, because $0$ clears all previous values when the time tick changes and hence does not include any temporal effect; and $1$ does not scale the CMS data structures at all.

	\textbf{Spatial Relations:} We would like to catch large groups of spatially nearby edges: e.g. a single source IP address suddenly creating a large number of edges to many destinations, or a small group of nodes suddenly creating an abnormally large number of edges between them. A simple intuition we use is that in either of these two cases, we expect to observe {\bfseries nodes} with a sudden appearance of a large number of edges. Hence, we can use CMS data structures to keep track of edge counts like before, except counting all edges adjacent to any node $u$. Specifically, we create CMS counters $\hat{a}_u$ and $\hat{s}_u$ to approximate the current and total edge counts adjacent to node $u$. Given each incoming edge $(u,v)$, we can then compute three anomaly scores: one for edge $(u,v)$, as in our previous algorithm; one for source node $u$, and one for destination node $v$. Finally, we combine the three scores by taking their maximum value. Another possibility of aggregating the three scores is to take their sum and we discuss the performance of summing the scores in Section \ref{sec:expmidas}. Algorithm \ref{alg:midasr} summarizes the resulting \midas-R algorithm.

	\begin{algorithm}
		\caption{\midas-R:\ Incorporating Relations \label{alg:midasr}}
		\KwIn{Stream of graph edges over time}
		\KwOut{Anomaly scores per edge}
		{\bfseries $\triangleright$ Initialize CMS data structures:} \\
		Initialize CMS for total count $s_{uv}$ and current count $a_{uv}$ \\
		Initialize CMS for total count $s_{u}, s_{v}$ and current count $a_{u}, a_{v}$ \\
		\While{new edge $e=(u,v,t)$ is received:}{
			{\bfseries $\triangleright$ Update Counts:} \\
			Update CMS data structures for the new edge $uv$, source node $u$ and destination node $v$\\
			{\bfseries $\triangleright$ Query Counts:} \\
			Retrieve updated counts $\hat{s}_{uv}$ and $\hat{a}_{uv}$\\
			Retrieve updated counts $\hat{s}_u,\hat{s}_v,\hat{a}_{u},\hat{a}_{v}$\\
			{\bfseries $\triangleright$ Compute Edge Scores:}\\
			$\text{score}(u,v,t) = (\hat{a}_{uv} - \frac{\hat{s}_{uv}}{t})^2 \frac{t^2}{\hat{s}_{uv}(t-1)}$\\
			{\bfseries $\triangleright$ Compute Node Scores:}\\
			$\text{score}(u,t) = (\hat{a}_{u} - \frac{\hat{s}_{u}}{t})^2 \frac{t^2}{\hat{s}_{u}(t-1)}$\\
			$\text{score}(v,t) = (\hat{a}_{v} - \frac{\hat{s}_{v}}{t})^2 \frac{t^2}{\hat{s}_{v}(t-1)}$\\
			{\bfseries $\triangleright$ Final Scores:}\\
			$\textbf{output} \max\{ \text{score}(u,v,t), \text{score}(u,t), \text{score}(v,t) \}$
		}
	\end{algorithm}

	\section{MIDAS-F: Filtering Anomalies}

    In \midas{} and \midas-R, in addition to being assigned an anomaly score, all normal and anomalous edges are also always recorded into the internal CMS data structures, regardless of their score. However, this inclusion of anomalous edges creates a `poisoning' effect which can allow future anomalies to slip through undetected.

	Let us consider a simplified case of a denial of service attack where a large number of edges arrive between two nodes within a short period of time. \midas{} and \midas-R analysis can be divided into three stages.

	In the first stage, when only a small number of such edges have been processed, the difference between the current count, $\hat{a}_{uv}$, and the expected count, $\frac{\hat{s}_{uv}}{t}$, is relatively small, so the anomaly score is low. This stage will not last long as the anomaly score will increase rapidly with the number of occurrences of anomalous edges.

	In the second stage, once the difference between these two counters becomes significant, the algorithm will return a high anomaly score for those suspicious edges.

	In the third stage, as the attack continues, i.e. anomalous edges continue to arrive, the expected count of the anomalous edge will increase. As a result, the anomaly score will gradually decrease, which can lead to false negatives, i.e. the anomalous edges being considered as normal edges, which is the `poisoning' effect due to the inclusion of anomalies in the CMS data structures.

    Therefore, to prevent these false negatives, we introduce the improved filtering \midas{} (\midas-F) algorithm. The following provides an overview:

	\begin{enumerate}
		\item {\bfseries Refined Scoring Function:} The new formula of the anomaly score only considers the information of the current time tick and uses the mean value of the previous time ticks as the expectation.

		\item {\bfseries Conditional Merge:} The current count $a$ for the source, destination and edge are no longer merged into the total count $s$ immediately. We determine whether they should be merged or not at the end of the time tick conditioned on the anomaly score.
	\end{enumerate}

	\subsection{Refined Scoring Function}

    During a time tick, while new edges continue to arrive, we only assign them a score, but do not directly incorporate them into our CMS data structures as soon as they arrive. This prevents anomalous edges from affecting the subsequent anomaly scores, which can possibly lead to false negatives. To solve this problem, we refine the scoring function to delay incorporating the edges to the end of the current time tick using a conditional merge as discussed in Section \ref{sec:ConditionalMerge}.

    As defined before, let $a_{uv}$ be the number of edges from $u$ to $v$ in the current time tick (but not including past time ticks). But unlike \midas{} and \midas-R, in \midas-F, we define $s_{uv}$ to be the total number of edges from $u$ to $v$ up to the previous time tick, not including the current edge count $a_{uv}$. By not including the current edge count immediately, we prevent a high $a_{uv}$ from being merged into $s_{uv}$ so that the anomaly score for anomalous edges is not reduced.

    In the \midas-F algorithm, we still follow the same assumption: that the mean level in the current time tick is the same as the mean level before the current time tick. However, instead of dividing the edges into two classes: past and current time ticks, we only consider the current time ticks. Similar to the chi-squared statistic of \cite{bhatia2020midas}, our statistic is as below.

	\begin{align*}
		X^2&=\frac{(\text{observed}-\text{expected})^2}{\text{expected}} \\
		&=\frac{\displaystyle\left(a_{uv}-\frac{s_{uv}}{t-1}\right)^2}{\displaystyle\frac{s_{uv}}{t-1}} \\
		&=\frac{\displaystyle\left[a_{uv}^2-\frac{2a_{uv}s_{uv}}{t-1}+\left(\frac{s_{uv}}{t-1}\right)^2\right](t-1)}{\displaystyle s_{uv}} \\
		&=\frac{\displaystyle a_{uv}^2(t-1)^2-2a_{uv}s_{uv}(t-1)+s_{uv}^2}{\displaystyle s_{uv}(t-1)} \\
		&=\frac{(a_{uv}+s_{uv}-a_{uv}t)^2}{s_{uv}(t-1)} \\
	\end{align*}

     Both $a_{uv}$ and $s_{uv}$ can be estimated by our CMS data structures, obtaining approximations $\hat{a}_{uv}$ and $\hat{s}_{uv}$ respectively. We will use this new score as the anomaly score for our \midas-F algorithm.

	\begin{defn}[MIDAS-F Anomaly Score]
		Given a newly arriving edge $(u,v,t)$, our anomaly score for this edge is computed as:
		\begin{equation}
			\label{eqn:FilteringCore.AnomalyScore}
			score(u,v,t)=\frac{(\hat{a}_{uv}+\hat{s}_{uv}-\hat{a}_{uv}t)^2}{\hat{s}_{uv}(t-1)}
		\end{equation}
	\end{defn}

	\subsection{Conditional Merge}\label{sec:ConditionalMerge}

    At the end of the current time tick, we decide whether to add $a_{uv}$ to $s_{uv}$ or not based on whether the edge $(u,v)$ appears normal or anomalous.

    We introduce $c_{uv}$ to keep track of the anomaly score. Whenever the time tick changes, if $c_{uv}$ is less than the pre-determined threshold $\varepsilon$, then the corresponding $a_{uv}$ will be added to $s_{uv}$; otherwise, the expected count, i.e., $\frac{s_{uv}}{t-1}$ will be added to $s_{uv}$ to keep the mean level unchanged. We add $a_{uv}$ only when the cached score $c_{uv}$ is less than the pre-determined threshold $\varepsilon$ to prevent anomalous instances of $a_{uv}$ from being added to the $s_{uv}$, which would reduce the anomaly score for an anomalous edge in the future time ticks.

	To store the latest anomaly score $c_{uv}$, we use a CMS-like data structure resembling the CMS data structure for $a$ and $s$ used in \midas{} and \midas-R. The only difference is that the updates to this data structure do not increment the existing occurrence counts, but instead override the previous values. Hereafter, we refer to this CMS-like data structure as CMS for convenience.

	To efficiently merge the CMS data structure for $a$ into the CMS data structure for $s$, we need to know which buckets in the same hash functions across the multiple CMS data structures correspond to a particular edge. However, the algorithm does not store the original edges after processing. Therefore it is necessary that for each entity (edge, source, destination), the three CMS data structures for $a$, $s$, $c$ use the same layout and the same hash functions for each hash table so that the corresponding buckets refer to the same edge and we can do a bucket-wise merge. In practice, the nine CMS data structures can be categorized into three groups, corresponding to the edges, source nodes, and destination nodes, respectively. Only the three CMS data structures within the same group need to share the same structure.

	The conditional merge step is described in Algorithm~\ref{alg:FilteringCore.Merge}.

	\begin{algorithm}[!htb]
		\caption{\textsc{Merge}}
		\label{alg:FilteringCore.Merge}
		\KwIn{CMS for $s$, $a$, $c$, threshold $\varepsilon$}
		\For{$\hat{s}$, $\hat{a}$, $\hat{c}$ from CMS buckets}{
			\If{$\hat{c}<\varepsilon$}{
				$\hat{s} = \hat{s}+\hat{a}$ \\
			}\ElseIf{$t\ne 1$}{
				$\hat{s} = \hat{s}+\dfrac{\hat{s}}{t-1}$ \tcp*{ $\hat{s}$ is up-to-date until $t-1$}
			}
		}
	\end{algorithm}

	We also incorporate temporal and spatial relations as done in \midas-R. For temporal relations, at the end of every time tick, rather than resetting our CMS data structures for $a_{uv}$, we scale all its counts by a fixed fraction $\alpha \in (0, 1)$. This allows past edges to count toward the current time tick, with a diminishing weight. For spatial relations, we use CMS data structures to keep track of the anomaly score of each edge like before, except considering all edges adjacent to any node $u$. Specifically, we create CMS counters $\hat{c}_u$ to keep track of the anomaly score for each node $u$ across all its neighbors. Given each incoming edge $(u,v)$, we can then compute three anomaly scores: one for edge $(u,v)$, as in \midas\ and \midas-R; one for source node $u$, and one for destination node $v$.

	Algorithm \ref{alg:FilteringCore.Algorithm} summarizes the resulting \midas-F algorithm. It can be divided into two parts: 1) regular edge processing in lines $13$ to $24$, where we compute anomaly scores for each incoming edge and update the relevant counts, and 2) scaling and merging steps in lines $6$ to $12$, where at the end of each time tick, we scale the current counts by $\alpha$ and merge them into the total counts.

	\begin{algorithm}
		\caption{\midas-F}
		\label{alg:FilteringCore.Algorithm}
		\KwIn{Stream of graph edges over time, threshold $\varepsilon$}
		\KwOut{Anomaly scores per edge}
		{\bfseries $\triangleright$ Initialize CMS data structures:} \\
		Initialize CMS data structure for total count $s_{uv}$, current count $a_{uv}$, anomaly score $c_{uv}$ \\
		Initialize CMS data structure for total count $s_{u}$, current count $a_{u}$, anomaly score $c_{u}$ \\
		Initialize CMS data structure for total count $s_{v}$, current count $a_{v}$, anomaly score $c_{v}$ \\
		\While{new edge $e=(u,v,t)$ is received}{
			\If(// Time tick changes){$t\ne t_{internal}$}{
				{\bfseries $\triangleright$ Merge Counts:} \\
				\Call{Merge}{$\hat{s}_{uv}$, $\hat{a}_{uv}$, $\hat{c}_{uv}$, $\varepsilon$} \\
				\Call{Merge}{$\hat{s}_u$, $\hat{a}_u$, $\hat{c}_u$, $\varepsilon$} \\
				\Call{Merge}{$\hat{s}_v$, $\hat{a}_v$, $\hat{c}_v$, $\varepsilon$} \\
				Scale CMS data structures for $a_{uv}$, $a_{u}$, $a_{v}$ by $\alpha$ \\
				$t_{internal} = t$ \\
			}
			{\bfseries $\triangleright$ Update Counts:} \\
			Update CMS data structure for $a$ for new edge $uv$ and nodes $u,v$ \\
			{\bfseries $\triangleright$ Query Counts:} \\
			Retrieve updated counts $\hat{s}_{uv}$ and $\hat{a}_{uv}$\\
			Retrieve updated counts $\hat{s}_u,\hat{s}_v,\hat{a}_{u},\hat{a}_{v}$\\
			{\bfseries $\triangleright$ Compute Scores:}\\
			$c_{uv} = \dfrac{(\hat{a}_{uv}+\hat{s}_{uv}-\hat{a}_{uv}t)^2}{\hat{s}_{uv}(t-1)}$ \\
			$c_u = \dfrac{(\hat{a}_u+\hat{s}_u-\hat{a}_{u}t)^2}{\hat{s}_u(t-1)}$ \\
			$c_v = \dfrac{(\hat{a}_v+\hat{s}_v-\hat{a}_{v}t)^2}{\hat{s}_v(t-1)}$ \\
			Update CMS data structure for $c$ for edge $uv$ and nodes $u,v$\\
			{\bfseries $\triangleright$ Final Scores:}\\
			$\textbf{output} \max\{ c_{uv},c_u,c_v\}$
		}
	\end{algorithm}

	\section{Time and Memory Complexity}

	In terms of memory, \midas, \midas-R, and \midas-F only need to maintain the CMS data structures over time, which are proportional to $O(wb)$, where $w$ and $b$ are the number of hash functions and the number of buckets in the CMS data structures; which is bounded with respect to the data size.

    For time complexity, the only relevant steps in Algorithms \ref{alg:midas}, \ref{alg:midasr} and \ref{alg:FilteringCore.Algorithm} are those that either update or query the CMS data structures, which take $O(w)$ (all other operations run in constant time). Thus, time complexity per update step is $O(w)$.

    For \midas{}-F, additionally, at the end of each time tick, $a$ is merged into $s$, as shown in Algorithm \ref{alg:FilteringCore.Merge}. At the end of each time tick, the algorithm needs to iterate over all hash functions and buckets. Thus, time complexity per time tick is $O(wb)$.

\FloatBarrier

	\section{Experiments}
	\label{sec:expmidas}

	In this section, we evaluate the performance of \midas{}, \midas-R, and \midas-F on dynamic graphs. We aim to answer the following questions:

	\begin{enumerate}[label=\textbf{Q\arabic*.}]
		\item {\bfseries Accuracy:} How accurately does \midas{} detect real-world anomalies compared to baselines, as evaluated using the ground truth labels? How will hyperparameters affect the accuracy?
		\item {\bfseries Scalability:} How does it scale with input stream length? How does the time needed to process each input compare to baseline approaches?
		\item {\bfseries Real-World Effectiveness:} Does it detect meaningful anomalies in case studies on \emph{Twitter} graphs?
	\end{enumerate}

	\textbf{Datasets:} \emph{DARPA}~\cite{lippmann1999results} is an intrusion detection dataset created in $1998$.
	It has $25K$ nodes, $4.5M$ edges, and $46K$ timestamps.
	The dataset records IP-IP connections from June $1$ to August $1$.
	Due to the relatively sparse time density, we use minutes as timestamps. \emph{CTU-13}~\cite{garcia2014empirical} is a botnet traffic dataset captured in the CTU University in $2011$.
	It consists of botnet samples from thirteen different scenarios.
	We mainly focus on those with denial of service attacks, i.e., scenarios $4$, $10$, and $11$.
	The dataset includes $371K$ nodes, $2.5M$ edges, and $33K$ timestamps, where the resolution of timestamps is one second. \emph{UNSW-NB15}~\cite{moustafa2015unsw} is a hybrid of real normal activities and synthetic attack behaviors.
	The dataset contains only $50$ nodes but has $2.5M$ records and $85K$ timestamps.
	Each timestamp in the dataset represents an interval of one second. \emph{TwitterSecurity}~\cite{rayana2016less} has $2.6M$ tweet samples for four months (May-Aug $2014$) containing Department of Homeland Security keywords related to terrorism or domestic security.
	Entity-entity co-mention temporal graphs are built on a daily basis. Ground truth contains the dates of major world incidents. \emph{TwitterWorldCup}~\cite{rayana2016less} has $1.7M$ tweet samples for the World Cup $2014$ season (June $12$-July $13$).
	The tweets are filtered by popular/official World Cup hashtags, such as \#worldcup, \#fifa, \#brazil, etc.
	Entity-entity co-mention temporal graphs are constructed on one hour sample rate.

	Note that we use different time tick resolutions for different datasets, demonstrating our algorithm is capable of processing datasets with various edge densities.

	\textbf{Baselines:}

		We compare with \sedanspot, PENminer, and F-FADE, however, as shown in Table~\ref{tab:comparisonmidas}, neither method aims to detect microclusters, or provides guarantees on false positive probability.

\begin{table}[!htb]
		\centering
		\caption{Comparison of relevant edge stream anomaly detection approaches.}
		\label{tab:comparisonmidas}
		\resizebox{\linewidth}{!}{
		\begin{tabular}{@{}lccc|c@{}}
			\toprule
			& {\sedanspot~\cite{eswaran2018sedanspot}}
			& {PENminer~\cite{belth2020mining}}
			& {F-FADE~\cite{chang2021f}}
			& {\bfseries {MIDAS}} \\
			& (ICDM'20) & (KDD'20) & (WSDM'21) & \\\midrule
			\textbf{Microcluster Detection} & & & & \CheckmarkBold \\
			\textbf{Guarantee on False Positive Probability} & & & & \CheckmarkBold \\
			\textbf{Constant Memory} & \Checkmark & & \Checkmark & \CheckmarkBold \\
			\textbf{Constant Update Time} & \Checkmark & \Checkmark & \Checkmark & \CheckmarkBold \\
			\bottomrule
		\end{tabular}}
	\end{table}

	\textbf{Evaluation Metrics:}
	All the methods output an anomaly score per edge (higher is more anomalous).
	We report the area under the receiver operating characteristic curve (ROC-AUC, higher is better).

	\subsection{Experimental Setup}

	All experiments are carried out on a $2.4 GHz$ Intel Core $i9$ processor, $32 GB$ RAM, running OS $X$ $10.15.2$.
	We implement our algorithm in C++ and use the open-source implementations of \sedanspot, PENminer, and F-FADE provided by the authors, following parameter settings as suggested in the original papers.

	We use $2$ hash functions for the CMS data structures, and set the number of CMS buckets to $1024$ to result in an approximation error of $\nu=0.003$.
	For \midas-R and \midas-F, we set the temporal decay factor $\alpha$ as $0.5$.
	For \midas-F, the default threshold $\varepsilon$ is $1000$. We discuss the influence of $\alpha$ and the threshold $\varepsilon$ in the following section. Unless otherwise specified, all experiments are repeated 21 times and the median performance (ROC-AUC, running time, etc.) is reported to minimize the influence of randomization in hashing. Also, note that the reported running time does not include I/O.

	\subsection{Accuracy}

	Table~\ref{tab:Experiment.AUROC} shows the ROC-AUC of \sedanspot, PENminer, F-FADE, \midas, \midas-R, and \midas-F on the \emph{DARPA}, \emph{CTU-13}, and \emph{UNSW-NB15} datasets since only these three datasets have ground truth available for each edge. On \emph{DARPA}, compared to the baselines, \midas\ algorithms increase the ROC-AUC by $6$\%-$53$\%, on \emph{CTU-13} by $13$\%-$62$\%, and on \emph{UNSW-NB15} by $12$\%-$30$\%.
	
	\begin{table*}[!htb]
		\centering
		\caption{ROC-AUC (standard deviation)}\label{tab:Experiment.AUROC}
		\resizebox{\linewidth}{!}{
		\begin{tabular}{lrrrlll}
			\toprule
			Dataset & PENminer & F-FADE & \sedanspot & \midas & \midas-R       & \midas-F       \\
			\midrule
			\emph{DARPA}      & 0.8267   & 0.8451 & 0.6442     & 0.9042 (0.0032) & 0.9514 (0.0012)         & \textbf{0.9873} (0.0009) \\
			\emph{CTU-13}     & 0.6041   & 0.8028 & 0.6397     & 0.9079 (0.0049) & 0.9703 (0.0009)         & \textbf{0.9843} (0.0004) \\
			\emph{UNSW-NB15}  & 0.7028   & 0.6858 & 0.7575     & 0.8843 (0.0079) & \textbf{0.8952} (0.0028) & 0.8517 (0.0013)         \\
			\bottomrule
		\end{tabular}}
	\end{table*}

	Figures \ref{fig:AUCdarpa}, \ref{fig:AUCctu}, and \ref{fig:AUCunsw}  plot the ROC-AUC vs. running time for the baselines and our methods on the \emph{DARPA}, \emph{CTU-13}, and \emph{UNSW-NB15} datasets respectively. Note that \midas, \midas-R, and \midas-F achieve a much higher ROC-AUC compared to the baselines, while also running significantly faster.

	\begin{figure}[!ht]
		\center
		\includegraphics[width=0.8\columnwidth]{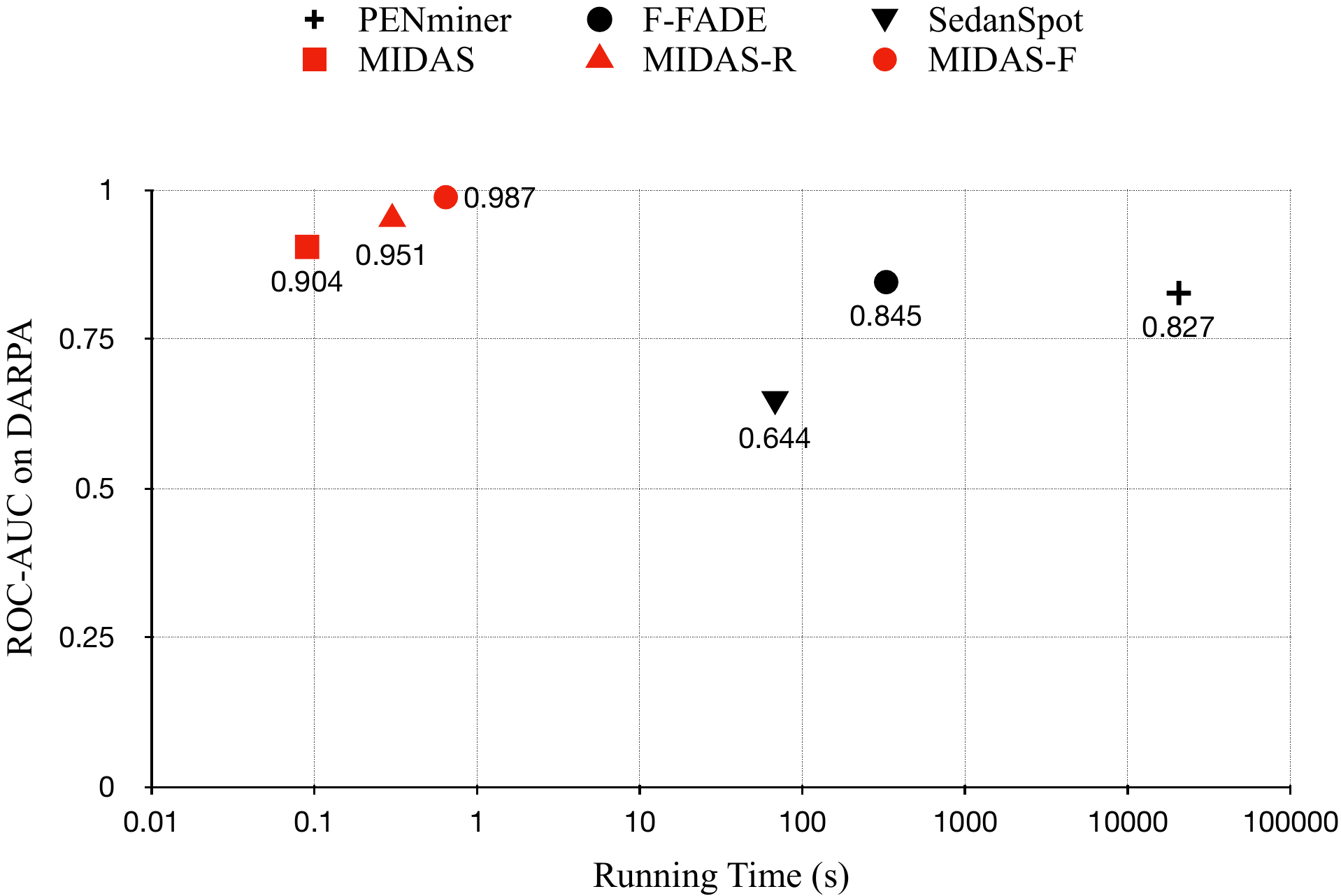}
		\caption{\label{fig:AUCdarpa} ROC-AUC vs. time on \emph{DARPA}}
	\end{figure}

	\begin{figure}[!ht]
		\center
		\includegraphics[width=0.8\columnwidth]{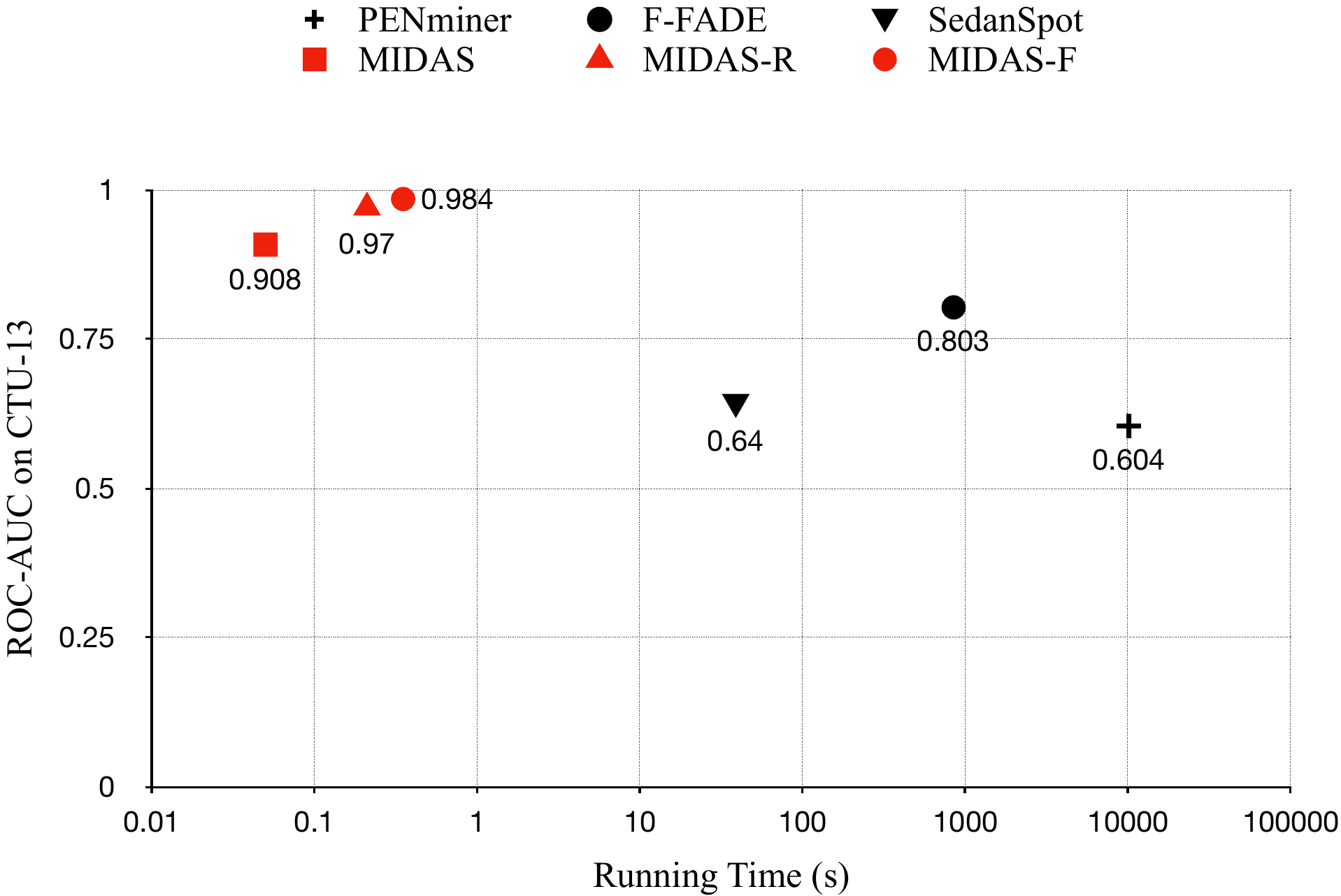}
		\caption{\label{fig:AUCctu} ROC-AUC vs. time on \emph{CTU-13}}
	\end{figure}

	\begin{figure}[!ht]
		\center
		\includegraphics[width=0.8\columnwidth]{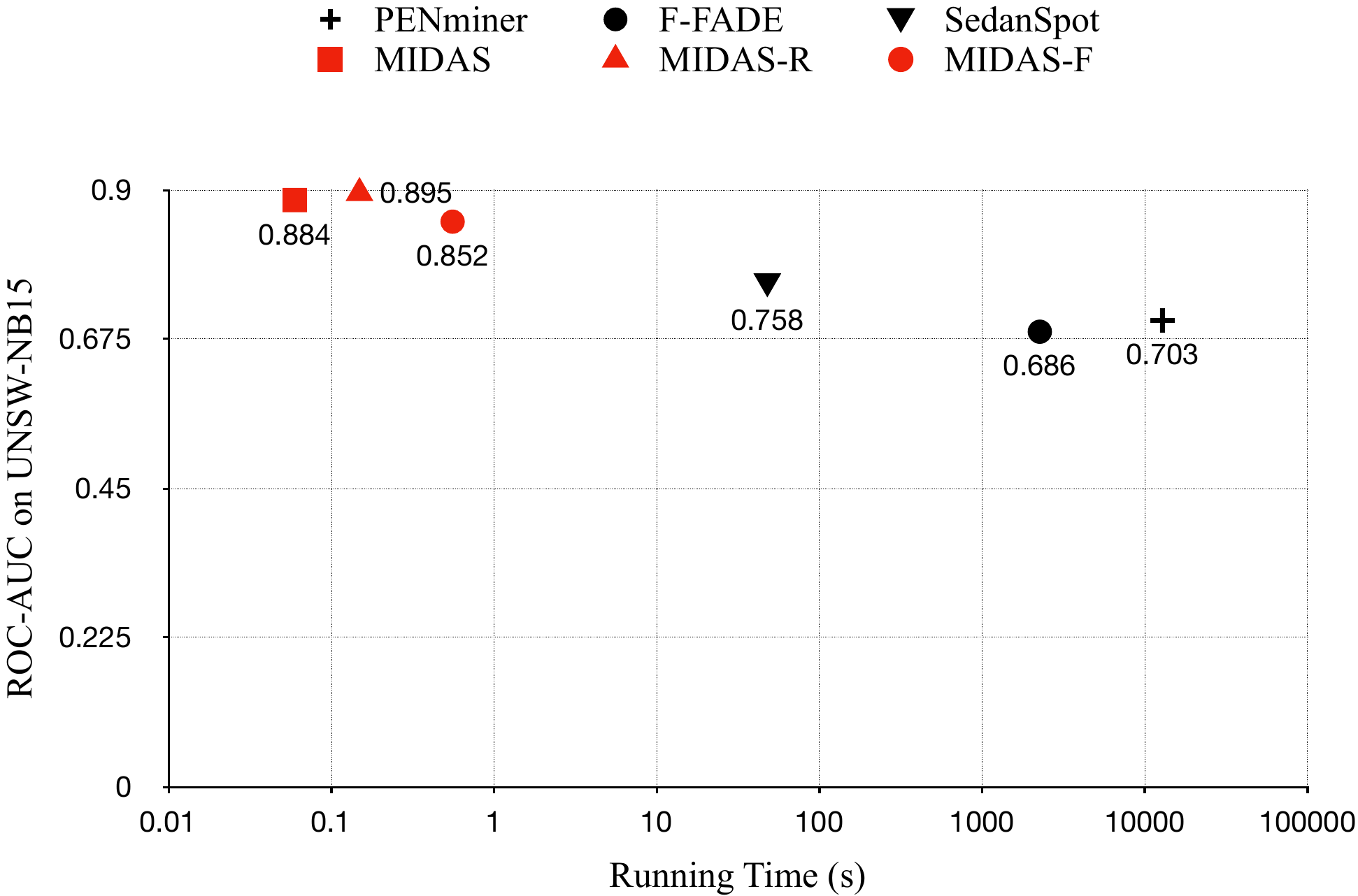}
		\caption{\label{fig:AUCunsw} ROC-AUC vs. time on \emph{UNSW-NB15}}
	\end{figure}

	Table~\ref{tab:FactorVsAUCmidas} shows the influence of the temporal decay factor $\alpha$ on the ROC-AUC for \midas-R and \midas-F in the \emph{DARPA} dataset. Note that instead of scaling the values in the CMS, \midas{} clears (or resets) values in the CMS data structure when the time tick changes; therefore, it is not included. We see that $\alpha=0.9$ gives the maximum ROC-AUC for \midas-R ($0.9657$) and \midas-F ($0.9876$).

	\begin{table}[!htb]
		\centering
		\caption{Influence of temporal decay factor $\alpha$ on the ROC-AUC in \midas-R and \midas-F}
		\label{tab:FactorVsAUCmidas}
		\begin{tabular}{@{}lrr@{}}
			\toprule
			$\alpha$ & \midas-R & \midas-F \\
			\midrule
		    $0.1$ & $0.9346$ & $0.9779$ \\
			$0.2$ & $0.9429$ & $0.9801$ \\
			$0.3$ & $0.9449$ & $0.9817$ \\
			$0.4$ & $0.9484$ & $0.9837$ \\
			$0.5$ & $0.9504$ & $0.9852$ \\
			$0.6$ & $0.9526$ & $0.9863$ \\
			$0.7$ & $0.9542$ & $0.9863$ \\
			$0.8$ & $0.9590$ & $0.9883$ \\
			$0.9$ & $0.9657$ & $0.9876$ \\
			\bottomrule
		\end{tabular}
	\end{table}

    Table~\ref{tab:ThresholdVsAUC} shows the influence of the threshold $\varepsilon$ on the ROC-AUC for \midas-F in the \emph{DARPA} dataset. If the threshold is too low, even normal edges can be rejected. On the other end, if the threshold is too high ($\varepsilon=10^7$), very few anomalous edges will be rejected, and \midas-F (ROC-AUC = $0.9572$) performs similar to \midas-R (ROC-AUC = $0.95$). We see that $\varepsilon=10^3$ achieves the maximum ROC-AUC of $0.9853$.

		\begin{table}[!htb]
		\centering
		\caption{Influence of threshold $\varepsilon$ on the ROC-AUC in \midas-F}
		\label{tab:ThresholdVsAUC}
		\begin{tabular}{@{}lr@{}}
			\toprule
			$\varepsilon$ & ROC-AUC \\
			\midrule
			$10^0$ & $0.9838$ \\
			$10^1$ & $0.9840$ \\
			$10^2$ & $0.9839$ \\
			$10^3$ & $0.9853$ \\
			$10^4$ & $0.9807$ \\
			$10^5$ & $0.9625$ \\
			$10^6$ & $0.9597$ \\
			$10^7$ & $0.9572$ \\
			\bottomrule
		\end{tabular}
	\end{table}

	Table~\ref{tab:BucketsVsAUC} shows the ROC-AUC vs. number of buckets ($b$) in CMSs on the \emph{UNSW-NB15} dataset.
	We can observe the increase in the performance, which indicates that increasing the buckets helps alleviate the effect of conflicts, and further reduce the false positive rate of the resulting scores.
	Also, note that the ROC-AUC does not change after $10,000$ buckets, one possible reason is that the number of columns is sufficiently high to negate the influence of conflicts.
	This also simulates the ``no-CMS'' situation, i.e., the edge counts are maintained in an array of infinite size.

	\begin{table}[!htb]
		\centering
		\caption{Influence of the number of buckets on the ROC-AUC in \midas, \midas-R, and \midas-F}
		\label{tab:BucketsVsAUC}
		\begin{tabular}{@{}lrrr@{}}
			\toprule
			$b$ & \midas & \midas-R & \midas-F \\
			\midrule
		    $10^2$ & $0.7978$ & $0.8161$ & $0.8653$ \\
            $10^3$ & $0.8732$ & $0.8418$ & $0.8863$ \\
            $10^4$ & $0.8842$ & $0.8517$ & $0.8952$ \\
            $10^5$ & $0.8842$ & $0.8517$ & $0.8952$ \\
            $10^6$ & $0.8842$ & $0.8517$ & $0.8952$ \\
            $10^7$ & $0.8842$ & $0.8517$ & $0.8952$ \\
			\bottomrule
		\end{tabular}
	\end{table}

	For \midas{}-R and \midas{}-F, we also test the effect of summing the three anomaly scores, one for the edge $(u,v)$, one for node $u$, and one for node $v$. The scores are not significantly different: with default parameters, the ROC-AUC is $0.95$ for \midas-R (vs. $0.95$ using maximum) and $0.98$ for \midas-F (vs. $0.99$ using maximum).

\FloatBarrier

	\subsection{Scalability}

	Table~\ref{tab:timesmidas} shows the running time for the baselines and \midas\ algorithms.	Compared to \sedanspot, on all the 5 datasets, \midas\ speeds up by $623-800\times$, \midas-R speeds up by $183-326\times$, and \midas-F speeds up by $85-286\times$. Compared to F-FADE, on all the 5 datasets, \midas\ speeds up by $806-37782\times$, \midas-R speeds up by $366-15112\times$, and \midas-F speeds up by $366-4047\times$. Compared to PENminer, on all the 5 datasets, \midas\ speeds up by $101419-214282\times$, and \midas-R speeds up by $46099-85712\times$, \midas-F speeds up by $22958-47324\times$.
	
		\begin{table*}[!htb]
		\centering
		\caption{Running time for different datasets in seconds}
		\label{tab:timesmidas}
		\resizebox{\linewidth}{!}{
		\begin{tabular}{lrrrrrr}
			\toprule
			Dataset & PENminer & F-FADE & \sedanspot & \midas & \midas-R       & \midas-F       \\
			\midrule
			\emph{DARPA}        & $20423$s & $325.1$s & $67.54$s   & $0.09$s & $0.30$s   & $0.64$s   \\
			\emph{CTU-13}       & $10065$s & $844.2$s & $38.73$s   & $0.05$s & $0.21$s   & $0.35$s   \\
			\emph{UNSW-NB15}      & $12857$s & $2267$s  & $48.03$s   & $0.06$s & $0.15$s   & $0.56$s   \\
			\emph{TwitterWorldCup}  & $3786$s  & $141.7$s & $22.92$s   & $0.03$s & $0.07$s   & $0.08$s   \\
			\emph{TwitterSecurity}  & $5071$s  & $40.34$s & $31.18$s   & $0.05$s & $0.11$s   & $0.11$s   \\
			\bottomrule
		\end{tabular}}
	\end{table*}

	\sedanspot{} requires several subprocesses (hashing, random-walking, reordering, sampling, etc), resulting in a large computation time. For PENminer and F-FADE, while the python implementation is a factor, the algorithm procedures also negatively affect their running speed. PENminer requires active pattern exploration and F-FADE needs expensive factorization operations.
	For \midas, the improvement of running speed is through both, the algorithm procedure as well as the implementation. The algorithm procedure is less complicated than baselines; for each edge, the only operations are updating CMSs (hashing) and computing scores, and both are within constant time complexity. The implementation is well optimized and utilizes techniques like auto-vectorization to boost execution efficiency.

	Figure~\ref{fig:scaling} shows the scalability of \midas, \midas-R, and \midas-F algorithms. We plot the time required to process the first $2^{16}, 2^{17},\ldots,2^{22}$ edges of the \emph{DARPA} dataset. This confirms the linear scalability of \midas{} algorithms with respect to the number of edges in the input dynamic graph due to its constant processing time per edge. Note that \midas{}, \midas{}-R and \midas-F can process $4.5M$ edges within $1$ second, allowing real-time anomaly detection.

	\begin{figure}[!htb]
		\center
		\includegraphics[width=0.7\columnwidth]{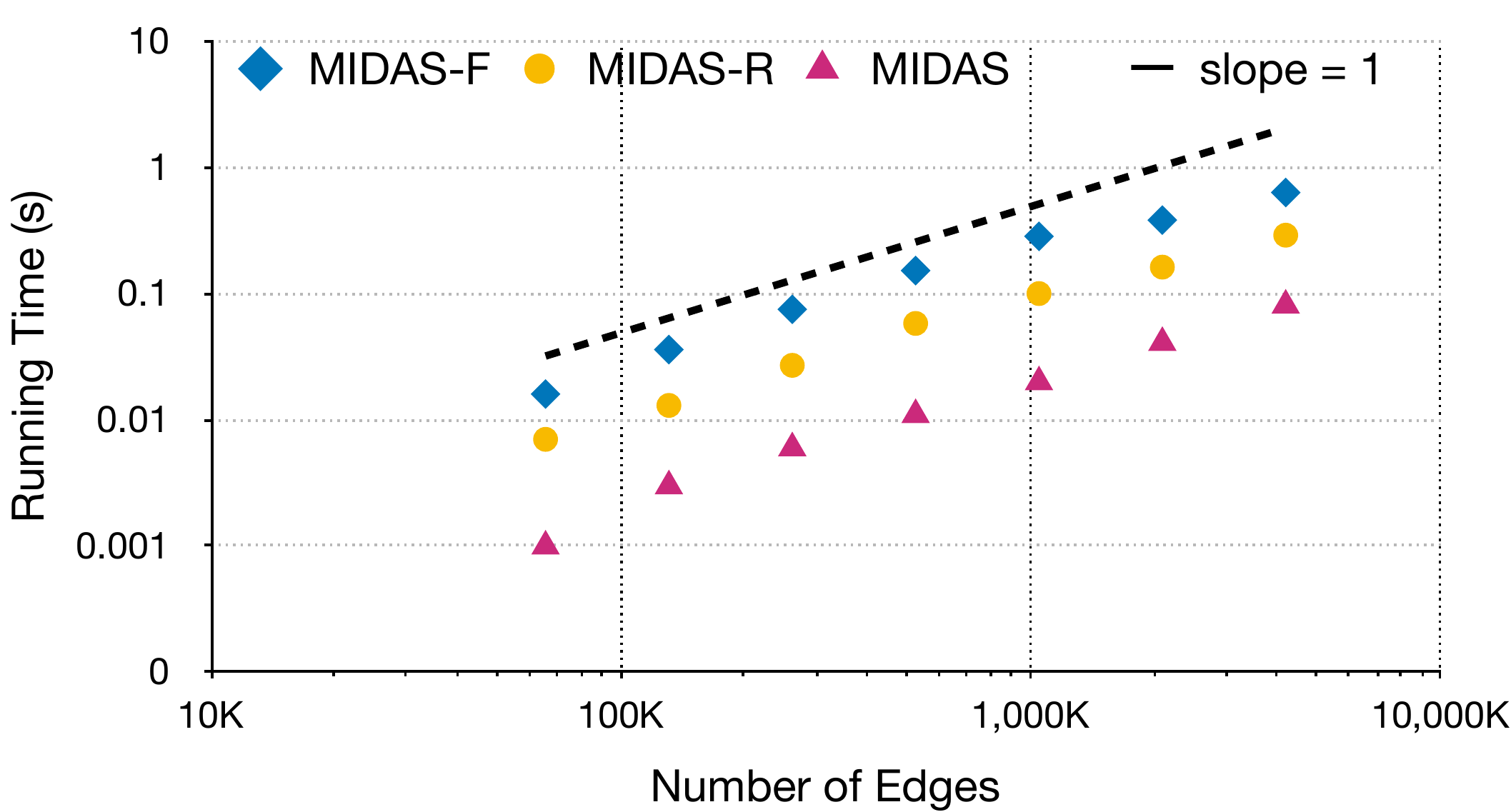}
		\caption{\label{fig:scaling} \midas{}, \midas-R and \midas-F scale linearly with the number of edges in the input dynamic graph.}
	\end{figure}

	Figure~\ref{fig:frequencymidas} plots the number of edges and the time to process each edge in the \emph{DARPA} dataset. Due to the limitation of clock accuracy, it is difficult to obtain the exact time of each edge. But we can approximately divide them into two categories, i.e., less than $1\mu s$ and greater than $1\mu s$. All three methods process majority of the edges within $1\mu s$.

	\begin{figure}[!htb]
		\center
		\includegraphics[width=0.7\columnwidth]{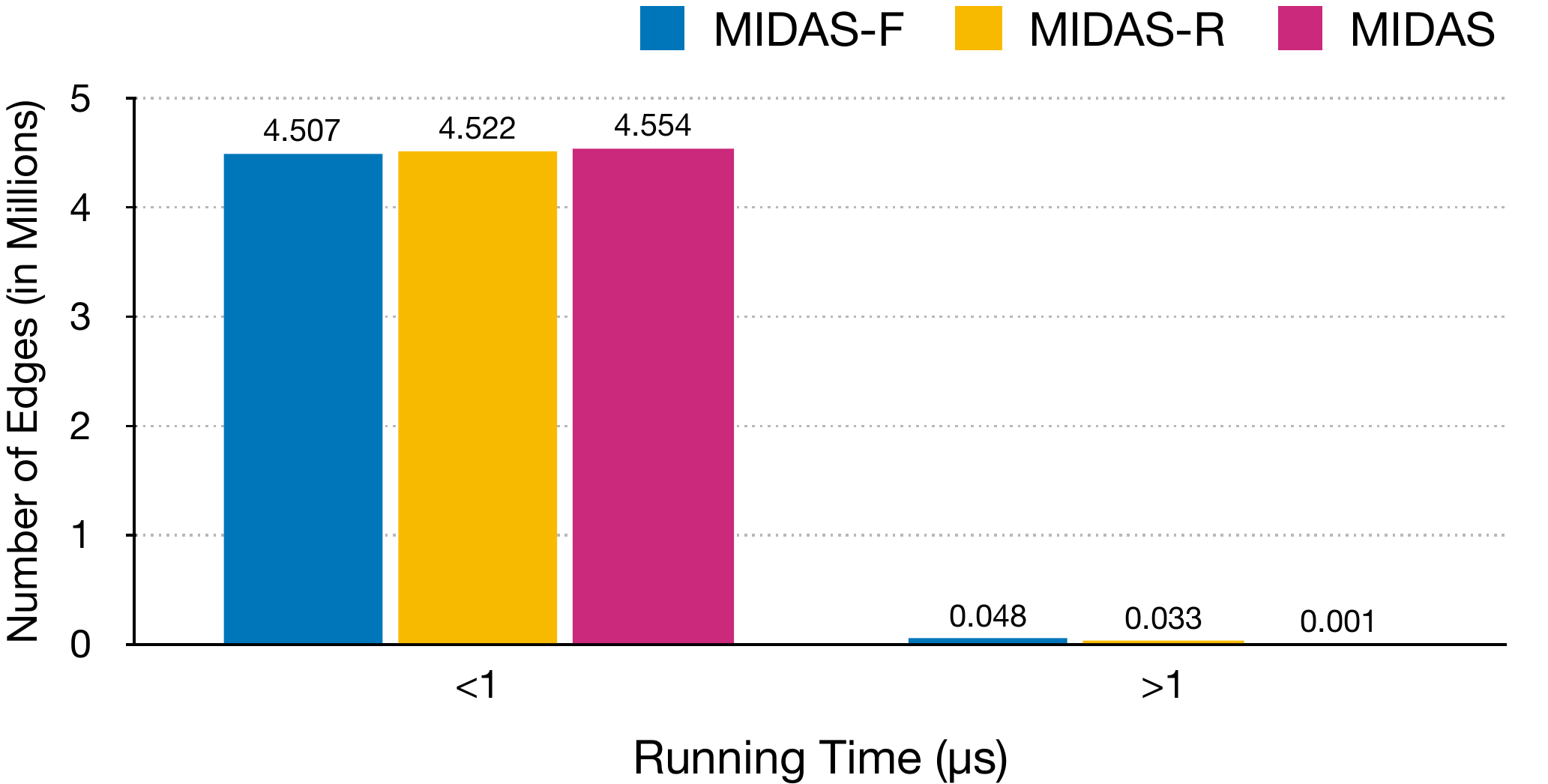}
		\caption{Distribution of processing times for $\sim4.5M$ edges of \emph{DARPA} dataset.}\label{fig:frequencymidas}
	\end{figure}

	Figure~\ref{fig:threshold} shows the dependence of the running time on the threshold for \midas-F. We observe that the general pattern is a line with a slope close to $0$. Therefore, the time complexity does not depend on the threshold.

	\begin{figure}[!htb]
		\center{\includegraphics[width=0.7\columnwidth]{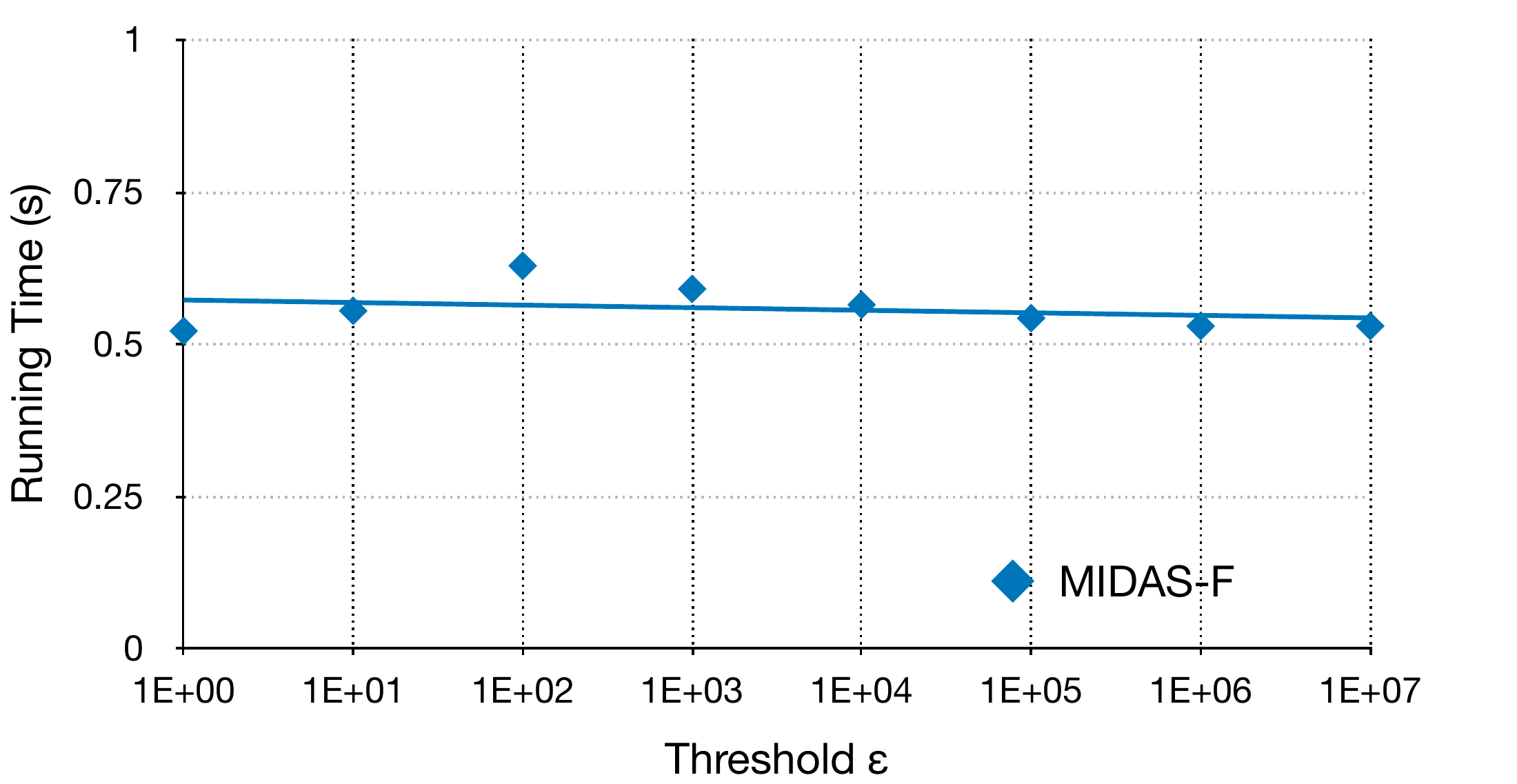}}
		\caption{\label{fig:threshold} Running time of \midas-F does not depend on the threshold $\varepsilon$.}
	\end{figure}

	Figure~\ref{fig:hashfunctions} shows the dependence of the running time on the number of hash functions and linear scalability.

	\begin{figure}[!htb]
		\center{\includegraphics[width=0.7\columnwidth]{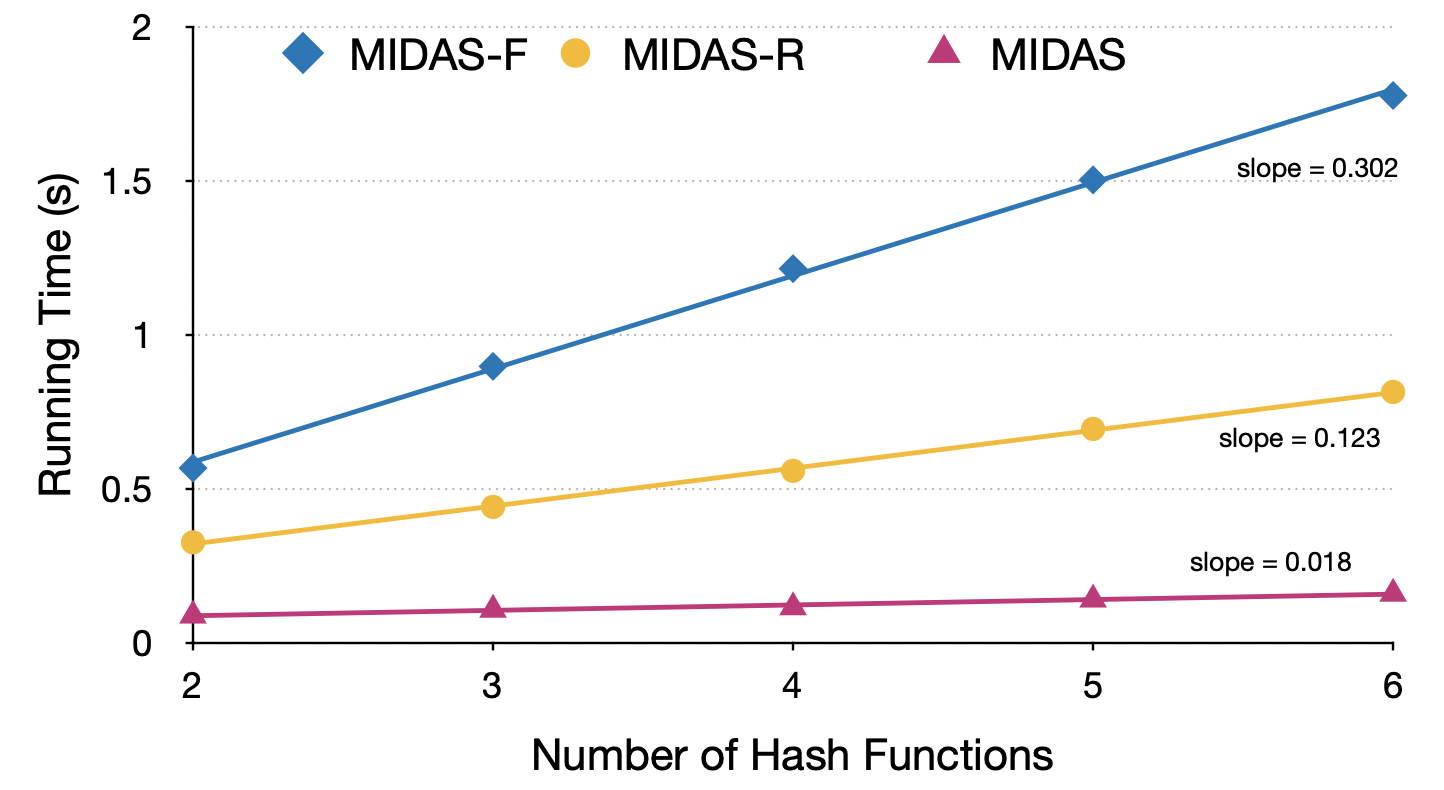}}
		\caption{\label{fig:hashfunctions} \midas{}, \midas-R and \midas-F scale linearly with the number of hash functions.}
	\end{figure}

	Figure~\ref{fig:buckets} shows the dependence of the running time on the number of buckets. In general, the time increases with the number of buckets, but \midas-F is more sensitive to the number of buckets. This is because \midas-F requires updating the CMS data structure, which, due to the nested selection operation, cannot be vectorized. On the other hand, in \midas{} and \midas-R, the clearing and $\alpha$ reducing operations can be efficiently vectorized.

	\begin{figure}[!htb]
		\center{\includegraphics[width=0.7\columnwidth]{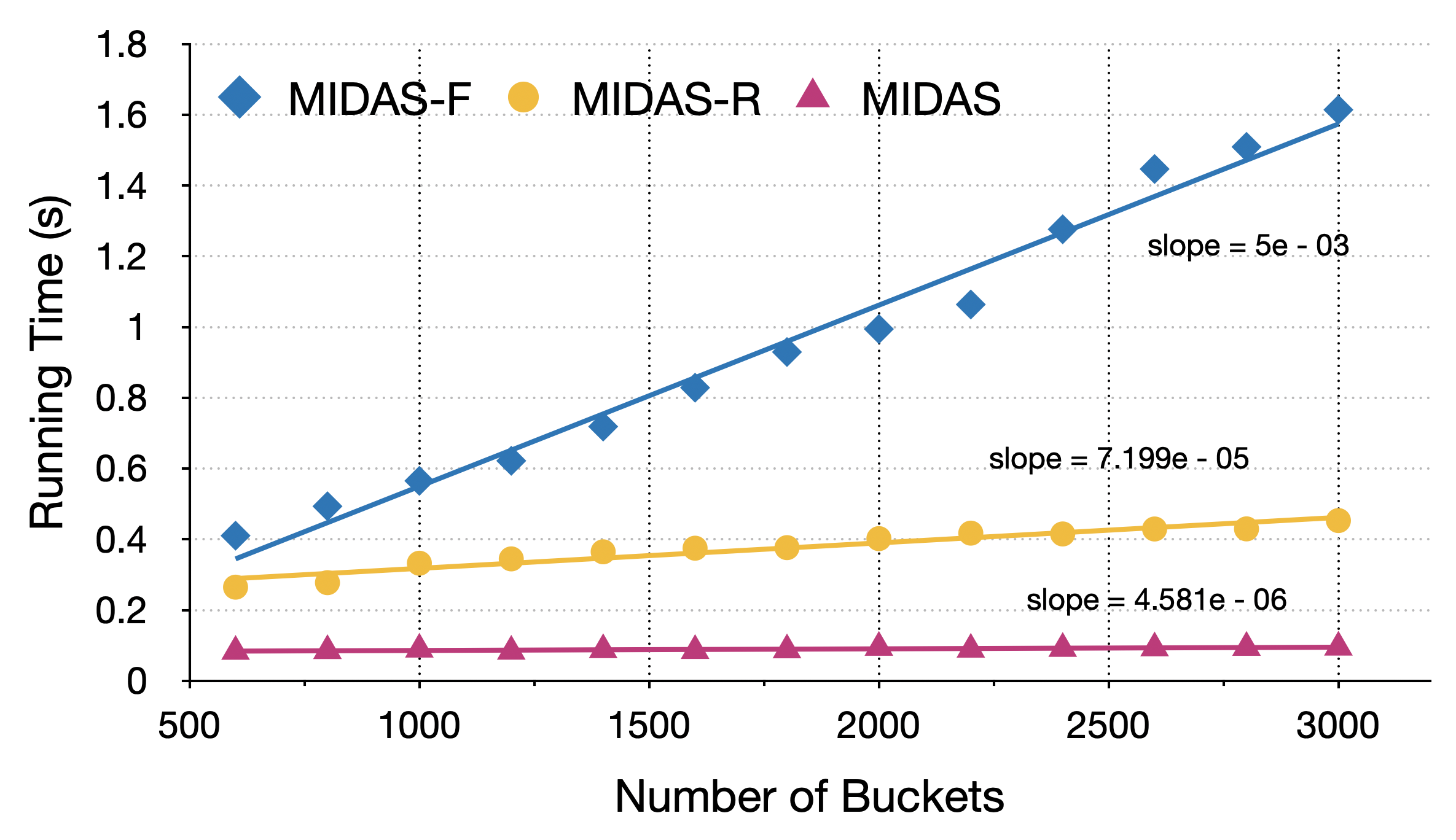}}
		\caption{\label{fig:buckets} \midas{}, \midas-R and \midas-F scale linearly with the number of buckets.}
	\end{figure}

\FloatBarrier

	\subsection{Real-World Effectiveness}

	We measure anomaly scores using \midas, \midas-R, \midas-F, \sedanspot, PENminer, and F-FADE on the \emph{TwitterSecurity} dataset.
	Figure~\ref{fig:security} plots the normalized anomaly scores vs. day (during the four months of 2014).
	We aggregate edges for each day by taking the highest anomaly score.
	Anomalies correspond to major world news such as the Mpeketoni attack (event 6) or the Soma Mine explosion (event 1).

	\begin{figure}[!htb]
		\centering
		\includegraphics[width=0.6\columnwidth]{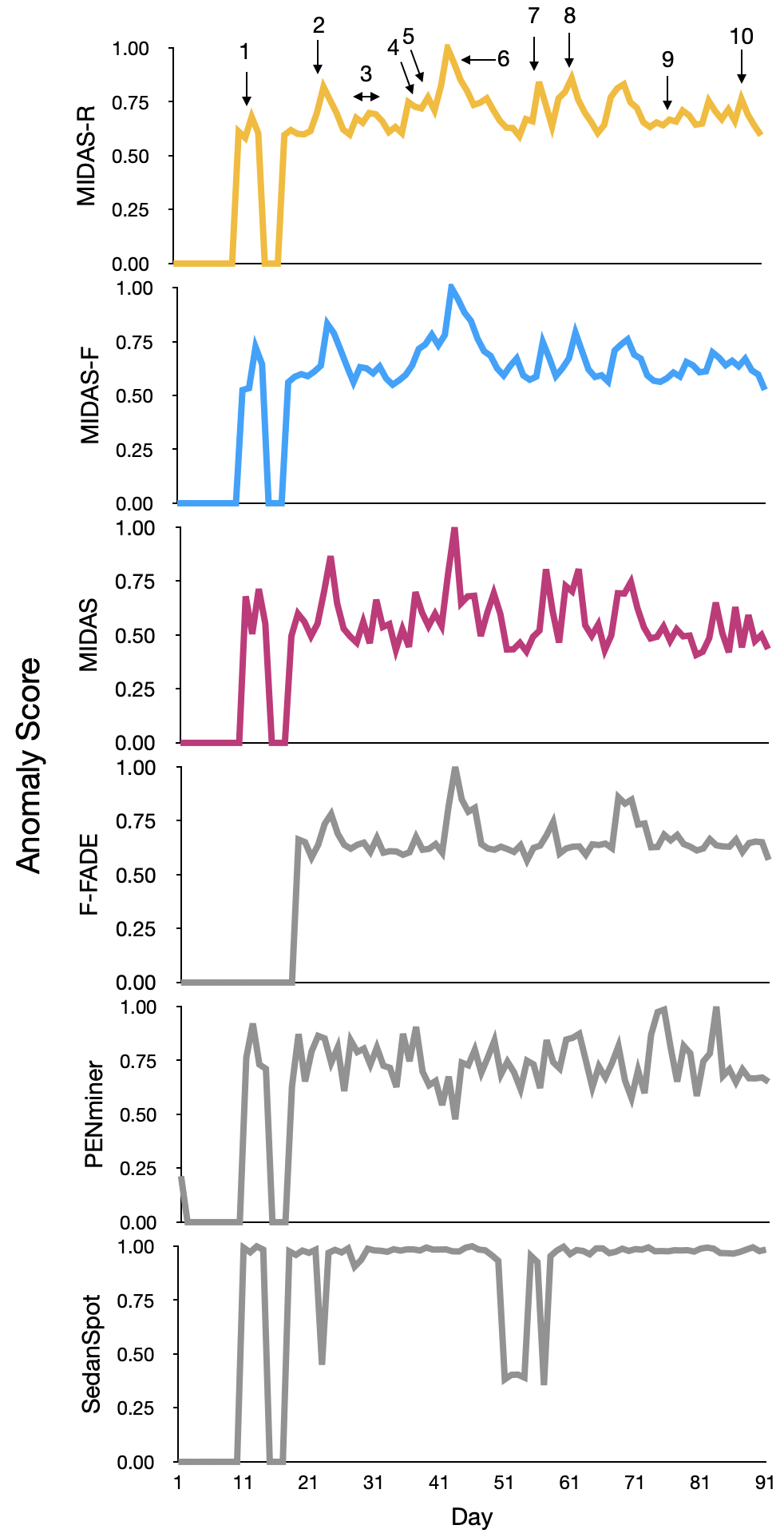}
		\caption{Anomalies detected by \midas, \midas-R and \midas-F correspond to major security-related events in \emph{TwitterSecurity}}\label{fig:security}
	\end{figure}

	\sedanspot\ gives relatively high scores for all days making it difficult to spot anomalies (events). F-FADE produces the highest score near event 6 and peaks at events 2 and 8. However, for other days, scores are maintained around a static level, which provides no useful information in detecting rest events.
	Also note that as F-FADE requires initial learning, thus there are no scores around event 1.
	PENminer's scores keep fluctuating during the four months.
	It would be hard to learn anomalies from the produced scores.
	For \midas\ and its variants, we can see four apparent peaks near major events like 2, 6, 7, 8, and at events 1 and 10, small peaks are also noticeable, though less obvious.
	Hence, we can see our proposed algorithm can extract more anomalous events from real-world social networks compared with baselines.

	The anomalies detected by \midas, \midas-R, and \midas-F coincide with the ground events in the \emph{TwitterSecurity} timeline as follows:

	\begin{enumerate}
	    \footnotesize
		\item 13-05-2014. Turkey Mine Accident, Hundreds Dead.
		\item 24-05-2014. Raid.
		\item 30-05-2014. Attack/Ambush.\\ 03-06-2014. Suicide bombing.
		\item 09-06-2014. Suicide/Truck bombings.
		\item 10-06-2014. Iraqi Militants Seized Large Regions.\\ 11-06-2014. Kidnapping.
		\item 15-06-2014. Attack.
		\item 26-06-2014. Suicide Bombing/Shootout/Raid.
		\item 03-07-2014. Israel Conflicts with Hamas in Gaza.
		\item 18-07-2014. Airplane with 298 Onboard was Shot Down over Ukraine.
		\item 30-07-2014. Ebola Virus Outbreak.
	\end{enumerate}

	\textbf{Microcluster anomalies:} Figure \ref{fig:micro} corresponds to Event $7$ in the \emph{TwitterSecurity} dataset. Single edges in the plot denote $444$ actual edges, while double edges in the plot denote $888$ actual edges between the nodes. This suddenly arriving (within $1$ day) group of suspiciously similar edges is an example of a microcluster anomaly which \midas, \midas-R\ and \midas-F detect, but \sedanspot{} misses.

	\begin{figure}[H]
		\centering
		\includegraphics[width=0.5\columnwidth]{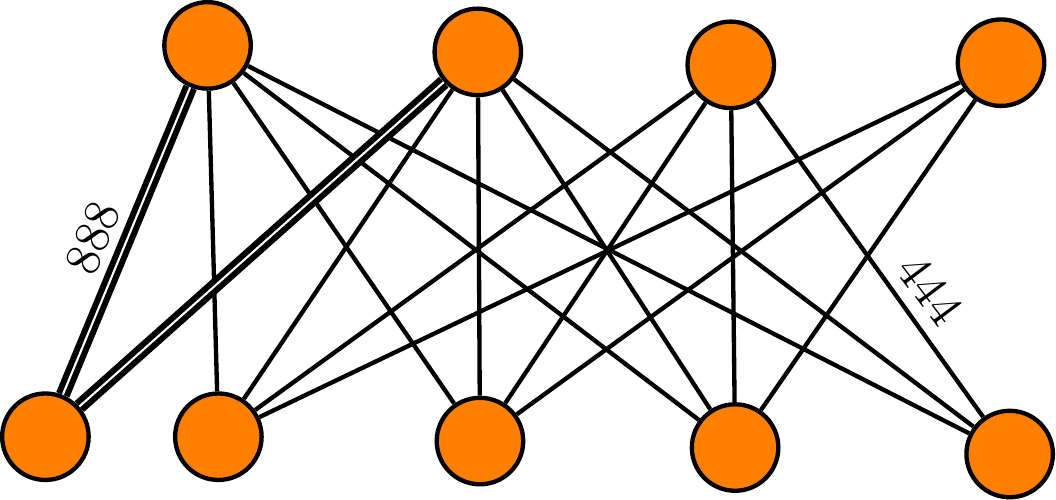}
		\caption{Microcluster Anomaly in \emph{TwitterSecurity}}\label{fig:micro}
	\end{figure}

\FloatBarrier

	\section{Conclusion}
	In this chapter, we proposed \midas, \midas-R, and \midas-F for microcluster based detection of anomalies in edge streams. Future work could consider more general types of data, including heterogeneous graphs or tensors.
	Our contributions are as follows:
	\begin{enumerate}
		\item {\bfseries Streaming Microcluster Detection:} We propose a novel streaming approach combining statistical (chi-squared test) and algorithmic (count-min sketch) ideas to detect microcluster anomalies, requiring constant time and memory.
		\item {\bfseries Theoretical Guarantees:} We show guarantees on the false positive probability of \midas.
		\item {\bfseries Effectiveness}: Our experimental results show that \midas{} outperforms baseline approaches by up to $62$\% higher ROC-AUC, and processes the data orders-of-magnitude faster than baseline approaches.
		\item {\bfseries Relations and Filtering}: We propose two variants, \midas-R that incorporates temporal and spatial relations, and \midas-F that aims to filter away anomalous edges to prevent them from negatively affecting the algorithm's internal data structures.
	\end{enumerate}
\SetPicSubDir{AnoGraph}
\SetExpSubDir{AnoGraph}

\chapter[Sketch-Based Anomaly Detection in Streaming Graphs][AnoGraph]{Sketch-Based Anomaly Detection in Streaming Graphs}
\label{ch:anograph}
\vspace{2em}

\SetKwInput{KwInput}{Input}                
\SetKwInput{KwOutput}{Output}              

\begin{mdframed}[backgroundcolor=magenta!20] 
Chapter based on work that is currently under submission \href{https://arxiv.org/pdf/2106.04486.pdf}{[PDF]}.
\end{mdframed}

\section{Introduction}
Given a stream of graph edges from a dynamic graph, how can we assign anomaly scores to both edges and subgraphs in an online manner, for the purpose of detecting unusual behavior, using constant memory and constant update time per newly arriving edge?

In streaming or online graph scenarios, some methods can detect the presence of anomalous edges, \cite{eswaran2018sedanspot,bhatia2020midas,belth2020mining,chang2021f}, while others can detect anomalous subgraphs \cite{shin2017densealert,eswaran2018spotlight,yoon2019fast}. However, all existing methods are limited to either anomalous edge or graph detection but are not able to detect both kinds of anomalies, as summarized in Table \ref{tab:comparison}.

We first extend the two-dimensional sketch to a \textit{higher-order sketch} to enable it to embed the relation between the source and destination nodes in a graph. A higher-order sketch has the useful property of preserving the dense subgraph structure; dense subgraphs in the input turn into dense submatrices in this data structure. Thus, the problem of detecting a dense subgraph from a large graph reduces to finding a dense submatrix in a constant size matrix, which can be achieved in constant time. The higher-order sketch allows us to propose several algorithms to detect both anomalous edges and subgraphs in a streaming manner.

We introduce two edge anomaly detection methods, \methodedge-G, and \methodedge-L, and two graph anomaly detection methods \methodgraph, and \methodgraph-K, that use the same data structure to detect the presence of a dense submatrix, and consequently anomalous edges, or subgraphs respectively. All our approaches process edges and graphs in constant time, and are independent of the graph size, i.e., they require constant memory. Moreover, our approach is the only streaming method that makes use of dense subgraph search to detect graph anomalies while only requiring constant memory and time. We also provide theoretical guarantees on the higher-order sketch estimate and the submatrix density measure. In summary, the main contributions of this chapter are:
\begin{enumerate}
    \item \textbf{Higher-Order Sketch (Section \ref{sec:sketch})}: We transform the dense subgraph detection problem into finding a dense submatrix (which can be achieved in constant time) by extending the count-min sketch (CMS) \cite{cormode2005improved} data structure to a higher-order sketch.
    \item \textbf{Streaming Anomaly Detection (Sections \ref{sec:edge},\ref{sec:graph})}: We propose four novel online approaches to detect anomalous edges and graphs in real-time, with constant memory and  update time. Moreover, this is the first streaming work that incorporates dense subgraph search to detect graph anomalies in constant memory/time.
    \item \textbf{Effectiveness (Section \ref{sec:exp})}: We outperform all state-of-the-art streaming edge and graph anomaly detection methods on four real-world datasets.
\end{enumerate}

{\bfseries Reproducibility}: Our code and datasets are available on \href{https://github.com/Stream-AD/AnoGraph}{https://github.com/Stream-AD/AnoGraph}.

\begin{table*}[!tb]
\centering
\caption{Comparison of relevant anomaly detection approaches.}
\label{tab:comparison}
\resizebox{\linewidth}{!}{
\addtolength{\tabcolsep}{-2pt}
\begin{tabular}{@{}r|ccccc|ccc|c@{}}
\toprule
{Property} 
& {DenseStream}
& {SedanSpot}
& {MIDAS-R}
& {PENminer}
& {F-FADE}
& {DenseAlert}
& {SpotLight}
& {AnomRank}
& {\bfseries {Our Method}} \\
& (KDD'17) & (ICDM'20) & (AAAI'20) & (KDD'20) & (WSDM'21) & (KDD'17) & (KDD'18) & (KDD'19) & ($2022$) \\\midrule
\textbf{Edge Anomaly} & \Checkmark & \Checkmark & \Checkmark & \Checkmark & \Checkmark & -- & -- & -- & \CheckmarkBold \\
\textbf{Graph Anomaly} & -- & -- & -- & -- & -- & \Checkmark & \Checkmark & \Checkmark & \CheckmarkBold \\
\textbf{Constant Memory} & -- & \Checkmark & \Checkmark & -- & \Checkmark & -- & \Checkmark & -- & \CheckmarkBold \\
\textbf{Constant Update Time} & -- & \Checkmark & \Checkmark & \Checkmark & \Checkmark & -- &  \Checkmark & -- & \CheckmarkBold \\
\textbf{Dense Subgraph Search} & \Checkmark & -- & -- & -- & -- & \Checkmark & -- & -- & \CheckmarkBold \\
\bottomrule
\end{tabular}}
\end{table*}

\section{Problem}
Let $\mathscr{E} = \{e_1, e_2, \cdots\}$ be a stream of weighted edges from a time-evolving graph $\mathcal{G}$. Each arriving edge is a tuple $e_i = (u_i, v_i, w_i, t_i)$ consisting of a source node $u_i \in \mathcal{V}$, a destination node $v_i \in \mathcal{V}$, a weight $w_i$, and a time of occurrence $t_i$, the time at which the edge is added to the graph. For example, in a network traffic stream, an edge $e_i$ could represent a connection made from a source IP address $u_i$ to a destination IP address $v_i$ at time $t_i$. We do not assume that the set of vertices $\mathcal{V}$ is known a priori: for example, new IP addresses or user IDs may be created over the course of the stream.

We model $\mathcal{G}$ as a directed graph. Undirected graphs can be handled by treating an incoming undirected edge as two simultaneous directed edges, one in each direction. We also allow $\mathcal{G}$ to be a multigraph: edges can be created multiple times between the same pair of nodes. Edges are allowed to arrive simultaneously: i.e. $t_{i+1} \ge t_i$, since in many applications $t_i$ is given as a discrete time tick.

The desired properties of our algorithm are as follows:

\begin{itemize}
\item {\bfseries Detecting Anomalous Edges:} To detect whether the edge is part of an anomalous subgraph in an online manner. Being able to detect anomalies at the finer granularity of edges allows early detection so that recovery can be started as soon as possible and the effect of malicious activities is minimized.

\item {\bfseries Detecting Anomalous Graphs:} To detect the presence of an unusual subgraph (consisting of edges received over a period of time) in an online manner, since such subgraphs often correspond to unexpected behavior, such as coordinated attacks. 

\item {\bfseries Constant Memory and Update Time:} To ensure scalability, memory usage and update time should not grow with the number of nodes or the length of the stream. Thus, for a newly arriving edge, our algorithm should run in constant memory and update time.

\end{itemize}

\section{Higher-Order Sketch \& Notations}
\label{sec:sketch}

Count-min sketches (CMS) \cite{cormode2005improved} are popular streaming data structures used by several online algorithms \cite{Mcgregor2014GraphSA}. CMS uses multiple hash functions to map events to frequencies, but unlike a hash table uses only sub-linear space, at the expense of overcounting some events due to collisions. Frequency is approximated as the minimum over all hash functions. CMS, shown in Figure \ref{fig:sketch}(a), is represented as a two-dimensional matrix where each row corresponds to a hash function and hashes to the same number of buckets (columns).

\begin{figure}[!b]
        \centering
        \includegraphics[width=0.8\columnwidth]{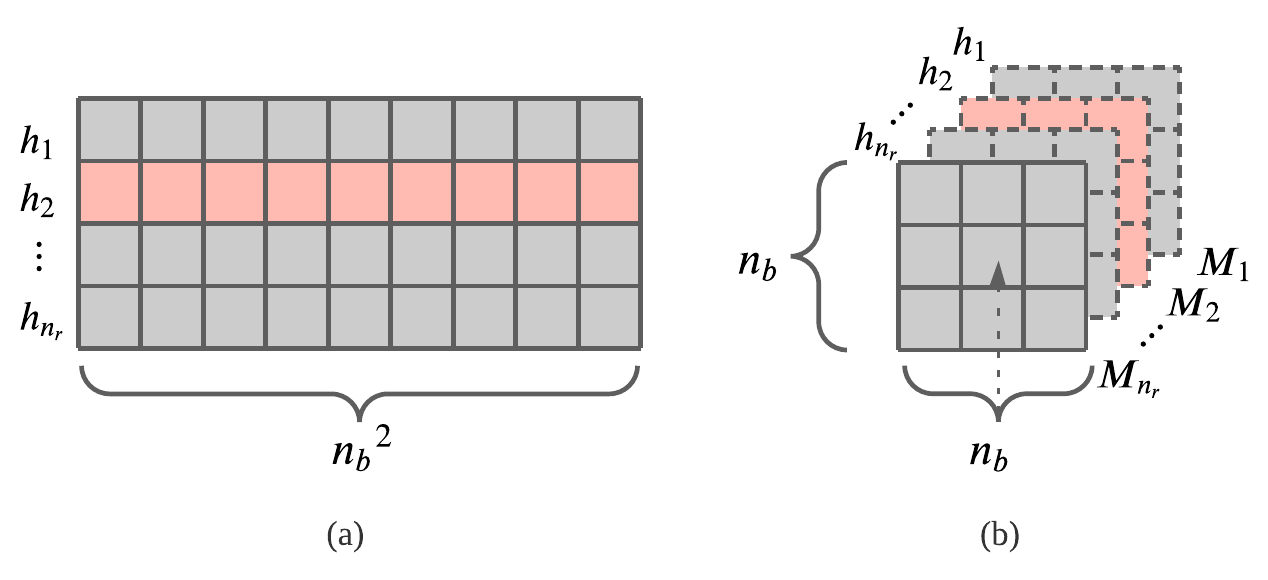} 
        \caption{(a) Original CMS with $n_b^{2}$ buckets for each hash function (b) Higher-order CMS with $n_b x n_b$ buckets for each hash function.}
        \label{fig:sketch}
\end{figure}

We introduce a Higher-order CMS (H-CMS) data structure where each hash function maps multi-dimensional input to a generic tensor instead of mapping it to a row vector. H-CMS enhances CMS by separately hashing the individual components of an entity thereby maintaining more information. Figure \ref{fig:sketch}(b) shows a 3-dimensional H-CMS that can be used to hash two-dimensional entities such as graph edges to a matrix. The source node is hashed to the first dimension and the destination node to the other dimension of the sketch matrix, as opposed to the original CMS that will hash the entire edge to a one-dimensional row vector as shown in Figure \ref{fig:sketch}(a).

We use a 3-dimensional H-CMS (operations described in Algorithm \ref{alg:hcmsoperations}) where the number of hash functions is denoted by $n_r$, and matrix $\mathcal{M}_j$ corresponding to $j$-th hash function $h_j$ is of dimension $n_b \times n_b$, i.e., a square matrix. For each $j \in [n_r]$, the $j$-th hash function denoted by $h_{j}(u,v)$ maps an edge $(u, v)$ to a matrix index $(h'_{j}(u), h''_{j}(v))$, i.e., the source node is mapped to a row index and the destination node is mapped to a column index. That is, $h_{j}(u,v)=(h'_{j}(u), h''_{j}(v))$. Therefore, each matrix in a 3-dimensional H-CMS captures the essence of a graph adjacency matrix. Dense subgraph detection can thus be transformed into a dense submatrix detection problem (as shown in Figure \ref{fig:hcms}) where the size of the matrix is a small constant, independent of the number of edges or the graph size.

\begin{algorithm}[!hb]
    \DontPrintSemicolon
    \SetNoFillComment
    \SetKwFunction{procinit}{INITIALIZE H-CMS}
    \SetKwFunction{procreset}{RESET H-CMS}
    \SetKwFunction{procupdate}{UPDATE H-CMS}
    \SetKwFunction{procdecay}{DECAY H-CMS}
    
	\setcounter{AlgoLine}{0}
    \SetKwProg{myproc}{Procedure}{}{}
    \myproc{\procinit{$n_r$, $n_b$}}{
        \For{$r\gets1$ ... $n_r$}{
            $h_r: \mathcal{V} \rightarrow [0, n_b)$ \tcp*{hash vertex}
            $M_r \rightarrow [0]_{n_b \times n_b}$
        }
    }
    \setcounter{AlgoLine}{0}
    \SetKwProg{myproc}{Procedure}{}{}
    \myproc{\procreset{$n_r$, $n_b$}}{
        \For{$r\gets1$ ... $n_r$}{
            $\mathcal{M}_r \leftarrow [0]_{n_b \times n_b}$ \tcp*{reset to zero matrix} 
        }
    }
    \setcounter{AlgoLine}{0}
    \SetKwProg{myproc}{Procedure}{}{}
    \myproc{\procupdate{$u, v, w$}}{
        \For{$r\gets1$ ... $n_r$} {
            $\mathcal{M}_r[h_r(u)][h_r(v)] \leftarrow \mathcal{M}_r[h_r(u)][h_r(v)] + w$
        }
    }
    \setcounter{AlgoLine}{0}
    \SetKwProg{myproc}{Procedure}{}{}
    \myproc{\procdecay{$\alpha$}}{
        \For{$r\gets1$ ... $n_r$}{
            $\mathcal{M}_r \leftarrow \alpha * \mathcal{M}_r$ \tcp*{decay factor: $\alpha$ }
        }
    }
    
	\caption{H-CMS Operations}
	\label{alg:hcmsoperations}
\end{algorithm}

\begin{figure}[!tb]
  \centering
  \includegraphics[width=0.7\columnwidth]{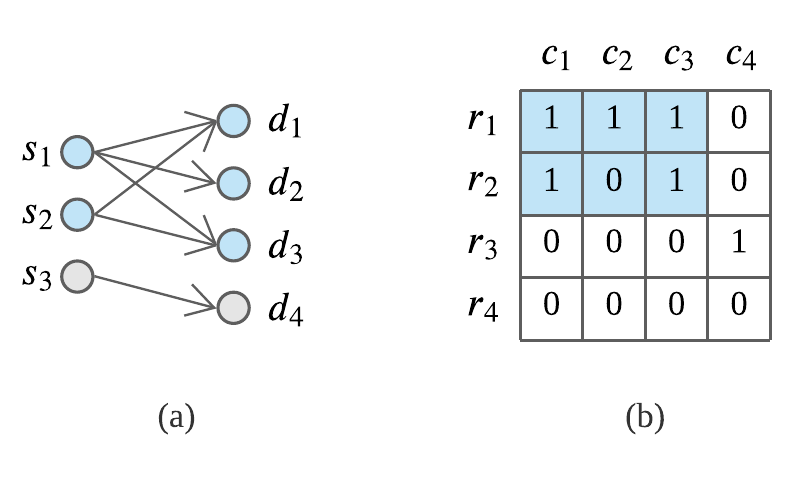}
  \caption{(a) Dense subgraph in the original graph between source nodes $s_1, s_2$, and destination nodes $d_1, d_2, d_3$ is transformed to a (b) Dense submatrix between rows $r_1,r_2$, and columns $c_1, c_2, c_3$ in the H-CMS.}
  \label{fig:hcms}
\end{figure}

For any $(u,v)$, let $y(u,v)$ be the true count of $(u,v)$ observed thus far and $\hat y(u,v)=\min_{j \in [n_r]} \mathcal{M}_j[h_j'(u)][h_j''(v)]$ be the estimate of the count via the 3-dimensional H-CMS. Since the H-CMS can overestimate the count by possible collisions (but not underestimate because we update and keep all the counts for every hash function), we have $y(u,v)\le \hat y(u,v)$. We define $M$ to be  the number of all observations so far; i.e., $M=\sum_{u,v} y(u,v)$.
The following theorem shows that the 3-dimensional H-CMS has estimate guarantees similar to the CMS:
\begin{theorem} \label{thm:1}
For all $k \in [n_r]$, let $h_k(u,v) = (h'_k(u), h''_k(v))$ where each of hash functions $h'_k$ and $h''_k$ is chosen uniformly at random from a pairwise-independent family. Here, we allow both cases of $h'=h''$ and $h'\neq h''$. Fix $\delta>0$ and set  $n_r = \ceil*{\ln \frac{1}{\delta}}$ and  $n_b^{}=\ceil{\frac{e}{\epsilon}}$.  Then, with probability at least $1-\delta$, $ \hat y(u,v)\le y(u,v)+\epsilon M$.
\end{theorem}
\begin{proof}
Fix $j \in [n_r]$. Let  $a=(u_{a}, v_{a})$ and $b=(u_{b}, v_{b})$ such that $a\neq b$. This implies that at least one of the following holds: $u_a \neq u_b$ or  $v_a \neq v_b$. Since $h'_j$ (and $h''_j$) is chosen uniformly at random from a pairwise-independent family, 
$P(h'_{j}(u_{a}) = h'_{j}(u_{b})) = \frac{1}{n_b}$ or $P(h''_{j}(v_{a}) = h''_{j}(v_{b})) = \frac{1}{n_b}$. If $P(h'_{j}(u_{a}) = h'_{j}(u_{b})) = \frac{1}{n_b}$, we have that $P(h_j(a) = h_j(b))=P(h'_j(u_{a}) = h'_j(u_{b}) \wedge h''_j(v_{a}) = h''_j(v_{b}))= P(h'_j(u_{a}) = h'_j(u_{b}))P(h''_j(u_{a}) = h''_j(u_{b})|h'_j(u_{a}) = h'_j(u_{b}) ) \le \frac{1}{n_b}=\frac{\epsilon}{e}$. Similarly, if $P(h''_{j}(v_{a}) = h''_{j}(v_{b})) = \frac{1}{n_b}$,  $P(h_j(a) = h_j(b))= P(h'_j(u_{a}) = h'_j(u_{b})|h''_j(u_{a}) = h''_j(u_{b}))P(h''_j(u_{a}) = h''_j(u_{b}) ) \le \frac{1}{n_b}=\frac{\epsilon}{e}$. 
Thus, in the both cases, the probability of the collision is $P(h_{j}(a) = h_{j}(b))=\frac{\epsilon}{e}$. Thus, by defining $X_{a,j}=\sum_{b} \one\{a \neq b  \wedge h_{j}(a) = h_{j}(b)\}y(b)$,
$
\EE[X_{a,j}]\le\sum_{b}^{} y(b)\EE[\one\{a \neq b  \wedge h_{j}(a) = h_{j}(b)\}]\le \frac{\epsilon}{e}M. 
$     
Since $\hat y(a)= \min_j y(a)+X_{a,j}$, this implies that $P(\hat y(a)> y(a)+\epsilon M)) =P(\min_j y(a)+X_{a,j}> y(a)+\epsilon M))=P(\min_j X_{a,j}>\epsilon M))\le P(\min_j X_{a,j}> e\EE[X_{a,j}])) $.
By  the Markov's inequality on the right-hand side, we have that $P(\hat y(a)> y(a)+\epsilon M)) \le P(\min_j X_{a,j}> e\EE[X_{a,j}]))\le e^{-d}\le \delta$.
\end{proof}

Theorem \ref{thm:1} shows that we have the estimate guarantee even if we use the same hash function for both the source nodes and the destination node (i.e., $h'=h''$). Thus, with abuse of notation, we write $h(u,v)=(h(u), h(v))$ when $h'=h''$ by setting $h=h'=h''$ on the right-hand side. On the other hand, in the case of $h' \neq h''$, it would be possible to improve the estimate guarantee in Theorem \ref{thm:1}. For example, if we can make $h$ to be chosen uniformly at random from a  weakly universal set of hash functions (by defining corresponding families of distributions for  $h'$ and  $h''$ under some conditions), then we can set $n_b=\ceil{\sqrt{\frac{e}{\epsilon}}}$ to have the same estimate guarantee as that of  Theorem \ref{thm:1} based on the proof of Theorem \ref{thm:1}. The analysis for such a potential improvement is left for future work as an open problem.

Frequently used symbols are discussed in Table \ref{tab:symbols}, and we leverage the subgraph density measure discussed in \cite{khuller2009finding} to define the submatrix $(S_x, T_x)$ density.

\begin{definition}
Given matrix $\mathcal{M}$, density of a submatrix of $\mathcal{M}$ represented by $S_x \subseteq S$ and $T_{x} \subseteq T$, is: 
\begin{equation}
\mathcal{D}(\mathcal{M}, S_x, T_x) = \frac{\sum_{s \in S_x}\sum_{t \in T_x}\mathcal{M}[s][t]}{\sqrt{|S_x||T_x|}}
\end{equation}
\label{def:density}
\end{definition}

\begin{table}[!htb]
\centering
\caption{Table of symbols.}
\label{tab:symbols}
\begin{tabular}{>{\centering\arraybackslash}p{3.2cm}|p{10cm}}
\toprule
\textbf{Symbol} & \textbf{Definition} \\
\midrule
$n_r$ & number of hash functions \\
$n_b$ & number of buckets \\
$h(u)$ & hash function $u \rightarrow [0, n_b)$\\ \midrule
$\mathcal{M}$ & a square matrix of dimensions $n_b \times n_b$\\
$\mathcal{M}[i][j]$ & element at row index i and column index j\\ \midrule
$S$ & set of all row indices\\
$S_{cur}$ & set of current submatrix row indices \\
$S_{rem}$ & set of remaining row indices \\
$T$ & set of all column indices \\
$T_{cur}$ & set of current submatrix column indices \\
$T_{rem}$ & set of remaining column indices \\
$[z]$ & set of all integers in the range $[1, z]$ \\\midrule
$\mathcal{D}(\mathcal{M}, S_x, T_x)$ & density of submatrix ($S_x$, $T_x$) \\ 
$\mathcal{E}(\mathcal{M}, S_x, T_x)$ & sum of elements of submatrix ($S_x$, $T_x$)\\ 
$\mathcal{R}(\mathcal{M}, u, T_x)$ & submatrix row-sum \\
& i.e. sum of elements of submatrix ($\{u\}$, $T_x$) \\ 
$\mathcal{C}(\mathcal{M}, S_x, v)$ & submatrix column-sum \\
& i.e. sum of elements of submatrix ($S_x$, $\{v\}$)\\ 
$\mathcal{L}(\mathcal{M}, u, v, S_x, T_x)$ & likelihood of index $(u, v)$ w.r.t. submatrix $(S_x, T_x)$ \\ 
$d_{max}$ & maximum reported submatrix density \\
\bottomrule
\end{tabular}
\end{table}

\section{Edge Anomalies}
\label{sec:edge}

In this section, using the H-CMS data structure, we propose \methodedge-G and \methodedge-L to detect edge anomalies by checking whether the received edge when mapped to a sketch matrix element is part of a dense submatrix. \methodedge-G finds a \textbf{G}lobal dense submatrix and performs well in practice while \methodedge-L maintains and updates a \textbf{L}ocal dense submatrix around the matrix element and therefore has better time complexity.

\subsection{\methodedge-G}

\methodedge-G, as described in Algorithm \ref{alg:AnoEdge-G}, maintains a \emph{temporally decaying} H-CMS, i.e. whenever 1 unit of time passes, we multiply all the H-CMS counts by a fixed factor $\alpha$ (lines 2,4). This decay simulates the gradual `forgetting' of older, and hence, more outdated information. When an edge $(u, v)$ arrives, $u$, $v$ are mapped to matrix indices $h(u)$, $h(v)$ respectively for each hash function $h$, and the corresponding H-CMS counts are updated (line 5). \textsc{Edge-Submatrix-Density} procedure (described below) is then called to compute the density of a dense submatrix around $(h(u), h(v))$. Density is reported as the anomaly score for the edge; a larger density implies that the edge is more likely to be anomalous.

\textsc{Edge-Submatrix-Density} procedure calculates the density of a dense submatrix around a given index $(h(u), h(v))$. A $1 \times 1$ submatrix represented by $S_{cur}$ and $T_{cur}$, is initialized with row-index $h(u)$ and column index $h(v)$ (line 9). The submatrix is iteratively expanded by greedily selecting a row $u_p$ from $S_{rem}$ (or a column $v_p$ from $T_{rem}$) that obtains the maximum row (or column) sum with the current submatrix (lines 11,12). This selected row $u_p$ (or column $v_p$) is removed from $S_{rem}$ (or $T_{rem}$), and added to $S_{cur}$ (or $T_{cur}$) (lines 14,16). The process is repeated until both $S_{rem}$ and $T_{rem}$ are empty (line 10). Density of the current submatrix is computed at each iteration of the submatrix expansion process and the maximum over all greedily formed submatrix densities is returned (lines 17,18).

\begin{algorithm}[!htb]
\caption{\methodedge-G : Streaming Anomaly Edge Scoring}
\label{alg:AnoEdge-G}
    \DontPrintSemicolon
    \SetNoFillComment
    \SetKwFunction{algo}{algo}\SetKwFunction{procf}{\textsc{Edge-Submatrix-Density}}\SetKwFunction{procs}{\textsc{AnoEdge-G}}

    
    \KwInput{Stream $\mathscr{E}$ of edges over time}
    \KwOutput{Anomaly score per edge}
    
    \setcounter{AlgoLine}{0}
    \SetKwProg{myproc}{Procedure}{}{}
    \myproc{\procs{$\mathscr{E}$}}{
    \tcc{H-CMS data structure}
    Initialize H-CMS matrix $\mathcal{M}$ for edge count \\
    \While{new edge $e = (u, v, w, t) \in \mathscr{E}$ is received} {
        \tcc{decay count}
        Temporal decay H-CMS with timestamp change \\
        Update H-CMS matrix $\mathcal{M}$ for new edge $(u, v)$ with value $w$ \tcp*{update count}
        \textbf{output} $score(e) \leftarrow$ \textsc{Edge-Submatrix-Density}($\mathcal{M}, h(u), h(v)$) 
    }
    }
    
    \SetKwProg{myproc}{Procedure}{}{}
    \myproc{\procf{$\mathcal{M}$, $u$, $v$}}{
    $S \leftarrow [n_b]; \enspace T \leftarrow [n_b]; \enspace S_{cur} \leftarrow \{u\}; \enspace T_{cur} \leftarrow \{v\}; \enspace S_{rem} \leftarrow S/\{u\}; \enspace T_{rem} \leftarrow T/\{v\}$ \;
    $d_{max} \leftarrow \mathcal{D}(\mathcal{M}, S_{cur}, T_{cur})$ \;
    \While{$S_{rem} \neq \emptyset \enspace \vee \enspace T_{rem} \neq \emptyset$} {
        \tcc{submatrix max row-sum index}
        $u_p \leftarrow \operatorname*{argmax}_{s_p \in S_{rem}} \mathcal{R}(\mathcal{M}, s_p, T_{cur})$ \\
        \tcc{submatrix max column-sum index}
        $v_p \leftarrow \operatorname*{argmax}_{t_p \in T_{rem}} \mathcal{C}(\mathcal{M}, S_{cur}, t_p)$ \\
        \If{$\mathcal{R}(\mathcal{M}, u_p, T_{cur}) > \mathcal{C}(\mathcal{M}, S_{cur}, v_p)$} {
            $S_{cur} \leftarrow S_{cur}\cup\{u_p\}; \enspace S_{rem} \leftarrow S_{rem}/\{u_p\} $\;
        } 
        \Else {
             $T_{cur} \leftarrow T_{cur}\cup\{v_p\}; \enspace T_{rem} \leftarrow T_{rem}/\{v_p\} $\;
        }
        $d_{max} \leftarrow max(d_{max}, \mathcal{D}(\mathcal{M}, S_{cur}, T_{cur}))$ \;
    }
  
    \KwRet $d_{max}$ \tcp*{dense submatrix density}} 

\end{algorithm}

\begin{proposition}\label{thm:AnoEdge-G-time}
Time complexity of Algorithm \ref{alg:AnoEdge-G} is $O(|\mathscr{E}|*n_r*n_b^2)$ \footnote{This is for processing all edges; the time per edge is constant.}. Memory complexity of Algorithm \ref{alg:AnoEdge-G} is $O(n_r*n_b^2)$.
\end{proposition}

\begin{proof}
Procedure \textsc{Edge-Submatrix-Density} removes rows (or columns) iteratively, and the total number of rows and columns that can be removed is $n_b + n_b - 2$. In each iteration, the approach performs the following three operations: (a) pick the row with minimum row-sum; (b) pick the column with minimum column-sum; (c) calculate density. We keep $n_b$-sized arrays for flagging removed rows (or columns), and for maintaining row-sums (or column-sums). Operations (a) and (b) take maximum $n_b$ steps to pick and flag the row with minimum row-sum (or column-sum). Updating the column-sums (or rows-sums) based on the picked row (or column) again takes maximum $n_b$ steps. Time complexity of (a) and (b) is therefore $O(n_b)$. Density is directly calculated based on subtracting the removed row-sum (or column-sum) and reducing the row-count (or column-count) from the earlier density value. Row-count and column-count are kept as separate variables. Therefore, the time complexity of the density calculation step is $O(1)$. Total time complexity of procedure \textsc{Edge-Submatrix-Density} is $O((n_b + n_b - 2)*(n_b + n_b + 1)) = O(n_b^2)$.

Time complexity to initialize and decay the H-CMS data structure is $O(n_r*n_b^2)$. Temporal decay operation is applied whenever the timestamp changes, and not for every received edge. Update counts operation updates a matrix element value ($O(1)$ operation) for $n_r$ matrices, and the time complexity of this step is $O(n_r)$. Anomaly score for each edge is based on the submatrix density computation procedure which is $O(n_b^2)$; the time complexity of $n_r$ matrices becomes $O(n_r*n_b^2)$. Therefore, the total time complexity of Algorithm \ref{alg:AnoEdge-G} is $O(|\mathscr{E}|*(n_r + n_r*n_b^2)) = O(|\mathscr{E}|*n_r*n_b^2)$.

For procedure \textsc{Edge-Submatrix-Density}, we keep an $n_b$-sized arrays to flag rows and columns that are part of the current submatrix, and to maintain row-sums and column-sums. Total memory complexity of \textsc{Edge-Submatrix-Density} procedure is $O(4*n_b) = O(n_b)$.

Memory complexity of H-CMS data structure is $O(n_r*n_b^2)$. Dense submatrix search and density computation procedure require $O(n_b)$ memory. For $n_r$ matrices, this becomes $O(n_r*n_b)$. Therefore, the total memory complexity of Algorithm \ref{alg:AnoEdge-G} is $O(n_r*n_b^2 + n_r*n_b) = O(n_r*n_b^2)$.
\end{proof}

\subsection{\methodedge-L}
Inspired by Definition \ref{def:density}, we define the likelihood measure of a matrix index $(h(u), h(v))$ with respect to a submatrix $(S_x, T_x)$, as the sum of the elements of submatrix $(S_x, T_x)$ that either share row with index $h''(v)$ or column with index $h'(u)$ divided by the total number of such elements.
\begin{definition}
Given matrix $\mathcal{M}$, likelihood of an index $h(u, v)$ with respect to a submatrix represented by $S_x \subseteq S$ and $T_{x} \subseteq T$, is:
\begin{equation}
\mathcal{L}(\mathcal{M}, u, v, S_x, T_x) = \frac{\sum_{(s, t) \; \in \; \; S_x \times \{h(v)\} \; \cup \; \{h(u)\} \times {T_x}}\mathcal{M}[s][t]}{|S_x \times \{h(v)\} \; \cup \; \{h(u)\} \times {T_x}|}
\end{equation}
\label{def:likelihood}
\end{definition}

\methodedge-L, as described in Algorithm \ref{alg:AnoEdge-L}, maintains a temporally decaying H-CMS to store the edge counts. We also initialize a mutable submatrix of size $1 \times 1$ with a random element, and represent it as $(S_{cur}, T_{cur})$. As we process edges, we greedily update $(S_{cur}, T_{cur})$ to maintain it as a dense submatrix. When an edge arrives, H-CMS counts are first updated, and the received edge is then used to check whether to \emph{expand} the current submatrix (line 7). If the submatrix density increases upon the addition of the row (or column), then the row-index $h(u)$ (or column-index $h(v)$) is added to the current submatrix, $(S_{cur}, T_{cur})$. To remove the row(s) and column(s) decayed over time, the process iteratively selects the row (or column) with the minimum row-sum (or column-sum) until removing it increases the current submatrix density. This ensures that the current submatrix is as \emph{condensed} as possible (line 9). As defined in Definition \ref{def:likelihood}, \methodedge-L computes the likelihood score of the edge with respect to $(S_{cur}, T_{cur})$ (line 10). A higher likelihood measure implies that the edge is more likely to be anomalous.

\begin{algorithm}[!htb]
\caption{\methodedge-L : Streaming Anomaly Edge Scoring}
    \label{alg:AnoEdge-L}
    \DontPrintSemicolon
    \SetNoFillComment
    \SetKwFunction{algo}{algo}\SetKwFunction{procs}{\textsc{AnoEdge-L}}
    
    \KwInput{Stream $\mathscr{E}$ of edges over time}
    \KwOutput{Anomaly score per edge}
    \setcounter{AlgoLine}{0}
    \SetKwProg{myproc}{Procedure}{}{}
    \myproc{\procs{$\mathscr{E}$}}{
    \tcc{H-CMS data structure}
    Initialize H-CMS matrix $\mathcal{M}$ for edges count \\
    \tcc{mutable submatrix}
    Initialize a randomly picked $1 \times 1$ submatrix $(S_{cur}, T_{cur})$ \\
    \While{new edge $e = (u, v, w, t) \in \mathscr{E}$ is received} {
        \tcc{decay count}
        Temporal decay H-CMS with timestamp change \\
        Update H-CMS matrix $\mathcal{M}$ for new edge $(u, v)$ with value $w$ \tcp*{update count}
        \textbf{$\triangleright$ Check and Update Submatrix:} \;
        Expand $(S_{cur}, T_{cur})$ \tcp*{expand submatrix}
        Condense $(S_{cur}, T_{cur})$ \tcp*{condense submatrix}
        \tcc{likelihood score from Definition \ref{def:likelihood}}
        \textbf{output} $score(e) \leftarrow \mathcal{L}(\mathcal{M}, h(u), h(v), S_{cur}, T_{cur})$ 
    }
    }
    
\end{algorithm}

\begin{proposition}\label{thm:AnoEdge-L-time}
Time complexity of Algorithm \ref{alg:AnoEdge-L} is $O(n_r*n_b^2 + |\mathscr{E}|*n_r*n_b)$. Memory complexity of Algorithm \ref{alg:AnoEdge-L} is $O(n_r*n_b^2)$.
\end{proposition}
\begin{proof}
As shown in Proposition \ref{thm:AnoEdge-G-time}, the time complexity of H-CMS is $O(n_r*n_b^2)$ and update operation is $O(n_r)$. Current submatrix $(S_{cur}, T_{cur})$ is updated based on \emph{expand} and \emph{condense} submatrix operations. (a) We keep an $n_b$-sized array to flag the current submatrix rows (or column), and also to maintain row-sums (or column-sums). Expand submatrix operation depends on the elements from row $h(u)$ and column $h(v)$, and the density is calculated by considering these elements, thus requiring maximum $n_b$ steps. Upon addition of the row (or column), the dependent column-sums (or row-sums) are also updated taking maximum $n_b$ steps. Time complexity of expand operation is therefore $O(n_b)$. (b) Condense submatrix operation removes rows and columns iteratively. A row (or column) elimination is performed by selecting the row (or column) with minimum row-sum (or column-sum) in $O(n_b)$ time. Removed row (or column) affects the dependent column-sums (or row-sums) and are updated in $O(n_b)$ time. Time complexity of a row (or column) removal is therefore $O(n_b)$. Condense submatrix removes rows (or columns) that were once added by the expand submatrix operation which in the worse case is $O|\mathscr{E}|$.

Expand and condense submatrix operations are performed for $n_r$ matrices. Likelihood score calculation depends on elements from row $h(u)$ and column $h(v)$, and takes $O(n_r*n_b)$ time for $n_r$ matrices. Therefore, the total time complexity of Algorithm \ref{alg:AnoEdge-L} is $O(n_r*n_b^2 + |\mathscr{E}|*n_r + |\mathscr{E}|*n_r*n_b + |\mathscr{E}|*n_r*n_b + |\mathscr{E}|*n_r*n_b) = O(n_r*n_b^2 + |\mathscr{E}|*n_r*n_b)$.

Memory complexity of the H-CMS data structure is $O(n_r*n_b^2)$. To keep current submatrix information, we utilize $n_b$-sized arrays similar to Proposition \ref{thm:AnoEdge-G-time}. For $n_r$ matrices, submatrix information requires $O(n_r*n_b)$ memory. Hence, total memory complexity of Algorithm \ref{alg:AnoEdge-L} is $O(n_r*n_b^2 + n_r*n_b) = O(n_r*n_b^2)$.
\end{proof}

\section{Graph Anomalies}
\label{sec:graph}
We now propose \methodgraph\ and \methodgraph-K to detect graph anomalies by first mapping the graph to a higher-order sketch, and then checking for a dense submatrix. These are the first streaming algorithms that make use of dense subgraph search to detect graph anomalies in constant memory and time. \methodgraph\ greedily finds a dense submatrix with a 2-approximation guarantee on the density measure. \methodgraph-K leverages \textsc{Edge-Submatrix-Density} from Algorithm \ref{alg:AnoEdge-G} to greedily find a dense submatrix around \textbf{$K$} strategically picked matrix elements performing equally well in practice.

\subsection{\methodgraph}
\methodgraph, as described in Algorithm \ref{alg:AnoGraph}, maintains an H-CMS to store the edge counts that are reset whenever a new graph arrives. The edges are first processed to update the H-CMS counts. \textsc{\methodgraph-Density} procedure (described below) is then called to find the dense submatrix. \methodgraph\ reports anomaly score as the density of the detected (dense) submatrix; a larger density implies that the graph is more likely to be anomalous.

\textsc{\methodgraph-Density} procedure computes the density of a dense submatrix of matrix $\mathcal{M}$. The current dense submatrix is initialized as matrix $\mathcal{M}$ and then the row (or column) from the current submatrix with minimum row (or column) sum is greedily removed. This process is repeated until $S_{cur}$ and $T_{cur}$ are empty (line 11). The density of the current submatrix is computed at each iteration of the submatrix expansion process and the maximum over all densities is returned (lines 18, 19). 

Algorithm \ref{alg:AnoGraph} is a special case of finding the densest subgraph in a directed graph problem \cite{khuller2009finding} where the directed graph is represented as an adjacency matrix and detecting the densest subgraph essentially means detecting dense submatrix. We now provide a guarantee on the density measure.

\begin{lemma}\label{lemma:2approx} Let $S^*$ and $T^*$ be the optimum densest sub-matrix solution of $\mathcal{M}$ with density $\mathcal{D}(\mathcal{M}, S^*, T^*) = d_{opt}$. Then $\forall u \in S^*$ and $\forall v \in T^*$,
\begin{equation}
    \mathcal{R}(\mathcal{M}, u, T^*) \ge \tau_{S^*}; \quad \mathcal{C}(\mathcal{M}, S^*, v) \ge \tau_{T^*}
\end{equation}
\begin{conditions}
    $\tau_{S^*}$ & $\mathcal{E}(\mathcal{M}, S^*, T^*)\left(1- \sqrt{1 - \frac{1}{|S^*|}}\right)$, \\
    $\tau_{T^*}$ & $\mathcal{E}(\mathcal{M}, S^*, T^*)\left(1- \sqrt{1 - \frac{1}{|T^*|}}\right)$ \\
\end{conditions}
\end{lemma}
\begin{proof}
Leveraging the proof from \cite{khuller2009finding}, let's assume that $\exists u \in S^*$ with $\mathcal{R}(\mathcal{M}, u, T^*) < \tau_{S^*}$. Density of submatrix after removing $u = \frac{\mathcal{E}(\mathcal{M}, S^*, T^*) - \mathcal{R}(\mathcal{M}, u, T^*)}{\sqrt{(|S^*-1|)|T^*|}}$ which is greater than $\frac{\mathcal{E}(\mathcal{M}, S^*, T^*) - \tau_{S^*}}{\sqrt{(|S^*-1|)|T^*|}}=d_{opt}$, and that is not possible. Hence, $\mathcal{R}(\mathcal{M}, u, T^*) \ge \tau_{S^*}$. $\mathcal{C}(\mathcal{M}, S^*, v) \ge \tau_{T^*}$ can be proved in a similar manner.
\end{proof}

\begin{theorem}\label{thm:2approx-supp}
\textsc{\methodgraph-Density} procedure in Algorithm \ref{alg:AnoGraph} achieves a 2-approximation guarantee for the densest submatrix problem.
\end{theorem}
\begin{proof} Leveraging the proof from \cite{khuller2009finding}, we greedily remove the row (or column) with minimum row-sum (or column-sum). At some iteration of the greedy process, $\;\forall u \in S_{cur}; \forall v \in T_{cur}$, $\;\mathcal{R}(\mathcal{M}, u, T_{cur}) \ge \tau_{S^*}$ and $\mathcal{C}(\mathcal{M}, S_{cur}, v) \ge \tau_{T^*}$. Therefore, $\mathcal{E}(\mathcal{M}, S_{cur}, T_{cur}) \ge |S_{cur}|\tau_{S^*}$ and $\mathcal{E}(\mathcal{M}, S_{cur}, T_{cur}) \ge |T_{cur}|\tau_{T^*}$. This implies that the density $\mathcal{D}(\mathcal{M}, S_{cur}, T_{cur}) \ge \sqrt{\frac{|S_{cur}|\tau_{S^*}|T_{cur}|\tau_{T^*}}{|S_{cur}||T_{cur}|}} = \sqrt{\tau_{S^*}\tau_{T^*}}$. Putting values of $\tau_{S^*}$ and $\tau_{T^*}$ from Lemma \ref{lemma:2approx}, and setting $|S^*| = \frac{1}{\sin^2\alpha}$, $|T^*| = \frac{1}{\sin^2\beta}$, we get $\mathcal{D}(\mathcal{M}, S_{cur}, T_{cur}) \ge \frac{\mathcal{E}(\mathcal{M, S^*, T^*})}{\sqrt{|S^*||T^*|}}\frac{\sqrt{(1-\cos\alpha)(1-\cos\beta)}}{\sin\alpha\sin\beta} \ge \frac{d_{opt}}{2\cos\frac{\alpha}{2}\cos\frac{\beta}{2}} \ge \frac{d_{opt}}{2}$.
\end{proof}

\begin{algorithm}[!htb]
\caption{\methodgraph: Streaming Anomaly Graph Scoring}
\label{alg:AnoGraph}
    \DontPrintSemicolon
    \SetNoFillComment
    \SetKwFunction{algo}{algo}\SetKwFunction{procf}{\textsc{\methodgraph-Density}}\SetKwFunction{procs}{\methodgraph}

    \KwInput{Stream $\mathscr{G}$ of edges over time}
    \KwOutput{Anomaly score per graph}
    
    \setcounter{AlgoLine}{0}
    \SetKwProg{myproc}{Procedure}{}{}
    \myproc{\procs{$\mathscr{G}$}}{
    \tcc{H-CMS data structure}
    Initialize H-CMS matrix $\mathcal{M}$ for graph edges count \\
    \While{new graph $G \in \mathscr{G}$ is received} {
        \tcc{reset count}
        Reset H-CMS matrix $\mathcal{M}$ for graph $G$ \\
        \For{edge $e = (u, v, w, t) \in G$}{
            Update H-CMS matrix $\mathcal{M}$ for edge $(u, v)$ with value $w$ \tcp*{update count}
        }
        \tcc{anomaly score}
        \textbf{output} $score(G) \leftarrow$ \methodgraph\textsc{-Density}($\mathcal{M}$) 
    }}
    
    \myproc{\procf{$\mathcal{M}$}}{
    $S_{cur} \leftarrow [n_b]; \enspace T_{cur} \leftarrow [n_b]$ \tcp*{initialize to size of $\mathcal{M}$}
    $d_{max} \leftarrow \mathcal{D}(\mathcal{M}, S_{cur}, T_{cur})$ \;
    \While{$S_{cur} \neq \emptyset \enspace \vee \enspace T_{cur} \neq \emptyset$} {
        \tcc{submatrix min row-sum index}
        $u_p \leftarrow \operatorname*{argmin}_{s_p \in S_{cur}} \mathcal{R}(\mathcal{M}, s_p, T_{cur})$ \\
        \tcc{submatrix min column-sum index}
        $v_p \leftarrow \operatorname*{argmin}_{t_p \in T_{cur}} \mathcal{C}(\mathcal{M}, S_{cur}, t_p)$ \\
        \If{$\mathcal{R}(\mathcal{M}, u_p, T_{cur}) < \mathcal{C}(\mathcal{M}, S_{cur}, v_p)$} {
            $S_{cur} \leftarrow S_{cur}/\{u_p\}$ \tcp*{remove row}
        } 
        \Else {
             $T_{cur} \leftarrow T_{cur}/\{v_p\}$ \tcp*{remove column}
        }
        $d_{max} \leftarrow max(d_{max}, \mathcal{D}(\mathcal{M}, S_{cur}, T_{cur}))$ \;
    }
    \KwRet $d_{max}$ \tcp*{dense submatrix density}}
\end{algorithm}

\begin{proposition}\label{thm:AnoGraph-time}
Time complexity of Algorithm \ref{alg:AnoGraph} is $O(|\mathscr{G}|*n_r*n_b^2 + |\mathscr{E}|*n_r)$. Memory complexity of Algorithm \ref{alg:AnoGraph} is $O(n_r*n_b^2)$.
\end{proposition}

\begin{proof}
Procedure \textsc{\methodgraph-Density} iteratively removes row (or column) with minimum row-sum (or column-sum). Maximum number of rows and columns that can be removed is $n_b + n_b - 2$. We keep $n_b$-sized arrays to store the current submatrix rows and columns, and row-sums and column-sums. At each iteration, selecting the row (or column) with minimum row-sum (or column-sum) takes $O(n_b)$ time, and updating the dependent row-sums (or column-sums) also $O(n_b)$ time. Density is calculated in $O(n_b)$ time based on the current submatrix row-sum and column-sum. Each iteration takes $O(n_b + n_b + n_b) = O(n_b)$ time. Hence, the total time complexity of \textsc{\methodgraph-Density} procedure is $O((n_b + n_b - 2)*n_b) = O(n_b^2)$.

Initializing the H-CMS data structure takes $O(n_r*n_b^2)$ time. When a graph arrives, \methodgraph: (a) resets counts that take $O(n_r*n_b^2)$ time; (b) updates counts taking $O(1)$ time for every edge update; (c) computes submatrix density that follows from procedure \textsc{\methodgraph-Density} and takes $O(n_b^2)$ time. Each of these operations is applied for $n_r$ matrices. Therefore, the total time complexity of Algorithm \ref{alg:AnoGraph} is $O(n_r*n_b^2 + |\mathscr{G}|*n_r*n_b^2 + |\mathscr{E}|*n_r + |\mathscr{G}|*n_r*n_b^2) = O(|\mathscr{G}|*n_r*n_b^2 + |\mathscr{E}|*n_r)$, where $|\mathscr{E}|$ is the total number of edges over graphs $\mathscr{G}$.

For procedure \textsc{\methodgraph-Density}, we keep $n_b$-sized array to flag rows and columns that are part of the current submatrix, and to maintain row-sums and column-sums. Hence, memory complexity of \textsc{\methodgraph-Density} procedure is $O(4*n_b) = O(n_b)$. 

H-CMS data structure requires $O(n_r*n_b^2)$ memory. Density computation relies on \textsc{\methodgraph-Density} procedure, and takes $O(n_b)$ memory. Therefore, the total memory complexity of Algorithm \ref{alg:AnoGraph} is $O(n_r*n_b^2)$.
\end{proof}

\subsection{\methodgraph-K}
Similar to \methodgraph, \methodgraph-K maintains an H-CMS which is reset whenever a new graph arrives. It uses the \textsc{\methodgraph-K-Density} procedure (described below) to find the dense submatrix. \methodgraph-K is summarised in Algorithm \ref{alg:AnoGraph-K}.

\textsc{\methodgraph-K-Density} computes the density of a dense submatrix of matrix $\mathcal{M}$. The intuition comes from the heuristic that the matrix elements with a higher value are more likely to be part of a dense submatrix. Hence, the approach considers $K$ largest elements of the matrix $\mathcal{M}$ and calls \textsc{Edge-Submatrix-Density} from Algorithm \ref{alg:AnoEdge-G} to get the dense submatrix around each of those elements (line 13). The maximum density over the considered $K$ dense submatrices is returned.

\begin{algorithm}[!htb]
\caption{\methodgraph-K: Streaming Anomaly Graph Scoring}
\label{alg:AnoGraph-K}
    \DontPrintSemicolon
    \SetNoFillComment
    \SetKwFunction{algo}{algo}\SetKwFunction{procf}{\methodgraph-K-Density}\SetKwFunction{procs}{\textsc{\methodgraph-K}}
    
    \KwInput{Stream $\mathscr{G}$ of edges over time}
    \KwOutput{Anomaly score per graph}
    
    \setcounter{AlgoLine}{0}
    \SetKwProg{myproc}{Procedure}{}{}
    \myproc{\procs{$\mathscr{G}, K$}}{
    \tcc{H-CMS data structure}
    Initialize H-CMS matrix $\mathcal{M}$ for graph edges count \\
    \While{new graph $G \in \mathscr{G}$ is received} {
        \tcc{reset count}
        Reset H-CMS matrix $\mathcal{M}$ for graph $G$ \\
        \For{edge $e = (u, v, w, t) \in G$}{
            Update H-CMS matrix $\mathcal{M}$ for edge $(u, v)$ with value $w$ \tcp*{update count}
        }
        \tcc{anomaly score}
        \textbf{output} $score(G) \leftarrow$ \methodgraph\textsc{-K-Density}($\mathcal{M}, K$) 
    }}
    
    \myproc{\procf{$\mathcal{M}$, $K$}}{
    $B \leftarrow [n_b] \times [n_b]$ \tcp*{set of all indices}
    $d_{max} \leftarrow 0$ \;
    \For{$j\gets1$ ... $K$}{
        \tcc{pick the max element}
        $u_p, v_p \leftarrow \operatorname*{argmax}_{(s_p, t_p) \in B} \mathcal{M}[s_p][t_p]$ \\
        $d_{max} \leftarrow max(d_{max}, \textsc{Edge-Submatrix-Density}({\mathcal{M}, u_p, v_p}))$ \;
        $B \leftarrow B/\{(u_p, v_p)\}$ \tcp*{remove max element index}
    }
    \KwRet $d_{max}$ \tcp*{dense submatrix density}}
\end{algorithm}

\begin{proposition}\label{thm:AnoGraph-K-time}
Time complexity of Algorithm \ref{alg:AnoGraph-K} is $O(|\mathscr{G}|*K*n_r*n_b^2 + |\mathscr{E}|*n_r)$. Memory complexity of Algorithm \ref{alg:AnoGraph-K} is $O(n_r*n_b^2)$.
\end{proposition}
\begin{proof}
Relevant operations in Procedure \textsc{\methodgraph-K-Density} directly follow from \textsc{Edge-Submatrix-Density} procedure, which has $O(n_b^2)$ time complexity. \textsc{Edge-Submatrix-Density} procedure is called $K$ times, therefore, the total time complexity of \textsc{\methodgraph-K-Density} procedure is $O(K*n_b^2)$.

For Algorithm \ref{alg:AnoGraph-K}, we initialize an H-CMS data structure that takes $O(n_r*n_b^2)$ time. When a graph arrives, \methodgraph-K: (a) resets counts that take $O(n_r*n_b^2)$ time; (b) updates counts taking $O(1)$ time for every edge update; (c) computes submatrix density that follows from procedure \textsc{\methodgraph-K-Density} and takes $O(K*n_b^2)$ time. Each of these operations is applied for $n_r$ matrices. Therefore, the total time complexity of Algorithm \ref{alg:AnoGraph-K} is $O(n_r*n_b^2 + |\mathscr{G}|*K*n_r*n_b^2 + |\mathscr{E}|*n_r + |\mathscr{G}|*n_r*n_b^2) = O(|\mathscr{G}|*K*n_r*n_b^2 + |\mathscr{E}|*n_r)$, where $|\mathscr{E}|$ is the total number of edges over graphs $\mathscr{G}$.

The density of $K$ submatrices is computed independently, and the memory complexity of Algorithm procedure \textsc{\methodgraph-K-Density} is the same as the memory complexity of \textsc{Edge-Submatrix-Density} procedure i.e. $O(n_b)$.

Maintaining the H-CMS data structure requires $O(n_r*n_b^2)$ memory. Density computation relies on \textsc{\methodgraph-K-Density} procedure, and it requires $O(n_b)$ memory. Therefore, the total memory complexity of Algorithm \ref{alg:AnoGraph-K} is $O(n_r*n_b^2)$.
\end{proof}

\FloatBarrier

\section{Experiments}
\label{sec:exp}

In this section, we evaluate the performance of our approaches as compared to all baselines discussed in Table \ref{tab:comparison}. 

Table~\ref{tab:Experiment.Dataset} shows the statistical summary of the four real-world datasets that we use: \emph{DARPA} \cite{lippmann1999results} and \emph{ISCX-IDS2012} \cite{shiravi2012toward} are popular datasets for graph anomaly detection used by baselines to evaluate their algorithms; \cite{ring2019survey} surveys more than $30$ datasets and recommends to use the newer \emph{CIC-IDS2018} and \emph{CIC-DDoS2019} datasets \cite{sharafaldin2018toward,sharafaldin2019developing} containing modern attack scenarios. $|E|$ corresponds to the total number of edge records, $|V|$ and $|T|$ are the number of unique nodes and unique timestamps, respectively. All edge (or graph)-based methods output an anomaly score per edge (or graph), a higher score implying more anomalousness. Similar to baseline papers, we report the Area under the ROC curve (AUC) and the running time. Unless explicitly specified, all experiments including those on the baselines are repeated $5$ times and the mean is reported. We aim to answer the following questions:
\begin{enumerate}[label=\textbf{Q\arabic*.}]
\item {\bfseries Edge Anomalies:} How accurately do \methodedge-G and \methodedge-L detect edge anomalies compared to baselines? Are they fast and scalable?
\item {\bfseries Graph Anomalies:} How accurately do \methodgraph\ and \methodgraph-K detect graph anomalies i.e. anomalous graph snapshots? Are they fast and scalable?
\end{enumerate}

\begin{table}[!htb]
		\centering
		\caption{Statistics of the datasets.}
		\begin{tabular}{lrrrrr}
			\toprule
			\textbf{Dataset} & $|V|$           & $|E|$             & $|T|$  \\
			\midrule
			DARPA               & \numprint{25525}  & \numprint{4554344}  & \numprint{46567}  \\
			ISCX-IDS2012         & \numprint{30917}  & \numprint{1097070}  & \numprint{165043} \\
			CIC-IDS2018      & \numprint{33176}  & \numprint{7948748}  & \numprint{38478} \\
			CIC-DDoS2019         & \numprint{1290}   & \numprint{20364525} & \numprint{12224} \\
			\bottomrule
		\end{tabular}
		\label{tab:Experiment.Dataset}
	\end{table}

\paragraph{Experimental Setup}\label{sec:setup}
	All experiments are carried out on a $2.4 GHz$ Intel Core $i9$ processor, $32 GB$ RAM, running OS $X$ $10.15.3$. For our approach, we keep $n_r=2$ and $n_b=32$ to have a fair comparison to MIDAS which uses ${n_b}^2=1024$ buckets. Temporal decay factor $\alpha=0.9$ for Algorithms $\ref{alg:AnoEdge-G}$ and $\ref{alg:AnoEdge-L}$. We keep $K=5$ for Algorithm $\ref{alg:AnoGraph-K}$. AUC for graph anomalies is shown with edge thresholds as $50$ for \emph{DARPA} and $100$ for other datasets. Time window is taken as $30$ minutes for \emph{DARPA} and $60$ minutes for other datasets.
 
\paragraph{Baselines}\label{sec:baselines}

We use open-source implementations of \densestream\ \cite{shin2017densealert} (Java), \sedanspot\ \cite{eswaran2018sedanspot} (C++), MIDAS-R \cite{bhatia2020midas} (C++), PENminer \cite{belth2020mining} (Python), F-FADE \cite{chang2021f} (Python), \densealert\ \cite{shin2017densealert} (Java), and \anomrank\ \cite{yoon2019fast} (C++) provided by the authors, following parameter settings as suggested in the original paper. For \spotlight\ \cite{eswaran2018spotlight}, we used open-sourced implementations of Random Cut Forest \cite{awsrando88:online} and Carter Wegman hashing \cite{carter1979universal}.
    
\subparagraph {Edge Anomalies}
    
    \begin{enumerate}
    \item {{\sedanspot:}}
		\texttt{sample\_size} $=10000$,
		\texttt{num\_walk} $=50$,
		\texttt{restart\_prob} $0.15$

    \item{{MIDAS:}} The size of CMSs is 2 rows by 1024 columns for all the tests.
	For MIDAS-R, the decay factor $\alpha=0.6$.

	\item{{PENminer:}}
		\texttt{ws} $=1$,
		\texttt{ms} $=1$,
		\texttt{view} = \texttt{id},
		\texttt{alpha} $=1$,
		\texttt{beta} $=1$,
		\texttt{gamma} $=1$

	\item{{\densestream:}} We keep default parameters, i.e., order $=3$.
	
	\item{{F-FADE:}}
		\texttt{embedding\_size} $=200$,
		\texttt{W\_upd} $=720$,
		\texttt{T\_th} $=120$,
		\texttt{alpha} $=0.999$,
		\texttt{M} $=100$

	For \texttt{t\_setup}, we always use the timestamp value at the $10^{th}$ percentile of the dataset.
	\end{enumerate}
	
\subparagraph {Graph Anomalies}
	
	\begin{enumerate}

	\item{{\spotlight:}}
	\texttt{K} $=50$,
	\texttt{p} $=0.2$,
	\texttt{q} $=0.2$

	\item{{\densealert:}} We keep default parameters, i.e., order $=3$ and window=$60$.
	
	\item{{\anomrank:}} We keep default parameters, i.e., damping factor c $= 0.5$, and L1 changes of node score vectors threshold epsilon $= 10^{-3}$.
	We keep ${1/4}^{th}$ number of graphs for initializing mean/variance as mentioned in the respective paper.
	
    \end{enumerate}   

\subsection{Edge Anomalies}

{\bfseries Accuracy:} Table \ref{tab:edgeanograph} shows the AUC of edge anomaly detection baselines, \methodedge-G, and \methodedge-L. We report a single value for \densestream\ and PENminer because these are non-randomized methods. PENminer is unable to finish on the large \emph{CIC-DDoS2019} within 24 hours; thus, that result is not reported. \sedanspot\ uses personalized PageRank to detect anomalies and is not always able to detect anomalous edges occurring in dense block patterns while PENminer is unable to detect structural anomalies. Among the baselines, MIDAS-R is the most accurate, however, it performs worse when there is a large number of timestamps as in \emph{ISCX-IDS2012}. Note that \methodedge-G and \methodedge-L outperform all baselines on all datasets.

\begin{table*}[!htb]
\centering
\caption{AUC and Running Time when detecting edge anomalies. Averaged over $5$ runs.}
\label{tab:edgeanograph}
\resizebox{\columnwidth}{!}{
\begin{tabular}{@{}lrrrrrrr}
\toprule
  Dataset & \densestream & \sedanspot & MIDAS-R & PENminer & F-FADE & \textbf{\methodedge-G} & \textbf{\methodedge-L} \\ 
\midrule
\multirow{2}{*}{DARPA} & $0.532$ & $0.647 \pm 0.006$ & $0.953 \pm 0.002$ &  0.872 & $0.919 \pm 0.005$ & $\mathbf 0.970 \pm 0.001$ & $ 0.964 \pm 0.001$ \\ 
& 57.7s & 129.1s  & 1.4s & 5.21 hrs & 317.8s & 28.7s & 6.1s \\ \midrule
\multirow{2}{*}{ISCX-IDS2012} & $0.551$ & $0.581 \pm 0.001$ & $0.820 \pm 0.050$ & 0.530 & $0.533 \pm 0.020$ & $0.954 \pm 0.000$ & $\mathbf 0.957 \pm 0.003$ \\ 
 & 138.6s & 19.5s & 5.3s & 1.3 hrs & 137.4s & 7.8s & 0.7s \\  \midrule
\multirow{2}{*}{CIC-IDS2018} & $0.756$ & $0.325 \pm 0.037$ & $0.919 \pm 0.019$ &  0.821 & $0.607 \pm 0.001$ & $\mathbf 0.963 \pm 0.014$ & $0.927 \pm 0.035$ \\ 
 & 3.3 hours & 209.6s & 1.1s & 10 hrs & 279.7s & 58.4s & 10.2s \\  \midrule
\multirow{2}{*}{CIC-DDoS2019} & $0.263$ & $0.567 \pm 0.004$ & $0.983 \pm 0.003$ &  --- & $0.717 \pm 0.041$ & $0.997 \pm 0.001$ & $\mathbf 0.998 \pm 0.001$ \\
 & 265.6s & 697.6s & 2.2s & > 24 hrs & 18.7s & 123.3s & 17.8s \\
\bottomrule
\end{tabular}
}
\end{table*}

{\bfseries Running Time:} Table \ref{tab:edgeanograph} shows the running time (excluding I/O) and real-time performance of \methodedge-G and \methodedge-L. Since \methodedge-L maintains a local dense submatrix, it is faster than \methodedge-G. \densestream\ maintains dense blocks incrementally for every coming tuple and updates dense subtensors when it meets an updating condition, limiting the detection speed. \sedanspot\ requires several subprocesses (hashing, random-walking, reordering, sampling, etc), PENminer and F-FADE need to actively extract patterns for every graph update, resulting in a large computation time. When there is a large number of timestamps like in \emph{ISCX-IDS2012}, MIDAS-R performs slower than \methodedge-L which is the fastest.

{\bfseries AUC vs Running Time:} Figure \ref{fig:edgea} plots accuracy (AUC) vs. running time (log scale, in seconds, excluding I/O) on \emph{ISCX-IDS2012} dataset. \methodedge-G and \methodedge-L achieve much higher accuracy compared to all baselines, while also running significantly faster.

\begin{figure}[!htb]
  \centering
    \includegraphics[width=0.7\textwidth]{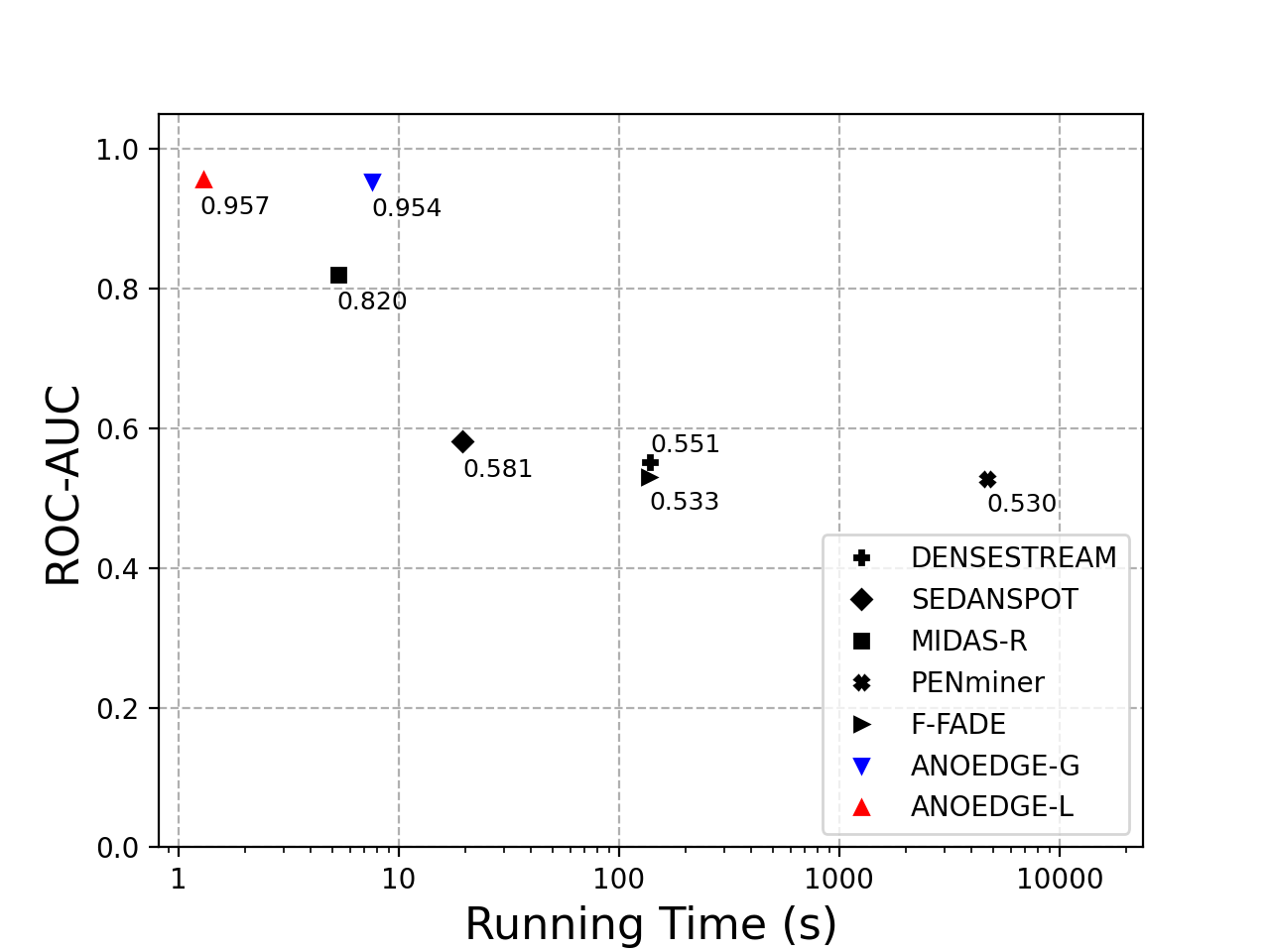}
   \caption{AUC vs running time when detecting edge anomalies on \emph{ISCX-IDS2012}.}
   \label{fig:edgea}
\end{figure}

{\bfseries Scalability:} Figures \ref{fig:edgeb} and \ref{fig:edgec} plot the running time with increasing number of hash functions and edges respectively, on the \emph{ISCX-IDS2012} dataset. This demonstrates the scalability of \methodedge-G and \methodedge-L.

\begin{figure}[!htb]
  \centering
    \includegraphics[width=0.7\textwidth]{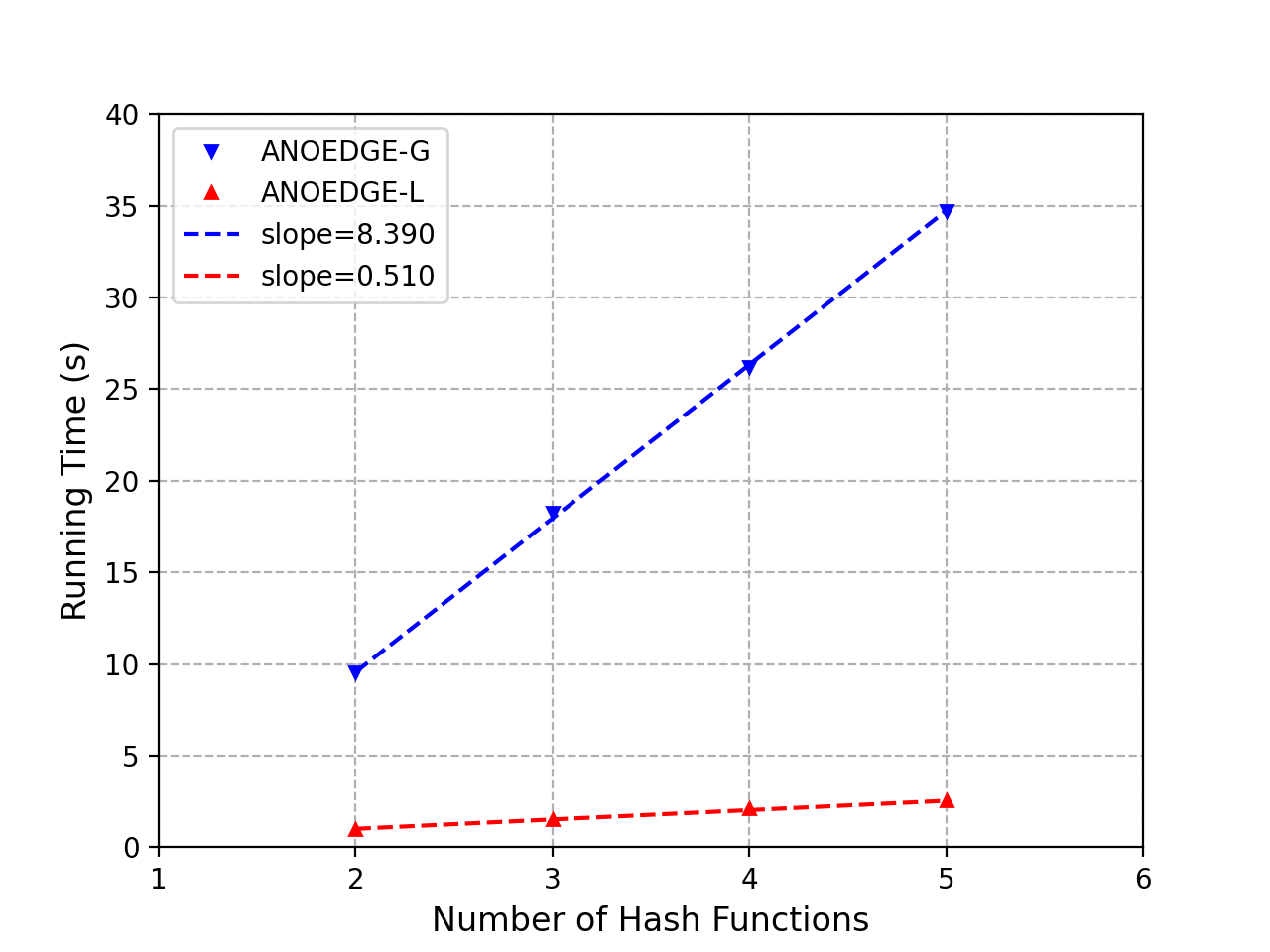}
   \caption{Linear scalability with number of hash functions on \emph{ISCX-IDS2012}.}
   \label{fig:edgeb}
\end{figure}

\begin{figure}[!htb]
  \centering
    \includegraphics[width=0.7\textwidth]{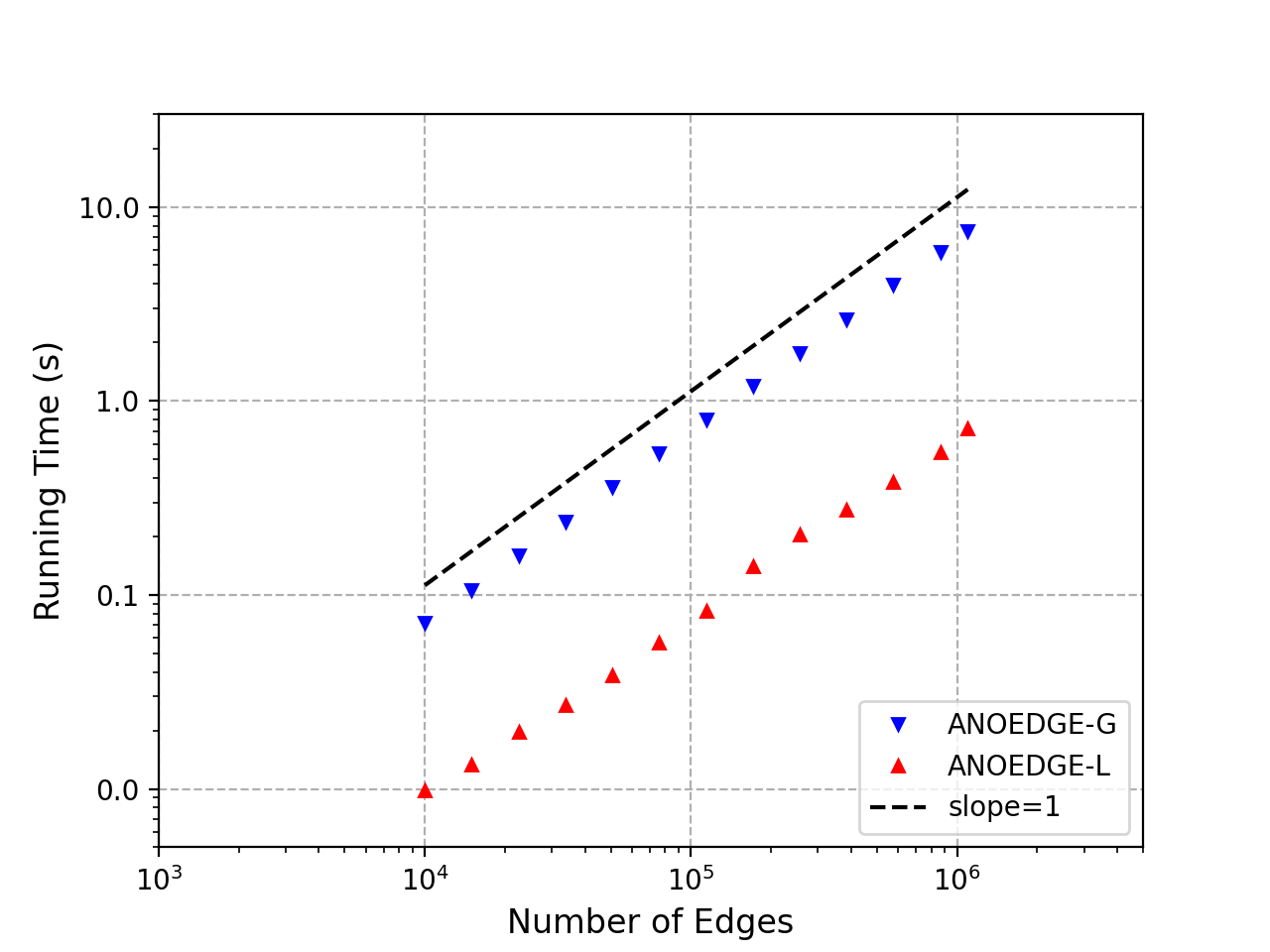}
   \caption{Linear scalability with number of edges on \emph{ISCX-IDS2012}.}
   \label{fig:edgec}
\end{figure}

\FloatBarrier

\subsection{Graph Anomalies}

{\bfseries Accuracy:} Table \ref{tab:graph} shows the AUC of graph anomaly detection baselines, \methodgraph, and \methodgraph-K. We report a single value for \densealert\ and \anomrank\ because these are non-randomized methods. \anomrank\ is not meant for a streaming scenario, therefore the low AUC. \densealert\ can estimate only one subtensor at a time and \spotlight\ uses a randomized approach without any actual search for dense subgraphs. Note that \methodgraph\ and \methodgraph-K outperform all baselines on all datasets while using a simple sketch data structure to incorporate dense subgraph search as opposed to the baselines. We provide results with an additional set of parameters in Table \ref{tab:graphaucsupp}.

\begin{table*}[!htb]
\centering
\caption{AUC and Running Time when detecting graph anomalies. Averaged over $5$ runs.}
\label{tab:graph}
\resizebox{\columnwidth}{!}{
\begin{tabular}{@{}lrrrrr}
\toprule
Dataset & \densealert & \spotlight & \anomrank & \textbf{\methodgraph} & \textbf{\methodgraph-K} \\ 
\midrule
\multirow{2}{*}{DARPA} & $0.833$ & $0.728 \pm 0.016$  & $0.754$ & $0.835 \pm 0.002$  & $\mathbf 0.839 \pm 0.002$ \\ 
& 49.3s & 88.5s & 3.7s & 0.3s & 0.3s \\ \midrule
\multirow{2}{*}{ISCX-IDS2012} & $0.906$ & $0.872 \pm 0.019$ & $0.194$ & $\mathbf0.950 \pm 0.001$ & $\mathbf 0.950 \pm 0.001$ \\ 
& 6.4s & 21.1s & 5.2s & 0.5s & 0.5s \\ \midrule
\multirow{2}{*}{CIC-IDS2018}  & $0.950$  & $0.835 \pm 0.022$ & $0.783$ & $\mathbf 0.957 \pm 0.000$ & $\mathbf 0.957 \pm 0.000$ \\
& 67.9s & 149.0s & 7.0s & 0.2s & 0.3s \\ \midrule
\multirow{2}{*}{CIC-DDoS2019} & $0.764$ & $0.468 \pm 0.048$ & $0.241$ & $0.946 \pm 0.002$ & $\mathbf 0.948 \pm 0.002$ \\   
& 1065.0s & 289.7s & 0.2s & 0.4s & 0.4s \\ 
\bottomrule
\end{tabular}
}
\end{table*}

{\bfseries Running Time:} Table \ref{tab:graph} shows the running time (excluding I/O). \densealert\ has $O(|\mathscr{E}|)$ worse case time complexity (per incoming edge). \anomrank\ needs to compute a global PageRank, which does not scale for stream processing. Note that \methodgraph\ and \methodgraph-K run much faster than all baselines.

{\bfseries AUC vs Running Time:} Figure \ref{fig:grapha} plots accuracy (AUC) vs. running time (log scale, in seconds, excluding I/O) on the \emph{CIC-DDoS2019} dataset. \methodgraph\ and \methodgraph-K achieve much higher accuracy compared to the baselines, while also running significantly faster.

\begin{figure}[!htb]
  \centering
    \includegraphics[width=0.7\textwidth]{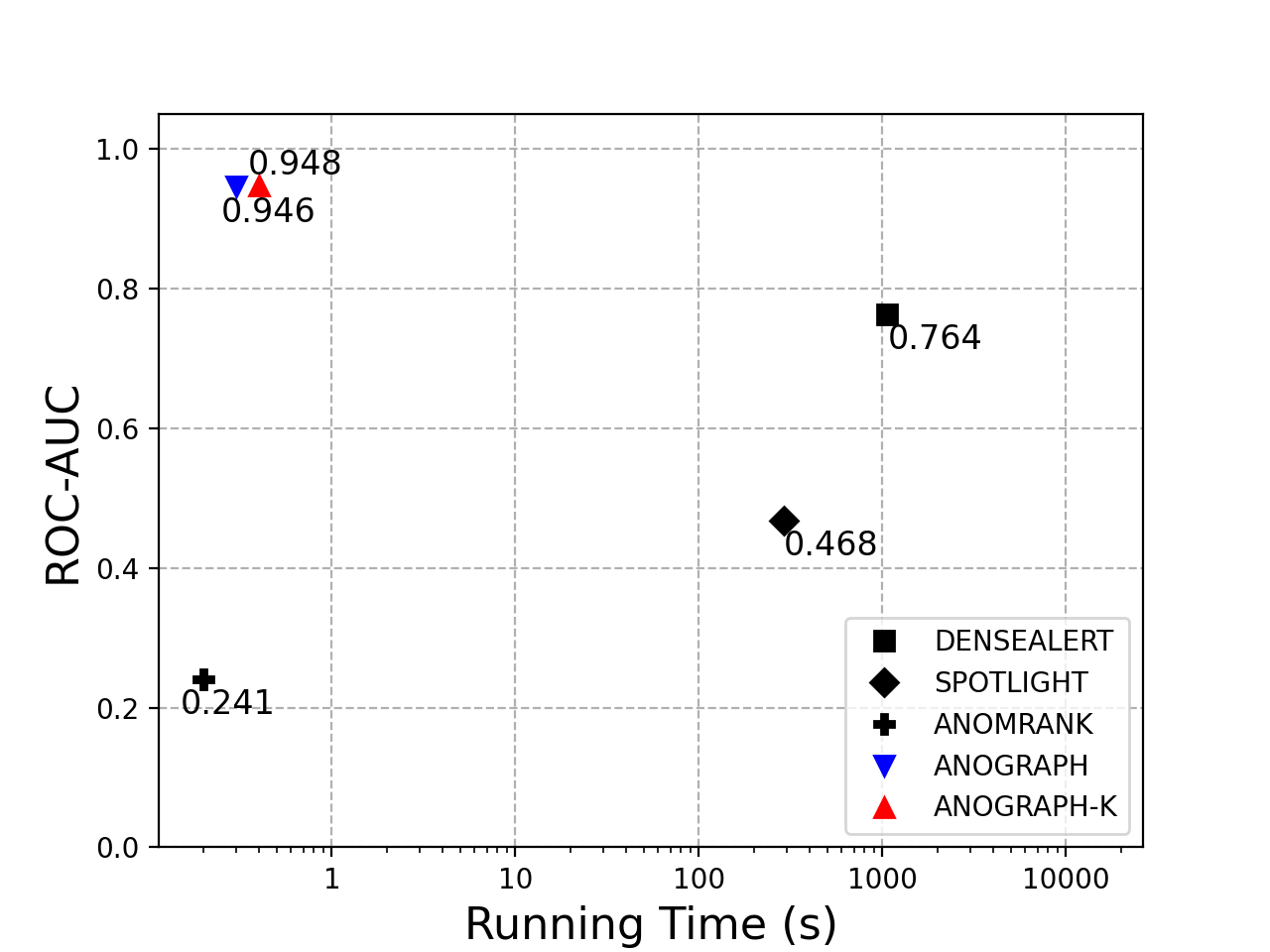}
  \caption{AUC vs running time when detecting graph anomalies on \emph{CIC-DDoS2019}.}
  \label{fig:grapha}
\end{figure}

{\bfseries Scalability:} Figures \ref{fig:graphb}, \ref{fig:graphc}, and \ref{fig:graphd} plot the running time with increasing factor $K$ (used for top-$K$ in Algorithm \ref{alg:AnoGraph-K}), number of hash functions and number of edges respectively, on the \emph{CIC-DDoS2019} dataset. This demonstrates the scalability of \methodgraph\ and \methodgraph-K.

\begin{figure}[!htb]
  \centering
    \includegraphics[width=0.7\textwidth]{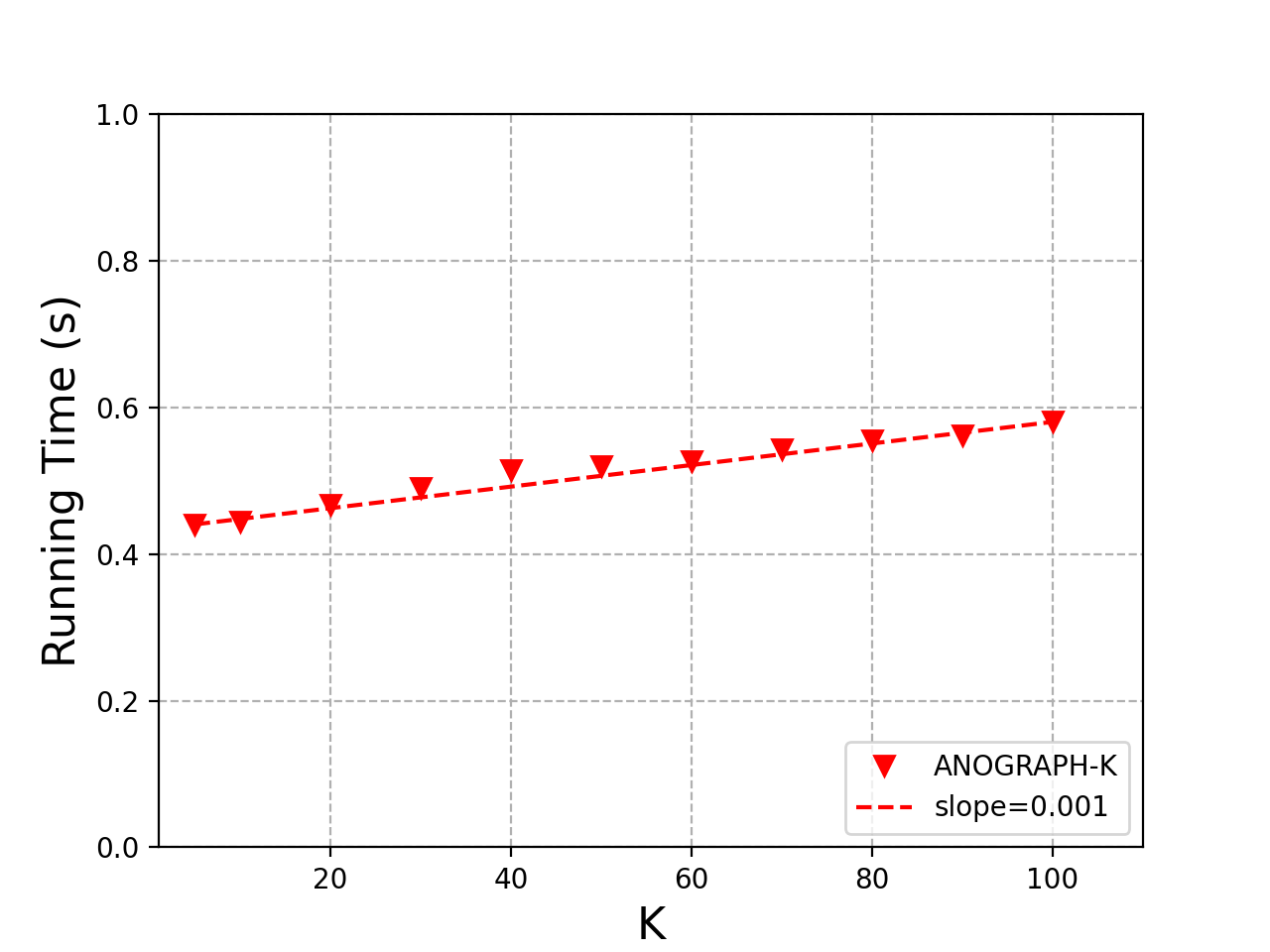}
  \caption{\methodgraph-K scales linearly with factor $K$ on \emph{CIC-DDoS2019}.}
  \label{fig:graphb}
\end{figure}

\begin{figure}[!htb]
  \centering
    \includegraphics[width=0.7\textwidth]{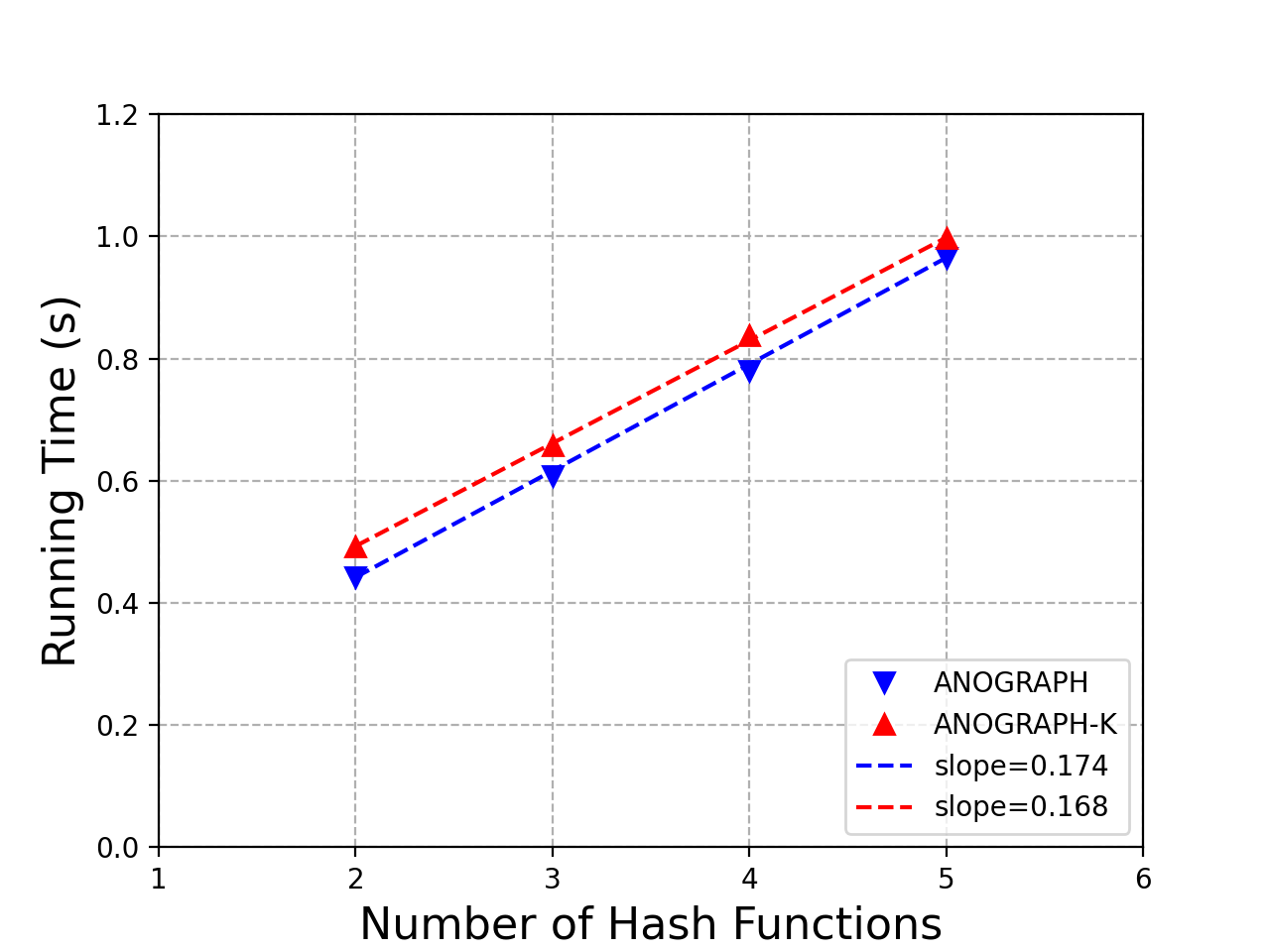}
  \caption{Linear scalability with number of hash functions on \emph{CIC-DDoS2019}.}
  \label{fig:graphc}
\end{figure}

\begin{figure}[!htb]
  \centering
    \includegraphics[width=0.7\textwidth]{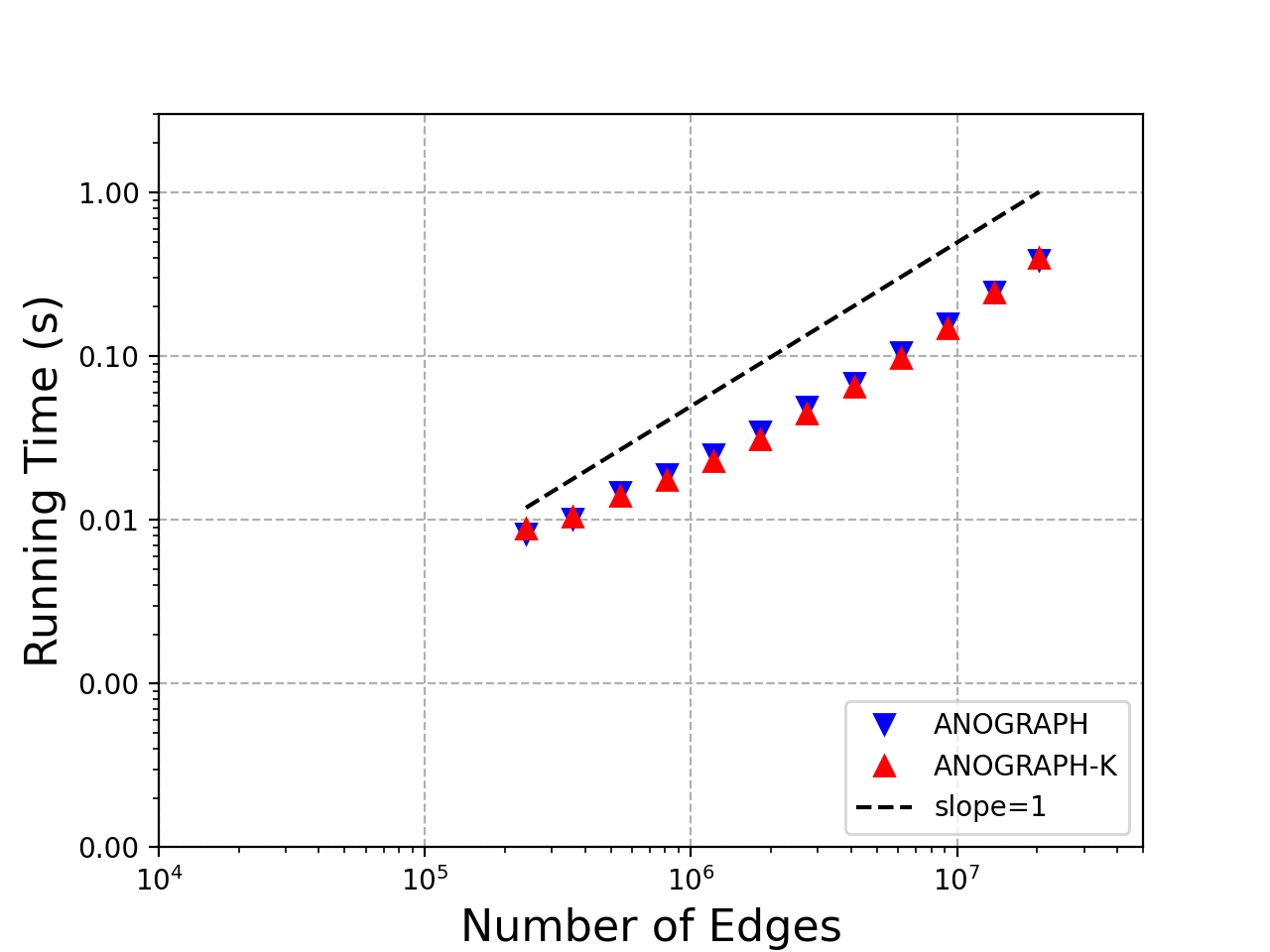}
  \caption{Linear scalability with number of edges on \emph{CIC-DDoS2019}.}
  \label{fig:graphd}
\end{figure}

\FloatBarrier

\subsection{Ablation Study}
\label{app:additionalresults}

Table \ref{tab:graphaucsupp} shows the performance of \methodgraph\ and \methodgraph-K for different time windows and edge thresholds. The edge threshold is varied in such a way that a sufficient number of anomalies are present within the time window. \methodgraph\ and \methodgraph-K perform similar to that in Table \ref{tab:graph}. Table \ref{tab:FactorVsAUC} shows the robustness of \methodedge-G and \methodedge-L as we vary the temporal decay factor $\alpha$.

\begin{table}[!htb]
\centering
\caption{Influence of time window and edge threshold on the ROC-AUC when detecting graph anomalies.}
\label{tab:graphaucsupp}
\begin{tabular}{@{}lrrrr}
\toprule
Dataset & Time & Edge & \textbf{\methodgraph} & \textbf{\methodgraph-K} \\ 
 & Window & Threshold & & \\
\midrule
\multirow{4}{*}{DARPA} & $15$ & $25$ & $0.835 \pm 0.001$ & $0.838 \pm 0.001$ \\ 
 & $30$ & $50$ & $0.835 \pm 0.002$ & $0.839 \pm 0.002$ \\
 & $60$ & $50$ & $0.747 \pm 0.002$ & $0.748 \pm 0.001$ \\ 
 & $60$ & $100$ & $0.823 \pm 0.000$ & $0.825 \pm 0.001$ \\ 
\midrule
\multirow{4}{*}{ISCX-IDS2012} & $15$ & $25$ & $0.945 \pm 0.001$ & $0.945 \pm 0.000$ \\  
 & $30$ & $50$ & $0.949 \pm 0.001$ & $0.948 \pm 0.000$ \\ 
 & $60$ & $50$ & $0.935 \pm 0.002$ & $0.933 \pm 0.002$ \\ 
 & $60$ & $100$ & $0.950 \pm 0.001$ & $0.950 \pm 0.001$ \\ 
\midrule
\multirow{4}{*}{CIC-IDS2018} & $15$ & $25$ & $0.945 \pm 0.004$ & $0.947 \pm 0.006$ \\ 
& $30$ & $50$ & $0.959 \pm 0.000$ & $0.959 \pm 0.001$ \\
& $60$ & $50$  & $0.920 \pm 0.001$ & $0.920 \pm 0.001$ \\ 
& $60$ & $100$ & $0.957 \pm 0.000$ & $0.957 \pm 0.000$ \\ 
\midrule
\multirow{4}{*}{CIC-DDoS2019} & $15$ & $25$ & $0.864 \pm 0.002$ & $0.863 \pm 0.003$ \\  
& $30$ & $50$ & $0.861 \pm 0.003$ & $0.861 \pm 0.003$ \\
& $60$ & $50$  & $0.824 \pm 0.004$ & $0.825 \pm 0.005$ \\ 
& $60$ & $100$ & $0.946 \pm 0.002$ & $0.948 \pm 0.002$ \\ 
\bottomrule
\end{tabular}
\end{table}

\begin{table}[!htb]
		\centering
		\caption{Influence of temporal decay factor $\alpha$ on the ROC-AUC in \methodedge-G and \methodedge-L on \emph{DARPA}.}
		\label{tab:FactorVsAUC}
		\begin{tabular}{@{}lrr@{}}
			\toprule
			$\alpha$ & \methodedge-G & \methodedge-L \\
			\midrule
			$0.1$ & $0.962$ & $0.957$ \\
            $0.2$ & $0.964$ & $0.957$ \\
            $0.3$ & $0.965$ & $0.958$ \\
            $0.4$ & $0.966$ & $0.959$ \\
            $0.5$ & $0.967$ & $0.960$ \\
            $0.6$ & $0.968$ & $0.961$ \\
            $0.7$ & $0.969$ & $0.962$ \\
            $0.8$ & $0.969$ & $0.964$ \\
            $0.9$ & $0.969$ & $0.966$ \\
            $0.95$ & $0.966$ & $0.966$ \\
			\bottomrule
		\end{tabular}
\end{table}

\FloatBarrier

\section{Conclusion}
In this chapter, we extend the CMS data structure to a higher-order sketch to capture complex relations in graph data and to reduce the problem of detecting suspicious dense subgraphs to finding a dense submatrix in constant time. We then propose four sketch-based streaming methods to detect edge and subgraph anomalies in constant time and memory. Furthermore, our approach is the first streaming work that incorporates dense subgraph search to detect graph anomalies in constant memory and time. We also provide a theoretical guarantee on the submatrix density measure and prove the time and space complexities of all methods. Experimental results on four real-world datasets demonstrate our effectiveness as opposed to popular state-of-the-art streaming edge and graph baselines. Future work could consider incorporating rectangular H-CMS matrices, node and edge representations, and more general types of data, including tensors.

\part{Multi-Aspect Data}
\SetPicSubDir{MSTREAM}
\SetExpSubDir{MSTREAM}

\chapter[MSTREAM: Fast Anomaly Detection in Multi-Aspect Streams][MSTREAM]{MSTREAM: Fast Anomaly Detection in Multi-Aspect Streams}
\label{ch:mstream}
\vspace{2em}



\newsavebox{\measurebox}

\begin{mdframed}[backgroundcolor=magenta!20] 
Chapter based on work that appeared at WWW'21 \cite{Bhatia2021MSTREAM} \href{https://dl.acm.org/doi/pdf/10.1145/3442381.3450023}{[PDF]}.
\end{mdframed}

\section{Introduction}

Given a stream of entries (i.e. records) in \emph{multi-aspect data} (i.e. data having multiple features or dimensions), how can we detect anomalous behavior, including group anomalies involving the sudden appearance of large groups of suspicious activity, in an unsupervised manner?

In this chapter, we propose \mstream, a method for processing a stream of multi-aspect data that detects \emph{group anomalies}, i.e. the sudden appearance of large amounts of suspiciously similar activity. Our approach naturally allows for similarity both in terms of categorical variables (e.g. a small group of repeated IP addresses creating a large number of connections), as well as in numerical variables (e.g. numerically similar values for average packet size).

\mstream\ is a streaming approach that performs each update in constant memory and time. This is constant both with respect to the stream length as well as in the number of attribute values for each attribute: this contrasts with tensor decomposition-based approaches such as STA and dense subtensor-based approaches such as \densealert, where memory usage grows in the number of possible attribute values. To do this, our approach makes use of locality-sensitive hash functions (LSH), which process the data in a streaming manner while allowing connections that form group anomalies to be jointly detected, as they consist of similar attribute values and hence are mapped into similar buckets by the hash functions. Finally, we demonstrate that the anomalies detected by \mstream\ are explainable.

To incorporate correlation between features, we further propose \mstream-PCA, \mstream-IB, and \mstream-AE which leverage Principal Component Analysis (PCA), Information Bottleneck (IB), and Autoencoders (AE) respectively, to map the original features into a lower-dimensional space and then execute \mstream\ in this lower-dimensional space. \mstream-AE is shown to provide better anomaly detection performance while also improving speed compared to \mstream, due to its lower number of dimensions.

In summary, the main contributions of our approach are:
\begin{enumerate}
    \item {\bfseries Multi-Aspect Group Anomaly Detection:} We propose a novel approach for detecting group anomalies in multi-aspect data, including both categorical and numeric attributes. Moreover, the anomalies detected by \mstream\ are explainable.
    \item {\bfseries Streaming Approach:} Our approach processes the data in a fast and streaming fashion, performing each update in constant time and memory.
    \item {\bfseries Effectiveness:} Our experimental results using \emph{KDDCUP99}, \emph{CICIDS-DoS}, \emph{UNSW-NB 15} and \emph{CICIDS-DDoS} datasets show that \mstream\ outperforms baseline approaches.
    \item {\bfseries Incorporating Correlation:} We propose \mstream-PCA, \mstream-IB and \mstream-AE to incorporate correlation between features.
\end{enumerate}
{\bfseries Reproducibility}: Our code and datasets are publicly available at \href{https://github.com/Stream-AD/MStream}{https://github.com/Stream-AD/MStream}.

\section{Problem}

Let $\mathcal{R} = \{r_1, r_2, \hdots\}$ be a stream of records, arriving in a streaming manner. Each record $r_i = (r_{i1}, \hdots, r_{id})$ consists of $d$ \emph{attributes} or dimensions, in which each dimension can either be categorical (e.g. IP address) or real-valued (e.g. average packet length). Note that since the data is arriving over time as a stream, we do not assume that the set of possible feature values is known beforehand; for example, in network traffic settings, it is common for new IP addresses to be seen for the first time at some point in the middle of the stream.

Our goal is to detect \emph{group anomalies}. Intuitively, group anomalies should have the following properties:

\begin{enumerate}
    \item {\bfseries Similarity in Categorical Attributes:} for categorical attributes, the group anomalies consist of a relatively small number of attribute values, repeated a suspiciously large number of times.
    \item {\bfseries Similarity in Real-Valued Attributes:} for real-valued attributes, the group anomalies consist of clusters of numerically similar attribute values.
    \item {\bfseries Temporally Sudden:} the group anomalies arrive suddenly, over a suspiciously short amount of time. In addition, their behavior (in terms of attribute values) should clearly differ from what we have observed previously, over the course of the stream.
\end{enumerate}

\section{Proposed Algorithm}

\subsection{Motivation}

Consider the toy example in Table \ref{tab:toy}, comprising a stream of connections over time. This dataset shows a clear block of suspicious activity from time $4$ to $5$, consisting of several IP addresses repeated a large number of times, as well as large packet sizes which seem to be anomalously large compared to the usual distribution of packet sizes.

The main challenge, however, is to detect this type of pattern in a {\bfseries streaming} manner, considering that we do not want to set any limits a priori on the duration of the anomalous activity we want to detect, or the number of IP addresses (or other attribute values) which may be involved in this activity.

As shown in Figure \ref{fig:mstream}, our approach addresses these problems through the use of a number of {\bfseries locality-sensitive hash functions}~\cite{charikar2002similarity}
which hash each incoming tuple into a fixed number of buckets. Intuitively, we do this such that tuples with many similar entries tend to be hashed into similar buckets. These hash functions are combined with a {\bfseries temporal scoring} approach, which takes into account how much overlap we observe between the buckets at any time: high amounts of overlap arriving in a short period of time suggest the presence of anomalous activity.

\begin{figure*}[!tb]
        \center{\includegraphics[width=\linewidth]
        {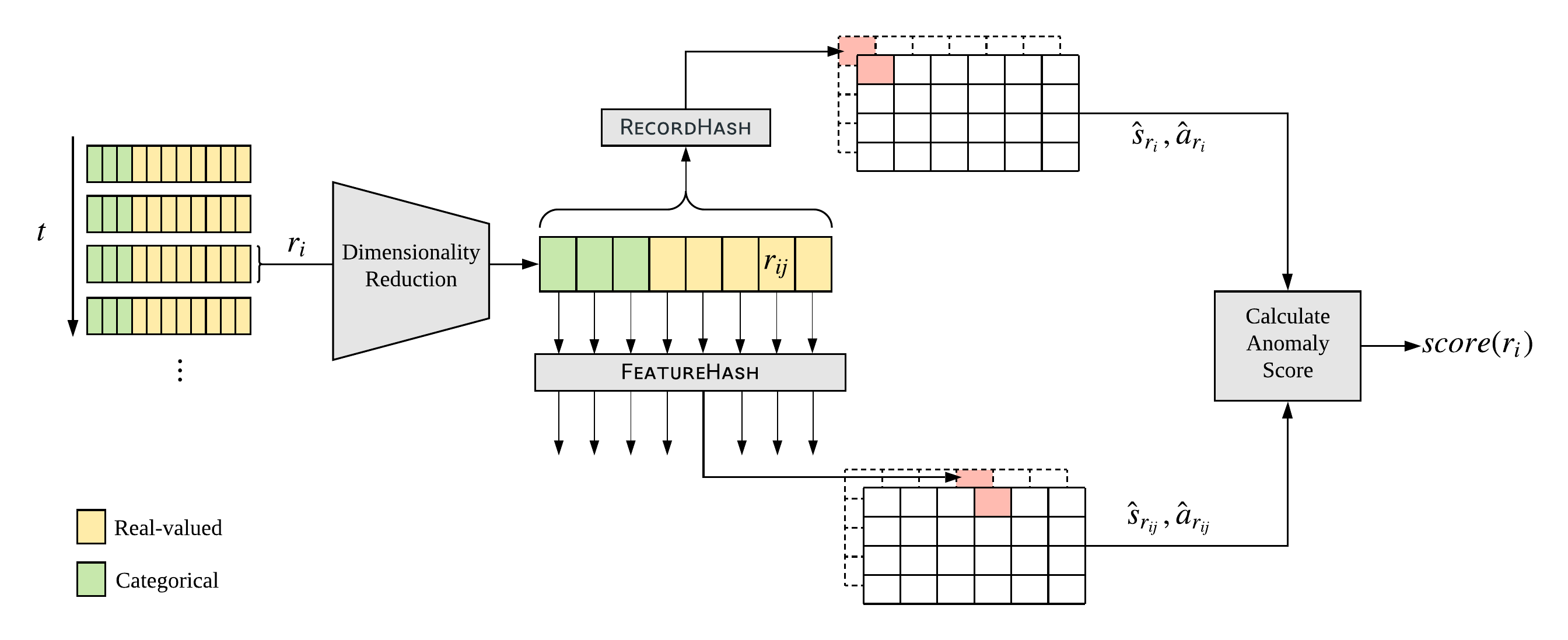}}
        \caption{\label{fig:mstream} Diagram of the proposed \mstream. The dimensionality reduction unit (Section \ref{sec:mstreamae}) takes in a record and outputs a lower-dimensional embedding. Two types of locality-sensitive hash functions are then applied. \textsc{FeatureHash} (Algorithm \ref{alg:featurehash}) hashes each individual feature and \textsc{RecordHash} (Algorithm \ref{alg:recordhash}) hashes the entire record jointly. These are then combined together using a temporal scoring approach to calculate the anomaly score for the record (Algorithm \ref{alg:mstream}).}
\end{figure*}

\begin{table}[!htb]
\centering
\caption{Simple toy example, consisting of a stream of multi-aspect connections over time.}
\label{tab:toy}
\addtolength{\tabcolsep}{-2pt}
\begin{tabular}{@{}ccccc@{}}
\toprule
{\bfseries Time} & {\bfseries Source IP} & {\bfseries Dest. IP} & {\bfseries Pkt. Size} & {\bfseries $\cdots$} \\ \midrule
$1$ & \ \ \ $194.027.251.021$ \ \ \ & $194.027.251.021$ & $100$ & $\cdots$ \\
$2$ & \ \ \ $172.016.113.105$ \ \ \ & $207.230.054.203$ & $80$ & $\cdots$ \\
\textcolor{red}{$4$} & \ \ \ \textcolor{red}{$194.027.251.021$} \ \ \ & \textcolor{red}{$192.168.001.001$} & \textcolor{red}{$1000$} & \textcolor{red}{$\cdots$} \\
\textcolor{red}{$4$} & \ \ \ \textcolor{red}{$194.027.251.021$} \ \ \ & \textcolor{red}{$192.168.001.001$} & \textcolor{red}{$995$} & \textcolor{red}{$\cdots$} \\
\textcolor{red}{$4$} & \ \ \ \textcolor{red}{$194.027.251.021$} \ \ \ & \textcolor{red}{$192.168.001.001$} & \textcolor{red}{$1000$} & \textcolor{red}{$\cdots$} \\
\textcolor{red}{$5$} & \ \ \ \textcolor{red}{$194.027.251.021$} \ \ \ & \textcolor{red}{$192.168.001.001$} & \textcolor{red}{$990$} & \textcolor{red}{$\cdots$} \\
\textcolor{red}{$5$} & \ \ \ \textcolor{red}{$194.027.251.021$} \ \ \ & \textcolor{red}{$194.027.251.021$} & \textcolor{red}{$1000$} & \textcolor{red}{$\cdots$} \\
\textcolor{red}{$5$} & \ \ \ \textcolor{red}{$194.027.251.021$} \ \ \ & \textcolor{red}{$194.027.251.021$} & \textcolor{red}{$995$} & \textcolor{red}{$\cdots$} \\
$6$ & \ \ \ $194.027.251.021$ \ \ \ & $194.027.251.021$ & $100$ & $\cdots$ \\
$7$ & \ \ \ $172.016.113.105$ \ \ \ & $207.230.054.203$ & $80$ & $\cdots$ \\
\bottomrule
\end{tabular}
\end{table}

In Sections \ref{sec:hashfunc} and \ref{sec:scoring}, we describe our \mstream\ approach, and in Section \ref{sec:mstreamae}, we describe our \mstream-PCA, \mstream-IB and \mstream-AE approaches which incorporate correlation between features in an unsupervised manner. \mstream-PCA uses principal component analysis, \mstream-IB uses information bottleneck, and \mstream-AE uses an autoencoder to first compress the original features and then apply \mstream\ in the compressed feature space.

\subsection{Hash Functions}
\label{sec:hashfunc}

Our approach uses two types of hash functions: \textsc{FeatureHash}, which hashes each feature individually, and \textsc{RecordHash}, which hashes an entire record jointly. We use multiple independent copies of each type of hash function, and explain how to combine these to produce a single anomalousness score.

\subsubsection{FeatureHash}
\label{sec:feature}

As shown in Algorithm \ref{alg:featurehash}, \textsc{FeatureHash} consists of hash functions independently applied to a single feature. There are two cases, corresponding to whether the feature is categorical (e.g. IP address) or real-valued (e.g. average packet length):

For \textbf{categorical} data, we use standard linear hash functions \cite{litwin1980linear} which map integer-valued data randomly into $b$ buckets, i.e. $\{0, \dots, b-1\}$, where $b$ is a fixed number.

For \textbf{real-valued} data, however, we find that randomized hash functions tend to lead to highly uneven bucket distributions for certain input datasets. Instead, we use a streaming log-bucketization approach. We first apply a log-transform to the data value (line 5), then perform min-max normalization, where the min and max are maintained in a streaming manner (line 7), and finally map it such that the range of feature values is evenly divided into $b$ buckets, i.e. $\{0, \dots, b-1\}$ (line 8). 
    
\begin{algorithm}
	\caption{\textsc{FeatureHash}: Hashing Individual Feature  \label{alg:featurehash}}
	\KwIn{$r_{ij}$ (Feature $j$ of record $r_{i}$)}
	\KwOut{Bucket index in $\{0, \dots, b-1\}$ to map $r_{ij}$ into}
	\textbf{if} $r_{ij}$ is categorical  \\
	\ \ \ \ \textbf{output} $\Call{HASH}{r_{ij}}$ \tcp*[f]{Linear Hash~\cite{litwin1980linear}} \\
	\textbf{else if} $r_{ij}$ is real-valued  \\
	\ \ \ \ {\bfseries $\triangleright$ Log-Transform} \\
	\ \ \ \ \ \ \ \ $\tilde{r}_{ij} = \log(1+r_{ij})$ \\
	\ \ \ \ {\bfseries $\triangleright$ Normalize} \\
	\ \ \ \ \ \ \ \ $\tilde{r}_{ij} \gets \frac{\tilde{r}_{ij} - min_{j}}{max_{j} - min_{j}}$ \tcp*[f]{Streaming Min-Max} \\
    \ \ \ \ {\bfseries output} $\lfloor \tilde{r}_{ij} \cdot b \rfloor  ($mod$ \ b)$ \tcp*[f]{Bucketization into $b$ buckets} \\
\end{algorithm}

\FloatBarrier

\subsubsection{RecordHash}
\label{sec:record}

As shown in Algorithm \ref{alg:recordhash}, in \textsc{RecordHash}, we operate on all features of a record simultaneously. We first divide the entire record $r_{i}$ into two parts, one consisting of the categorical features $\mathcal{C}$, say $r_{i}^{cat}$, and the other consisting of real-valued features $\mathcal{R}$, say $r_{i}^{num}$. We then separately hash $r_{i}^{cat}$ to get $bucket_{cat}$, and $r_{i}^{num}$ to get $bucket_{num}$. Finally we take the sum modulo $b$ of $bucket_{cat}$ and $bucket_{num}$ to get a bucket for $r_{i}$. We hash $r_{i}^{cat}$ and $r_{i}^{num}$ as follows:

\begin{enumerate}
\item $r_{i}^{cat}$: We use standard linear hash functions \cite{litwin1980linear} to map $\forall j \in \mathcal{C}$ each of the individual features $r_{ij}$ into $b$ buckets, and then combine them by summing them modulo $b$ to compute the bucket index $bucket_{cat}$ for $r_{i}^{cat}$ (line 3).

\item $r_{i}^{num}$: To compute the hash of a real-valued record $r_{i}^{num}$ of dimension $p=|\mathcal{R}|$, we choose $k$ random vectors $\vec{a_{1}},\vec{a_{2}},..,\vec{a_{k}}$ each having $p$ dimensions and independently sampled from a Gaussian distribution $\mathcal{N}_{p}(\vec{0},\vec{I_p})$, where $k=\lceil \log_{2}(b) \rceil$. We compute the scalar product of $r_{i}^{num}$ with each of these vectors (line 6). We then map the positive scalar products to $1$ and the non-positive scalar products to $0$ and then concatenate these mapped values to get a $k$-bit string, then convert it from a bitset into an integer $bucket_{num}$ between $0$ and $2^k-1$. (line 10).
\end{enumerate}

\begin{algorithm}
	\caption{\textsc{RecordHash}: Hashing Entire Record  \label{alg:recordhash}}
	\KwIn{Record $r_{i}$}
	\KwOut{Bucket index in $\{0, \dots, b-1\}$ to map $r_{i}$ into}
	{\bfseries $\triangleright$ Divide $r_{i}$ into its categorical part, $r_{i}^{cat}$, and its numerical part, $r_{i}^{num}$} \\
	{\bfseries $\triangleright$ Hashing $r_{i}^{cat}$} \\
	\ \ \ \ $bucket_{cat} = (\sum_{j \in \mathcal{C} }${$\Call{HASH}{r_{ij}})$} (mod $b$) \tcp*[f]{Linear Hash~\cite{litwin1980linear}} \\
	{\bfseries $\triangleright$ Hashing $r_{i}^{num}$} \\
	\ \ \ \ \textbf{for} $id \gets$ $1$ to $k$  \\
	\ \ \ \ \ \ \ \	\textbf{if} $\langle{r_{i}^{num} , \vec{a_{id}}}\rangle > 0$  \\
	\ \ \ \ \ \ \ \ \ \ \ \ $bitset[id] = 1$ \\
	\ \ \ \ \ \ \ \ \textbf{else}  \\
	\ \ \ \ \ \ \ \ \ \ \ \ $bitset[id] = 0$ \\
	\ \ \ \ $bucket_{num} = \Call{INT}{bitset}$ \tcp*[f]{Convert bitset to integer} \\

    {\bfseries output}\ $(bucket_{cat}+bucket_{num})($mod$ \ b)$
\end{algorithm}

\FloatBarrier

\subsection{Temporal Scoring}
\label{sec:scoring}

\textsc{Midas} \cite{bhatia2020midas} uses two types of CMS data structures to maintain approximate counts $\hat{s}_{uv}$ and $\hat{a}_{uv}$ which estimate $s_{uv}$ and $a_{uv}$ respectively. The anomaly score for an edge in \textsc{Midas} is then defined as:

\begin{equation}
score(u,v,t) = \left(\hat{a}_{uv} - \frac{\hat{s}_{uv}}{t}\right)^2 \label{eqn:eq1} \frac{t^2}{\hat{s}_{uv}(t-1)}
\end{equation}

\textsc{Midas} is designed to detect anomalous edges, which are two-dimensional records (consisting of source and destination node index). 
Therefore, it cannot be applied in the high-dimensional setting of multi-aspect data. Moreover, \textsc{Midas} treats variables of the dataset as categorical variables, whereas multi-aspect data can contain arbitrary mixtures of categorical variables (e.g. source IP address) and numerical variables (e.g. average packet size).

We extend \textsc{Midas} to define an anomalousness score for each record and detect anomalous records in a streaming manner. Given each incoming record $r_{i}$ having $j$ features, we can compute $j+1$ anomalousness scores: one for the entire record $r_{i}$ and one for each individual feature $r_{ij}$. We compute each score by computing the chi-squared statistic over the two categories: current time tick and past time ticks. Anomaly scores for individual attributes are useful for interpretability, as they help explain which features are most responsible for the anomalousness of the record.
Finally, we combine these scores by taking their sum.

\begin{defn}[Anomaly Score]
Given a newly arriving record $(r_{i},t)$, our anomalousness score is computed as:
\begin{align}
score (r_{i},t) = \left(\hat{a}_{r_{i}} - \frac{\hat{s}_{r_{i}}}{t}\right)^2 \frac{t^2}{\hat{s}_{r_{i}}(t-1)} + \sum_{j=1}^{d} score(r_{ij},t)
\end{align}
where,
\begin{align}
\text{score}(r_{ij},t) = \left(\hat{a}_{r_{ij}} - \frac{\hat{s}_{r_{ij}}}{t}\right)^2 \frac{t^2}{\hat{s}_{r_{ij}}(t-1)}
\end{align}
\end{defn}
and $\hat{a}_{r_{i}} ($or $\hat{a}_{r_{ij}})$ is an approximate count of  $r_{i} ($or $r_{ij})$ at current time $t$ and $\hat{s}_{r_{i}} ($or $\hat{s}_{r_{ij}})$ is an approximate count of $r_{i} ($or $r_{ij})$ up to time $t$.

We also allow temporal flexibility of records, i.e. records in the recent past count towards the current anomalousness score. This is achieved by reducing the counts $\hat{a}_{r_{i}}$ and $\hat{a}_{r_{ij}} \forall j \in \{1,..,d\}$ by a factor of $\alpha \in (0, 1)$ rather than resetting them at the end of each time tick. This results in past records counting towards the current time tick, with a diminishing weight.

\mstream\ is summarised in Algorithm \ref{alg:mstream}.

\begin{algorithm}
	\caption{\mstream:\ Streaming Anomaly Scoring \label{alg:mstream}}
	\KwIn{Stream of records over time}
	\KwOut{Anomaly scores for each record}
	{\bfseries $\triangleright$ Initialize data structures:} \\
	\ \ \ \ Total record count $\hat{s}_{r_{i}}$ and total attribute count $\hat{s}_{r_{ij}} \forall j \in \{1,..,d\}$ \\
	\ \ \ \ Current record count $\hat{a}_{r_{i}}$ and current attribute count  $\hat{a}_{r_{ij}} \forall j \in \{1,..,d\}$ \\
	\While{new record $(r_i,t) = (r_{i1}, \hdots, r_{id},t)$ is received:}{
	{\bfseries $\triangleright$ Hash and Update Counts:} \\

	\ \ \ \ \textbf{for} $j \gets$ $1$ to $d$  \\
	\ \ \ \ \ \ \ \ $bucket_{j} = \Call{FeatureHash}{r_{ij}}$ \\
    \ \ \ \ \ \ \ \ Update count of $bucket_{j}$ \\
	
	\ \ \ \ $bucket= \Call{RecordHash}{r_{i}}$\\
	\ \ \ \ Update count of $bucket$\\
	{\bfseries $\triangleright$ Query Counts:} \\
	\ \ \ \ Retrieve updated counts $\hat{s}_{r_{i}}$, $\hat{a}_{r_{i}}$, $\hat{s}_{r_{ij}}$ and $\hat{a}_{r_{ij}} \forall j \in \{1..d\}$\\
	{\bfseries $\triangleright$ Anomaly Score:}\\
	\ \ \ \ {\bfseries output} $score (r_{i},t) = \left(\hat{a}_{r_{i}} - \frac{\hat{s}_{r_{i}}}{t}\right)^2 \frac{t^2}{\hat{s}_{r_{i}}(t-1)} + \sum_{j=1}^{d} score(r_{ij},t)$\\
	}
\end{algorithm}

\FloatBarrier

\subsection{Incorporating Correlation Between Features}
\label{sec:mstreamae}

In this section, we describe our \mstream-PCA, \mstream-IB, and \mstream-AE approaches where we run the \mstream\ algorithm on a lower-dimensional embedding of the original data obtained using Principal Component Analysis (PCA) \cite{pearson1901liii}, Information Bottleneck (IB) \cite{tishby2000information} and Autoencoder (AE) \cite{hinton1994autoencoders} methods in a streaming manner.

Our motivation for combining PCA, IB, and AE methods with \mstream\ is two-fold. Firstly, the low-dimensional representations learned by these algorithms incorporate correlation between different attributes of the record, making anomaly detection more effective. Secondly, a reduction in the dimensions would result in faster processing per record.

For all three methods, we first learn the dimensionality reduction transformation using a very small initial subset of $256$ records from the incoming stream. We then compute the embeddings for the subsequent records and pass them to \mstream\ to detect anomalies in an online manner.

\paragraph{\textbf{Principal Component Analysis}}

We choose PCA because it only requires one major parameter to tune: namely the dimension of the projection space. Moreover, this parameter can be set easily by analysis of the explained variance ratios of the principal components. Hence \mstream-PCA can be used as an off-the-shelf algorithm for streaming anomaly detection with dimensionality reduction.

\paragraph{\textbf{Information Bottleneck}}
Information bottleneck for dimensionality reduction can be posed as the following optimization problem:
$$
\min _{p(t | x)} I(X ; T)-\beta I(T ; Y)
$$
where $X$, $Y$, and $T$ are random variables. $T$ is the compressed representation of $X$, $I(X ; T)$ and $I(T ; Y)$ are the mutual information of $X$ and $T$, and of $T$ and $Y$, respectively, and $\beta$ is a Lagrange multiplier.
In our setting, $X$ denotes the multi-aspect data, $Y$ denotes whether the data is anomalous and $T$ denotes the dimensionally reduced features that we wish to find. Our implementation is based on the Neural Network approach for Nonlinear Information Bottleneck \cite{kolchinsky2017nonlinear}.

\paragraph{\textbf{Autoencoder}}
Autoencoder is a neural network based approach for dimensionality reduction. An autoencoder network consists of an encoder and a decoder. The encoder compresses the input into a lower-dimensional space, while the decoder reconstructs the input from the low-dimensional representation. Our experimental results in Section \ref{sec:experiment} show that even with a simple 3-layered autoencoder, \mstream-AE outperforms both \mstream-PCA and \mstream-IB.
 
\subsection{Time and Memory Complexity}
\label{sec:timecomplexity}
In terms of memory, \mstream\ only needs to maintain data structures over time, which requires memory proportional to $O(wbd)$, where $w$, $b$, and $d$ are the number of hash functions, the number of buckets in the data structures and the total number of dimensions; which is bounded with respect to the stream size.

For time complexity, the only relevant steps in Algorithm \ref{alg:mstream} are those that either update or query the data structures, which take $O(wd)$ (all other operations run in constant time). Thus, the time complexity per update step is $O(wd)$.

\FloatBarrier

\section{Experiments}
\label{sec:experiment}
In this section, we evaluate the performance of \mstream\ and \mstream-AE compared to \elliptic, LOF, I-Forest, \rcf\ and \densealert\  on multi-aspect data streams. We aim to answer the following questions:

\begin{enumerate}[label=\textbf{Q\arabic*.}]
\item {\bfseries Anomaly Detection Performance:} How accurately does \mstream\ detect real-world anomalies compared to baselines, as evaluated using the ground truth labels?
\item {\bfseries Scalability:} How does it scale with input stream length and number of dimensions? How does the time needed to process each input compare to baseline approaches?
\item {\bfseries Real-World Effectiveness:} Does it detect meaningful anomalies? Does it detect group anomalies?

\end{enumerate}

\paragraph{\textbf{Datasets}}

\emph{KDDCUP99} dataset \cite{KDDCup192:online} is based on the DARPA dataset and is among the most extensively used datasets for intrusion detection. Since the proportion of data belonging to the `attack' class is much larger than the proportion of data belonging to the `non-attack' class, we downsample the `attack' class to a proportion of $20\%$. \emph{KDDCUP99} has $42$ dimensions and $1.21$ million records.

\cite{ring2019survey} surveys more than 30 intrusion detection datasets and recommends to use the newer CICIDS \cite{sharafaldin2018toward} and UNSW-NB15 \cite{moustafa2015unsw} datasets. These contain modern-day attacks and follow the established guidelines for reliable intrusion detection datasets (in terms of realism, evaluation capabilities, total capture, completeness, and malicious activity) \cite{sharafaldin2018toward}.

\emph{CICIDS} 2018 dataset was generated at the Canadian Institute of Cybersecurity. Each record is a flow containing features such as Source IP Address, Source Port, Destination IP Address, Bytes, and Packets. These flows were captured from a real-time simulation of normal network traffic and synthetic attack simulators. This consists of the \emph{CICIDS-DoS} dataset ($1.05$ million records, 80 features) and the \emph{CICIDS-DDoS} dataset ($7.9$ million records, 83 features). \emph{CICIDS-DoS} has $5\%$ anomalies whereas \emph{CICIDS-DDoS} has $7\%$ anomalies.

\emph{UNSW-NB 15} dataset was created by the Cyber Range Lab of the Australian Centre for Cyber Security (ACCS) for generating a hybrid of real modern normal activities and synthetic contemporary attack behaviors. This dataset has nine types of attacks, namely, Fuzzers, Analysis, Backdoors, DoS, Exploits, Generic, Reconnaissance, Shellcode, and Worms. It has $49$ features and $2.5$ million records including $13\%$ anomalies.

\begin{table*}[!htb]
\centering
\caption{Comparison of relevant multi-aspect anomaly detection approaches.}
\label{tab:comparisonmstream}
\resizebox{\linewidth}{!}{
\begin{tabular}{@{}rcccccccc|c@{}}
\toprule
& {Elliptic }
& {LOF }
& {I-Forest }
& {STA }
& {MASTA }
& {STenSr }
& {\rcf} 
& {\densealert } 
& {\bfseries {\mstream}} \\ 
& ($1999$) & ($2000$) & ($2008$) & ($2006$) & ($2015$) & ($2015$) & ($2016$) & ($2017$) & ($2021$) \\\midrule
\textbf{Group Anomalies} & & & & & & & & \Checkmark & \CheckmarkBold \\
\textbf{Real-valued Features} & \Checkmark & \Checkmark & \Checkmark & & & & \Checkmark & & \CheckmarkBold \\
\textbf{Constant Memory} & & & & & & & \Checkmark & \Checkmark & \CheckmarkBold \\
\textbf{Const. Update Time} & & & & \Checkmark & \Checkmark & \Checkmark & \Checkmark & \Checkmark & \CheckmarkBold \\
\bottomrule
\end{tabular}}
\end{table*}

\paragraph{\bfseries Baselines}

We consider unsupervised algorithms \lof, \iso, \elliptic, STA, MASTA, STenSr, \densealert\ and \rcf. Of these, only \densealert\ performs group anomaly detection (by detecting dense subtensors); however, as shown in Table \ref{tab:comparisonmstream}, it cannot effectively handle real-valued features (as it treats all features as discrete-valued). Due to a large number of dimensions, even sparse tensor versions of STA/MASTA/STenSr run out of memory on these datasets. So, we compare with \elliptic, \lof, \iso, \densealert\ and \rcf.

\paragraph{\bfseries Evaluation Metrics}
All the methods output an anomaly score per edge (higher is more anomalous). We plot the ROC curve, which compares the True Positive Rate (TPR) and False Positive Rate (FPR), without needing to fix any threshold. We also report the ROC-AUC (Area under the ROC curve).

\paragraph{\bfseries Experimental Setup}
All experiments are carried out on a $2.4 GHz$ Intel Core $i9$ processor, $32 GB$ RAM, running OS $X$ $10.15.2$. We implement \mstream\ in C\texttt{++}. We use $2$ independent copies of each hash function, and we set the number of buckets to 1024. We set the temporal decay factor $\alpha$ as $0.85$ for \emph{KDDCUP99}, $0.95$ for \emph{CICIDS-DoS} and \emph{CICIDS-DDoS}, and $0.4$ for \emph{UNSW-NB 15} due to its higher time granularity. Note that \mstream\ is not sensitive to variation of $\alpha$ parameter as shown in Table \ref{tab:FactorVsAUCmstream}. Since \emph{KDDCUP99} dataset does not have timestamps, we apply the temporal decay factor once every 1000 records. We discuss the influence of temporal decay factor $\alpha$ on the ROC-AUC in Section \ref{sec:ablations}.

To demonstrate the robustness of our proposed approach, we set the output dimension of \mstream-PCA, \mstream-IB and \mstream-AE for all datasets to a common value of $12$ instead of searching individually on each method and dataset. We reduce the real-valued columns to $12$ dimensions and then pass these along with the categorical columns to \mstream. Results on varying the number of output dimensions can be found in Section \ref{sec:ablations}. For \mstream-PCA we use the open-source implementation of PCA available in the scikit-learn \cite{scikit-learn} library. Parameters for \mstream-AE and \mstream-IB are described in Section \ref{sec:ablations}.

We use open-sourced implementations of \densealert\ and \rcf, provided by the authors, following parameter settings as suggested in the original papers. For \elliptic, \lof\ and \iso\ we use the open-source implementation available in the scikit-learn \cite{scikit-learn} library. We also pass the true anomaly percentage to \elliptic, \lof\ and \iso\ methods, while the remainder of the methods do not require the anomaly percentage.

All the experiments, unless explicitly specified, are performed $5$ times for each parameter group, and the mean and standard deviation values are reported.

\subsection{Anomaly Detection Performance}

Figure \ref{fig:ROC} plots the ROC curve for \mstream, \mstream-PCA, \mstream-IB\ and \mstream-AE along with the baselines, \elliptic, \lof, \iso, \densealert\ and \rcf\ on \emph{CICIDS-DoS} dataset. 
We see that \mstream, \mstream-PCA, \mstream-IB\ and \mstream-AE achieve a much higher ROC-AUC ($0.92-0.95$) compared to the baselines. \mstream\ and its variants achieve at least $50\%$ higher AUC than \densealert, $11\%$ higher than \rcf\ $26\%$ higher than \iso, $23\%$ higher than \elliptic\ and  $84\%$ higher than \lof.

\begin{figure}[!htb]
        \center{\includegraphics[width=0.8\columnwidth]
        {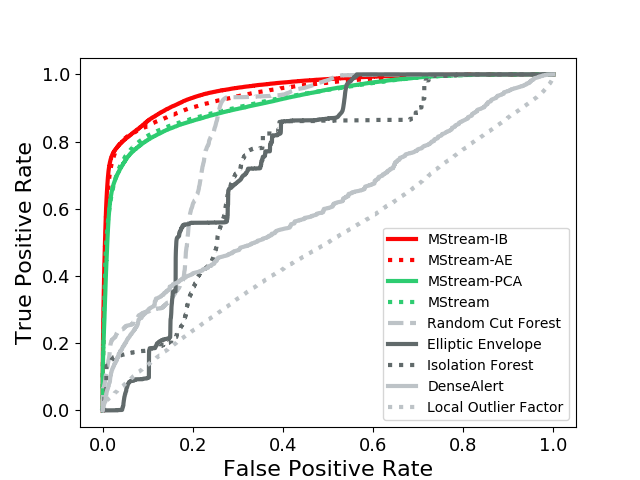}}
        \caption{\label{fig:ROC} ROC on \emph{CICIDS-DoS} dataset.}
\end{figure}

Table \ref{tab:aucmstream} shows the AUC of \elliptic, \lof, \iso, \densealert, \rcf\ and \mstream\ on \emph{KDDCUP99}, \emph{CICIDS-DoS}, \emph{UNSW-NB 15} and \emph{CICIDS-DDoS} datasets. We report a single value for \lof\ and \densealert\ because these are non-randomized methods. We also report a single value for \rcf\ because we use the parameters and random seed of the original implementation. \densealert\ performs well on small-sized datasets such as \emph{KDDCUP99} but as the dimensions increase, its performance decreases. On the large \emph{CICIDS-DDoS} dataset \densealert\ runs out of memory. We observe that \mstream\ outperforms all baselines on all datasets. By learning the correlation between features, \mstream-AE achieves higher ROC-AUC than \mstream, and performs comparably or better than \mstream-PCA and \mstream-IB. We also discuss evaluating the ROC-AUC in a streaming manner in Section \ref{sec:ablations}.

\begin{table*}[!htb]
\centering
\caption{AUC of each method on different datasets.}
\label{tab:aucmstream}
\resizebox{\linewidth}{!}{
\begin{tabular}{@{}rccccccccc@{}}
\toprule
& Elliptic
 & LOF
 & I-Forest
 &  DAlert
 & RCF
 & \textbf{\mstream}
 & \textbf{\mstream-PCA}
 & \textbf{\mstream-IB}
 & \textbf{\mstream-AE} \\ \midrule
 \textbf{KDD} & $0.34 \pm 0.025$ & $0.34$ &  $0.81 \pm 0.018$ & $0.92$ & $0.63$ &   $0.91 \pm 0.016$ & $0.92 \pm 0.000$ & $\mathbf{0.96} \pm 0.002$ & $\mathbf{0.96} \pm 0.005$ \\
 \textbf{DoS} & $0.75 \pm 0.021$ & $0.50$ & $0.73 \pm 0.008$ & $0.61$  & $0.83$ & $0.93 \pm 0.001$ & $0.92 \pm 0.001$ & $\mathbf{0.95} \pm 0.003$ & $0.94 \pm 0.001$ \\
  \textbf{UNSW} & $0.25 \pm 0.003$ & $0.49$ & $0.84 \pm 0.023$ & $0.80$ & $0.45$ & $0.86 \pm 0.001$ & $0.81 \pm 0.001$ & $0.82 \pm 0.001$ & $\mathbf{0.90} \pm 0.001$ \\
\textbf{DDoS} & $0.57 \pm 0.106$ & $0.46$ & $0.56 \pm 0.021$ & $--$ & $0.63$ &  $0.91 \pm 0.000$ & $\mathbf{0.94} \pm 0.000$ & $0.82 \pm 0.000$ & $0.93 \pm 0.000$ \\
\bottomrule
\end{tabular}}
\end{table*}

Figure \ref{fig:AUC} plots ROC-AUC vs. running time (log-scale, in seconds, excluding I/O) for the different methods on the \emph{CICIDS-DoS} dataset. We see that \mstream, \mstream-PCA, \mstream-IB and \mstream-AE achieve $11\%$ to $90\%$ higher AUC compared to baselines, while also running almost two orders of magnitude faster.

\begin{figure}[!htb]
        \center{\includegraphics[width=\columnwidth]
        {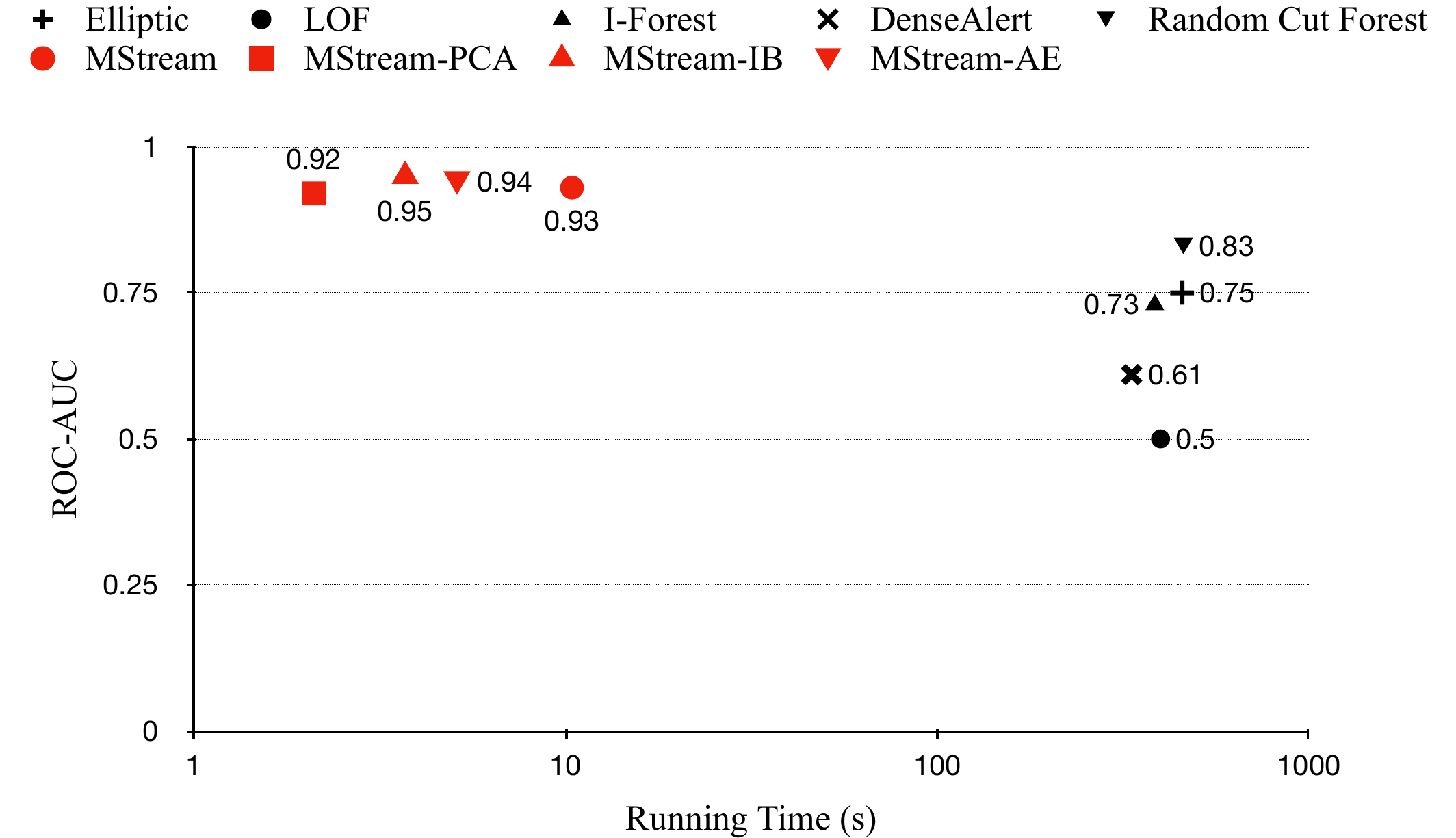}}
        \caption{\label{fig:AUC} ROC-AUC vs time on \emph{CICIDS-DoS} dataset.}
\end{figure}

\FloatBarrier

\subsection{Scalability}

Table \ref{tab:timesmstream} shows the time it takes \elliptic, \lof, \iso, \densealert, \rcf, \mstream\ and \mstream-AE to run on \emph{KDDCUP99}, \emph{CICIDS-DoS}, \emph{UNSW-NB 15} and \emph{CICIDS-DDoS} datasets. We see that \mstream\ runs much faster than the baselines: for example, \mstream\ is $79$ times faster than \densealert\ on the \emph{KDDCUP99} dataset. \mstream-PCA, \mstream-IB and \mstream-AE incorporate dimensionality reduction and are therefore faster than \mstream: for example, \mstream-AE is $1.38$ times faster than \mstream\ and $110$ times faster than \densealert\ on the \emph{KDDCUP99} dataset.

\begin{table*}[!htb]
\centering
\caption{Running time of each method on different datasets in seconds.}
\label{tab:timesmstream}
\resizebox{\linewidth}{!}{
\begin{tabular}{@{}rccccccccc@{}}
\toprule
& Elliptic
 & LOF
 & I-Forest
 &  DAlert
 & RCF
 & \textbf{\mstream}
 & \textbf{\mstream-PCA}
 & \textbf{\mstream-IB}
 & \textbf{\mstream-AE} \\ \midrule
 \textbf{KDD} & $216.3$ & $1478.8$ &  $230.4$ & $341.8$ & $181.6$ & $4.3$ & ${2.5}$ & ${3.1}$ & ${3.1}$ \\
 \textbf{DoS} & $455.8$ & $398.8$ & $384.8$ & $333.4$  & $459.4$ & $10.4$ & ${2.1}$ & $3.7$ & $5.1$ \\
 \textbf{UNSW} & $654.6$ & $2091.1$ & $627.4$ & $329.6$ & $683.8$ & $12.8$ & ${6.6}$ & $8$ & $8$ \\
\textbf{DDoS} & $3371.4$ & $15577s$ & $3295.8$ & $--$ & $4168.8$ & $61.6$ & ${16.9}$ & $25.6$ & $27.7$ \\
\bottomrule
\end{tabular}}
\end{table*}

Figure \ref{fig:scaling1} shows the scalability of \mstream\ with respect to the number of records in the stream (log-scale). We plot the time needed to run on the (chronologically) first $2^{12}, 2^{13}, 2^{14},...,2^{20}$ records of the \emph{CICIDS-DoS} dataset. Each record has $80$ dimensions. This confirms the linear scalability of \mstream\ with respect to the number of records in the input stream due to its constant processing time per record.

\begin{figure}[!htb]
        \center{\includegraphics[width=0.7\columnwidth]
        {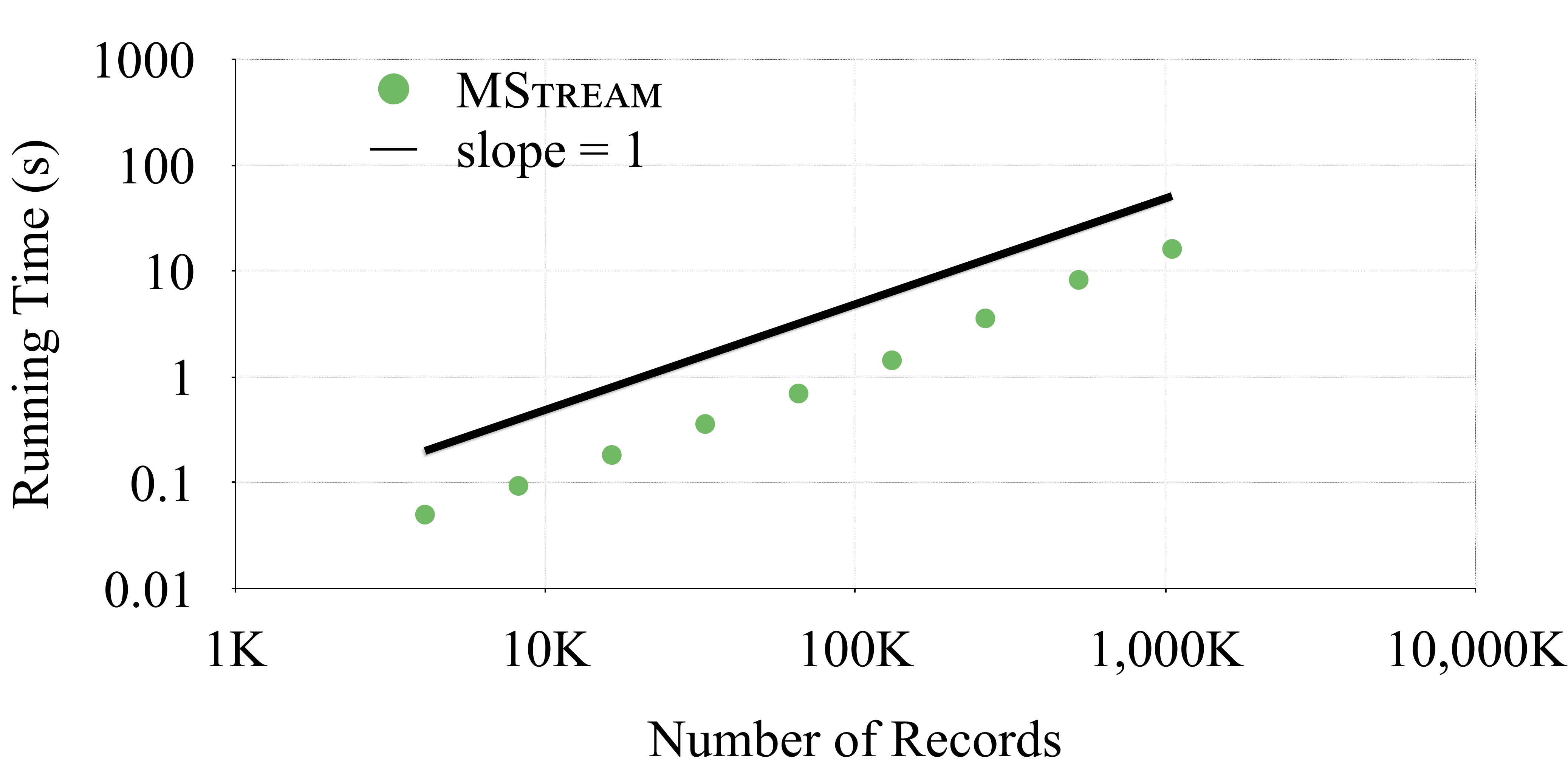}}
        \caption{\label{fig:scaling1} \mstream\ scales linearly with the number of records in \emph{CICIDS-DoS}.}
\end{figure}

Figure \ref{fig:scaling2} shows the scalability of \mstream\ with respect to the number of dimensions (linear-scale). We plot the time needed to run on the first $10, 20, 30,...,80$ dimensions of the \emph{CICIDS-DoS} dataset. This confirms the linear scalability of \mstream\ with respect to the number of dimensions in the input data.

\begin{figure}[!htb]
        \center{\includegraphics[width=0.7\columnwidth]
        {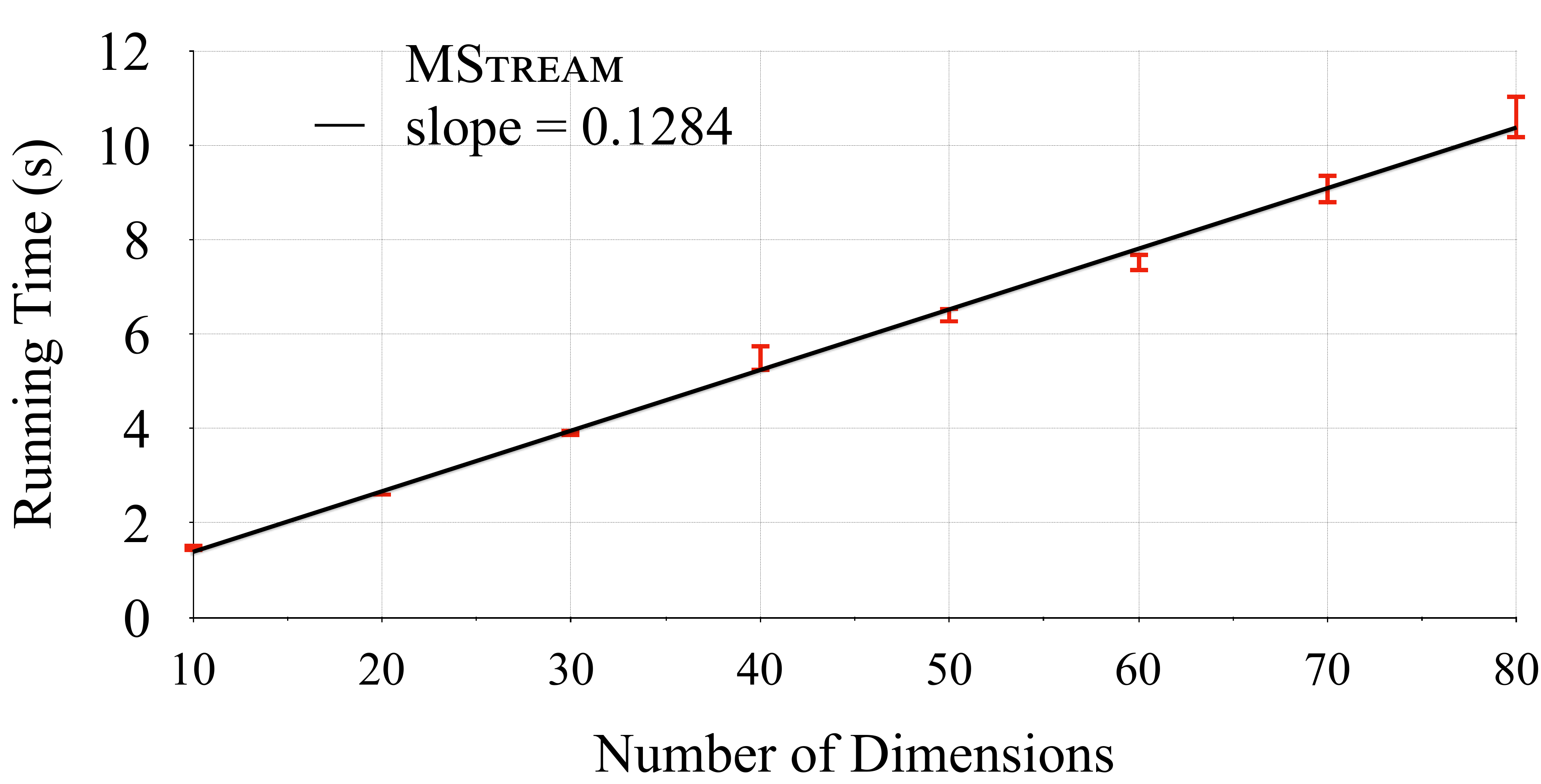}}
        \caption{\label{fig:scaling2} \mstream\ scales linearly with the number of dimensions in \emph{CICIDS-DoS}.}
\end{figure}

Figure \ref{fig:rows} shows the scalability of \mstream\ with respect to the number of hash functions (linear-scale). We plot the time taken to run on the \emph{CICIDS-DoS} dataset with $2, 3, 4$ hash functions. This confirms the linear scalability of \mstream\ with respect to the number of hash functions.

\begin{figure}[!htb]
        \center{\includegraphics[width=0.7\columnwidth]
        {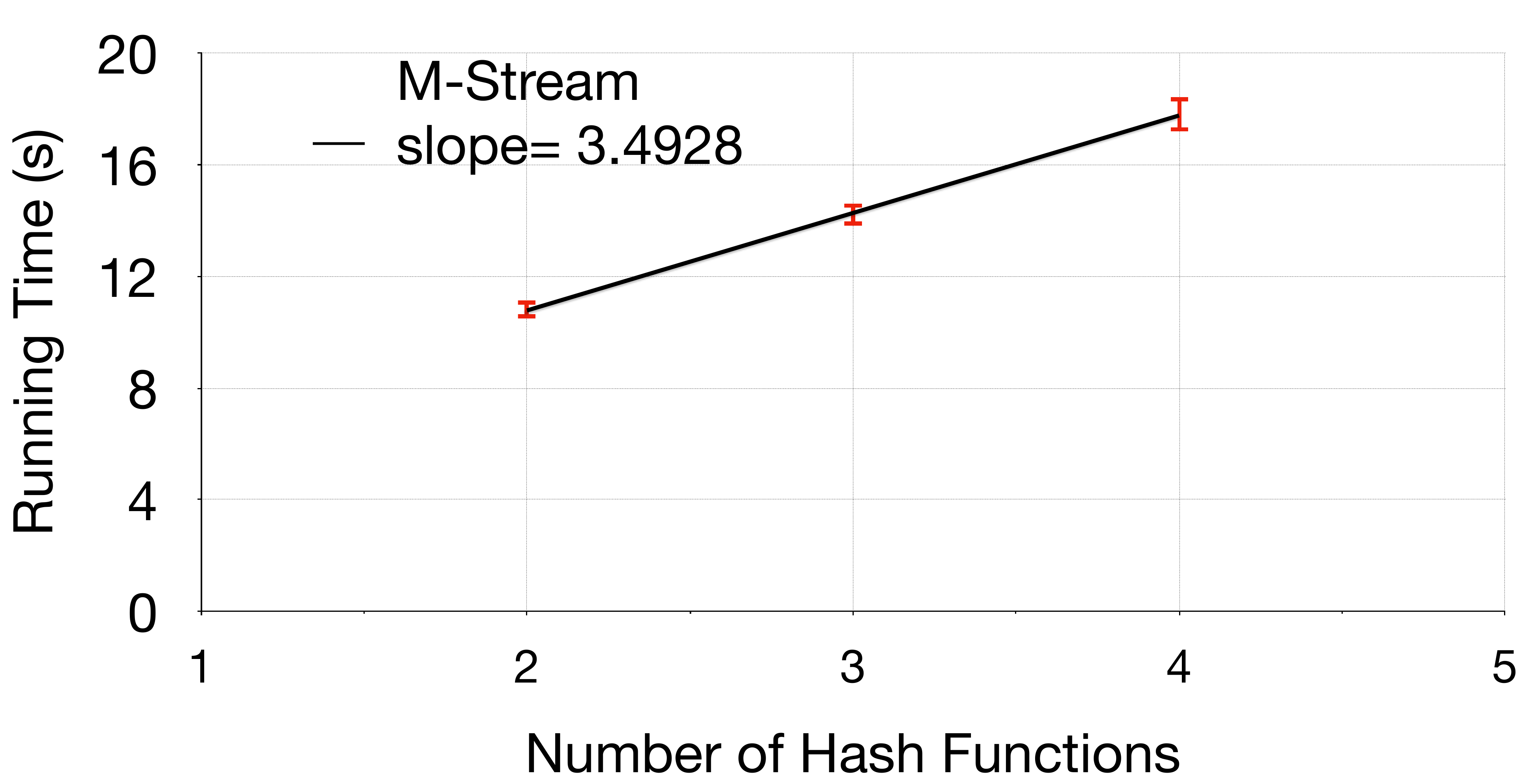}}
        \caption{\label{fig:rows} \mstream\ scales linearly with the number of hash functions in \emph{CICIDS-DoS}.}
\end{figure}

Since \mstream-PCA, \mstream-IB\ and \mstream-AE apply \mstream\ on the lower-dimensional features obtained using an autoencoder, they are also scalable.

Figure \ref{fig:frequencymstream} plots a frequency distribution of the time taken (in microseconds) to process each record in the \emph{CICIDS-DoS} dataset. \mstream\ processes $957K$ records within $10\mu s$ each, $60K$ records within $100\mu s$ each, and the remaining $30K$ records within $1000\mu s$ each.

\begin{figure}[!htb]
        \center{\includegraphics[width=0.7\columnwidth]
        {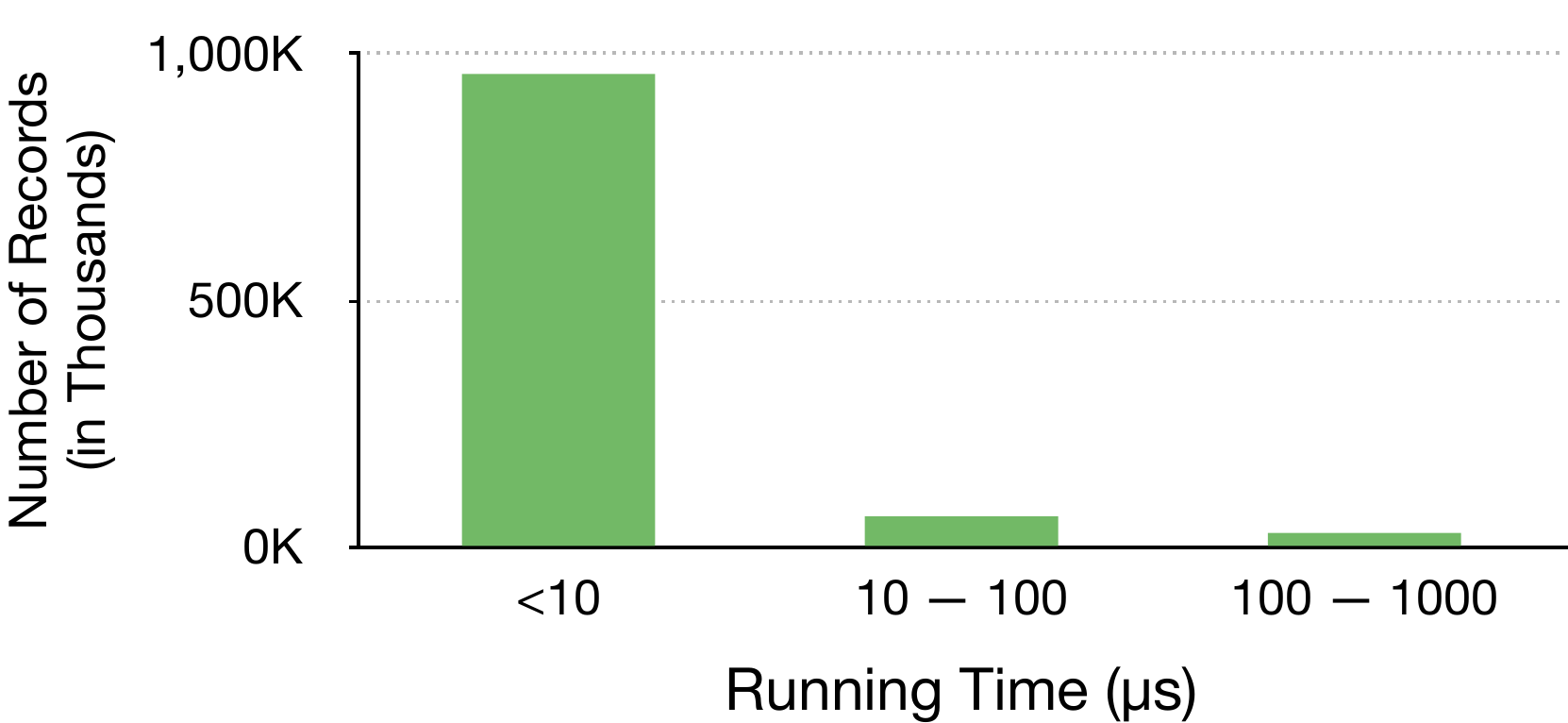}}
        \caption{\label{fig:frequencymstream} Distribution of processing times for $\sim1.05M$ records of the \emph{CICIDS-DoS} dataset.}
\end{figure}

\FloatBarrier

\subsection{Discoveries}
We plot normalized anomaly scores over time using \elliptic, \lof, \iso, \densealert, \rcf\ and \mstream\ on the \emph{CICIDS-DoS} dataset in Figure \ref{fig:dos}. To visualize, we aggregate records occurring in each minute by taking the max anomaly score per minute, for a total of $565$ minutes. Ground truth values are indicated by points plotted at $y=0$ (i.e. normal) or $y=1$ (anomaly).

\lof\ and \densealert\ miss many anomalies whereas \elliptic, \iso\ and \rcf\ output many high scores unrelated to any attacks. This is also reflected in Table \ref{tab:aucmstream} and shows that \mstream\ is effective in catching real-world anomalies.

\begin{figure}[!htb]
        \center{\includegraphics[width=0.7\columnwidth]
        {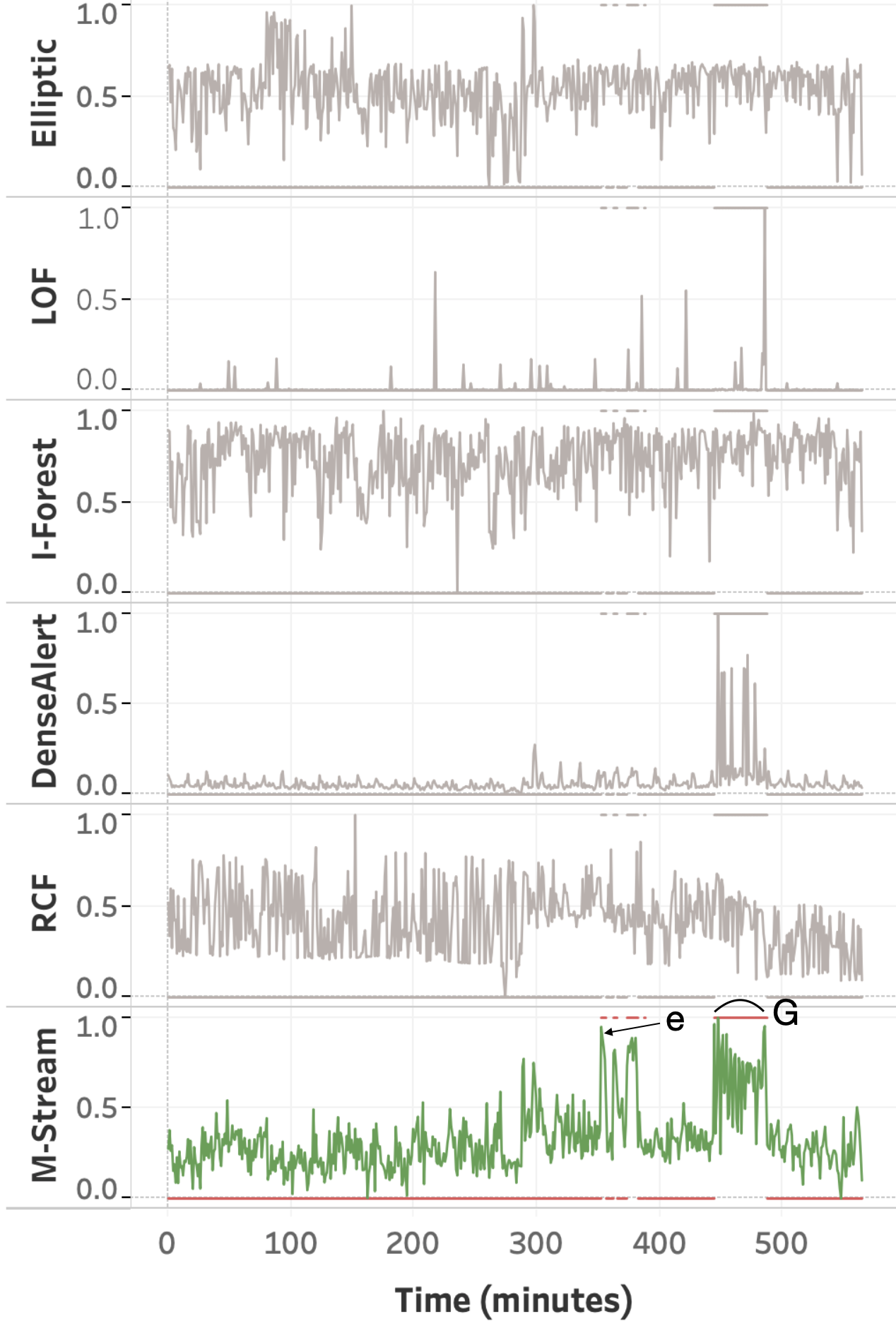}}
        \caption{\label{fig:dos} Plots of anomaly scores over time; spikes for \mstream\ correspond to the ground truth events in \emph{CICIDS-DoS}, but not for baselines.}
\end{figure}

\textbf{Group anomaly detection:}
In Figure \ref{fig:dos}, $G$ is a group anomaly that \mstream\ is able to detect, whereas \elliptic, \lof\ and \iso\ completely miss it. \densealert\ and \rcf\ partially catch it, but are also not fully effective in such high-dimensional datasets. This shows the effectiveness of \mstream\ in catching group anomalies such as DoS and DDoS attacks.

\textbf{Explainability:}
As \mstream\ estimates feature-specific anomaly scores before aggregating them, it is interpretable. For a given anomaly, we can rank the features according to their anomaly scores. We can then explain which features were most responsible for the anomalousness of a record in an unsupervised setting.

For example, in Figure \ref{fig:dos}, \mstream\ finds that $e$ is an anomaly that occurs due to the Flow IAT Min feature. This agrees with \cite{sharafaldin2018toward}, which finds that the best feature set for DoS using a Random Forest approach (supervised learning; in contrast, our approach does not require labels) are B.Packet Len Std, Flow IAT Min, Fwd IAT Min, and Flow IAT Mean.

\FloatBarrier

\section{Ablations}
\label{sec:ablations}
\subsection{Influence of temporal decay factor}
\label{app:3mstream}
Table~\ref{tab:FactorVsAUCmstream} shows the influence of the temporal decay factor $\alpha$ on the ROC-AUC for \mstream\ on \emph{CICIDS-DoS} dataset. We see that $\alpha=0.95$ gives the maximum ROC-AUC for \mstream\ ($0.9326 \pm 0.0006$), as also shown in Table~\ref{tab:aucmstream}.

\begin{table}[!htb]
		\centering
		\caption{Influence of temporal decay factor $\alpha$ on the
	    ROC-AUC in \mstream\  on \emph{CICIDS-DoS} dataset.}
		\label{tab:FactorVsAUCmstream}
		\begin{tabular}{@{}ccc@{}}
			\toprule
			$\alpha$ & ROC-AUC \\
			\midrule
		$0.1$ & $0.9129 \pm 0.0004$ \\
        $0.2$ & $0.9142 \pm 0.0009$\\
        $0.3$ & $0.9156	\pm 0.0006$ \\
        $0.4$ & $0.9164	\pm 0.0014$ \\
        $0.5$ & $0.9163	\pm 0.0005$ \\
        $0.6$ & $0.917 \pm 0.0005$ \\
        $0.7$ & $0.9196	\pm 0.0015$ \\
        $0.8$ & $0.9235	\pm 0.0003$ \\
        $0.9$ & $0.929 \pm	0.0003$ \\
        $0.95$ & $0.9326 \pm 0.0006$ \\
			\bottomrule
		\end{tabular}
		
\end{table}

\subsection{Influence of dimensions}
\label{app:5mstream}
Table~\ref{tab:dimensions} shows the influence of the output dimensions on the ROC-AUC for \mstream-PCA, \mstream-IB, and \mstream-AE \emph{KDDCUP99} dataset. We see that all methods are robust to the variation in output dimensions.

\begin{table}[H]
\centering
\caption{Influence of Output Dimensions on the ROC-AUC of \mstream-PCA, \mstream-IB, and \mstream-AE on \emph{KDDCUP99} dataset.}
\label{tab:dimensions}
\begin{tabular}{@{}rcccc@{}}
\toprule
\textbf{Dimensions}
 & \textbf{\mstream-PCA}
 & \textbf{\mstream-IB}
 & \textbf{\mstream-AE} \\ \midrule
 $4$ & $0.93$ & $0.95$ & $0.95$\\
 $8$ & $0.94$ & $0.95$ & $0.93$ \\
$12$ & $0.92$ & $0.96$ & $0.96$ \\
 $16$ & $0.87$ & $0.96$ & $0.96$ \\
\bottomrule
\end{tabular}
\end{table}

\subsection{Dimensionality Reduction}
\label{sec:app1}
For \mstream-IB, we used an online implementation, \url{https://github.com/burklight/nonlinear-IB-PyTorch} for the underlying Information Bottleneck algorithm with $\beta=0.5$ and the variance parameter set to a constant value of $1$. The network was implemented as a $2$ layer binary classifier. For \mstream-AE, the encoder and decoder were implemented as single layers with ReLU activation.

Table \ref{tab:aearch} shows the network architecture of the autoencoder. Here $n$ denotes the batch size, and $d$ denotes the input data dimensions. The input data dimensions for each dataset are described in Section \ref{sec:experiment}.  
\begin{table}[!htb]
\begin{center}
\caption{Autoencoder Architecture}
\label{tab:aearch}
\begin{tabular}{@{}rccc@{}}
\toprule
\textbf{Index} & \textbf{Layer} & \textbf{Output Size}  \\ \midrule
$1$ \ \ \ \ & Linear & $n\times12$\\
$2$ \ \ \ \ & ReLU & $n\times12$\\
$3$ \ \ \ \ & Linear & $n\times d$\\
\bottomrule
\end{tabular}
\end{center}
\end{table}

We used Adam Optimizer to train both these networks with $\beta_1 = 0.9$ and $\beta_2 = 0.999$. Grid Search was used for hyperparameter tuning:  Learning Rate was searched on $[1\mathrm{e}-2, 1\mathrm{e}-3, 1\mathrm{e}-4, 1\mathrm{e}-5]$, and number of epochs was searched on $[100, 200, 500, 1000]$. The final values for these can be found in Table \ref{tab:ibparam}.
\begin{table}[!htb]
\centering
\caption{\mstream-IB parameters for different datasets.}
\label{tab:ibparam}
\begin{tabular}{@{}rcccc@{}}
\toprule
& \multicolumn{2}{c}{\mstream-IB} & \multicolumn{2}{c}{\mstream-AE} \\
 \textbf{Dataset} & Learning Rate  &  Epochs & Learning Rate  & Epochs  \\ \midrule
 \textbf{KDD} & $1\mathrm{e}-2$ & $100$ & $1\mathrm{e}-2$ & $100$\\
 \textbf{DoS} & $1\mathrm{e}-5$ &  $200$ & $1\mathrm{e}-2$ &  $1000$\\
 \textbf{UNSW} & $1\mathrm{e}-2$ & $100$ & $1\mathrm{e}-2$ & $100$\\
\textbf{DDoS} & $1\mathrm{e}-3$ & $200$ & $1\mathrm{e}-3$ & $100$\\
\bottomrule
\end{tabular}
\end{table}

\subsection{Evaluating ROC-AUC in a streaming manner}
\label{app:4mstream}
Table~\ref{tab:streamingauc} shows the ROC-AUC for \mstream-AE on \emph{KDDCUP99} when evaluated over the stream. The evaluation is done on all records seen so far and is performed after every $100K$ records. We see that as the stream length increases, ROC-AUC for \mstream-AE converges to $0.96$, as also shown in Table~\ref{tab:aucmstream}.

\begin{table}[H]
\centering
\caption{Evaluating ROC-AUC of \mstream-AE in a streaming manner on \emph{KDDCUP99} dataset.}
\label{tab:streamingauc}
\begin{tabular}{@{}rcc@{}}
\toprule
	Stream Size & ROC-AUC \\ \midrule
 $100K$ & $0.912488$ \\
        $200K$ & $0.895391$ \\
        $300K$ & $0.855598$ \\
        $400K$ & $0.934532$ \\
        $500K$ & $0.965250$ \\
        $600K$ & $0.953906$ \\
        $700K$ & $0.947531$ \\
        $800K$ & $0.961340$ \\
        $900K$ & $0.973217$ \\
        $1000K$ & $0.970212$ \\
        $1100K$ & $0.967215$ \\
        $1200K$ & $0.959664$ \\
\bottomrule
\end{tabular}
\end{table}

\FloatBarrier

\section{Conclusion}
In this chapter, we proposed \mstream\ for detecting group anomalies in multi-aspect streams, and \mstream-PCA, \mstream-IB, and \mstream-AE\ which incorporate dimensionality reduction to improve accuracy and speed. Future work could consider more complex combinations (e.g. weighted sums) of anomaly scores for individual attributes. Our contributions are:
\begin{enumerate}
   \item {\bfseries Multi-Aspect Group Anomaly Detection:} We propose a novel approach for detecting group anomalies in multi-aspect data, including both categorical and numeric attributes. Moreover, the anomalies detected by \mstream\ are explainable.
    \item {\bfseries Streaming Approach:} Our approach processes the data in a fast and streaming fashion, performing each update in constant time and memory.
    \item {\bfseries Effectiveness:} Our experimental results using \emph{KDDCUP99}, \emph{CICIDS-DoS}, \emph{UNSW-NB 15} and \emph{CICIDS-DDoS} datasets show that \mstream\ outperforms baseline approaches.
    \item {\bfseries Incorporating Correlation:} We propose \mstream-PCA, \mstream-IB, and \mstream-AE to incorporate correlation between features.
\end{enumerate}

\SetPicSubDir{MemStream}
\SetExpSubDir{MemStream}

\chapter[MemStream: Memory-Based Streaming Anomaly Detection][MemStream]{MemStream:\:Memory-Based Streaming Anomaly Detection}
\label{ch:memstream}
\vspace{2em}

\begin{mdframed}[backgroundcolor=magenta!20] 
Chapter based on work that appeared at WWW'22 \cite{bhatia2022memstream} \href{https://dl.acm.org/doi/pdf/10.1145/3485447.3512221}{[PDF]}.
\end{mdframed}

\section{Introduction}
Given a stream of entries over time in a multi-dimensional data setting where concept drift is present, how can we detect anomalous activities?

To handle concept drift in a streaming setting, our approach uses an explicit memory module. For anomaly detection, this memory can be used to store the trends of normal data that act as a baseline with which to judge incoming records. A read-only memory, in a drifting setting, is of limited use and thus should be accompanied by an appropriate memory update strategy. The records arrive over time; thus, older records in the memory might no longer be relevant to the current trends suggesting a First-In-First-Out memory replacement strategy. The introduction of memory, with an appropriate update strategy, seems to tackle some of the issues in streaming anomaly detection with concept drift. However, the system described so far does not provide a fail-safe for when an anomalous sample enters the memory and is thus susceptible to memory poisoning.

We, therefore, propose \memstream, which uses a denoising autoencoder \cite{denoisingae} to extract features, and a memory module to learn the dynamically changing trend, thereby avoiding the over-generalization of autoencoders (i.e. the problem of autoencoders reconstructing anomalous samples well). Our streaming framework is resilient to concept drift and we prove a theoretical bound on the size of memory for effective drift handling. Moreover, we allow quick retraining when the arriving stream becomes sufficiently different from the training data.

We also discuss two architectural design choices to make \memstream\ robust to memory poisoning. The first modification prevents anomalous elements from entering the memory, and the second modification deals with how the memory can be self-corrected and recovered even if it harbors anomalous elements. Finally, we discuss the effectiveness of \memstream\ compared to state-of-the-art streaming baselines.

In summary, our main contributions are:
\begin{enumerate}
    \item {\bfseries Streaming Anomaly Detection:} We propose a novel streaming approach using a denoising autoencoder and a memory module, for detecting anomalies. \memstream\ is resilient to concept drift and allows quick retraining.
    \item {\bfseries Theoretical Guarantees:} In Proposition \ref{prop:1}, we discuss the optimum memory size for effective concept drift handling. In Proposition \ref{prop:2}, we discuss the motivation behind our architecture design.
    \item {\bfseries Robustness to Memory Poisoning:} \memstream\ prevents anomalies from entering the memory and can self-correct and recover from bad memory states.
    \item {\bfseries Effectiveness:} Our experimental results show that \memstream\ convincingly outperforms $11$ state-of-the-art baselines using $2$ synthetic datasets (that we release as open-source) and $11$ popular real-world datasets.
\end{enumerate}

{\bfseries Reproducibility}: Our code and datasets are available on \href{https://github.com/Stream-AD/MemStream}{https://github.com/Stream-AD/MemStream}.

\section{Problem}

Let $\mathcal{X} = \{x_1, x_2, \cdots\}$ be records arriving in a streaming manner. Each entry $x_i = (x_{i1}, \cdots, x_{id})$ consisting of $d$ \emph{attributes} or dimensions, where each dimension can either be categorical (e.g. IP address) or real-valued (e.g. average packet length).

Our goal is to detect anomalies in streaming data. A common phenomenon in real-world data is that the nature of the stream changes over time. These changes are generally described in terms of the statistical properties of the stream, such as the mean changes across some or all features. As the definition of the ``concept" of normal behavior changes, so does the definition of an anomaly. Thus, we need a model that is able to adapt to the dynamic trend and thereby recognize anomalous records.

\section{Algorithm}

\subsection{Motivation}

Consider an attacker who hacks a particular IP address and uses it to launch denial of service attacks on a server. Modern cybersecurity systems are trained to detect and block such attacks, but this is made more challenging by changes over time, e.g. in the identification of attacking machines. This is a ``concept" drift and the security system must learn to identify such changing trends to mitigate the attacks. Consider the toy example in Table \ref{tab:toymemstream}, comprising a multi-dimensional temporal data stream. There is a sudden distribution change and concept drift in all attributes from time $t=5$ to $t=6$.

\begin{table}[!ht]
\centering
\caption{Simple toy example, consisting of a stream of records over time with a trend shift at $t=6$.}
\label{tab:toymemstream}
\begin{tabular}{@{}rrrrr@{}}
\toprule
{\bfseries Time} & {\bfseries Feature 1} & {\bfseries Feature 2} & {\bfseries Feature 3} & {\bfseries ...} \\ \midrule
$1$ & $8.39$ & $1.44$ & $4.16$ & $\cdots$ \\
$2$ & $6.72$ & $4.55$ & $3.49$ & $\cdots$ \\
$3$ & $3.49$ & $2.10$ & $1.56$ & $\cdots$ \\
$4$ & $4.28$ & $0.64$ & $1.22$ & $\cdots$ \\
$5$ & $5.54$ & $2.40$ & $6.55$ & $\cdots$ \\
$6$ & $183.75$ & $132.03$ & $9.86$ & $\cdots$ \\
$7$ & $146.47$ & $128.49$ & $16.52$ & $\cdots$ \\
$8$ & $197.96$ & $97.16$ & $15.05$ & $\cdots$ \\
$9$ & $192.50$ & $89.95$ & $12.46$ & $\cdots$ \\
$10$ & $158.32$ & $10.37$ & $15.76$ & $\cdots$ \\
\bottomrule
\end{tabular}
\end{table}

The main challenge for the algorithm is to detect these types of patterns in a {\bfseries streaming} manner within a suitable timeframe. That is, the algorithm should not give an impulsive reaction to a short-lived change in the base distribution, but also should not take too long to adapt to the dynamic trend. Note that we do not want to set any limits a priori on the duration of the anomalous activity we want to detect, or the window size after which the model should be updated to account for the concept drift.

\subsection{Overview}
As shown in Figure \ref{fig:teaser}, the proposed \memstream\ algorithm addresses these problems through the use of a memory augmented feature extractor that is initially trained on a small subset of normal data. The memory acts as a reserve of encodings of normal data. At a high level, the role of the feature extractor is to capture the structure of normal data. An incoming record is then scored by \emph{calculating the discounted score} based on the similarity of its encoding as evaluated against those in memory. Based on this score, if the record is deemed normal, then it is used to \emph{update the memory}. To adapt to the changing data trend, memory is required to keep track of the data drift from the original distribution. Since concept drift is generally a gradual process, the memory should maintain the temporal contiguity of records. This is achieved by following a First-In-First-Out (FIFO) memory replacement policy.

\begin{figure*}[!t]
\begin{centering}
  \includegraphics[width=\textwidth]{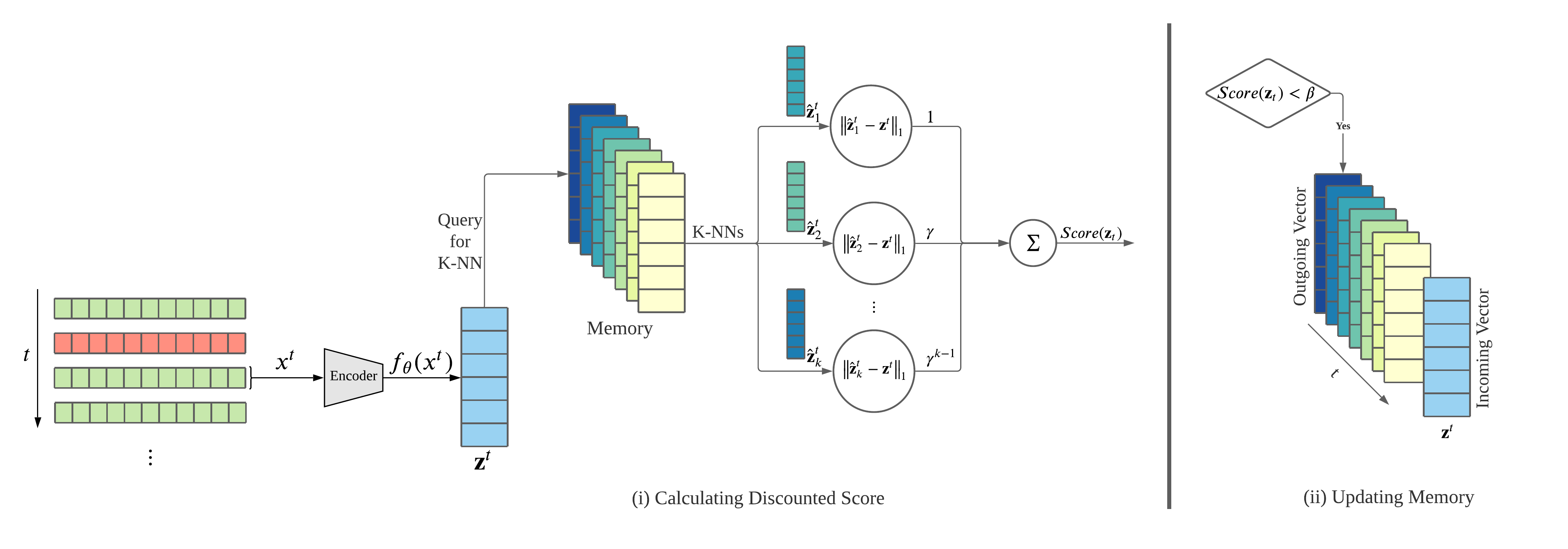}
\caption{
After initial training of the feature extractor on a small subset of normal data, \memstream\ processes records in two steps: (i) It outputs anomaly scores for each record by querying the memory for $K$-nearest neighbors to the record encoding and calculating a discounted distance and (ii) It updates the memory, in a FIFO manner, if the anomaly score is within an update threshold $\beta$.}
  \label{fig:teaser}
  \end{centering}
\end{figure*}

\subsection{Feature Extraction}
\label{sec:representation}
Neural Networks can learn representations using an autoencoder consisting of two parts - an encoder and a decoder \cite{Goodfellow-et-al-2016}. The encoder forms an intermediate representation of the input samples and the decoder is trained to reconstruct the input samples from their intermediate representations. Denoising autoencoders~\cite{denoisingae} partially corrupt the input data before passing it through the encoder. Intuitively, this ``forces" the network to capture the useful structure in the input distribution, pushing it to learn more robust features of the input. In our implementation, we use an additive isotropic Gaussian noise model.

\memstream\ allows flexibility in the choice of the feature extraction backbone. We consider Principal Component Analysis (PCA) and Information Bottleneck (IB) \cite{tishby2000information,kolchinsky2019nonlinear} as alternatives to autoencoders for feature extraction~\cite{Bhatia2021MSTREAM}. PCA-based methods are effective for off-the-shelf learning, with little to no hyperparameter tuning. Information Bottleneck can be used for learning useful features by posing the following optimization problem:
$$
\min _{p(t | x)} I(X ; T)-\beta I(T; Y)
$$
where $X$, $Y$, and $T$ are random variables. $T$ is the compressed representation of $X$, $I(X ; T)$ and $I(T ; Y)$ are the mutual information of $X$ and $T$, and of $T$ and $Y$, respectively, and $\beta$ is a Lagrange multiplier. The problem configuration and the available data greatly influence the choice of the feature extraction algorithm. We evaluate the methods to extract features in Section \ref{exp:representation}.

\subsection{Memory}
\paragraph{\textbf{Memory-based Representation:}} The memory $\boldsymbol{M}$ is a collection of $N$ real-valued $D$ dimensional vectors where $D$ is the dimension of the encodings $\vec{z}$. Given a representation $\vec{z}$, the memory is queried to retrieve the $K$-nearest neighbors $\{{\vec{\hat{z}}_1^{t},\vec{\hat{z}}_2^{t} ... \vec{\hat{z}}_K^{t}\}}$ of $\vec{z}$ in $\boldsymbol{M}$ under the $\ell_1$ norm such that:
$$||\vec{\hat{z}}_1^{t} - \vec{z}||_1 \leq ... \leq ||\vec{\hat{z}}_K^{t} - \vec{z}||_1$$
The hyper-parameter $N$ denotes the memory size. The performance of the algorithm varies depending on the value of $N$; very large or small values of $N$ would hinder the performance.

\paragraph{\textbf{Memory Update:}} Fixed memory trained on limited samples of streaming data will not be able to handle concept drift; therefore, continuous memory update is necessary. Different memory update strategies can be used such as Least Recently Used (LRU), Random Replacement (RR), and First-In-First-Out (FIFO). We observe that the FIFO memory update policy wherein the new element to be added replaces the earliest added element in the memory works well in practice. It can easily handle concept drift in streaming data as the memory retains the most recent non-anomalous samples from the distribution. We compare FIFO with LRU and RR strategies in more detail in Section \ref{exp:memory}. It is also interesting to note that \memstream\ can easily handle periodic patterns by adjusting the memory size: a memory of size greater than the product of the period and the sampling frequency should be sufficient to avoid flagging periodic changes as anomalies. Section \ref{sec:drift} evaluates \memstream's ability to detect anomalies in a periodic setting.

As shown in Algorithm \ref{alg:memstream}, the autoencoder is initially trained with a small amount of data $\mathcal{D}$ to learn how to generate data embeddings (line 2). The memory is initialized with the same training dataset (line 3). We also store the mean and standard deviation of this small training dataset. As new records arrive, the encoder performs normalization using the stored mean and standard deviation and computes the compressed representation $\vec{z}^t$ (line 6). It then computes the $K$-nearest neighbors ($\hat{\vec{z}}^t_1, \cdots,  \hat{\vec{z}}^t_K$) by querying the memory (line 8), and calculates their $\ell_1$ distance with $\vec{z}^t$ (line 10). The final discounted score is calculated as an exponentially weighted average (weighting factor $\gamma$) (line 12). This helps in making the autoencoder more robust. The discounted score is then compared against a user-defined threshold $\beta$ (line 14) and the new record is updated into the memory in a FIFO manner if the score falls within $\beta$ (line 15). This step ensures that anomalous records do not enter the memory. If the memory is updated, then the stored mean and standard deviation are also updated accordingly. The discounted score is returned as the anomaly score for the record $\vec{x}^t$ (line 17).

\begin{algorithm}
	\caption{\memstream\ \label{alg:memstream}}
	\KwIn{Stream of data records}
	\KwOut{Anomaly scores for each record}
	{\bfseries $\triangleright$ Initialization} \\
	Feature Extractor, $f_\theta$, trained using small subset of data $\mathcal{D}$ \\
	Memory, $M$, initialized as $f_\theta(\mathcal{D})$\\
	\While{new sample $\vec{x}^t$ is received:}{
	{\bfseries $\triangleright$ Extract features:} \\
	$\vec{z}^t = f_\theta(\vec{x}^t)$\\
	{\bfseries $\triangleright$ Query memory:} \\
	$\{\vec{\hat{z}_1}^{t},\vec{\hat{z}_2}^{t} ... \vec{\hat{z}_K}^{t}\}  =$ $K$-nearest neighbors of $\vec{z}^t$ in $M$\\
	{\bfseries $\triangleright$ Calculate distance:} \\
	$R(\vec{z}^t,\vec{\hat{z}_i}^t) = ||\vec{z}^t-\vec{\hat{z}_i}^t||_1$ \algorithmicforall\ $i \in 1..K$\\
	{\bfseries $\triangleright$ Assign discounted score:} \\
	$Score(\vec{z}^t) = \dfrac{\sum_{i=1}^K{\gamma^{i-1}R(\vec{z}^t,\vec{\hat{z}_i}^t)}}{\sum_{i=1}^K{\gamma^{i-1}}}$\\
	{\bfseries $\triangleright$ Update Memory:} \\
	\If{$Score(\vec{z}^t) < \beta$}
	{
	    Replace earliest added element in $\boldsymbol{M}$ with $\vec{z}^t$ 
	}
	{\bfseries $\triangleright$ Anomaly Score:}\\
	{\bfseries output} $Score(\vec{z}^t)$\\
	}
\end{algorithm}

\FloatBarrier

\subsection{Theoretical Analysis}

\subsubsection{Relation between Memory Size and Concept Drift}
\label{sec:theorymemsize}
Our analysis of the relation between memory size and concept drifts suggests that the memory size should be proportional to (the spread of data distributions) / (the speed of concept drifts).

As we increase the size of memory,  we can decrease the possibility of a  false positive (falsely classifying a normal sample as an anomaly). This is because it is more likely for a new data point to have a close point in a larger memory. Therefore, on the one hand, in order to decrease the false \textit{positive} rate, we want to increase the memory size. On the other hand, in order to minimize a false \textit{negative} rate  (i.e.,  failing to raise an alarm when an anomaly did happen), Proposition \ref{prop:1} suggests that the memory size should be smaller than some quantity proportional to (standard deviations of distributions) / (the speed of distributional drifts). That is, it suggests that the memory size should be  smaller than  $2 \sigma \sqrt{d(1+\epsilon)}/\alpha$, where $d$ is the input dimension, $\alpha $ measures the speed of distributional drifts, $\sigma$ is the standard deviation of distributions, and $\epsilon \in (0,1)$.  More concretely, under drifting normal distributions, the proposition shows that a new distribution after $\tau$  drifts and an original distribution before the $\tau$ drifts are sufficiently dissimilar whenever $\tau>2 \sigma \sqrt{d(1+\epsilon)}/\alpha$, so that the memory should forget about the original distribution to minimize a false-negative rate. We also discuss this effect of increasing the memory size in Section \ref{exp:memlen}.

\begin{proposition} \label{prop:1}
Define   $S_{t,\epsilon}=\{x\in \RR^d : \|x-\mu_{t}\|_2 \le \sigma\sqrt{d(1+\epsilon)}\}$.
Let $(\mu_t)_t$ be the sequence such that there exits a positive real number $\alpha$ for which $\|\mu_t- \mu_{t'}\|_2\ge (t'-t)\alpha$ for any $t<t'$.  Let  $\tau > \frac{2 \sigma \sqrt{d(1+\epsilon)}}{\alpha}$ and    ${\displaystyle x_{t}\sim \ {\mathcal {N}}(\mu_t, \sigma I)}$ for all $t\in \NN^+$. Then, for any $\epsilon > 0$ and $t \in \NN^+ $, with probability at least $1-2 \exp(-d\epsilon^2/8)$, the following holds:  $x_t \in S_{t,\epsilon}$ and $x_{t+\tau} \notin S_{t,\epsilon}$.
\end{proposition}
\begin{proof}
Let us write $\bd(x,x')=\|x-x'\|_2$. Then, by the triangle inequality,
\begin{align} \label{eq:1}
\bd(\mu_t, \mu_{t+\tau}) \le \bd(\mu_t, x_{t+\tau})+\bd( x_{t+\tau}, \mu_{t+\tau}). 
\end{align} 
By using the property of  the Gaussian distribution with   ${\displaystyle z_{t+\tau}\sim \ {\mathcal {N}}(0,  I)}$, we have that 
\begin{align*}
& \Pr(\| x_{t+\tau}-\mu_{t+\tau}\|_{2}< \sigma\sqrt{d(1+\epsilon)} ) \\
& = \Pr(\| \sigma z_{t+\tau}+\mu_{t+\tau}-\mu_{t+\tau}\|_{2}< \sigma\sqrt{d(1+\epsilon)} ) \\
& =\Pr(\| z_{t+\tau}\|_{2}^{2}< d(1+\epsilon) ). 
\end{align*}

Thus, using the Chernoff bound for the Standard normal distribution  for  ${\displaystyle z_{t+\tau}\sim \ {\mathcal {N}}(0,  I)}$, we have that
\begin{align*}
\Pr(\| x_{t+\tau}-\mu_{t+\tau}\|_{2}>\sigma\sqrt{d(1+\epsilon)} ) \le \exp\left(-\frac{d\epsilon^2}{8}\right).
\end{align*}
Similarly, 
\begin{align*}
\Pr(\| x_{t}-\mu_{t}\|_{2}>\sigma\sqrt{d(1+\epsilon)} ) \le \exp\left(-\frac{d\epsilon^2}{8}\right).
\end{align*}
By tanking union hounds, we have  that with probability at least $1-2 \exp(-d\epsilon^2/8)$,
\begin{align} \label{eq:2}
\| x_{t+\tau}-\mu_{t+\tau}\|_{2}\le\sigma\sqrt{d(1+\epsilon)},
\end{align}
and
\begin{align} \label{eq:3}
\| x_{t}-\mu_{t}\|_{2}\le\sigma\sqrt{d(1+\epsilon)}.
\end{align}
By using the upper bound of \eqref{eq:2} in \eqref{eq:1}, we have that $\bd(\mu_t, \mu_{t+\tau}) \le \bd(\mu_t, x_{t+\tau})+\sigma\sqrt{d(1+\epsilon)}$,
which implies that  
\begin{align*} 
\bd(\mu_t, \mu_{t+\tau}) -\sigma\sqrt{d(1+\epsilon)}\le \bd(\mu_t, x_{t+\tau}).
\end{align*}
Using the assumption on  $(\mu_t)_t$, 
\begin{align*} 
\tau\alpha-\sigma\sqrt{d(1+\epsilon)}\le \bd(\mu_t, x_{t+\tau}).
\end{align*}
Using the definition of $\tau$, 
\begin{align*} 
 \sigma \sqrt{d(1+\epsilon)}< \bd(\mu_t, x_{t+\tau}).
\end{align*}
This means that $x_{t+\tau} \notin S_{t,\epsilon}$. On the other hand, equation \eqref{eq:3} shows that $ x_{t} \in S_{t,\epsilon}$. 
\end{proof}

\subsubsection{Architecture Choice}
\label{sec:theoryarch}
In the following, we provide one reason why we use an architecture with $d \le D$, where $d$ is the input dimension and $D$ is the embedding dimension. Namely, Proposition \eqref{prop:2} shows that if $d>D$, then there exists an anomaly constructed through perturbation of a normal sample such that the anomaly is not detectable. The construction of an anomaly in the proof is indeed unique to the case of  $d>D$, and is not applicable to the case of   $d \le D$. This provides the motivation for why we may want to use the architecture of $d \le D$, to avoid such an undetectable anomaly.

Let $\theta$ be fixed. Let  $f_\theta$ be  a  deep neural network $f_\theta: \RR^d\rightarrow \RR^D$ with ReLU and/or max-pooling as:
$f_\theta(x)=\sigma^{[L]}\big(z^{[L]}(x,\theta)\big),  z^{[l]}(x,\theta) = W^{[l]} \sigma^{(l-1)} \left(z^{[l-1]}(x,\theta)\right)$, for  $l=1,2,\dots, L$, where $\sigma^{(0)} \left(z^{[0]}(x,\theta)\right)= x$, $\sigma$ represents nonlinear function due to ReLU and/or max-pooling, and $W^{[l]}\in \RR^{N_l \times N_{l-1}}$ is a matrix of weight parameters connecting the $(l-1)$-th layer to the $l$-th layer. For the nonlinear function $\sigma$ due to ReLU and/or max-pooling, we can define $\dot \sigma^{[l]}(x,\theta)$ such that $\dot \sigma^{[l]}(x,\theta)$ is a diagonal matrix with each element being $0$ or $1$, and $\sigma^{[l]} \left(z^{[l]}(x,\theta)\right)=\dot \sigma^{[l]}(x,\theta) z^{[l]}(x,\theta)$. For any differentiable point $x$ of $f_\theta$, define $\Omega(x) = \{x'\in \RR^d:\forall l, \  \dot \sigma^{[l]}(x',\theta)= \dot\sigma^{[l]}(x,\theta)\}$ and $\Bcal_{r}(x)=\{x'\in \RR^d : \|x - x'\|_2 \le r\}$.

\begin{proposition} \label{prop:2}
Let $x$ be a differentiable point of $f_\theta$ such that   $\mathcal{B}_{r}(x) \subseteq \Omega(x)$ for some $r>0$. If $d>D$, then there exists a  $\delta \in \RR^d$ such that for any $\hx\in\RR^d$ and $\bbeta >0$,  the following holds: $\|\delta\|_{2}= r$ and 
$$
R(x,\hx) <\bbeta \implies R(x+\delta, \hx)< \bbeta.
$$
\end{proposition}
\begin{proof}
 We can rewrite the output of the function as
$
f_\theta(x)=\dot \sigma^{[L]}(x,\theta) W^{[L]} \dot \sigma^{[L-1]}(x,\theta)  W^{[L-1]}  \cdots 
W^{[2]}\dot \sigma^{[1]}(x,\theta) W^{[1]} x.
$
Thus, for any $\delta$ such that $(x+\delta) \in \mathcal{B}_{r}(x) \subseteq \Omega(x)$, we have
\begin{align*}
f_\theta(x+\delta)& =\dot \sigma^{[L]}(x+\delta,\theta) W^{[L]} \dot \sigma^{[L-1]}(x+\delta,\theta)  W^{[L-1]}  \cdots \\
& W^{[2]}\dot \sigma^{[1]}(x+\delta,\theta) W^{[1]} (x+\delta)
\\
& =\sigma^{[L]}(x,\theta) W^{[L]} \dot \sigma^{[L-1]}(x,\theta)  W^{[L-1]}  \cdots \\
& W^{[2]}\dot \sigma^{[1]}(x,\theta) W^{[1]} (x+\delta)
 \\ & = M x + M \delta 
\end{align*}
where $M = \sigma^{[L]}(x,\theta) W^{[L]} \dot \sigma^{[L-1]}(x,\theta)  W^{[L-1]}  \\ \cdots W^{[2]}\dot \sigma^{[1]}(x,\theta) W^{[1]}$. Notice that $M$ is a matrix of size $D$ by $d$. Thus,  ff $d>D$, there the nulls space (or the kernel space) of $M$ is not $\{0\}$ and there exists $\delta'\in\RR^d $ in the null space of $M$ such that $\|\delta'\|\neq 0$ and $M (r'\delta')=0$ for all $r'>0$. Thus, there exists a  $\delta \in \RR^d$ such that  $(x+\delta) \in \mathcal{B}_{r}(x) \subseteq \Omega(x)$,   $\|\delta\|_2 = r$, and $M \delta=0$, yielding
$$
f_\theta(x+\delta)=M x=f_\theta(x). 
$$
This implies the statement of this proposition. 
\end{proof}

\FloatBarrier
\section{Experiments}
\label{sec:expmemstream}

In this section, we aim to answer the following questions:

\begin{enumerate}[label=\textbf{Q\arabic*.}]
\item {\bfseries Comparison to Streaming Methods:} How accurately does \memstream\ detect real-world anomalies as compared to state-of-the-art streaming baseline methods?
\item {\bfseries Concept Drift:} How fast can \memstream\ adapt under concept drift?
\item {\bfseries Retraining:} What effect does retraining \memstream\ have on the accuracy and time?
\item {\bfseries Self-Correction and Recovery:} Does \memstream\ provide a self-correction mechanism to recover from ``bad" memory states?
\end{enumerate}

\paragraph{Experimental Setup}
All methods output an anomaly score for every record (higher is more anomalous). We report the ROC-AUC (Area under the Receiver Operating Characteristic curve). All experiments, unless explicitly specified, are performed $5$ times for each parameter group, and the mean values are reported. All experiments are carried out on a $2.6 GHz$  Intel Core \textit{i}$7$ system with $16 GB$ RAM and running Mac OS Catalina $10.15.5$. Following \textsc{MStream}, we take the output dimension as $8$ for PCA and IB. For \memstream-PCA, we use the open-source implementation available in the scikit-learn \cite{scikit-learn} library of Principal Component Analysis. For \memstream-IB, we used an online implementation \footnote{\url{https://github.com/burklight/nonlinear-IB-PyTorch}} for the underlying Information Bottleneck algorithm with $\beta=0.5$ and the variance parameter set to $1$. The network was implemented as a $2$ layer binary classifier. For \memstream, the encoder and decoder were implemented as single layer Neural Nets with ReLU activation. We used Adam Optimizer to train both these networks with $\beta_1 = 0.9$ and $\beta_2 = 0.999$. Grid Search was used for hyperparameter tuning:  Learning Rate was set to $1\mathrm{e}-2$, and the number of epochs was set to $5000$. The memory size $N$, and the value of the threshold $\beta$, can be found in Table \ref{tab:params} in the Appendix. Memory size for each intrusion detection dataset was searched in $\{256, 512, 1024, 2048\}$. For multi-dimensional point datasets, if the size of the dataset was less than $2000$, $N$ was searched in $\{4, 8, 16, 32, 64\}$, and if it was greater than $2000$, then $N$ was searched in $\{128, 256, 512, 1024, 2048\}$. The threshold $\beta$, is an important parameter in our algorithm, and hence we adopt a finer search strategy. For each dataset, and method, $\beta$ was searched in $\{10, 1, 0.1, 0.001, 0.0001\}$. Unless stated otherwise, AE was used for feature extraction with output dimension $D=2d$, and with a FIFO memory update policy. The KNN coefficient $\gamma$  was set to $0$ for all experiments. For the synthetic dataset, we use a memory size of $N=16$. For all methods, across all datasets, the number of training samples used is equal to the memory size.

\paragraph{Datasets:}

Table \ref{tab:datasets} contains the datasets that we use for evaluation. We briefly describe how these datasets are prepared for anomaly detection. Table \ref{tab:params} shows the memory size $N$, and the value of the threshold $\beta$.

\begin{table*}[!htb]
\centering
\caption{Statistics of the datasets.}
\resizebox{\linewidth}{!}{
\begin{tabular}{@{}lccccccccccccc@{}}
\toprule
 
 & KDD99
 & NSL
 & UNSW
 & DoS
 & Syn.
 & Ion.
 & Cardio
 & Sat.
 & Sat.-2
 & Mamm.
 & Pima
 & Cover \\ \midrule

\textbf{Records} & $494,021$ & $125,973$ & $2,540,044$ & $1,048,575$ & $10,000$ & $351$ & $1831$ & $6435$ & $5803$ & $11183$ & $768$ & $286048$ \\

\textbf{Dimensions} & $121$ & $126$ & $122$ & $95$ & $1$ & $33$ & $21$ & $36$ & $36$ & $6$ & $8$ & $10$ \\

\bottomrule
\label{tab:datasets}
\end{tabular}
}
\end{table*}

\begin{table*}[!htb]
\centering
\caption{Memory Length and Update Threshold used for the different datasets}
\resizebox{\linewidth}{!}{
\begin{tabular}{@{}lccccccccccccc@{}}
\toprule
 \textbf{Method}
 & KDD99
 & NSL
 & UNSW
 & DoS
 & Syn.
 & Ion.
 & Cardio
 & Sat.
 & Sat.-2
 & Mamm.
 & Pima
 & Cover \\ \midrule

$N$ & $256$ & $2048$ & $2048$ & $2048$ & $16$ & $4$ & $64$ & $32$ & $256$ & $128$ & $64$ & $2048$\\

$\beta$ & $1$ & $0.1$ & $0.1$ & $0.1$ & $1$ & $0.001$ & $1$ & $0.01$ & $10$ & $0.1$ & $0.001$ & $0.0001$  \\

\bottomrule
\label{tab:params}
\end{tabular}
}
\end{table*}

\begin{enumerate}
    \item \emph{KDDCUP99} \cite{KDDCup192:online} is based on the DARPA data set and is amongst the most extensively used data sets for multi-aspect anomaly detection. The original dataset contains samples of $41$ dimensions, $34$ of which are continuous and $7$ are categorical, and also displays concept drift \cite{minku2011ddd}. We use one-hot representation to encode the categorical features, and eventually, we obtain a dataset of $121$ dimensions. For the \emph{KDDCUP99} dataset, we follow the settings in \cite{zong2018deep}. As $20\%$ of data samples are labeled as ``normal" and $80\%$ are labeled as ``attack", normal samples are in a minority group; therefore, we treat normal ones as anomalous in this experiment, and the $80\%$ samples labeled as attack in the original dataset are treated as normal samples.

    \item \emph{NSL-KDD} \cite{tavallaee2009detailed} solves some of the inherent problems of the \emph{KDDCUP99} dataset such as redundant and duplicate records and is considered more enhanced as compared to \emph{KDDCUP99}.

    \item \emph{CICIDS-DoS} \cite{sharafaldin2018toward} was created by the Canadian Institute of Cybersecurity. Each record is a flow containing features such as source IP address, source port, destination iP address, bytes, and packets. These flows were captured from a real-time simulation of normal network traffic and synthetic attack simulators. This consists of the \emph{CICIDS-DoS} dataset ($1.05$ million records). \emph{CICIDS-DoS} has $5\%$ anomalies and contains samples of $95$ dimensions with a mixture of numeric and categorical features. For categorical features, we further used binary encoding to represent them because of the high cardinality. \cite{ring2019survey} surveys more than 30 intrusion detection datasets and recommends to use the newer CICIDS \cite{sharafaldin2018toward} and UNSW-NB15 \cite{moustafa2015unsw} datasets.

    \item \emph{UNSW-NB15} \cite{moustafa2015unsw} was created by the Cyber Range Lab of the Australian Centre for Cyber Security (ACCS) for generating a hybrid of real modern normal activities and synthetic contemporary attack behaviors. This dataset has nine types of attacks, namely, Fuzzers, Analysis, Backdoors, DoS, Exploits, Generic, Reconnaissance, Shellcode, and Worms. It has $13\%$ anomalies.

    \item Ionosphere \cite{oddsdatasets} is derived using the ionosphere dataset from the UCI ML repository \cite{ucidatasets} which is a binary classification dataset with dimensionality $34$. There is one attribute having values of all zeros, which is discarded. So the total number of dimensions is $33$. The `bad' class is considered as outliers class and the `good' class as inliers.
    
    \item Cardio \cite{oddsdatasets} is derived using the Cardiotocography (Cardio) dataset from the UCI ML repository \cite{ucidatasets} which consists of measurements of fetal heart rate (FHR) and uterine contraction (UC) features on cardiotocograms classified by expert obstetricians. This is a classification dataset, where the classes are normal, suspect, and pathologic. For outlier detection, the normal class formed the inliers, while the pathologic (outlier) class is downsampled to $176$ points. The suspect class is discarded.
    
    \item Satellite \cite{oddsdatasets} is derived using the Statlog (Landsat Satellite) dataset from the UCI ML repository \cite{ucidatasets} which is a multi-class classification dataset. Here, the training and test data are combined. The smallest three classes, i.e. $2, 4, 5$ are combined to form the outliers class, while all the other classes are combined to form an inlier class.
    
    \item Satimage-2 \cite{oddsdatasets} is derived using the Statlog (Landsat Satellite) dataset from the UCI ML repository \cite{ucidatasets} which is also a multi-class classification dataset. Here, the training and test data are combined. Class $2$ is down-sampled to $71$ outliers, while all the other classes are combined to form an inlier class. The modified dataset is referred to as Satimage-$2$.
    
    \item Mammography \cite{oddsdatasets} is derived from openML\footnote{https://www.openml.org/}. The publicly available openML dataset has $11,183$ samples with $260$ calcifications. If we look at predictive accuracy as a measure of goodness of the classifier for this case, the default accuracy would be $97.68\%$ when every sample is labeled non-calcification. But, it is desirable for the classifier to predict most of the calcifications correctly. For outlier detection, the minority class of calcification is considered as the outlier class and the non-calcification class as inliers.
    
    \item Pima \cite{oddsdatasets} is the same as Pima Indians diabetes dataset of the UCI ML repository \cite{ucidatasets} which is a binary classification dataset. Several constraints were placed on the selection of instances from a larger database. In particular, all patients here are females at least $21$ years old of Pima Indian heritage.

    \item ForestCover \cite{oddsdatasets} is the ForestCover/Covertype dataset from the UCI ML repository \cite{ucidatasets} which is a multiclass classification dataset. It is used in predicting forest cover type from cartographic variables only. This dataset has $54$ attributes ($10$ quantitative variables, $4$ binary wilderness areas, and $40$ binary soil type variables). Here, an outlier detection dataset is created using only $10$ quantitative attributes. Instances from class $2$ are considered as normal points and instances from class $4$ are anomalies. The anomalies ratio is $0.9\%$. Instances from the other classes are omitted.

\end{enumerate}

Apart from these standard datasets, we also create and use a synthetic dataset (that we plan to release publicly), \emph{Syn} with $10\%$ anomalies and $T=10000$ samples. This dataset is constructed as a superposition of a linear wave with slope $2 \times 10^{-3}$, two sinusoidal waves with time periods $0.2T$ and $0.3T$ and amplitudes $8$ and $4$, altogether with an additive Gaussian noise from a standard normal distribution. $10\%$ of the samples are chosen at random and are perturbed with uniform random noise from the interval $[3, 6]$ to simulate anomalous data. Figure \ref{fig:syndata} shows a scatterplot of the synthetic data. Anomalous samples constitute $10\%$ of the data and are represented by red dots in the scatter plot.

    \begin{figure}[!htb]
    \begin{centering}
  \includegraphics[width=0.7\columnwidth]{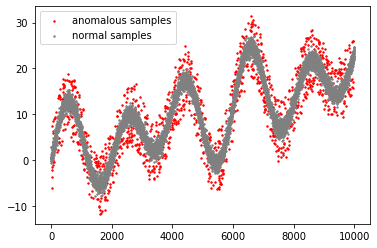}
  \caption{Scatterplot of the Synthetic Dataset.}
  \label{fig:syndata}
 \end{centering}
    \end{figure}
    
     By construction, the synthetic data distribution changes significantly over time. The presence of this concept drift makes the task challenging resulting in poor performance by baseline approaches, as seen in the Experiments. However, \memstream, through the use of explicit memory, can adapt to the drift in the distribution, proving its effectiveness in concept drift settings.

\paragraph{Baseline Parameters}

STORM: window\_size=$10000$, max\_radius=$0.1$\\
HS-Tree: window\_size=$100$, num\_trees=$25$, max\_depth=$15$, initial\_window\_X=None\\
iForestASD:	window\_size=$100$, n\_estimators=$25$, anomaly\_threshold=$0.5$, drift\_threshold=$0.5$\\
RS-Hash: sampling\_points=$1000$, decay=$0.015$, num\_components=$100$, num\_hash\_fns=$1$\\
RCF: num\_trees=$4$, shingle\_size=$4$, tree\_size=$256$\\
LODA: num\_bins=$10$, num\_random\_cuts=$100$\\
Kitsune: max\_size\_ae=$10$, learning\_rate=$0.1$, hidden\_ratio=$0.75$, grace\_feature\_mapping=grace\_anomaly\_detector=$10\%$ of data\\
DILOF: window size = $400$, thresholds = [0.1f, 1.0f, 1.1f, 1.15f, 1.2f, 1.3f, 1.4f, 1.6f, 2.0f, 3.0f] , K = $8$\\
\textsc{xStream}: projection size=$50$, number of chains=$50$, depth=$10$, rowstream=$0$, nwindows=$0$, initial sample size=\# rows in data, scoring batch size=$100000$\\
\textsc{MStream}: alpha = $0.85$\\
Ex. IF: ntrees=$200$, sample\_size=$256$, limit=None, ExtensionLevel=$1$

\subsection{Comparison to Streaming Methods}

Table \ref{tab:auc} shows the AUC of \memstream\ and state-of-the-art streaming baselines. We use open-sourced implementations of DILOF \cite{Na2018DILOFEA}, \textsc{xStream} \cite{Manzoor2018xStreamOD}, \textsc{MStream} \cite{Bhatia2021MSTREAM}, Extended Isolation Forest (Ex. IF) \cite{Hariri2021ExtendedIF}, provided by the authors, following parameter settings as suggested in the original papers. For STORM \cite{Angiulli2007DetectingDO}, HS-Tree \cite{Tan2011FastAD}, iForestASD \cite{Ding2013AnAD}, RS-Hash \cite{Sathe2016SubspaceOD}, Random Cut Forest (RCF) \cite{guha2016robust}, LODA \cite{Pevn2015LodaLO}, Kitsune \cite{Mirsky2018KitsuneAE}, we use the open-source library PySAD \cite{pysad} implementation, following original parameters. LODA could not process the large \emph{UNSW} dataset. Ex. IF and Kitsune are unable to run on datasets with just one field, therefore their results with \emph{Syn} are not reported.

\begin{table*}[!htb]
\centering
\caption{AUC of \memstream\ and Streaming Baselines. Averaged over $5$ runs.}
\resizebox{\linewidth}{!}{
\begin{tabular}{@{}lccccccccccccc@{}}
\toprule
 \textbf{Method}
 & KDD99
 & NSL
 & UNSW
 & DoS
 & Syn.
 & Ion.
 & Cardio
 & Sat.
 & Sat.-2
 & Mamm.
 & Pima
 & Cover \\ \midrule

STORM (CIKM'07) & $0.914$ & $0.504$ & $0.810$ & $0.511$ & $0.910$ & $0.637$ & $0.507$ & $0.662$ & $0.514$ & $0.650$ & $0.528$ & $0.778$\\

{HS-Tree (IJCAI'11)} & $0.912$ & $0.845$ & $0.769$ & $0.707$ & $0.800$ & $0.764$ & $0.673$ & $0.519$ & $0.929$ & $0.832$ & $0.667$ & $0.731$ \\

{iForestASD (ICONS'13)} & $0.575$ & $0.500$ & $0.557$ & $0.529$ & $0.501$ & $0.694$ & $0.515$ & $0.504$ & $0.554$ & $0.574$ & $0.525$ & $0.603$ \\
 
{RS-Hash (ICDM'16)} & $0.859$ & $0.701$ & $0.778$ & $0.527$ & $0.921$ & $0.772$ & $0.532$ & $0.675$ & $0.685$ & $0.773$ & $0.562$ & $0.640$ \\
  
{RCF (ICML'16)} & $0.791$ & $0.745$ & $0.512$ & $0.514$ & $0.774$ & $0.675$ & $0.617$ & $0.552$ & $0.738$ & $0.755$ & $0.571$ & $0.586$ \\

{LODA (ML'16)} & $0.500$ & $0.500$ & $---$ & $0.500$ & $0.506$ & $0.503$ & $0.501$ & $0.500$ & $0.500$ & $0.500$ & $0.502$ & $0.500$\\
 
Kitsune (NDSS'18) & $0.525$ & $0.659$ & $0.794$ & $0.907$ & $---$ & $0.514$ & $0.966$ & $0.665$ & $0.973$ & $0.592$ & $0.511$ & $0.888$\\
 
{DILOF (KDD'18)} & $0.535$ & $0.821$ & $0.737$ & $0.613$ & $0.703$ & $\mathbf {0.928}$ & $0.570$ & $0.561$ & $0.563$ & $0.733$ & $0.543$ & $0.688$ \\

{\textsc{xStream} (KDD'18)} & $0.957$ & $0.552$ & $0.804$ & $0.800$ & $0.539$ & $0.847$ & $0.918$ & $0.677$ & $\mathbf {0.996}$ & $0.856$ & $0.663$ & $0.894$ \\
 
{\textsc{MStream} (WWW'21)} & $0.844$ & $0.544$ & $0.860$ & $0.930$ & $0.505$ & $0.670$ & $\mathbf {0.986}$ & $0.563$ & $0.958$ & $0.567$ & $0.529$ & $0.874$ \\
 
{Ex. IF (TKDE'21)} & $0.874$ & $0.767$ & $0.541$ & $0.734$ & $---$ & $0.872$ & $0.921$ & $0.716$ & $0.995$ & $0.867$ & $0.672$ & $0.902$ \\

{\textbf{\memstream}} & $\mathbf {0.980}$ & $\mathbf {0.978}$ & $\mathbf {0.972}$ & $\mathbf {0.938}$ & $\mathbf {0.955}$ & $0.821$ & $0.884$ & $\mathbf {0.727}$ & $0.991$ & $\mathbf {0.894}$ & $\mathbf {0.742}$ & $\mathbf {0.952}$ \\
 
\bottomrule
\label{tab:auc}
\end{tabular}
}
\end{table*}

Random subspace generation in RS-Hash includes many irrelevant features into subspaces while omitting relevant features in high-dimensional data. The objective of random projection in LODA retains the pairwise distances of the original space, therefore it fails to provide accurate outlier estimation. \textsc{xStream} performs well in \emph{KDD99}, \textsc{MStream} performs well in \emph{DoS}, however, note that \memstream\ achieves statistically significant improvements in AUC scores over baseline methods. Moreover, baselines are unable to catch complicated drift scenarios in \emph{NSL}, \emph{UNSW} and \emph{Syn}.

Table \ref{tab:time} reports the running AUC-PR scores of \memstream\ and baseline methods on the \emph{NSL-KDD} dataset, as well as their corresponding running times. Note that not only does \memstream\ greatly outperform baselines on AUC-PR, but also does so in a time-efficient manner.

\begin{table}[!htb]
\centering
\caption{AUC-PR and Time required to run \memstream\ and Streaming Baselines on \emph{NSL-KDD}. \memstream\ provides statistically significant (p-value $< 0.001$) improvements over baseline methods.}
\begin{tabular}{@{}lcc@{}}
\toprule
 \textbf{Method}
 & \textbf{AUC-PR}
 & \textbf{Time (s)}
 \\ \midrule
STORM & $0.681 \pm 0.000$ & $754$\\
{HS-Tree} & $0.709 \pm 0.063$ & $306$\\
{iForestASD} & $0.534 \pm 0.000$ & $19876$\\
{RS-Hash} & $0.500 \pm 0.140$  & $892$\\
{RCF} & $0.664 \pm 0.006$ & $665$\\
{LODA} & $0.734 \pm 0.067$ & $2617$\\
Kitsune & $0.673 \pm 0.000$ & $821$\\
{DILOF} & $0.822 \pm 0.000$ & $260$\\
{\textsc{xStream}} & $0.541 \pm 0.070$ & $34$\\
{\textsc{MStream}} & $0.510 \pm 0.000$ & $0.08$\\
{Ex. IF} & $0.659 \pm 0.014$ & $889$\\
\textbf{\memstream} & $\mathbf{0.959} \pm 0.002$ & $55$\\
\bottomrule
\label{tab:time}
\end{tabular}
\end{table}

\FloatBarrier

\subsection{Concept Drift}
\label{sec:drift}

We next investigate \memstream's performance under concept drift, particularly how fast it can adapt. As shown in Figure \ref{fig:drift} (top), we create a synthetic data set which covers a wide variety of drifts scenarios: (a) point anomalies: $T=19000$ (b) sudden frequency change: $T\in[5000, 10000]$ (c) continuous concept drift: $T\in [15000, 17500]$ (d) sudden concept drift due to mean change: $T\in [12500, 15000]$. Anomaly scores are clipped at $T=12500$ and $T=19000$ for better visibility.

\begin{figure}[!htb]
\centering
  \includegraphics[width=0.8\columnwidth]{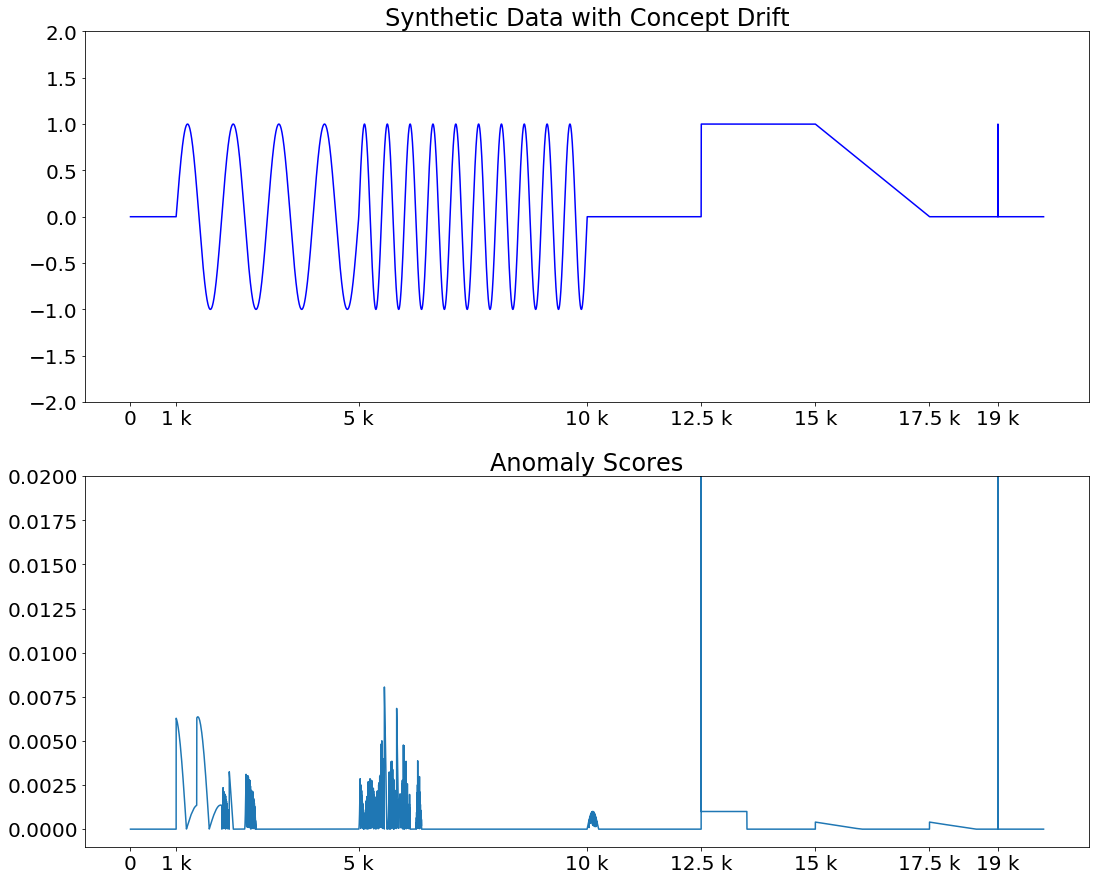}
  \captionof{figure}{(Top): Synthetic data with drift. (Bottom): Anomaly Scores output by \memstream\ demonstrating resilience to drift.}
  \label{fig:drift}
\end{figure}

\memstream\ is able to handle all the above-mentioned concept drift scenarios as is evident in Figure \ref{fig:drift} (bottom). We observe that \memstream\ assigns high scores corresponding to trend-changing events (e.g. $T=1000, 5000, 10000$, etc.) which produce anomalies, then with a gradual decrease in scores thereafter as it \emph{adapts} successfully to the new distribution. Note that \memstream\ can also adapt to periodic streams. For the first cycle of the sine wave $T\in[1000, 2000]$, the anomalous scores are relatively high. However, as more and more normal samples are seen from the sine distribution, \memstream\ adapts to it.

\FloatBarrier

\subsection{Retraining}
\label{sec:retrain}
 The need for re-training is especially prevalent in very long drifting streams where the feature extractor, trained on the small subset of the initial normal data $\mathcal{D}$, starts facing record data sufficiently different from its training data. In this experiment, we test the ability of \memstream\ to accommodate this more challenging setting by periodically retraining its feature extractor. Fine-tuning is performed at regular intervals distributed uniformly across the stream, i.e. to implement $k$ fine-tunings on a stream of size $S$, the first fine-tuning occurs at $\left \lfloor \frac{S}{k+1} \right \rfloor$. Figure \ref{fig:retrain} shows the AUC and time taken to fine-tune \memstream\ on \emph{CICIDS-DoS} with a stream size greater than $1M$ records. Note that as we increase the number of times \memstream\ is fine-tuned, we observe large gains in AUC with a negligible time difference.

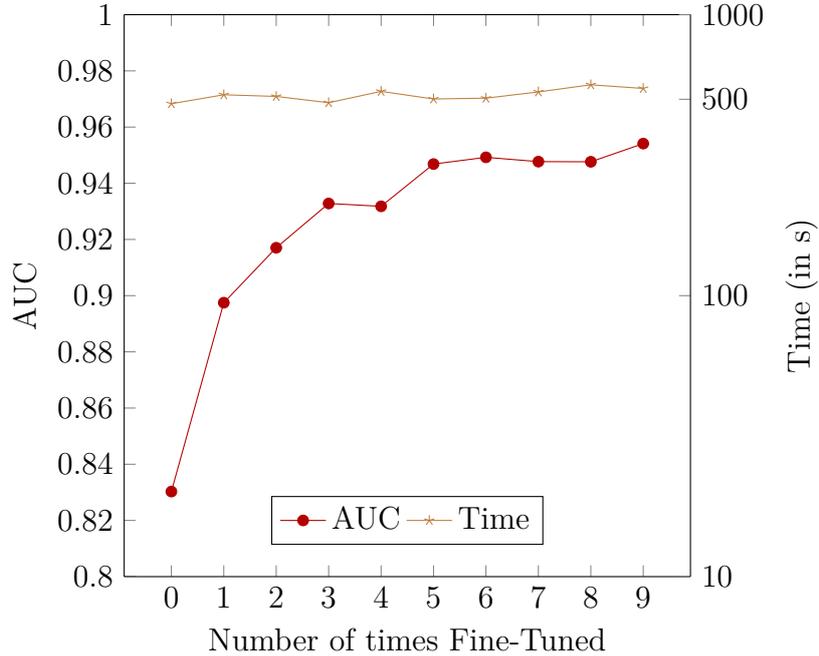
\begin{figure}[!htb]
\centering
\begin{tikzpicture}
\pgfplotsset{
      scale only axis,
      width=0.5\textwidth,height=0.5\textwidth
  }

  \begin{axis}[
    axis y line*=left,
    xlabel=Number of times Fine-Tuned,
    xtick = {0, 1, 2, 3, 4, 5, 6, 7, 8, 9},
    ymin=0.8,
    ymax=1.0,
    tick label style={font=\normalsize},
    ylabel = {AUC},
  ]
    \addplot[red!70!black,mark=*] coordinates {(0, 0.8302751955)
(1, 0.8974884823)
(2, 0.9170457649)
(3, 0.9328035165)
(4, 0.9317918228)
(5, 0.9468332062)
(6, 0.9492323078)
(7, 0.9477016701)
(8, 0.947640945)
(9, 0.9541160801)
    };\label{auc}
    \end{axis}

    \begin{axis}[
      axis y line*=right,
      axis x line=none,
      xtick = {0, 1, 2, 3, 4, 5, 6, 7, 8, 9},
      ymode = log,
      log ticks with fixed point,
        ytick = {10, 100, 500, 1000},
        ymin=10,
        ymax=1000,
        yticklabels = {$10$, $100$, $500$, $1000$},
        tick label style={font=\normalsize},
        ylabel = {Time (in s)},
        legend style={at={(0.5,0.1)},legend columns=2,fill=none,draw=black,anchor=center,align=center},
    ]

    \addlegendimage{/pgfplots/refstyle=auc}\addlegendentry{AUC}
    \addplot[brown,mark=star] coordinates {(0, 481.9602451)
(1, 518.6869633)
(2, 511.7169969)
(3, 486.0186646)
(4, 533.2838132)
(5, 501.0282154)
(6, 504.4203665)
(7, 530.9628208)
(8, 562.8917241)
(9, 546.230355)
    };     
    \addlegendentry{Time}; 
  \end{axis}

\end{tikzpicture}
\caption{Retraining effect on the AUC and time for \emph{CICIDS-DOS}.}
\label{fig:retrain}
\end{figure}

\FloatBarrier

\subsection{Self-Correction and Recovery}
\label{sec:recovery}

Consider the scenario where an anomalous element enters the memory. A particularly catastrophic outcome of this event could be the cascading effect where more and more anomalous samples replace the normal elements in the memory due to their similarity. This can ultimately lead to a situation where the memory solely consists of anomalous samples. These ``Group Anomaly" events are fairly common in intrusion detection settings. We show that this issue is mitigated by the use of $K$-nearest neighbors in our approach. We simulate the above setting by adding the first labeled anomalous element in memory during the initialization.

In Table \ref{tab:recovery}, a high $\beta$ allows anomalous elements to also enter the memory. In the absence of $K$-nearest neighbor discounting (i.e. $\gamma=0$), a high $\beta$ value algorithm succumbs to the above-described scenario resulting in poor performance. On the other hand, with discounting (i.e. $\gamma \neq 0$), the algorithm is able to ``recover" itself, and as a result, the performance does not suffer considerably. Note that when the threshold $\beta$ is in its appropriate range, the algorithm is robust to the choice of discount factor $\gamma$.

\begin{table}[!htb]
\centering
\caption{Performance of \memstream\ on \emph{NSL-KDD} dataset after adding an anomalous element in memory when $K=3$ and for different values of discount factor $\gamma$.}
\label{tab:recovery}
\begin{tabular}{lrr}
    \toprule
    \textbf{$\gamma$} & High $\beta (=1)$ & Appropriate $\beta (=0.001)$ \\
    \midrule
    $0$ & $0.771$ & $0.933$\\ 
    $0.25$ & $0.828$ & $0.966$\\ 
    $0.5$ & $0.848$ & $0.967$\\ 
    $1$ & $0.888$ & $0.965$\\
    \bottomrule
  \end{tabular}

\end{table}

\FloatBarrier

\subsection{Ablations}

\begin{table}[!htb]
    \centering
    \caption{Ablation study for different components of \memstream\ on \emph{KDDCUP99}.}
    \label{tab:ablations}
    \begin{tabular}{ll|rrrr}
    \toprule
         & \textbf{Component} & \multicolumn{4}{c}{\bfseries Ablations}\\
         \midrule
         (a) & Memory & None & LRU & RR & FIFO\\
         & Update & $0.938$& $0.946$& $0.946$& $0.980$\\
         \midrule
         (b) & Feature & Identity &PCA & IB & AE\\
         & Extraction & $0.822$& $0.863$ & $0.959$& $0.980$\\
         \midrule
         (c) & Memory & 128 & 256 & 512 & 1024\\
         & Length ($N$) & $0.950$& $0.980$& $0.946$& $0.811$\\
         \midrule
         (d) & Output & $d/2$ & $d$ & $2d$ & $5d$\\
         & Dimension ($D$) & $0.951$& $0.928$& $0.980$& $0.983$\\
         \midrule
         (e) & Update & 1 & 0.1 & 0.01 & 0.001\\
         & Threshold ($\beta$) & $0.980$& $0.938$& $0.938$& $0.938$ \\
         \midrule
         (f) & KNN & 0 & 0.25 & 0.5 & 1\\
         & coefficient ($\gamma$) & $0.980$& $0.939$& $0.937$ & $0.936$ \\
    \bottomrule
    \end{tabular}
\end{table}

\textbf{(a) Memory Update:}
\label{exp:memory}
Taking inspiration from the work done in cache replacement policies in computer architecture, we replace the FIFO memory update policy with Least Recently Used (LRU) and Random Replacement (RR) policies. Table \ref{tab:ablations}(a) reports results with these three and when no memory update is performed on the \emph{KDDCUP99} dataset. Note that FIFO outperforms other policies. This is due to the temporal locality preserving property of the FIFO policy to keep track of the current trend. LRU and RR policies do not maintain a true snapshot of the stream in the memory and are thus unable to learn the changing trend.

\textbf{(b) Feature Extraction:}
\label{exp:representation}
Table \ref{tab:ablations}(b) shows experiments with different methods for feature extraction discussed in Section \ref{sec:representation}. Autoencoder outperforms both PCA and Information Bottleneck approaches.

\textbf{(c) Memory Length ($N$):}
\label{exp:memlen} 
As we noted in Section \ref{sec:theorymemsize}, increasing $N$ can decrease the false positive rate, but also increase the false negative rate. We observe this effect empirically in Table \ref{tab:ablations}(c), where the sweet spot is found at $N=256$, and increasing memory length further degrades performance.  An additional experiment demonstrating the effect of memory size is discussed in Table \ref{tab:memoryeffect}. We note that very large or very small values of N would hinder the algorithm performance as the memory will not be able to capture the current trend properly. A very large `N' will not ensure that the current trend is learned exclusively and the memory would always be contaminated by representatives of the previous trend. On the other hand, a very small `N' will not allow enough representatives from the current trend and thus in both cases, the performance of the algorithm will be sub-optimal.

\begin{table*}[!htb]
\centering
\caption{Effect of Memory Size on the AUC in \memstream\ on \emph{NSL-KDD} dataset.}
\label{tab:memoryeffect}
\resizebox{\linewidth}{!}{
\begin{tabular}{@{}rcccccccccccc@{}}
\toprule
\textbf{Memory Size} & $2^4$ & $2^5$ & $2^6$ & $2^7$ & $2^8$ & $2^9$ & $2^{10}$ & $2^{11}$ & $2^{12}$ & $2^{13}$ & $2^{14}$ \\

\midrule
\textbf{AUC} & $0.670$ & $0.649$ & $0.932$ & $0.936$ & $0.923$ & $0.950$ & $0.972$ & $0.976$ & $0.985$ & $0.989$ & $0.991$ \\
\bottomrule
\end{tabular}
}
\end{table*}

\textbf{(d) Output Dimension ($D$):}
\label{exp:output}
In Section \ref{sec:theoryarch}, we motivate why we use an architecture with $D >= d$. In Table \ref{tab:ablations}(d), we compare architectures with different output dimension $D$ as a function of the input dimension $d$. We find that $D=d/2$ outperforms an architecture with $D=d$, owing to the features learning by dimensionality reduction. Note that \memstream\ performs well for large $D$.

\textbf{(e) Update Threshold ($\beta$):}
\label{exp:threshold}
The update threshold is used to judge records based on their anomaly scores and determine whether they should update the memory. A high $\beta$ corresponds to frequent updates to the memory, whereas a low $\beta$ seldom allows memory updates. Thus, $\beta$ can capture our belief about how frequently the memory should be updated, or how close is the stream to the initial data distribution. From Table \ref{tab:ablations}(e), we notice that for \emph{KDDCUP99}, a drifting dataset, a more flexible threshold ($\beta=1$) performs well, and more stringent thresholds perform similar to no memory updates (Table \ref{tab:ablations}(a)).

\textbf{(f) KNN coefficient ($\gamma$):}
\label{exp:knncoeff}
In Section \ref{sec:recovery}, we discussed the importance of the KNN coefficient $\gamma$ in the Self-Recovery Mechanism. Table \ref{tab:ablations}(f) compares different settings of $\gamma$, without memory poisoning.

\FloatBarrier

\section{Conclusion}
We propose \memstream, a novel memory augmented feature extractor framework for streaming anomaly detection in multi-dimensional data and concept drift settings. \memstream\ uses a denoising autoencoder to extract features and a memory module with a FIFO replacement policy to learn the dynamically changing trends. Moreover, \memstream\ allows quick retraining when the arriving stream becomes sufficiently different from the training data. We give a theoretical guarantee on the relation between the memory size and the concept drift. Furthermore, \memstream\ prevents memory poisoning by using (1) a discounting $K$-nearest neighbor memory leading to a unique self-correcting and recovering mechanism; (2) a theoretically motivated architecture design choice. \memstream\ outperforms $11$ state-of-the-art streaming methods. Future work could consider more tailored memory replacement policies, e.g. by assigning different weights to the memory elements.
\part{ Conclusion and Future Work}
\chapter{Conclusion and Future Work}
\label{ch:concl}

\section{Summary and Overarching Themes}
This dissertation was organized into six chapters. Chapter \ref{ch:intro} motivated the need for real-time anomaly detection and summarized the contributions. Chapter \ref{ch:related} categorizes and discusses the related work in graph and multi-aspect data settings.

Chapter \ref{ch:midas} introduced MIDAS which used a count-min sketch data structure to detect \emph{microcluster anomalies}, or suddenly arriving groups of suspiciously similar edges, in edge streams, using constant time and memory. In addition, by using a principled hypothesis testing framework, \textsc{Midas} provided theoretical bounds on the false positive probability, which previous methods do not provide. We also proposed two variants, \midas-R which incorporated temporal and spatial relations, and \midas-F which filtered away anomalous edges to prevent them from negatively affecting the algorithm's internal data structures.

In Chapter \ref{ch:anograph}, we extended the count-min sketch to a higher-order sketch data structure to capture complex relations in graph data. This higher-order sketch has the useful property of preserving the dense subgraph structure (dense subgraphs in the input turn into dense submatrices in the data structure). We then proposed four online algorithms that utilize this enhanced data structure to detect both edge and graph anomalies in constant memory and constant update time. Furthermore, our approach was the first streaming work that incorporates dense subgraph search to detect graph anomalies in constant memory and constant update time per newly arriving edge. We also provided theoretical guarantees on the higher-order sketch estimate and the submatrix density measure.

We then broadened the graph setting to a multi-aspect data stream in Chapter \ref{ch:mstream} and proposed \mstream\ to detect anomalous records in multi-aspect data streams including both categorical and numeric attributes. \mstream\ is online, thus processing each record in constant time and constant memory. We further proposed \mstream-PCA, \mstream-IB, and \mstream-AE to incorporate correlation between features and demonstrated how the anomalies detected by \mstream\ are explainable.

Finally, in Chapter \ref{ch:memstream}, we considered multi-aspect data streams with concept drift and proposed \memstream\ to detect anomalous records. \memstream\ leveraged the power of a denoising autoencoder to learn representations and a memory module to learn the dynamically changing trend in data without the need for labels. We proved a theoretical bound on the size of memory for effective drift handling. In addition, we allow quick retraining when the arriving stream becomes sufficiently different from the training data. Furthermore, \memstream\ made use of two architecture design choices to be robust to memory poisoning.

In Appendix \ref{ch:exgan}, we propose ExGAN for adversarial generation of extreme/anomalous data. Appendix \ref{ch:sess} incorporates semi-supervision in streaming anomaly detection.

Throughout this dissertation, we have described a number of different methods, designed to detect anomalies in a specific setting. How can we distill these into a coherent framework? The anomaly detection approaches can be categorized based on both the data setting, as well as the type of anomaly we wish to detect, as follows.

\begin{enumerate}[label=\textbf{Q\arabic*.}]
\item {\bfseries Graphs}
\begin{enumerate}
    \item How can we detect anomalous edges in dynamic graphs using constant time and memory? {\textsc{MIDAS/AnoEdge}}
    \item How can we detect anomalous subgraphs in dynamic graphs using constant time and memory? {\textsc{AnoGraph}}
\end{enumerate}

\item {\bfseries Multi-Aspect Data:}
\begin{enumerate}
    \item How can we detect anomalous behavior in multi-aspect data streams, including group anomalies involving the sudden appearance of large groups of suspicious activity, in an unsupervised manner? {\textsc{MStream}}
    \item How can we detect anomalous activities in multi-aspect data streams where concept drift is present? {\textsc{MemStream}}
\end{enumerate}
\end{enumerate}

\section{Future Work}
We list a few potential directions for future work in the area of streaming anomaly detection.

\begin{itemize}
\item Advanced data structures and algorithms:
Future work can extend our symmetrical higher-order sketch to a rectangular matrix and try more complex combinations (e.g. weighted sums) of anomaly scores for individual attributes. Moreover, one can consider more tailored memory replacement policies as well, e.g. by assigning different weights to the memory elements. Graph Neural Networks are an effective way of learning from complex input data and an exciting future direction is to incorporate embedding-based approaches and node and edge representations in streaming anomaly detection. A heterogeneous graph setting consisting of different types of entities also provides a greater challenge.

\item Faster data streams: Analysing the data stream rate is an important aspect of the design and performance of streaming anomaly detection systems. In general, faster data streams will require more time and memory to process, as the anomaly detection system will need to analyze the data more quickly and will need to store more data in memory. This can be a challenge, as the amount of time and memory available to the system may be limited, and the system may need to be redesigned to handle a wide range of data stream rates, for example using parallel computing.

\item Exploring new applications: Streaming anomaly detection is currently used in a variety of applications, such as network security and fraud detection. However, there is potential to expand the use of these techniques to a wider range of applications, such as predictive maintenance, where it could be used to identify potential issues with equipment or systems before they fail. This could involve analyzing data streams from sensors and other monitoring systems to identify anomalies that could indicate potential problems, and using this information to schedule maintenance or other interventions to prevent failure. Other applications include environmental monitoring, social media data streams, medical data, and other types of complex and dynamic data streams.

\item More powerful models: We plan to investigate a hybrid approach of deep learning models and streaming data structures that combines the strength of both, by using deep learning models to extract rich and detailed representations of the data, and then combining them with streaming data structures to process and analyze these representations in real-time.

During the course of the dissertation, we moved from a graph to a multi-aspect data setting in trying to combine multiple sources and richer inputs to detect the anomalies more accurately. Continuing in this direction, we want to gradually build more powerful models that can capture complex types of input data, for example, it will be interesting to analyze accompanying textual data using recent innovations in natural language processing models, and gradually expand to more multi-modal approaches. 

\end{itemize}

\begin{appendices}
\part*{Appendix}
\SetPicSubDir{ExGAN}
\SetExpSubDir{ExGAN}

\chapter[ExGAN: Adversarial Generation of Extreme Samples][ExGAN]{ExGAN: Adversarial Generation of Extreme Samples}
\label{ch:exgan}
\vspace{2em}

\begin{mdframed}[backgroundcolor=magenta!20] 
Chapter based on work that appeared at AAAI'21 \cite{bhatia2021exgan} \href{https://arxiv.org/pdf/2009.08454.pdf}{[PDF]}.
\end{mdframed}

\section{Introduction}
Modelling extreme events in order to evaluate and mitigate their risk is a fundamental goal with a wide range of applications, such as extreme weather events, financial crashes, and managing unexpectedly high demand for online services. A vital part of mitigating this risk is to be able to understand or generate a wide range of extreme scenarios. For example, in many applications, stress-testing is an important tool, which typically requires testing a system on a wide range of extreme but realistic scenarios, to ensure that the system can successfully cope with such scenarios. This leads to the question: how can we generate a wide range of extreme but realistic scenarios, for the purpose of understanding or mitigating their risk?

\begin{figure*}[!t]
\centering
\includegraphics[width=\textwidth]{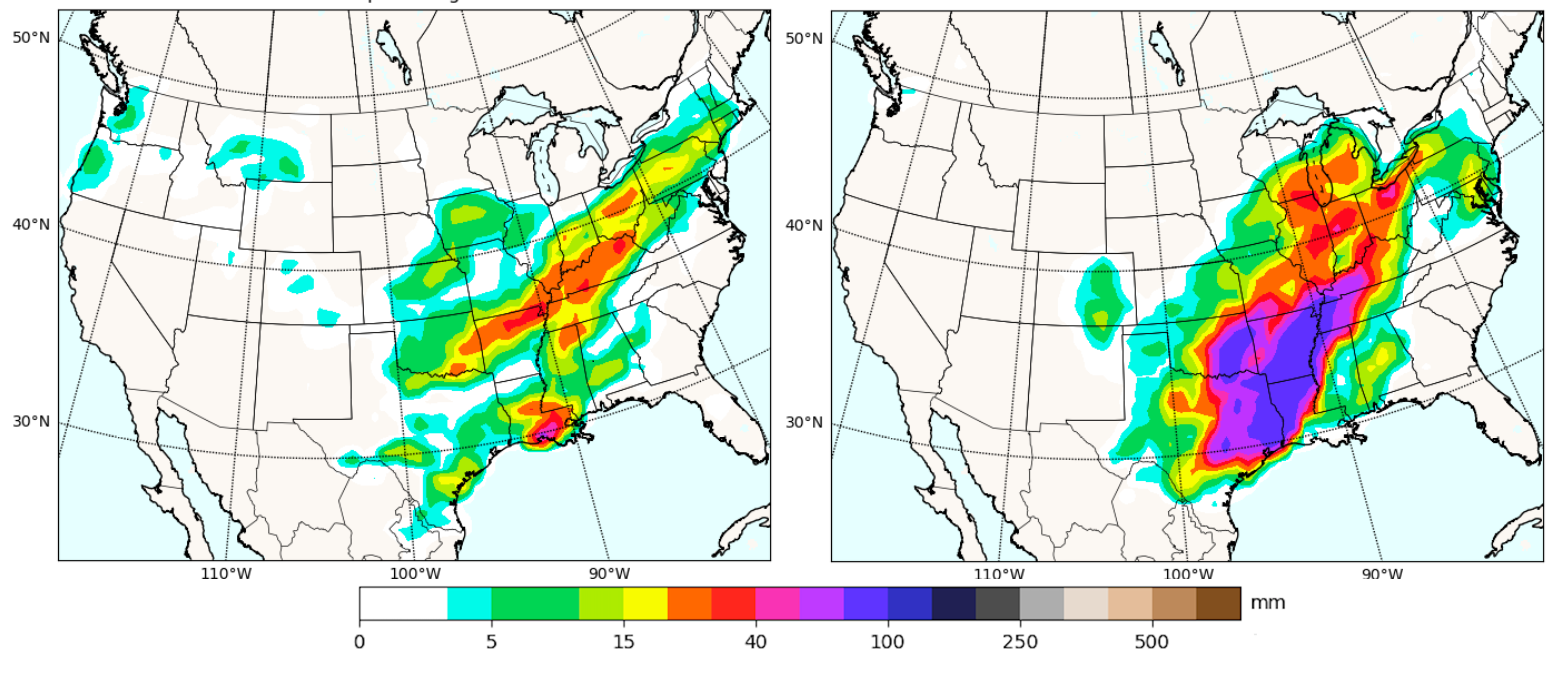}
    \caption{Our goal is to generate samples which are both \emph{realistic} and \emph{extreme}, based on any user-specified extremeness criteria (in this case, high total rainfall). \textit{Left:} Existing GAN-based approaches generate typical rainfall patterns, which have low (green) to moderate (red) rainfall. \textit{Right:} Extreme samples generated by our approach have extreme (violet) rainfall, and realistic spatial patterns resembling that of real floods.}
    \label{fig:teaserexgan}
\end{figure*}

Recently, Generative Adversarial Networks (GANs) and their variants have led to tremendous interest, due to their ability to generate highly realistic samples. On the other hand, existing GAN-based methods generate \emph{typical} samples, i.e. samples that are similar to those drawn from the bulk of the distribution. Our work seeks to address the question: how can we design deep learning-based models which can generate samples that are not just realistic, but also extreme (with respect to any user-specified measure)? Answering this question would allow us to generate extreme samples that can be used by domain experts to assist in their understanding of the nature of extreme events in a given application. Moreover, such extreme samples can be used to perform stress-testing of existing systems, to ensure that the systems remain stable under a wide range of extreme but realistic scenarios. 

Our work relates to the recent surge of interest in making deep learning algorithms reliable even for safety-critical applications such as medical applications, self-driving cars, aircraft control, and many others. Toward this goal, our work explores how deep generative models can be used for understanding and generating the extremes of a distribution, for any user-specified extremeness probability, rather than just generating typical samples as existing GAN-based approaches do. 

More formally, our problem is as follows: Given a data distribution and a criterion to measure extremeness of any sample in this data, can we generate a diverse set of realistic samples with any given extremeness probability? Consider a database management setting with queries arriving over time; users are typically interested in resilience against high query loads, so they could choose to use the number of queries per second as a criterion to measure extremeness. Then using this criterion, we aim to simulate extreme (i.e. rapidly arriving) but realistic query loads for the purpose of stress testing. Another example is rainfall data over a map, as in Figure \ref{fig:teaserexgan}. Here, we are interested in flood resilience, so we can choose to measure extremeness based on total rainfall. Then, generating realistic extreme samples would mean generating rainfall scenarios with spatially realistic patterns that resemble rainfall patterns in actual floods, such as in the right side of Figure \ref{fig:teaserexgan}, which could be used for testing the resilience of a city's flood planning infrastructure.

To model extremeness in a principled way, our approach draws from Extreme Value Theory (EVT), a probabilistic framework designed for modelling the extreme tails of distributions. However, there are two additional aspects to this problem that make it challenging. The first issue is the lack of training examples: in a moderately sized dataset, the rarity of ``extreme" samples means that it is typically infeasible to train a generative model only on these extreme samples. The second issue is that we need to generate extreme samples at any given, user-specified extremeness probability.

One possible approach is to train a GAN, say DCGAN \cite{radford2016unsupervised}, over all the images in the dataset regardless of their extremeness. A rejection sampling strategy can then be applied, where images are generated repeatedly until an example satisfying the desired extremeness probability is found. However, as we show in Section \ref{experiments}, the time taken to generate extreme samples increases rapidly with increasing extremeness, resulting in poor scalability.

Our approach, ExGAN, relies on two key ideas. Firstly, to mitigate the lack of training data in the extreme tails of the data distribution, we use a novel \textbf{distribution shifting} approach, which gradually shifts the data distribution in the direction of increasing extremeness. This allows us to fit a GAN in a robust and stable manner, while fitting the tail of the distribution, rather than its bulk. Secondly, to generate data at any given extremeness probability, we use \textbf{EVT-based conditional generation}: we train a conditional GAN, conditioned on the extremeness statistic. This is combined with EVT analysis, along with keeping track of the amount of distribution shifting performed, to generate new samples at the given extremeness probability.

We present a thorough analysis of our approach, ExGAN, on the US precipitation data. This dataset consists of daily precipitation data over a spatial grid across the lower $48$ United States (Continental United States), Puerto Rico, and Alaska. The criteria used to define extremeness is the total rainfall, and, as explained above, an extreme scenario would correspond to a flood. We show that we are able to generate realistic and extreme rainfall patterns.

Figure \ref{fig:2} shows images of rainfall patterns from the data, both normal and extreme samples, and images sampled from DCGAN and ExGAN simulating normal and extreme conditions.

\begin{figure*}[t!]
\begin{subfigure}{\textwidth}
\includegraphics[width=\linewidth]{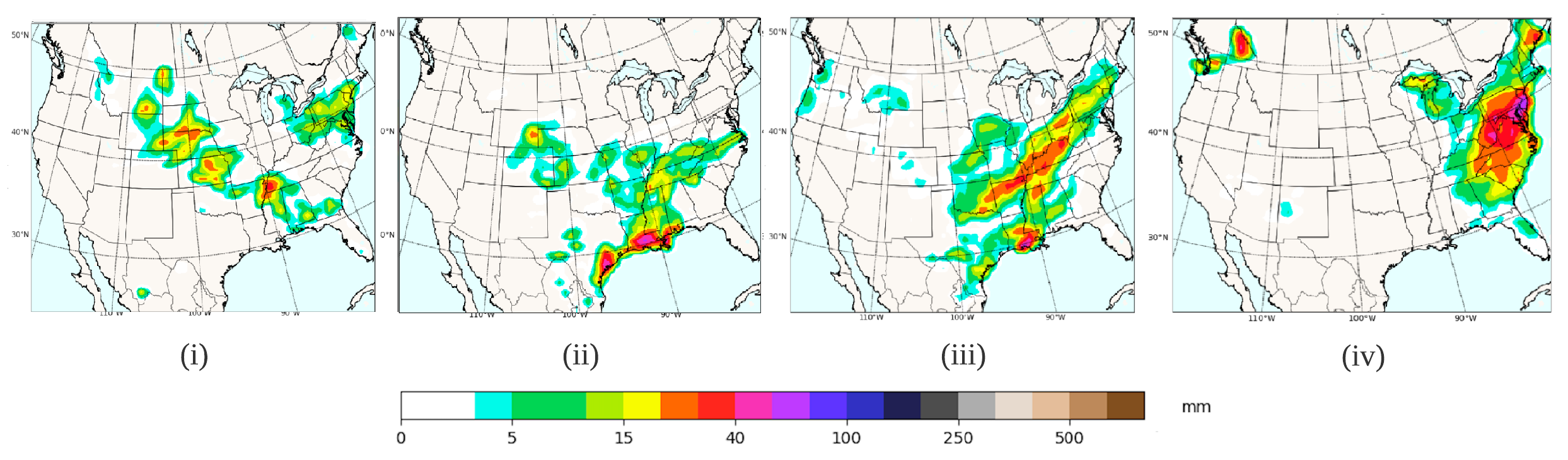}
\caption{Normal samples ((i) and (ii)) from the original dataset show low and moderate rainfall. Samples generated using DCGAN ((iii) and (iv)) are similar to normal samples from the original dataset.} \label{fig:1a}
\end{subfigure}
\begin{subfigure}{\textwidth}
\includegraphics[width=\linewidth]{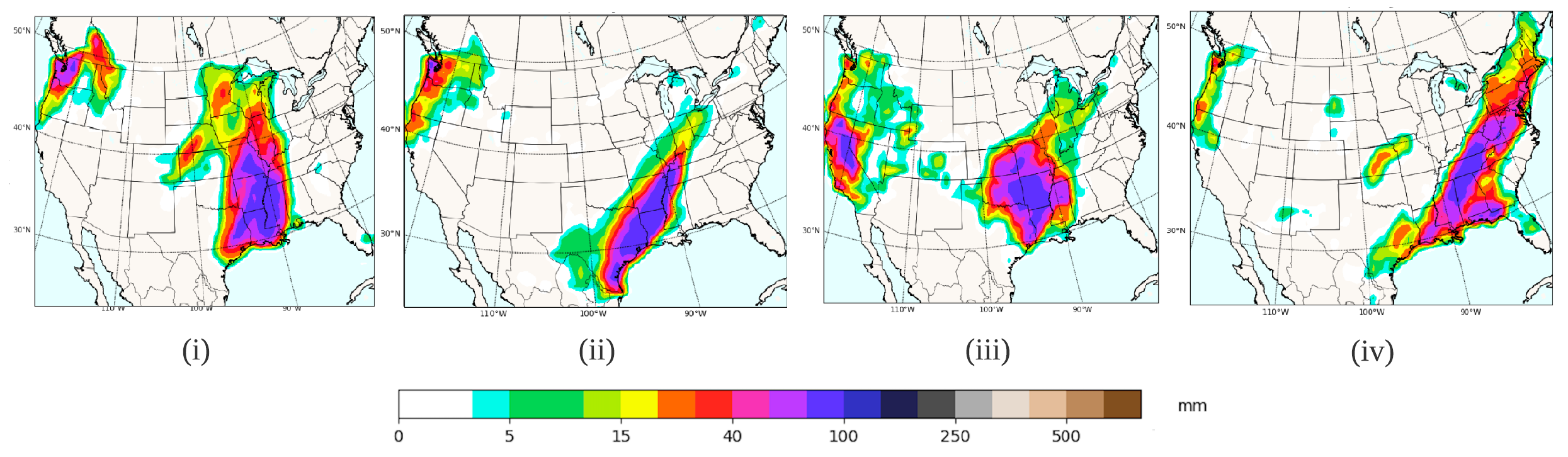}
\caption{Extreme samples ((i) and (ii)) from the original dataset showing high rainfall. Samples generated using ExGAN ((iii) and (iv)) are similar to extreme samples from the original dataset.} \label{fig:1b}
\end{subfigure}
\caption{Comparison between DCGAN (which generates normal samples), and ExGAN (which generates extreme samples).} \label{fig:2}
\end{figure*}

In summary, the main contributions of our approach are:
\begin{enumerate}
    \item {\bfseries Generating Extreme Samples:} We propose a novel deep learning-based approach for generating extreme data using distribution-shifting and EVT analysis.
    \item {\bfseries Constant Time Sampling:} We demonstrate how our approach is able to generate extreme samples in constant-time (with respect to the extremeness probability $\tau$), as opposed to the $\mathcal{O}(\frac{1}{\tau})$ time taken by the baseline approach.
    \item {\bfseries Effectiveness:} Our experimental results show that ExGAN generates realistic samples based on both visual inspection and quantitative metrics, and is faster than the baseline approach by at least three orders of magnitude for extremeness probability of $0.01$ and beyond.
\end{enumerate}

{\bfseries Reproducibility}: Our code and datasets are publicly available at \href{https://github.com/Stream-AD/ExGAN}{https://github.com/Stream-AD/ExGAN}.

\FloatBarrier

\section{Related Work}
\subsection{Conditional Generative Adversarial Networks}
Conditional GANs (CGANs), introduced in \cite{cgan}, allow additional information as input to GAN which makes it possible to direct the data generation process. Conditional DCGAN (CDCGAN) \cite{cdcgan}, is a modification of CGAN using the conditional variables but with a  convolutional architecture. These methods are briefly discussed in Section \ref{background}. 
There has also been a significant amount of work done on GAN-based models for conditioning on different types of inputs such as images \cite{imagegan1,imagegan2}, text \cite{textgan2}, and multi-modal conditional GANs \cite{multimodal}.

\subsection{Data Augmentation}
Data Augmentation using GANs \cite{antoniou2017data,shmelkov2018good,tran2017bayesian,tran2020towards,yamaguchi2019effective,karras2020training} has been extensively used in different domains, such as anomaly detection \cite{lim2018doping}, time series \cite{zhou2019beatgan,ramponi2018t}, speech processing \cite{zhang2019dada}, NLP \cite{chang2018code,yu2017seqgan, fedus2018maskgan}, emotion classification \cite{zhu2018emotion, luo2018eeg_emotion}, 
medical applications \cite{Zheng2017UnlabeledSG, han2019learning, hu2018prostategan, calimeri2017biomedical} and 
computer vision \cite{karras2019style, odena2017conditional, perez2017effectiveness, sixt2018rendergan, choi2019self, siarohin2019appearance} as a solution for tackling class imbalance \cite{bagan} and generating cross-domain data \cite{crossdomain}. However, these methods do not provide any control over the extremeness of the generated data.

\subsection{Extreme Value Theory}
Extreme value theory~\cite{gumbel2012statistics,theorem2} is a statistical framework for modelling extreme deviations or tails of probability distributions. EVT has been applied to a variety of machine learning tasks including anomaly detection \cite{evtad1, Siffer2017AnomalyDI, evtad3, evtad4, evtad5}, graph mining \cite{hooi2020telltail} and local intrinsic dimensionality estimation \cite{evtlid1, evtlid2}. \cite{evtnips} use EVT to develop a probabilistic framework for classification in extreme regions, \cite{robustnessevt} use it to design an attack-agnostic robustness metric for neural networks.

EVT typically focuses on modelling univariate or low-dimensional~\cite{tawn1990modelling} distributions. A few approaches, such as dimensionality-reduction based~\cite{chautru2015dimension,sabourin2014bayesian}, exist for moderate dimensional vectors (e.g. $20$). A popular approach for multivariate extreme value analysis is Peaks-over-Threshold with specific definitions of exceedances \cite{rootzen2006, ferreira2014, engelke}, and \cite{dombry} showed it can be modelled by r-Pareto processes. \cite{davison2018, fondeville2020functional} presented an inference method on r-Pareto processes applicable to higher dimensions compared to previous works on max-stable processes \cite{asadi} and Pareto processes \cite{thibaud}.

To the best of our knowledge, there has not been any work on extreme sample generation using deep generative models.

\section{Background}  \label{background}

\label{app:4exgan}
\subsection{GAN and DCGAN:}
Generative Adversarial Network (GAN) is a framework to train deep generative models. The training is done using a minimax game, where a generator $G$ producing synthetic samples plays against a discriminator $D$ that attempts to discriminate between real data and samples created by $G$. The goal of the generator is to learn a distribution $P_G$ which matches the data distribution $P_{data}$. Instead of explicitly estimating $P_{G}$, $G$ learns to transform noise variables $z \sim P_{noise}$, where $P_{noise}$ is the distribution of noise, into synthetic samples $ x \sim G(z)$. The discriminator $D$ outputs $D(x)$ representing the probability of a sample $x$ coming from the true data distribution. In practice, both $G(z;\theta_g)$ and $D(x;\theta_d)$ are parameterized by neural networks. $G$ and $D$ are simultaneously trained by using the minimax game objective $V_{GAN}(D,G)$:
\begin{align*}
   \min _{G} \max _{D} V_{GAN}(D, G)=\mathbb{E}_{x \sim P_{data}}[\log D(x)] \\
   + \ \mathbb{E}_{z \sim P_{n o i s e}}[\log (1-D(G(z)))]
\end{align*}
The stability in training and the effectiveness in learning unsupervised image representations are some of the reasons that make Deep Convolutional GAN, or DCGAN, \cite{radford2016unsupervised} one of the most popular and successful network designs for GAN, especially when dealing with image data. The DCGAN model uses strided convolutions in the discriminator and fractional strided convolutions in the generator along with a bunch of tricks to stabilize training. 
\subsection{CGAN and CDCGAN:}
CGAN extends GANs to conditional models by adding auxiliary information, or conditionals, to both the generator and discriminator. It is done by feeding the conditional, $y$, as an additional input layer. The modified objective is given by  
\begin{align*}
    \min _{G} \max _{D} V(D, G)=\mathbb{E}_{{x} \sim P_{data}}[\log D({x} | {y})] \\
   + \ \mathbb{E}_{{z} \sim P_{noise}}[\log (1-D(G({z} | {y})))] 
\end{align*}
The implementation of CGAN consists of linear or fully connected layers. cDCGAN improves on CGAN by using the DCGAN architecture along with the additional conditional input. The use of convolutional layers generates samples with much better image quality compared to CGAN.

\subsection{Extreme Value Theory (EVT)}

The Generalized Pareto Distribution (GPD) ~\cite{coles2001introduction} is a commonly used distribution in EVT. The parameters of GPD are its \textbf{scale} $\sigma$, and its \textbf{shape} $\xi$. The cumulative distribution function
 (CDF) of the GPD is:
\begin{align}
G_{\sigma, \xi}(x) = 
\begin{cases}
    1-(1+\frac{\xi\cdot x}{\sigma})^{-1/\xi} & \text{if } \xi \neq 0 \\
    1-\exp(-\frac{x}{\sigma}) & \text{if } \xi = 0
\end{cases}
\end{align}

A useful property of the GPD is that it generalizes both Pareto distributions (which have heavy tails) and exponential distributions (which have exponentially decaying tails). In this way, the GPD can model both heavy tails and exponential tails, and smoothly interpolate between them. Another property of the GPD is its `universality' property for tails: intuitively, it can approximate the tails of a large class of distributions following certain smoothness conditions, with error approaching $0$. Thus, the GPD is particularly suitable for modelling the tails of distributions.

\cite{theorem2,theorem1} show that the excess over a sufficiently large threshold $u$, denoted by $X - u$, is likely to follow a Generalized Pareto Distribution (GPD) with parameters $\sigma(u),\xi$. This is also known as the Peaks over Threshold method. In practice, the threshold $u$ is commonly set to a value around the $95^{th}$ percentile, while the remaining parameters can be estimated using maximum likelihood estimation~\cite{grimshaw1993computing}.

\begin{theo} \cite{theorem2}\cite{theorem1}. For a large class of distributions, a function $\sigma(u)$ can be found such that
\begin{equation}
\lim _{u \rightarrow \bar{x}} \sup _{0 \leq x<\bar{x}-u}\left|F_{u}(x)-G_{\sigma(u),\xi}\right|=0
\end{equation}
where $\bar{x}$ is the rightmost point of the distribution, $u$ is a threshold, and $F_{u}$ is the \emph{excess distribution function}, i.e. $F_{u}(x)=P(X-u \leq x | X>u)$.
\end{theo}

\section{ExGAN: Extreme Sample Generation Using GANs}
\label{methods}
\subsection{Problem}
We are given a training set $\mathbf{x}_1, \cdots, \mathbf{x}_n \sim \mathcal{D}$, along with $\mathsf{E}(\mathbf{x})$, a user-defined \emph{extremeness measure}: for example, in our running example of rainfall modelling, the extremeness measure is defined as the total rainfall in $\mathbf{x}$, but any measure could be chosen in general. We are also given a user-specified \emph{extremeness probability} $\tau \in (0, 1)$, representing how extreme the user wants their sampled data to be: for example, $\tau=0.01$ represents generating an event whose extremeness measure is only exceeded $1\%$ of the time.\footnote{In hydrology, the notion of a \emph{$100$-year flood} is a well-known concept used for flood planning and regulation, which is defined as a flood that has a $1$ in $100$ chance of being exceeded in any given year. Given daily data, generating a 100-year flood then corresponds to setting $\tau=\frac{1}{365 \times 100}.$ }

Given these, our goal is to generate synthetic samples $\mathbf{x}'$ that are both 1) \emph{realistic}, i.e. hard to distinguish from the training data, and 2) \emph{extreme} at the given level: that is, 
$P_{\mathbf{x} \sim \mathcal{D}}(\mathsf{E}(\mathbf{x}) > \mathsf{E}(\mathbf{x}'))$ should be as close as possible to $\tau$.

\subsection{Distribution Shifting}
\label{secdistshift}
An immediate issue we face is that we want our trained model to mimic the extreme tails, not the bulk of the distribution; however, most of the data lies in its bulk, with much fewer samples in its tails. While data augmentation could be employed, techniques like image transform may not be applicable: for example, in the US precipitation data, each pixel captures the rainfall distribution at some fixed location; altering the image using random transforms would change this correspondence.

To address this issue, we propose a novel Distribution Shifting approach in Algorithm \ref{alg:distshift}, parameterized by a shift parameter $c \in (0, 1)$. Our overall approach is to repeatedly `shift' the distribution by filtering away the less extreme $(1-c)$ proportion of the data, then generating data to return the dataset to its original size. In addition, to maintain the desired proportion of original data points from $\mathcal{X}$, we adopt a `stratified' filtering approach, where the original and generated data are filtered separately.

 \begin{algorithm}
	\caption{Distribution Shifting\ \label{alg:distshift}}
	{\bfseries Input}: dataset $\mathcal{X}$, extremeness measure $\mathsf{E}$, shift parameter $c$, iteration count $k$ \\
	Sort $\mathcal{X}$ in decreasing order of extremeness\\
	Initialize $\mathcal{X}_s \gets \mathcal{X}$ \\
    \For{$i \gets 1 \text{ to } k$}{
    {\bfseries $\triangleright$ Shift the data distribution by a factor of $c$:}\\
    Train DCGAN $G$ and $D$ on $\mathcal{X}_s$\\
	$\mathcal{X}_s \gets$ top $\lfloor c^{i}\cdot n\rfloor$ extreme samples of $\mathcal{X}$ \\
    Generate $\lceil (n-\lfloor c^{i}\cdot n\rfloor) \cdot \dfrac{1}{c}\rceil$ data points using $G$, and insert most extreme $n-\lfloor c^{i}\cdot n\rfloor$ samples into $\mathcal{X}_s$\\
    }
 \end{algorithm}

Specifically, we first sort our original dataset $\mathcal{X}$ in decreasing order of extremeness (Line 2), then initialize our shifted dataset $\mathcal{X}_s$ as $\mathcal{X}$ (Line 3). Next, each iteration $i$ of a Distribution Shift operation works as follows. We first fit a DCGAN to $\mathcal{X}_s$ (Line 6). We then replace our shifted dataset $\mathcal{X}_s$ with the top $\lfloor c^i\cdot n\rfloor$ extreme data points from $\mathcal{X}$ (Line 7). Next, we use the DCGAN to generate additional $\lceil (n-\lfloor c^{i}\cdot n\rfloor) \cdot \dfrac{1}{c}\rceil$ data samples and add the most extreme $n-\lfloor c^{i}\cdot n\rfloor$ samples to $\mathcal{X}_s$ (Line 8). This ensures that we choose the most extreme $c$ proportion of the generated data while bringing the dataset back to its original size of $n$ data points. Each such iteration shifts the distribution toward its upper tail by a factor of $c$. We perform $k$ iterations, aiming to shift the distribution sufficiently so that $\tau$ is no longer in the extreme tail of the resulting shifted distribution. Iteratively shifting the distribution in this way ensures that we always have enough data to train the GAN in a stable manner, while allowing us to gradually approach the tails of the distribution.

In addition, during the shifting process, we can train successive iterations of the generator via `warm start', by initializing its parameters using the previously trained model, for the sake of efficiency.

\FloatBarrier

\subsection{EVT-based Conditional Generation}
The next issue we face is the need to generate samples at the user-given extremeness probability of $\tau$. Our approach will be to train a conditional GAN using extremeness as a conditioning variable. To generate samples, we then use EVT analysis, along with our knowledge of how much shifting has been performed, to determine the necessary extremeness level we should condition on, to match the desired extremeness probability.

Specifically, first note that after $k$ shifts, the corresponding extremeness probability in the shifted distribution that we need to sample at becomes $\tau' = \tau / c^k$. Thus, it remains to sample from the shifted distribution at the extremeness probability of $\tau'$, which we will do using EVT. Algorithm \ref{alg:condgenEVT} describes our approach: we first compute the extremeness values using $\mathsf{E}$ on each point in $\mathcal{X}_s$: i.e. $e_i = \mathsf{E}(\mathbf{x_i}) \ \forall \ \mathbf{x_i} \in \mathcal{X}_s$ (Line 2). Then we perform EVT Analysis on $e_1, \cdots, e_n$: we fit Generalized Pareto Distribution (GPD) parameters $\sigma, \xi$ using maximum likelihood estimation~\cite{grimshaw1993computing} to $e_1, \cdots, e_n$ (Line 3). Next, we train a conditional DCGAN (Generator $G_s$ and Discriminator $D_s$) on $\mathcal{X}_s$, with the conditional input to $G_s$ (within the training loop of $G_s$) sampled from a GPD with parameters $\sigma, \xi$ (Line 4). In addition to the image, $D_s$ takes in a second input which is $e$ for a generated image $G_s(\mathbf{z}, e)$ and $\mathsf{E}(\mathbf{x})$ for a real image $\mathbf{x}$. An additional loss $\mathcal{L}_{\text{ext}}$ is added to the GAN objective:
\begin{align} 
\mathcal{L}_{\text{ext}} = \mathbb{E}_{\mathbf{z},e}\left[\dfrac{|e-\mathsf{E}({G_s(\mathbf{z}, e)})|}{e}\right]
\end{align}
where $\mathbf{z}$ is sampled from a multivariate standard normal distribution and $e$ is sampled from a GPD with parameters $\sigma, \xi$. Note that training using $\mathcal{L}_{\text{ext}}$ requires $\mathsf{E}$ to be differentiable.

$\mathcal{L}_{\text{ext}}$ minimizes the distance between the desired extremeness ($e$) and the extremeness of the generated sample ($\mathsf{E}(G_s(z, e)$). This helps reinforce the conditional generation property and prevents the generation of samples with unrelated extremeness. Using the inverse CDF of the GPD, we determine the extremeness level $e'$ that corresponds to an extremeness probability of $\tau'$:
\begin{align} 
    e' = G^{-1}_{\sigma, \xi}(1-\tau')
    \label{eq:gpd}
\end{align}
where $G^{-1}_{\sigma, \xi}$ is the inverse CDF of the fitted GPD (Line 5). 
Finally, we sample from our conditional DCGAN at the desired extremeness level $e'$ (Line 6). 

\begin{algorithm}
	\caption{EVT-based Conditional Generation \label{alg:condgenEVT}}
	{\bfseries Input}: shifted dataset $\mathcal{X}_s$, extremeness measure $\mathsf{E}$, adjusted extremeness probability $\tau'$ \\
	Compute extremeness values $e_i = \mathsf{E}(\mathbf{x_i}) \ \forall \ \mathbf{x_i} \in \mathcal{X}_s$ \\
	Fit GPD parameters $\sigma, \xi$ using maximum likelihood~\cite{grimshaw1993computing} on $e_1, \cdots, e_n$ \\
	Train conditional DCGAN ($G_s$ and $D_s$) on $\mathcal{X}_s$ where the conditioning input for $G_s$ is sampled from a GPD with parameters $\sigma, \xi$ \\
	Extract required extremeness level: $e' \gets G^{-1}_{\sigma, \xi}(1-\tau')$\\
	Sample from $G_s$ conditioned on extremeness level $e'$\\
\end{algorithm}

\FloatBarrier

\section{Experiments} \label{experiments}
In this section, we evaluate the performance of ExGAN compared to DCGAN on the US precipitation data. We aim to answer the following questions:

\begin{enumerate}[label=\textbf{Q\arabic*.}]
\item {\bfseries Realistic Samples (Visual Inspection):} Does ExGAN generate realistic extreme samples, as evaluated by visual inspection of the images?
\item {\bfseries Realistic Samples (Quantitative Measures):} Does ExGAN generate realistic extreme samples, as evaluated using suitable GAN metrics?
\item {\bfseries Speed:} How fast does ExGAN generate extreme samples compared to the baseline? Does it scale with high extremeness?
\end{enumerate}

\paragraph{Dataset:}
We use the US precipitation dataset \footnote{https://water.weather.gov/precip/}. The National Weather Service employs a multi-sensor approach to calculate the observed precipitation with a spatial resolution of roughly $4\times4$ km on an hourly basis. We use the daily spatial rainfall distribution for the duration of January 2010 to December 2016 as our training set, and for the duration of January 2017 to August 2020 as our test set. We only retain those samples in our test set which are more extreme, i.e. have higher total rainfall, than the $95^{th}$ percentile in the train set. Images with original size $813\times1051$ are resized to $64\times64$ and normalized between $-1$ and $1$.

\paragraph{Baseline:}
The baseline is a DCGAN \cite{radford2016unsupervised} trained over all the images in the dataset, combined with rejection sampling. Specifically, to generate at a user-specified level $\tau$, we use EVT as in our framework (i.e. Eq. \eqref{eq:gpd}) to compute the extremeness level $e=G^{-1}_{\sigma, \xi}(1-\tau)$ that corresponds to an extremeness probability of $\tau$. We then repeatedly generate images until one is found that satisfies the extremeness criterion within $10\%$ error; that is, we reject the image $\mathbf{x}$ if 
$\left|\dfrac{e - \mathsf{E}(\mathbf{x})}{e}\right| > 0.1$.

\paragraph{Evaluation Metrics:}
We evaluate how effectively the generator is able to mimic the tail of the distribution using FID and Reconstruction Loss metrics. Fréchet Inception Distance (FID) \cite{fid} is a common metric used in the GAN literature to evaluate image samples and has been found to be consistent with human judgement. Intuitively, it compares the distributions of real and generated samples based on their activation distributions in a pre-trained network. However, an ImageNet-pretrained Inception network which is usually used to calculate FID is not suitable for our dataset. Hence, we construct an autoencoder trained on test data, as described above, and use the statistics on its bottleneck activations to compute the FID:
\begin{align*}
\mathrm{FID}=\left\|\bm{\mu_{r}}-\bm{\mu_{g}}\right\|^{2}+\operatorname{Tr}\left(\bm{\Sigma_{r}}+\bm{\Sigma_{g}}-2\left(\bm{\Sigma_{r}} \bm{\Sigma_{g}}\right)^{1 / 2}\right)
\end{align*}
where $\operatorname{Tr}$ denotes the trace of a matrix, $\left(\bm{\mu_{r}}, \bm{\Sigma_{r}}\right)$ and $\left(\bm{\mu_{g}}, \bm{\Sigma_{g}}\right)$ are the mean and covariance of the bottleneck activations for the real and generated samples respectively.

We further evaluate our model on its ability to reconstruct unseen extreme samples by computing a reconstruction loss on the test set \cite{Xiang2017OnTE}.

Letting $\mathbf{\tilde{x}}_1, \cdots, \mathbf{\tilde{x}}_m$ denote the test images, the reconstruction loss for an unconditional generator $G$ is given by, 
$$\mathcal{L}_{\mathrm{rec}}=\frac{1}{m} \sum_{i=1}^{m} \min _{\mathbf{z_i}}\left\|G(\mathbf{z_i})-\mathbf{\tilde{x}}_i\right\|_{2}^{2}$$
where $\mathbf{z_i}$ are the latent space vectors

For an extremeness conditioned generator $G$,
$$\mathcal{L}_{\mathrm{rec\_ext}}=\frac{1}{m} \sum_{i=1}^{m} \min _{\mathbf{z_i}}\left\|G(\mathbf{z_i}, \mathsf{E}(\mathbf{\tilde{x}}_i))-\mathbf{\tilde{x}}_i\right\|_{2}^{2}$$

To compute the reconstruction loss, we initialize the latent space vectors $\mathbf{z_i}$ as the zero vector and perform gradient descent on it to minimize the objective defined above. We use similar parameters as \cite{Xiang2017OnTE} to calculate the reconstruction loss, i.e. learning rate was set to $0.001$ and the number of gradient descent steps was set to $2000$, while we use Adam optimizer instead of RMSprop.

We also evaluate how accurately our method is able to condition on the extremeness of the samples. We use Mean Absolute Percentage Error (MAPE), where the error is calculated between the extremeness used to generate the sample ($e$) and the extremeness of the generated sample ($\mathsf{E}(G(\mathbf{z}, e))$).

\begin{align} 
\text{MAPE} = \mathbb{E}_{\mathbf{z},e}\left[\dfrac{|e-\mathsf{E}({G_s(\mathbf{z}, e)})|}{e}\right] \times 100\%
\end{align}
where $\mathbf{z}$ is sampled from a multivariate standard normal distribution and $e$ is sampled from a GPD with parameters $\sigma, \xi$.

\paragraph{Experimental Setup:}
All experiments are carried out on a $2.6 GHz$  Intel Xeon CPU, $256 GB$ RAM, $12 GB$ Nvidia GeForce RTX 2080 Ti GPU running Debian GNU/Linux $9$.

Images are upsampled from $64\times64$ to $813\times1051$ to plot the rainfall maps. We also apply techniques introduced in the literature to stabilize GAN training such as label smoothing, noisy inputs to the discriminator, lower learning rate for the discriminator, label flipping, and gradient clipping~\cite{wgan,noisy}. Details of these techniques can be found in the Implementation Details.

\paragraph{Network Architectures}
\label{app:2exgan}

Let ConvBlock denote the sequence of layers Conv$4\times4$, InstanceNorm\cite{instancenorm}, LeakyReLU with appropriate sizes. Similarly let ConvTBlock denote the sequence of layers ConvTranspose4x4, InstanceNorm, LeakyRelu with appropriate sizes. Let $n$ be the batch size.

Tables \ref{tab:exgang}, \ref{tab:exgand}, \ref{tab:dcgang}, \ref{tab:dcgand}, and \ref{tab:extremeae} show the architectures for ExGAN Generator, ExGAN Discriminator, DCGAN Generator, DCGAN Discriminator, and FID Autoencoder respectively.

\begin{table}[!ht]
\centering
\caption{Architecture for ExGAN Generator.}
\label{tab:exgang}
\begin{center}
\begin{tabular}{@{}rccc@{}}
\toprule
 \textbf{Index} & \textbf{Layer} & \textbf{Output Size}  \\ \midrule
$1$ \ \ \ \ & ConvTBlock & $n\times512\times4\times4$\\
$2$ \ \ \ \ & ConvTBlock & $n\times256\times8\times8$\\
$3$ \ \ \ \ & ConvTBlock & $n\times128\times16\times16$\\
$4$ \ \ \ \ & ConvTBlock & $n\times64\times32\times32$\\
$5$ \ \ \ \ & ConvTranpose$4\times4$ & $n\times1\times64\times64$\\
$6$ \ \ \ \ & Tanh & $n\times1\times64\times64$\\
\bottomrule
\end{tabular}

\end{center}
\end{table}

\begin{table}[htb!]
\caption{Architecture for ExGAN Discriminator.}
\label{tab:exgand}
\begin{center}
\begin{tabular}{@{}rccc@{}}
\toprule
 \textbf{Index} & \textbf{Layer} & \textbf{Output Size}  \\ \midrule
$1$ \ \ \ \ & ConvBlock & $n\times64\times32\times32$\\
$2$ \ \ \ \ & ConvBlock & $n\times128\times16\times16$\\
$3$ \ \ \ \ & ConvBlock & $n\times256\times8\times8$\\
$4$ \ \ \ \ & ConvBlock & $n\times512\times4\times4$\\
$5$ \ \ \ \ & Conv$4\times4$ & $n\times64\times1\times1$\\
$6$ \ \ \ \ &Reshape &$n\times64$\\
$7$ \ \ \ \ &Concat &$n\times65$\\
$8$ \ \ \ \ & Linear &$n\times1$\\
$9$ \ \ \ \ & Sigmoid &$n\times1$\\
\bottomrule
\end{tabular}
\end{center}
\end{table}

\begin{table}[!ht]
\caption{Architecture for DCGAN Generator.}
\label{tab:dcgang}
\begin{center}
\begin{tabular}{@{}rccc@{}}
\toprule
\textbf{Index} & \textbf{Layer} & \textbf{Output Size}  \\ \midrule
$1$ \ \ \ \ & ConvTBlock & $n\times512\times4\times4$\\
$2$ \ \ \ \ & ConvTBlock & $n\times256\times8\times8$\\
$3$ \ \ \ \ & ConvTBlock & $n\times128\times16\times16$\\
$4$ \ \ \ \ & ConvTBlock & $n\times64\times32\times32$\\
$5$ \ \ \ \ & ConvTranpose$4\times4$ & $n\times1\times64\times64$\\
$6$ \ \ \ \ & Tanh &$n\times1\times64\times64$\\
\bottomrule
\end{tabular}

\end{center}
\end{table}

\begin{table}[htb!]
\caption{Architecture for DCGAN Discriminator.}
\label{tab:dcgand}
\begin{center}
\begin{tabular}{@{}rccc@{}}
\toprule
\textbf{Index} & \textbf{Layer} & \textbf{Output Size}  \\ \midrule
$1$ \ \ \ \ & ConvBlock & $n\times64\times32\times32$\\
$2$ \ \ \ \ & ConvBlock & $n\times128\times16\times16$\\
$3$ \ \ \ \ & ConvBlock & $n\times256\times8\times8$\\
$4$ \ \ \ \ & ConvBlock & $n\times512\times4\times4$\\
$5$ \ \ \ \ & Conv$4\times4$ & $n\times64\times1\times1$\\
$6$ \ \ \ \ &Reshape &$n\times64$\\
$7$ \ \ \ \ & Linear &$n\times1$\\
$8$ \ \ \ \ & Sigmoid &$n\times1$\\
\bottomrule
\end{tabular}
\end{center}
\end{table}

\begin{table}[htb!]
\caption{Architecture for FID Autoencoder}
\label{tab:extremeae}
\begin{center}
\begin{tabular}{@{}rccc@{}}
\toprule
\textbf{Index} & \textbf{Layer} & \textbf{Output Size}  \\ \midrule
$1$ \ \ \ \ & Linear & $n\times128$\\
$2$ \ \ \ \ & ReLU & $n\times128$\\
$3$ \ \ \ \ & Dropout(0.5) & $n\times128$\\
$4$ \ \ \ \ & Linear & $n\times4096$\\
\bottomrule
\end{tabular}
\end{center}
\end{table}

\FloatBarrier

\paragraph{Implementation Details}
\label{app:3exgan}
The following settings were common to both DCGAN and ExGAN. All convolutional layer weights were initialized from $\mathcal{N}(0, 0.02)$. We sample the noise, or latent inputs, from a standard normal distribution instead of uniform distribution with the latent dimension = $20$. Alpha for LeakyReLU was set to $0.2$. Adam optimizer was used with parameters, Learning rate for $G = 0.0002$, $D = 0.0001$, and betas = ($0.5$, $0.999$). Noisy labels were used, i.e. the Real and Fake labels used for training had values in [$0.7$, $1.2$] and [$0, 0.3$] instead of $1$ and $0$ respectively \cite{noisy}. The Real and Fake labels were flipped with a probability of $0.05$. Gradient clipping was employed restricting the gradients of $G$ and $D$ to be in [-20, 20]. Noise was added to the input of the $D$ starting from $1e-5$ and linearly decreased to $0$. Batch Size was $256$.\\
{\bfseries Distribution Shifting}: Unless stated otherwise, $c$ was set to 0.75, $k$ was set to 10. For the initial iteration, where the network is trained on all data, the learning rates for $G$ and $D$ were set to $0.0002$ and $0.0001$ respectively, and the network was trained for $500$ epochs. For subsequent iterations, learning rates for $G$ and $D$ were lowered to $0.00002$ and $0.00001$ respectively, and the network was trained for 100 epochs. \\
FID Autoencoder: The Autoencoder was optimized using Adam with a learning rate $0.001$, trained for $50$ epochs with standard L1 Loss. To ensure a fair comparison, we only compare the most extreme samples from DCGAN with ExGAN. Specifically, if ExGAN generates $n$ samples where the extremeness probabilities are sampled uniformly from $(0, \tau]$, then we generate $\lceil\frac{n}{\tau}\rceil$ samples from DCGAN and retain the most extreme $n$ samples for comparison.

\subsection{Realistic Samples (Visual Inspection)}
Figure \ref{fig:3} shows the extreme samples generated by ExGAN corresponding to extremeness probability $\tau = 0.001$ and $0.0001$. We observe that ExGAN generates samples that are similar to the images of rainfall patterns from the original data in Figure \ref{fig:1b}. As we change $\tau$ from $0.001$ to $0.0001$, we observe the increasing precipitation in the generated samples. The typical pattern of radially decreasing rainfall in real data is learned by ExGAN. ExGAN also learns that coastal areas are more susceptible to heavy rainfall.

\begin{figure*}[t!]
\begin{subfigure}{\textwidth}
\includegraphics[width=\linewidth]{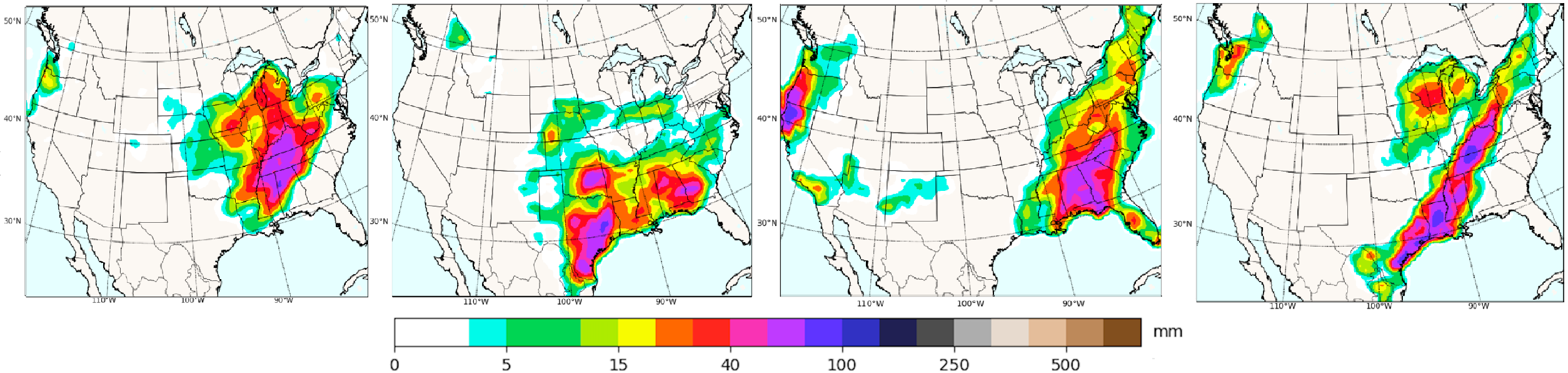}
\caption{Samples from ExGAN for extremeness probability $\tau = 0.001$. Time taken to sample = $0.002s$} \label{fig:3a}
\end{subfigure}
\begin{subfigure}{\textwidth}
\includegraphics[width=\linewidth]{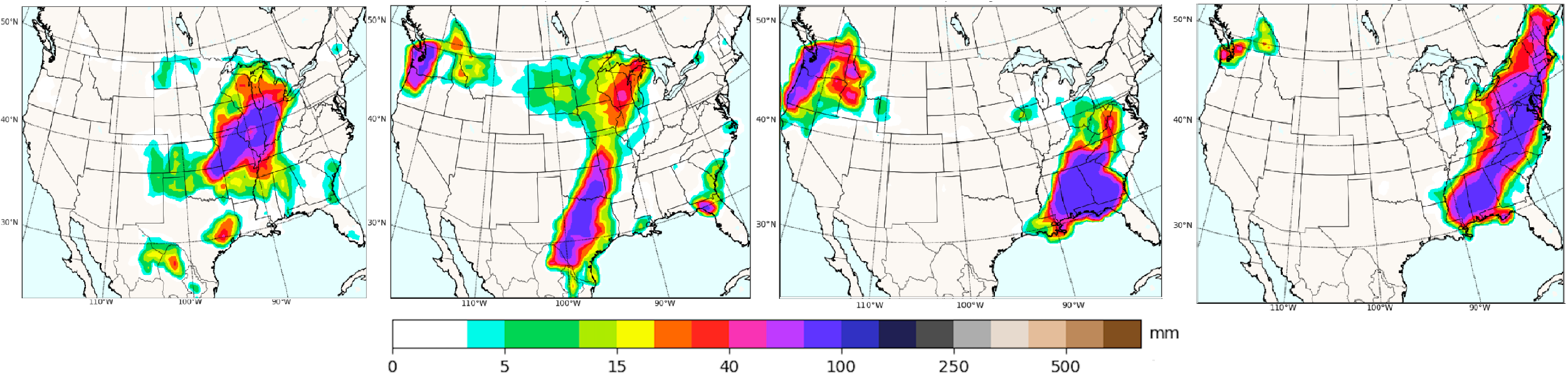}
\caption{Samples from ExGAN for extremeness probability $\tau = 0.0001$. Time taken to sample = $0.002s$} \label{fig:3b}
\end{subfigure}
\begin{subfigure}{\textwidth}
\includegraphics[width=\linewidth]{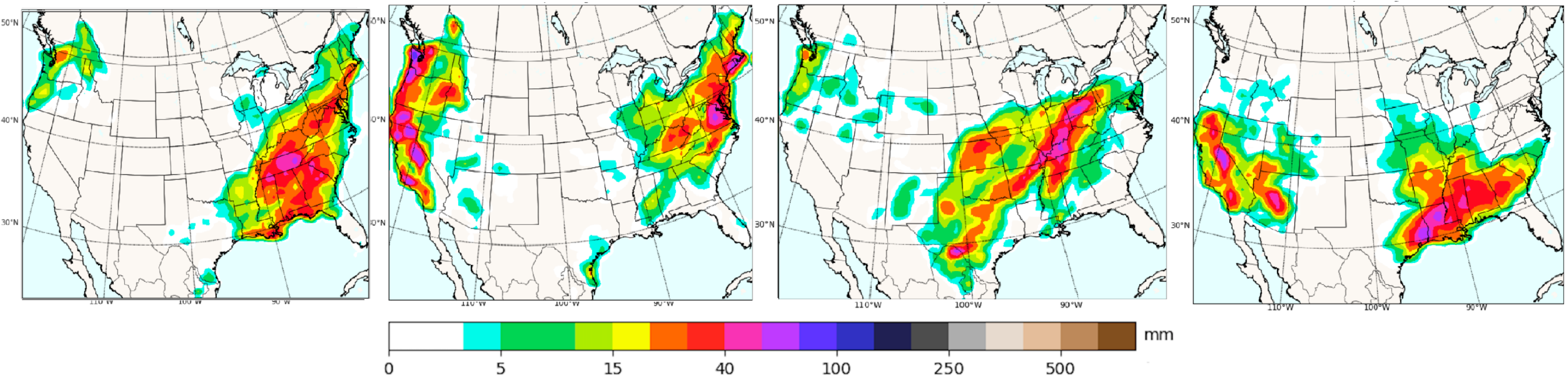}
\caption{Samples from DCGAN for extremeness probability $\tau = 0.01$. Time taken to sample = $7.564s$. DCGAN is unable to generate samples in $1$ hour when $\tau = 0.001$ or $0.0001$.} \label{fig:3c}
\end{subfigure}
\caption{ExGAN generates images that are realistic, similar to the original data samples, in constant time.}
\label{fig:3}
\end{figure*}

Figure \ref{fig:3c} shows the extreme samples generated by DCGAN for extremeness probability $\tau = 0.01$. When $\tau = 0.001$ or $0.0001$, DCGAN is unable to generate even one sample, within $10$\% error, in $1$ hour (as we explain further in Section \ref{sec:speed}).

\FloatBarrier

\subsection{Realistic Samples (Quantitative Measures)}

The GAN is trained for $100$ epochs in each iteration of distribution shifting. For distribution shifting, we set $c=0.75$, $k=10$ and use warm start. MAPE for DCGAN can be upper bounded by the rejection strategy used for sampling, and this bound can be made tighter at the expense of sampling time. For our experiment, we upper bound the MAPE for DCGAN by $10\%$ as explained above. MAPE for ExGAN is $3.14\% \pm 3.08\%$.

Table \ref{tab:metrics} reports the FID (lower is better) and reconstruction loss (lower is better). ExGAN is able to capture the structure and extremeness in the data, and generalizes better to unseen extreme scenarios, as shown by the lower reconstruction loss and lower FID score (loss = $0.0172$ and FID = $0.0236 \pm 0.0037$) as compared to DCGAN (loss = $0.0292$ and FID = $0.0406 \pm 0.0063$).

\begin{table}[ht]
\begin{center}
\caption{FID, and Reconstruction Loss, for DCGAN and ExGAN (averaged over $5$ runs). For FID, the p-value for significant improvement of ExGAN over the baseline is $0.002$, using a standard two-sample t-test.}
\begin{tabular}{ccc}
\toprule

{\bfseries Method} & {\bfseries FID} & {\bfseries Reconstruction Loss} \\
\midrule
\textbf{DCGAN}       & $0.0406 \pm 0.0063$ & $0.0292$ \\
\textbf{ExGAN}       & $0.0236 \pm 0.0037$ & $0.0172$ \\

\bottomrule
\end{tabular}
\label{tab:metrics}
\end{center}
\end{table}

Table \ref{tab:apptable2} reports the reconstruction loss, MAPE and FID for ExGAN for different values of $c$ and $k$. To ensure a fair comparison, we select the parameters $c$ and $k$ for distribution shifting, such that the amount of shift, $c^k$, is approximately similar.
Intuitively, we would expect higher $c$ to correspond to slower and more gradual shifting, which in turn helps the network smoothly interpolate and adapt to the shifted distribution, leading to better performance. This trend is observed in Table \ref{tab:apptable2}. However, these performance gains with higher $c$ values come at the cost of training time.

\begin{table}[!ht]
\begin{center}
\caption{Reconstruction Loss, MAPE and FID values for ExGAN for different $c$ and $k$ (averaged over $5$ runs).}
\begin{tabular}{ccccc}
\toprule

$c$&$k$&{\bfseries Rec. Loss}&{\bfseries MAPE}&{\bfseries FID}\\

\midrule

$0.24$ & $2$ & $0.0173$ & $3.43 \pm 3.01$ & $0.0367 \pm 0.0096$\\
$0.49$ & $4$ & $0.0173$ & $3.32 \pm 3.10$ & $0.0304 \pm 0.0109$\\
$0.75$ & $10$ & $0.0172$ & $3.14 \pm 3.08$ & $0.0236 \pm 0.0037$\\
$0.90$ & $27$ & $0.0169$ & $3.05 \pm 3.14$ & $0.0223 \pm 0.0121$\\

\bottomrule
\end{tabular}
\label{tab:apptable2}
\end{center}
\end{table}

\FloatBarrier

\subsection{Speed} \label{sec:speed}

The time taken to generate $100$ samples for different extremeness probabilities is reported in Table \ref{tab:times}. Note that ExGAN is scalable and generates extreme samples in constant time as opposed to the $\mathcal{O}(\frac{1}{\tau})$ time taken by DCGAN to generate samples with extremeness probability $\tau$. DCGAN could not generate even one sample for extremeness probabilities $\tau = 0.001$ and $\tau = 0.0001$ in $1$ hour. Hence, we do not report sampling times on DCGAN for these two values.

\begin{table}[ht!]
\centering
\caption{Sampling times for DCGAN and ExGAN for different extremeness probabilities (in seconds).}
\begin{center}
\begin{tabular}{ccccc}
\toprule
\multirow{2}{*}{\bfseries Method} & \multicolumn{4}{c}{\bfseries Extremeness Probability ($\tau)$}\\
  &$0.05$  &$0.01$ &$0.001$ &$0.0001$\\  \midrule
\textbf{DCGAN}    & $1.230s$ & $7.564s$ & $-$ & $-$\\
\textbf{ExGAN}   & $0.002s$ & $0.002s$ & $0.002s$ & $0.002s$\\
\bottomrule
\end{tabular}
\label{tab:times}
\end{center}
\end{table}

\FloatBarrier

\subsection{Ablation Results}
\label{app:5exgan}
To evaluate the advantage of distribution shifting, we construct a model with an architecture similar to ExGAN but trained over all images in the dataset, i.e. no Distribution Shifting has been applied. This model is then evaluated in the same manner as described in the chapter.

Without distribution shifting, the reconstruction loss remains almost the same as ExGAN ($0.0166$ compared to $0.0172$). However, we observe that the FID score increases significantly ($0.0493 \pm 0.0097$ compared to $0.0236 \pm 0.0037$), showing the need for distribution shifting.

\FloatBarrier

\section{Ethical Impact}
Modelling extreme events in order to evaluate and mitigate their risk is a fundamental goal in a wide range of applications, such as extreme weather events, financial crashes, and managing unexpectedly high demand for online services. Our method aims to generate realistic and extreme samples at any user-specified probability level, for the purpose of planning against extreme scenarios, as well as stress-testing existing systems. Our work also relates to the goal of designing robust and reliable algorithms for safety-critical applications such as medical applications, aircraft control, and many others, by exploring how we can understand and generate the extremes of a distribution.

Our work explores the use of deep generative models for generating realistic extreme samples, toward the goal of building robust and reliable systems. Possible negative impact can arise if these samples are not truly representative or realistic enough, or do not cover a comprehensive range of possible extreme cases. Hence, more research is needed, such as for ensuring certifiability or verifiability, as well as evaluating the practical reliability of our approach for stress-testing in a wider range of real-world settings.

\FloatBarrier

\section{Conclusion}
In this chapter, we propose ExGAN, a novel deep learning-based approach for generating extreme data. We use (a) distribution shifting to mitigate the lack of training data in the extreme tails of the data distribution; (b) EVT-based conditional generation to generate data at any given extremeness probability.

We demonstrate how our approach is able to generate extreme samples in constant-time (with respect to the extremeness probability $\tau$), as opposed to the $\mathcal{O}(\frac{1}{\tau})$ time taken by the baseline. Our experimental results show that ExGAN generates realistic samples based on both visual inspection and quantitative metrics, and is faster than the baseline approach by at least three orders of magnitude for extremeness probability of $0.01$ and beyond.

The flexibility and realism achieved by the inclusion of GANs, however, come at the cost of theoretical guarantees. While our algorithmic steps (e.g. Distribution Shifting) are designed to approximate the tails of the original distribution in a principled way, it is difficult to provide guarantees due to its GAN framework. Future work could consider different model families (e.g. Bayesian models), toward the goal of deriving theoretical guarantees, as well as incorporating neural network based function approximators to learn a suitable extremeness measure ($\mathsf{E}$).

\SetPicSubDir{SESS}
\SetExpSubDir{SESS}

\chapter[Semi-Supervised Anomaly Detection via Sketches][SESS]{Semi-Supervised Anomaly Detection via Sketches}
\label{ch:sess}

\begin{mdframed}[backgroundcolor=magenta!20] 
Chapter based on work that is currently under submission.
\end{mdframed}

\section{Introduction}
In this chapter, we initiate the study of semi-supervision of sketch-based anomaly detection algorithms.
Anomaly detection of aggregate objects (e.g., graphs) where the input is a sequence of simple information (edges) has received increased attention in recent years, for example,
 \midas\ \cite{bhatia2020midas} and \spotlight\ \cite{eswaran2018spotlight}. \midas\ detects edge anomalies in real-time data streams, where the input is a sequence of edges and the goal is to detect anomalous edges (defined via bursty behavior) based on sketches (embeddings) of the input graph. In \spotlight\ the goal is to discover anomalous graphs (defined by a collection of observed input edges in an interval) that have dense bi-cliques for sub-intervals of time corresponding to a bursty behavior.
We note that anomaly detection is a
multifaceted problem \cite{chandola2009anomaly,akoglu2015graph} and a detailed treatment of that topic is beyond the scope of this
manuscript. However, the two mentioned applications correspond to
anomaly detection performed via sketches of the dynamic input data
stream. In most
anomaly detection scenarios, typically few ground truth labels are
available -- making anomaly detection an enticing application of
semi-supervision. In this chapter, we address the question of augmenting such sketch-based anomaly detection with semi-supervision and show a few surprising results.

Semi-supervision is a celebrated principle in machine learning that often improves the performance of models in scenarios where large corpora of labeled data are difficult to find, and a rich and impactful literature exists on this topic e.g. \cite{zhou2004learning,li2016graph}. It has been well established that for static data analysis, the availability of a few labeled examples can greatly improve performance in many settings.  However, straightforward off-the-shelf applications of standard semi-supervision do not often complete execution within reasonable time limits for streaming data.
The dynamic streaming aspect has received little
attention in semi-supervision with a few notable recent exceptions
\cite{wagner2018semi,siddiqui2018feedback,zheng2019addgraph,zhu2020semi} -- however, even for these applications, the type of objects seen in a stream and the objects for which semi-supervised feedback is provided are identical. A consequence of this uniformity of feedback is that more feedback (over randomly chosen subsets, which is non-adversarial) is almost always beneficial. However, we show that for certain extremely simple two-state bursty streams, with full observation, the performance of an ``optimum'' algorithm given an inexact statistical estimate of the stream, can decrease with increased feedback (again over non-adversarial/random subsequences). Such a modular operation, assuming incomplete knowledge and stepwise optimization, is typical in many off-the-shelf learning approaches -- but any assumption that a streaming algorithm over a large number of edges and nodes has reasonably accurate statistics seems to be inapplicable. The fact that stream characteristics remain stable over such a large number of observations may simply not be true.
 This immediately demonstrates that semi-supervision over streaming data creates a tension between learning data characteristics and learning a decision boundary. Indeed, as a contrast point, approximate algorithms based on Thomson sampling can be formally proven to not exhibit this behavior in that same two-state scenario.

The above observation alone is sufficient to mandate more investigation of semi-supervision of algorithms that use sketches or implicit parameter estimation as substeps. However, anomaly detection presents a yet more fascinating surprise: determining whether an observation is an anomaly can be easier than determining that a point is not an anomaly. In other words, anomalies are often self-evident. This implies that the semi-supervised feedback may be one-sided, or that propensity of label errors can skew in one direction. We show that in such a case, for the same two-state bursty stream, the performance of both the optimum and the Thompson Sampling algorithm (both algorithms being given the correct underlying statistic) decreases with more feedback! If two state systems can create such an unusual phenomenon, it stands to reason that semi-supervised graph anomaly detection over a large number of nodes requires significantly more investigation.  This problem is typically seen in 
learning algorithms that also have to decide on which points get feedback -- 
inappropriate operations on feedback can relegate a learning algorithm to be stuck in a bad region of the decision space. In a streaming context, that same phenomenon arises from the decision of ``which points to forget'' which may impact the relevancy of subsequent feedback even if the feedback was provided via agnostic non-adversarial random sampling.

At the same time, both for this simple system as well as for real data, this chapter shows that algorithms can be designed to achieve significant benefits with semi-supervision. One avenue of this improvement is the use of sketches. Sketches are not just useful embeddings but also operational data structures that approximately summarize and aggregate a data stream. Most sketches have a natural notion of an update algorithm corresponding to an update of the input. Such an update has an obvious parallel in semi-supervision where labels are updated. The aforementioned tension of learning an accurate distribution and an accurate decision boundary can be expressed as a single joint problem in a sketching setting. 
Streaming algorithms typically decide to ``forget'' elements in the stream and the semi-supervised feedback can help a streaming algorithm decide better on which pieces of information it chooses to forget. One can notice the parallel of such a process with the celebrated multiplicative weight update algorithm where expert feedback helps us find optimum solutions to convex optimization problems -- we show that a similar style of algorithm can greatly improve semi-supervised graph anomaly detection algorithms (\midas,\spotlight) in contrast to state-of-the-art streaming semi-supervisions algorithms such as \cite{wagner2018semi} which implement streaming label propagation. We note that the issue of providing feedback to objects different from the objects seen in the stream need no longer be important in a sketched/embedded representation because all objects are inexact. Such generalizations of semi-supervision are achieved automatically.

\looseness=-1
Note that semi-supervision is broadly connected to information acquisition in constrained systems. The example of the two-state process is a common Partially Observable Markov Decision Process (POMDP) used in wireless routing \cite{Zhao2008OnMS} and stochastic control  
\cite{Ny2008MultiUAVDR} literature. Often these systems exemplify restless bandits \cite{Kaelbling1998PlanningAA,Guha2010ApproximationAF} and standard techniques of information acquisition such as bandit problems, for example \cite{Abernethy2016ThresholdBW}, do not apply. Moreover, in the context of anomaly detection, the restless bandit setup corresponds to a two-arm, unbalanced classification setting and sketch variants of probability matching/Thompson Sampling \cite{Kuzborskij2019EfficientLB},
do not apply.

{\bfseries Contributions:} To summarize, while sketching techniques (a) often preserve unknown manifolds defined by dynamic data, (b) are amenable to easy
updates, (c) are defined for structured objects such as graphs, and
(d) can be harnessed to provide anytime semi-supervised algorithms  -- care is required to apply these ideas and off-the-shelf methods may not be a fit. In this work, we first investigate a conceptual
system model in Section~\ref{sec:conceptual} where anomalies are
bursty. We ignore all connections to graph data -- and observe a
single edge in isolation. Using the intuition of counting based
summary of edges, we then switch gears to graphs and investigate a
combination of sketching and semi-supervised learning. We show that
state-of-the-art streaming graph anomaly detection algorithms like
\midas\ \cite{bhatia2020midas} and \spotlight\
\cite{eswaran2018spotlight} which rely on count-based sketches can be
improved significantly with semi-supervision, using real-life public
datasets.

In the context of \midas, we
propose \sess\ that significantly improves upon \midas\ when
the classes are imbalanced (as is the case in anomaly detection)
without sacrificing the inherent efficiency of \midas. We then propose \sess-3D  which is capable of incorporating node feedback and improves upon the processing by being cache-aware and using higher-order sketches.
The performance of these algorithms is significantly better (in
accuracy and computational efficiency) than using state-of-the-art streaming semi-supervision algorithms such as
\cite{wagner2018semi}. Note that non-streaming semi-supervised
algorithms such as those based on label propagation
\cite{zhu2003semi} do not finish on these large datasets in a reasonable time.

In the context of \spotlight, we note that its performance can be
improved in the weakly semi-supervised setting where only edge
feedback is available. This provides a realistic example of weakly
correlated feedback because the presence of dense bicliques is only
weakly correlated with the provided feedback over edges.

While there has been work on sketch-based
classification problems \cite{Talukdar2014ScalingGS} for individual input points, we are not sure
how that applies to semi-supervision over a graph defined by the input points.
To the best of our knowledge, this direction of exploring sketching algorithms for
semi-supervised unbalanced classification has not been considered heretofore.

\section{Related Work}

Streaming or online algorithms vary significantly from their static counterparts in terms of space and time management strategies due to the strict restrictions posed by the streaming nature of data. Many algorithms \cite{bateni2018optimal,Hao2020SparseAL,Huang2018NearOF,Menon2007AnID,Braverman2016BeatingCF,Coleman2019RACESM,Bury2018SketchA,Cohen2007SketchingUD,Tai2018SketchingLC,Shi2020HigherOrderCS} make use of data sketches to maintain item-counts owing to their compact structure and yet bounded error estimates. Sketches have also been used for faster anomaly detection \cite{Abry2007InvitedTS,Gopalan2018FasterAD,kumagai2020SSAD_on_graphs,bruschi2020discovering,Chen2021MakingOS,Li2006DetectionAI,zhang2020augsplicing}. See \cite{Mcgregor2014GraphSA} for an extensive survey.

Semi-supervised algorithms have been explored in various domains like vision \cite{Li2020SemisupervisedCW}, text \cite{Yang2017BridgingCF} and graph data \cite{esfandiari2018streaming,morvan2018graph,wan2020contrastive,Song2021GraphbasedSL}. Although the specific form of application manifests differently according to the domain and its constraints, at its core, semi-supervision ideas are realized in three different categories.

\begin{enumerate}
    \item \textbf{Consistency Regularization:} Unlabeled data produce perturbed, unlabeled input samples relying on the assumption that the model should output similar predictions. Generative Modeling is a famous technique to generate perturbed data samples for consistency regularization \cite{meng2019semi}. \cite{kingma2014semi} bootstraps the dataset by predicting the labels of the unlabeled data points by using a generative model. \cite{alberti2019synthetic} introduces noise in unlabeled data samples to increase performance. \cite{sohn2020fixmatch} also generates pseudo labels using the model's predictions on weakly augmented images. \cite{berthelot2019mixmatch} guesses low-entropy labels for data-augmented unlabeled examples and mixes labeled and unlabeled data using a sharpening function. Many problems in the field of Natural Language Processing have found a semi-supervision learning based solution e.g. \cite{cer2018universal_sentence_encodings}.

    \item \textbf{Entropy Minimization:} The core idea of Entropy Minimization is that the decision boundary of the classifier should not pass through high density regions of the data space. As predictions near the decision boundary are more uncertain, entropy minimization seeks to make the model more confident in its predictions by moving the boundary away from the data. \cite{grandvalet2005entropy_min} introduces a loss function to learn the model parameters by minimizing entropy in the prediction, additionally to the supervised loss. \cite{berthelot2019mixmatch} reduces the entropy by employing a sharpening function.
    
    \item \textbf{Graph-Based:} \cite{zhu2003semi,zhou2004learning,li2016graph,liu2019deep} have a long history of work and propagate limited label information to unlabeled examples following clustering or manifold assumptions. By taking advantage of the progress of deep learning including graph neural networks and graph convolutional networks, these methods have achieved state-of-the-art results on various semi-supervised node classification tasks \cite{kipf2017semi,yang2016revisiting,rong2019dropedge,xu2020graph,luo2018smooth,iscen2019label,jia2019graph}. However, these methods do not assume class imbalance and thus are not immediately applicable in anomaly detection.
\end{enumerate}

Semi-supervision has either directly been used or can be modified for Anomaly Detection in \cite{gornitz2013toward,wu2018imverde,zhou2018sparc,Ruff2020DeepSA,amid2015kernel,feng2019beyond,JU2020167,Liu2016DetectingAI}. However, all of these approaches cannot be used in a streaming setting. There is some active learning related work used for anomaly detection including \cite{siddiqui2018feedback}, but it is also not clear how to use these in a streaming manner. PENminer \cite{belth2020mining} detects burst anomalies, however, it does not consider semi-supervision. Online graph-based semi-supervision has generated considerable interest recently \cite{zhu2009some,Valko2010OnlineSL,ravi2016large,zheng2019addgraph,zhu2020semi,bera2012advanced,Hu2019HierarchicalGC}, but the processing time and memory are still proportional to the stream length.

Closest in spirit to our work is \cite{wagner2018semi}, which runs semi-supervision on streams with sub-linear memory. We show how \sess\ is significantly better and runs in real-time while requiring constant space.

\section{A Conceptual System}
\label{sec:conceptual}

In this section, we show that for an optimal algorithm with incorrectly estimated parameters, the errors may increase with increasing feedback. We also show that for probability matching/Thompson Sampling type methods, with incorrectly estimated parameters, the errors decrease with increasing feedback (which is the desirable phenomenon) for two-sided feedback (all classes being observable). And yet, the error may dramatically increase in one-sided observations. Finally, we show that the benefit of semi-supervision is significantly higher in unbalanced settings in comparison to balanced settings.

The specific system is a POMDP, widely studied in wireless scheduling \cite{Bertsimas2001PerformanceOM,Zhao2008OnMS} and unmanned aerial vehicle (UAV) routing \cite{Ny2008MultiUAVDR}.

As shown in Figure \ref{fig:statefigure}, consider a machine $T$, with transition probabilities
$P[N\rightarrow A]=p$, $P[A \rightarrow N] = q$, $P[A \rightarrow A] =
1-q$ and $P[N \rightarrow N] = 1-p$.

\begin{figure}[!ht]
        \center{\includegraphics[scale=0.9]
        {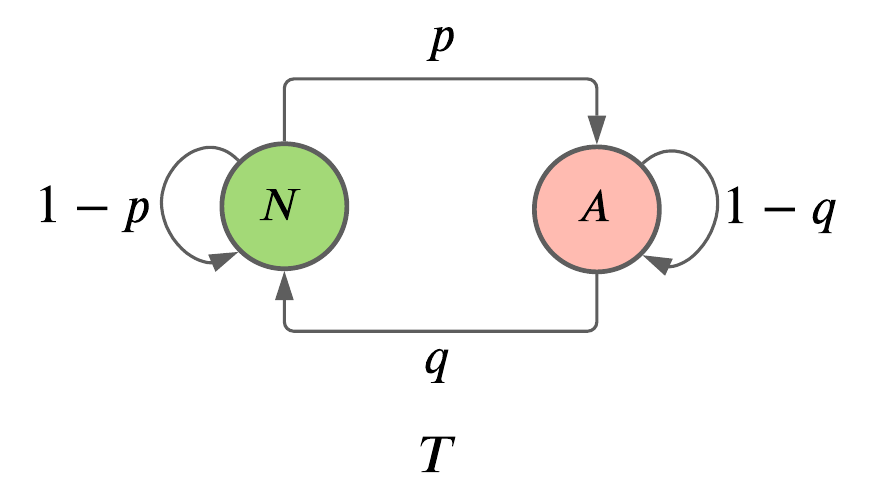}}
        \caption{\label{fig:statefigure} A machine $T$ with transition probabilities $p$ and $q$ between normal and anomalous states.}
\end{figure}

$A$ corresponds to an anomalous
state and $N$ corresponds to a normal state. For the transition
probabilities $p,q \ll 1$ and $p+q <1$, the states are sticky and anomalies/normal
points exhibit a bursty behavior. The expected long-run probability of
observing state $A$ is $p/(p+q)$, and an assumption of $20 p < q$ is suitable for anomaly detection application;
corresponding to $\approx 95\%$ normal (N) observations. The goal of an algorithm is to produce predictions $\{ A,N\}$ for each
time sequence -- while observing the true state (feedback) of $T$ for a few
select time steps. We introduce the following definitions:

\begin{definition}
 If the algorithm is allowed to inspect the true state of $T$
 irrespective of its own prediction or the true labels, then define the feedback to be {\bfseries two-sided}.
\end{definition}

\begin{definition}
 If the algorithm can only inspect the true state of $T$
   when $T$ is in state $A$, then define the feedback to be {\bfseries one-sided}.
  \end{definition}

The two-sided scenario is most typical and captures the experimental
measurement when a random subset of true labels are
provided in an online manner to a streaming algorithm. Note that one-sided feedback is easier to measure. Many other definitions of sidedness may exist -- based on specifics of the application, which we omit in this presentation. However, the above notions are the most natural in the context of an algorithm seeking feedback.

 \subsection{Two Illustrative Algorithms}
Consider a simple probability matching type algorithm {\textsc{ Imitate}}: Suppose the
algorithm has an estimate $\hat{p},\hat{q}$ for the true parameters
$p,q$. It uses the parameters to predict $A$/$N$ independently
of the true process; except that on receiving feedback, it resets to the state provided in the feedback.

\smallskip
Consider an optimal algorithm {\textsc{ Opt}}, that has no foreknowledge of
future feedback, with estimations $\hat{p},\hat{q}$ for the true parameters
$p,q$. First note that:

\begin{theorem}
\label{firsttheorem}
  Suppose the locations of the feedback were chosen independently of
  {\textsc{ Opt}}, and {\textsc{ Opt}} has no knowledge when the next feedback
  would arrive.
  If the last feedback was $N$, Algorithm {\textsc{ Opt}} continues to predict $N$ till the next feedback. If the last feedback was $A$,
  then {\textsc{ Opt}} predicts $A$ for a fixed number of steps $L$ (to be determined) and switches to predicting $N$. 
  \end{theorem}
  
  \begin{proof}
The first part of the proof follows from the fact that conditioned on
last observing $N$, the probability that $T$ is in $N$ is higher than
$T$ being in state $A$, since $q>p$. If the optimum algorithm
predicted $A$ for a particular time step, then it could predict $N$
(keeping every other prediction the same) and improve its mistake
bound in expectation.

If the last seen state was $A$, the algorithm {\textsc{ Opt}} should (1)
eventually start predicting $N$ and (2) once it starts predicting $N$,
it should continue to predict $N$. To observe (1), note that the long-run probability of $T$ being in $N$ is $q/(p+q)$ which is higher than
being in $A$. Moreover, $T$ is expected to transition to $N$ after an
expected $1/q$ number of steps. Even though {\textsc{ Opt}} may not know
$q$, not switching to $N$ after a long period of time is clearly
suboptimal. For (2), observe that if {\textsc{ Opt}} predicts an $A$
following an $N$; then switching the order of those two predictions
(keeping other predictions the same) improves the expected mistake
bound since the probability of observing $A$ decreases monotonically
with time $t$ (a consequence of $1> p+q$ and $p<q$).  Therefore, not
knowing when the next feedback would arrive, {\textsc{ Opt}}'s strategy
would correspond to a distribution over steps it waits at $A$ before
switching to $N$. Since the time steps are discrete, one of those time
steps would provide a minimum number of mistakes. That number of steps
determines $L$. 
 \end{proof}

We note that, given the knowledge {\textsc{ Opt}} has, its best action
corresponds to $L = 1/\hat{q}$. We now discuss the difference between the two
algorithms {\textsc{ Imitate}} and {\textsc{ Opt}}.
We make the simplistic assumption that the locations
of the feedback are chosen at random (agnostic of both algorithms).
We begin with the following theorems.

\begin{theorem}
  \label{imitatetheorem}
For a fixed stream length, the number of mistakes made by {\textsc{
  Imitate}} decreases with increasing feedback, for locations
chosen randomly.
\end{theorem}

\begin{proof}
We first observe that for any fixed chunk length (between two
feedbacks) where there has been
no feedback; the number of mistakes cannot decrease if the chunk
length $C$ increases by $1$. This is best seen by a coupling where a
sample path (corresponding to the transcript of states of both {\textsc{
  Imitate}} and the true process $T$) of length $C$, is increased by $1$.
The mistake bound holds for each sample path. By induction, this
extends to any $C' >C$. As the number of feedback  increases,
the increased feedback corresponds to a distribution of lengths ${\mathcal D}(C)$
which is stochastically dominated by a distribution ${\mathcal
  D}(C')$.  The theorem follows.
\end{proof}

\begin{theorem}
  \label{opttheorem}
For a fixed stream length and randomly chosen feedback location, the number of mistakes made by {\textsc{
 Opt}} can increase when $\hat{q} > q^2/(p+q)$ especially when the
fraction of feedback $\phi \rightarrow 0$. Note that since $p <
q$, this corresponds to a small overestimation of $q$; and a small
underestimation in the number of anomalies.
\end{theorem}
\paragraph{One-sided Feedback}
\label{app:proof3}

\begin{proof}
Prediction machine $P$ predicts state $A$ for $L$ steps and then returns to state $N$ until feedback is provided. $P$ is given a true label of $1$ with probability $\phi$. This process of receiving feedback forms a geometric distribution and the expected number of timesteps between two true labels is $\frac{1}{\phi}$.

If $L > \frac{1}{\phi}$, $P$ always remains in state $A$ no matter how much feedback is provided. Therefore expected accuracy is the probability that $T$ is in state $A$ which is $\frac{p}{p+q}$. Let us now consider $L < \frac{1}{\phi}$. As shown in Figure \ref{fig:optproof1}, the algorithm accuracy in such a block can be calculated in two parts: when $P$ is in state $A$ for $L$ steps and when $P$ returns to state $N$ after $L$ steps and remains in state $N$ for $K$ steps.

\begin{figure}[!ht]
        \center{\includegraphics[width=0.7\columnwidth]
        {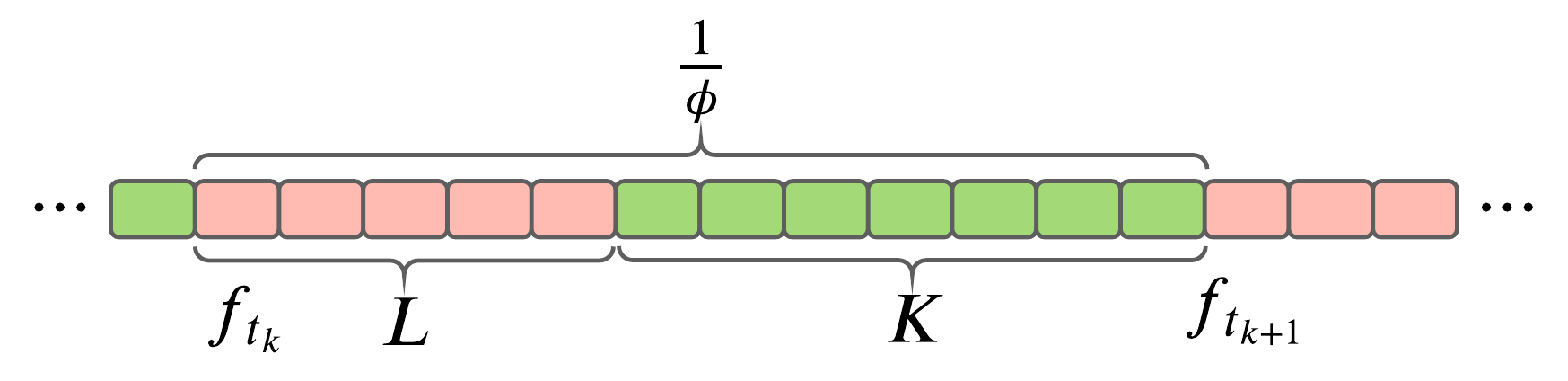}}
        \caption{\label{fig:optproof1} Analysing {\textsc {Opt}} for one partition when $L < \frac{1}{\phi}$.}
\end{figure}

(1) When true label $1$ is provided to $P$, $T$ is in state $A$ and it is expected that it will remain in state $A$ for $\frac{1}{q}$ expected steps before going to $N$. This is because the mean first passage time \cite{levin2017markov} for the state $N$ is $\frac{1}{Pr(A \rightarrow N)} = \frac{1}{q}$. The expected number of times $P$'s prediction matches with $T$ is $min\left\{L, \frac{1}{q}\right\}$ given $L < \frac{1}{p}$.
    
(2) When $P$ comes back to state $N$ (after completing its stay in state $A$ for $L$ steps), it outputs $0$ in every timestep until new feedback arrives. It will stay in state $N$ for an estimated $K = \frac{1}{\phi} - L$ steps. $P$ correctly predicts $0$ for an estimated $K*Pr[s=N] = K*\frac{q}{p+q}$ number of times.

Thus the accuracy of $P$, as a function of $L$ and $\phi$ is 

\[   
Acc(L,\phi) = 
     \begin{cases}
        \frac{K\cdot \frac{q}{p+q} + L}{K+L} &\quad\text{if } L \le \frac{1}{q}\\
        \frac{K\cdot \frac{q}{p+q} + \frac{1}{q}}{K+L} &\quad\text{if } \frac{1}{q} \le L < \frac{1}{p}\\ 
     \end{cases}
\]

Note that $K + L = \frac{1}{\phi}$. Simplifying the expressions, we get

\[   
Acc(L,\phi) = 
     \begin{cases}
        \frac{q}{p+q} + \frac{\phi L p}{p+q} &\quad\text{if } L \le \frac{1}{q}\\
        \frac{q}{p+q} + \phi\left[\frac{1}{q} - \frac{qL}{p+q}\right] &\quad\text{if } \frac{1}{q} \le L < \frac{1}{p}\\ 
     \end{cases}
\]

We can thus infer that the accuracy in one partition reaches its maximum at the expected length $L = \frac{1}{q}$. Furthermore, if $L > \frac{1+\frac{p}{q}}{q}$, then as feedback ($\phi$) increases, the accuracy decreases as the coefficient of $\phi$ $\left(\frac{1}{q} - \frac{qL}{p+q} \right)$ becomes negative. Note that when no feedback is provided ($\phi=0$), accuracy is $\frac{q}{p+q}$.
\end{proof}

\paragraph{Two-sided Feedback}

\begin{proof}
Suppose the algorithm received true label $1$ with probability $r$ and true label $0$ with probability $1-r$ conditioned on being given feedback. The probability of giving feedback is $\phi$. Let us first consider $L < \frac{1}{\phi}$.

(1) When $L < \frac{1}{q}$, we calculate the final accuracy of $P$ by considering two cases:
(1) $f_{t_{k}} = 1$ (2) $f_{t_{k}} = 0$.
    For the first case, the accuracy was already calculated in Appendix \ref{app:proof3} to be  $\frac{q}{p+q} + \frac{\phi L p}{p+q}$. For the second case when $f_{t_{k}} = 0$, we can calculate the final accuracy as $Pr[s = N] = \frac{q}{p+q}$. Since we are receiving $1$s with probability $r$ and $0$s with probability $1-r$, the combined accuracy of $P$ becomes:
    
    \begin{align*}
        &r*\left[\frac{q}{p+q} + \frac{\phi L p}{p+q}\right] + (1-r)*\frac{q}{p+q} = \frac{q}{p+q} + \frac{\phi L p r}{p+q}
    \end{align*}
    
(2) When $\frac{1}{q} < L < \frac{1}{p}$, we break down the calculation of final accuracy of $P$ by considering two cases: (1) $f_{t_{k}} = 1$ (2) $f_{t_{k}} = 0$. Similar to when $L < \frac{1}{q}$, we refer to Appendix \ref{app:proof3} to get the accuracy of $P$ when $f_{t_{k}} = 1$ as $\frac{q}{p+q} + \phi\left[\frac{1}{q} - \frac{qL}{p+q}\right]$. When $f_{t_{k}} = 0$, we calculate the final accuracy as $Pr[s = N] = \frac{q}{p+q}$. Summing up the two cases by considering the probability of their occurrences, we get the final accuracy of $P$ as $\frac{q}{p+q} + r \phi \left[\frac{1}{q} - \frac{qL}{p+q}\right]$.

When $L>\frac{1}{\phi}$, if $P$ is in state $A$ having received a positive true label previously, it will get interrupted by the next stream feedback even before it finishes its term in the $A$ state. This means that we can analyze the performance of $P$ as before but by just replacing $L$ with $\frac{1}{\phi}$. Thus the performance of $P$ will be independent of $L$.

Thus the accuracy of $P$ will be $\frac{q}{p+q} + \frac{pr}{p+q}$ when $\frac{1}{\phi} < \frac{1}{q}$ and $\frac{q}{p+q} + r\phi \left[ \frac{1}{q} - \frac{q}{\phi(p+q)}\right]$ when $\frac{1}{q} < \frac{1}{\phi} < \frac{1}{p}$.

Now we derive the performance of $P$ for a general $L$. Assuming that $P$ gets positive feedback, $T$ starts in state $A$. We conduct an expected case analysis where $T$ remains in state $A$ for $\frac{1}{q}$ number of steps, then it switches to state $N$ and remains there for an expected $\frac{1}{p}$ number of steps and so on. Since $P$ is predicting $1$ all throughout, the calculation of its accuracy just means that we calculate the number of times $T$ stays in state $A$. We break down the analysis into two cases:

(1) When $T$ ends in state $A$ - There are multiple ways in which $T$ can end up in state $A$, when $L < \frac{1}{q}$, $\frac{1}{p} + \frac{1}{q}< L < \frac{1}{p} + \frac{2}{q}$, $\frac{2}{p} + \frac{2}{q}< L < \frac{2}{p} + \frac{3}{q}$ and so on. Let $T_A = \left \lfloor \frac{1}{q} \right \rfloor$ and $T_N = \left \lfloor \frac{1}{p} \right \rfloor$. For a general $L$, $T$ will remain in state $A$ for an expected $\frac{1}{q} \left \lfloor \frac{L}{T_A + T_N}  \right \rfloor + (L \mod T_A)$ steps.

(2) When $T$ ends in state $N$ - As before there are multiple ways in which $T$ can end up in state $N$, when $\frac{1}{q} < L < \frac{1}{p}$, $\frac{1}{p} + \frac{1}{q}< L < \frac{2}{p} + \frac{1}{q}$, $\frac{2}{p} + \frac{2}{q}< L < \frac{3}{p} + \frac{2}{q}$ and so on. For a general $L$, $T$ will remain in state $A$ for an expected $\frac{1}{q} + \frac{1}{q}\left \lfloor \frac{L}{T_A + T_N}  \right \rfloor$ steps.

Note that $\lfloor x \rfloor \leq x$ and $y \mod x < x$ for any $x,y \in \mathcal{Z^{+}}$. We apply these inequalities to the above expressions. If $L \mod (T_A + T_N) < T_A$ i.e. $T$ ends up in state $A$, then $Acc_L < \frac{1}{L} \left[ \frac{pL}{p+q} + \frac{1}{q} \right]$ and if $L \mod (T_A + T_N) > T_A$ i.e. $T$ ends up in state $N$, then $Acc_L < \frac{1}{L} \left[ \frac{pL}{p+q} + \frac{1}{q} \right]$. Thus $Acc_L < \frac{p}{p+q} + \frac{1}{Lq}$ for all $L$.

\end{proof}
It is worth noting that in semi-supervised learning, feedback is
typically small i.e. $\phi \rightarrow 0$. It is not unexpected that phenomena such as Theorem~\ref{opttheorem} arise when the amount
of feedback is large. However, a non-monotone behavior at the initial
stages appears to be more problematic.
Further contrast Theorems~\ref{imitatetheorem}, ~\ref{opttheorem}
with the following observation:

\begin{observation}
In the absence of feedback, the accuracy (fraction of correct
prediction of $A$, $N$) of {\textsc{ Imitate}}
is $(\hat{p}p + \hat{q}q)/((\hat{p} +\hat{q})*(p+q))$, based on the mixing
probability of the two markov chains corresponding to the real and
imitated processes. An algorithm that always answers $N$ has accuracy
$q/(p+q)$. 
\end{observation}

It is surprising that the performance of a (supposedly) ``optimal''
algorithm decreases with feedback (at least initially) as shown in
Theorem~\ref{opttheorem} and demonstrated in 
Table \ref{tab:app2}.
This is due to the fact that the parameters $\hat{p},\hat{q}$ are estimated incorrectly but the ``optimal'' algorithm could not correct for that incorrect estimation. In contrast, as Theorem~\ref{imitatetheorem} and
Table \ref{tab:app4}
 shows, {\textsc{ Imitate}} does not have this
undesirable property and performs better with more feedback, even when
parameters are estimated inaccurately. At the same time, the
performance of the optimum algorithm for correctly estimated parameters
can be higher. 
For (randomized) one-sided
feedback, as 
Table ~\ref{tab:app1}
shows, the performance of the
optimum can degrade significantly; while the performance of {\textsc{
  Imitate}} does not 
  (Table \ref{tab:app3}).
  These observations seem to indicate that {\textsc{ Imitate}} is desirable.

\begin{table*}[!ht]
\centering
\caption{Optimum: Two-Sided Feedback. Stream size $1,000,000$, created with $p = 0.001$, $q = 0.02$ (Averaged over $10$ runs).}
\label{tab:app2}
\resizebox{\columnwidth}{!}{
\begin{tabular}{@{}rccccccccc@{}}
\toprule
$L$ & $\phi=0$ & $\phi=0.001$ & $\phi=0.002$ & $\phi=0.003$ & $\phi=0.004$ & $\phi=0.005$ & $\phi=0.01$ & $\phi=0.02$\\
\midrule
$10$ & $0.951$	& $0.952$ &	$0.952$ & $0.952$ & $0.953$ & $0.953$ & $0.955$ & $0.959$ \\
$20$ & $0.951$	& $0.952$ &	$0.952$ & $0.953$ & $0.954$ & $0.954$ & $0.957$ & $0.962$ \\
$30$ & $0.951$	& $0.952$ &	$0.952$ & $0.953$ & $0.954$ & $0.955$ & $0.958$ & $0.963$ \\
$40$ & $0.951$	& $0.952$ &	$0.952$ & $0.953$ & $0.954$ & $0.955$ & $0.958$ & $0.963$ \\
$50$ & $0.951$	& $0.952$ &	$0.952$ & $0.953$ & $0.954$ & $0.954$ & $0.957$ & $0.962$ \\
$60$ & $0.951$	& $0.952$ &	$0.952$ & $0.953$ & $0.954$ & $0.954$ & $0.957$ & $0.961$ \\
$70$ & $0.951$	& $0.951$ &	$0.951$ & $0.952$ & $0.953$ & $0.953$ & $0.956$ & $0.96$ \\
$100$ & $0.951$	& $0.951$ &	$0.951$ & $0.95$ & $0.951$ & $0.951$ & $0.952$ & $0.956$ \\
$150$ & $0.951$	& $0.949$ &	$0.949$ & $0.946$ & $0.946$ & $0.946$ & $0.946$ & $0.952$ \\
$200$ & $0.951$	& $0.947$ &	$0.947$ & $0.943$ & $0.943$ & $0.942$ & $0.942$ & $0.951$ \\
\bottomrule
\end{tabular}}
\end{table*}

\begin{table*}[ht]
\centering
\caption{Imitate: Two-Sided Feedback. Stream size $1,000,000$, created with $p = 0.001$, $q = 0.02$ (Averaged over $10$ runs).}
\label{tab:app4}
\resizebox{\columnwidth}{!}{
\begin{tabular}{@{}rcccccccccc@{}}
\toprule
$\hat{p}$ & $\hat{q}$ & $\phi=0$ & $\phi=0.001$ & $\phi=0.002$ & $\phi=0.003$ & $\phi=0.004$ & $\phi=0.005$ & $\phi=0.01$ & $\phi=0.02$\\
\midrule
$0.001$ & $0.01$ & $0.871$ & $0.874$ & $0.88$ & $0.886$ & $0.89$ & $0.894$ & $0.91$ & $0.928$ \\
$0.001$ & $0.02$ & $0.908$ & $0.911$ & $0.912$ & $0.914$ & $0.917$ & $0.918$ & $0.925$ & $0.937$ \\
$0.001$ & $0.03$ & $0.922$ & $0.923$ & $0.925$ & $0.926$ & $0.927$ & $0.928$ & $0.933$ & $0.942$ \\
$0.001$ & $0.04$ & $0.929$ & $0.93$ & $0.931$ & $0.932$ & $0.933$ & $0.934$ & $0.937$ & $0.945$ \\
$0.001$ & $0.05$ & $0.934$ & $0.935$ & $0.935$ & $0.935$ & $0.936$ & $0.937$ & $0.94$ & $0.946$ \\
$0.001$ & $0.1$ & $0.942$ & $0.943$ & $0.943$ & $0.943$ & $0.944$ & $0.944$ & $0.946$ & $0.949$ \\
\bottomrule
\end{tabular}}
\end{table*}

\begin{table*}[!ht]
\centering
\caption{Optimum: One-Sided Feedback. Stream size $1,000,000$, created with $p = 0.001$, $q = 0.02$ (Averaged over $10$ runs).}
\label{tab:app1}
\resizebox{\columnwidth}{!}{
\begin{tabular}{@{}rccccccccc@{}}
\toprule
$L$ & $\phi=0$ & $\phi=0.001$ & $\phi=0.002$ & $\phi=0.003$ & $\phi=0.004$ & $\phi=0.005$ & $\phi=0.01$ & $\phi=0.02$\\
\midrule
$10$ & $0.951$ & $0.952$ & $0.952$ & $0.952$ & $0.953$ & $0.953$ &	$0.955$ & $0.958$ \\
$20$ & $0.951$ & $0.952$ & $0.952$ & $0.953$ & $0.954$ & $0.954$ &	$0.957$ & $0.961$ \\
$30$ & $0.951$ & $0.952$ & $0.952$ & $0.953$ & $0.954$ & $0.954$ &	$0.957$ & $0.962$ \\
$40$ & $0.951$ & $0.952$ & $0.952$ & $0.953$ & $0.954$ & $0.955$ &	$0.957$ & $0.96$ \\
$50$ & $0.951$ & $0.952$ & $0.952$ & $0.953$ & $0.954$ & $0.954$ &	$0.956$ & $0.958$ \\
$60$ & $0.951$ & $0.952$ & $0.952$ & $0.952$ & $0.953$ & $0.953$ &	$0.955$ & $0.955$ \\
$70$ & $0.951$ & $0.951$ & $0.951$ & $0.952$ & $0.953$ & $0.953$ &	$0.953$ & $0.951$ \\
$100$ & $0.951$ & $0.951$ & $0.951$ & $0.949$ & $0.95$ & $0.949$ &	$0.946$ & $0.94$ \\
$150$ & $0.951$ & $0.949$ & $0.949$ & $0.944$ & $0.944$ & $0.942$ &	$0.933$ & $0.919$ \\
$200$ & $0.951$ & $0.947$ & $0.947$ & $0.94$ & $0.938$ & $0.934$ &	$0.92$ & $0.899$ \\
\bottomrule
\end{tabular}}
\end{table*}

\begin{table*}[ht]
\centering
\caption{Imitate: One-Sided Feedback. Stream size $1,000,000$, created with $p = 0.001$, $q = 0.02$ (Averaged over $10$ runs).}
\label{tab:app3}
\resizebox{\columnwidth}{!}{
\begin{tabular}{@{}rcccccccccc@{}}
\toprule
$\hat{p}$ & $\hat{q}$ & $\phi=0$ & $\phi=0.001$ & $\phi=0.002$ & $\phi=0.003$ & $\phi=0.004$ & $\phi=0.005$ & $\phi=0.01$ & $\phi=0.02$\\
\midrule
$0.001$ & $0.01$ & $0.87$ & $0.869$ & $0.867$ & $0.865$ & $0.868$ & $0.866$ & $0.861$ & $0.855$ \\
$0.001$ & $0.02$ & $0.908$ & $0.907$ & $0.908$ & $0.908$ & $0.909$ & $0.908$ & $0.909$ & $0.909$ \\
$0.001$ & $0.03$ & $0.922$ & $0.922$ & $0.923$ & $0.923$ & $0.923$ & $0.923$ & $0.925$ & $0.927$ \\
$0.001$ & $0.04$ & $0.929$ & $0.93$ & $0.93$ & $0.93$ & $0.931$ & $0.931$ & $0.933$ & $0.935$ \\
$0.001$ & $0.05$ & $0.934$ & $0.934$ & $0.934$ & $0.935$ & $0.935$ & $0.935$ & $0.937$ & $0.939$ \\
$0.001$ & $0.1$ & $0.942$ & $0.942$ & $0.943$ & $0.943$ & $0.943$ & $0.944$ & $0.945$ & $0.948$ \\
\bottomrule
\end{tabular}}
\end{table*}

\FloatBarrier

We investigate the accuracy of the two algorithms for a
particular run over a sequence of length $1M$, as more data is ingested by the
algorithms with correct estimates of parameters. Note that as shown in
Figure \ref{fig:optimum},
one-sided {\textsc{ Opt}} does much better than {\textsc{ Imitate}} shown in
Figure \ref{fig:imitate}.
However, there is a minimal improvement between one-sided and two-sided feedback even for feedback as large as $2\%$.

\begin{figure}[!htb]
        \center{\includegraphics[width=0.7\columnwidth]
        {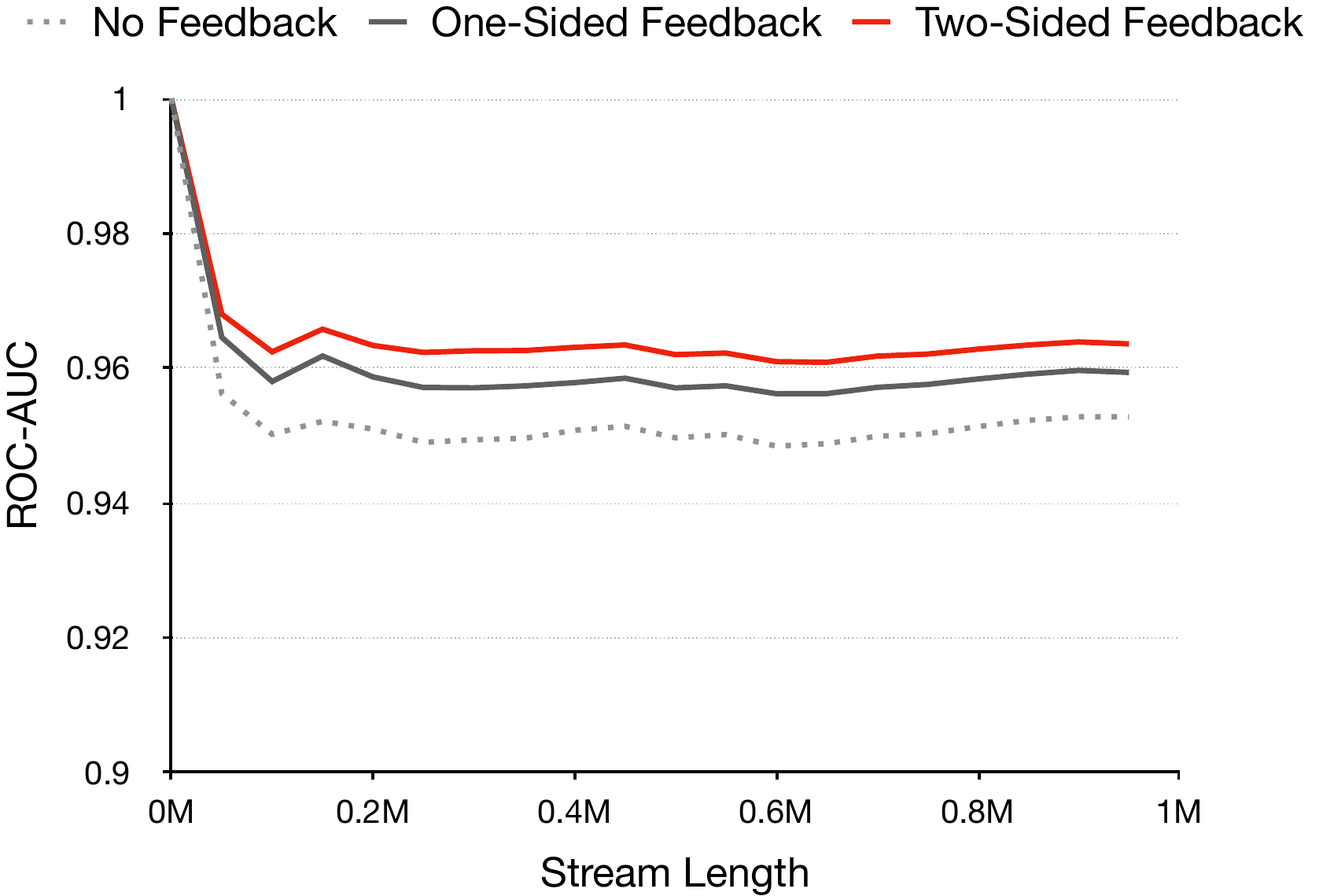}}
        \caption{\label{fig:optimum} Accuracy profile of Optimum.}
\end{figure}

\begin{figure}[!htb]
        \center{\includegraphics[width=0.7\columnwidth]
        {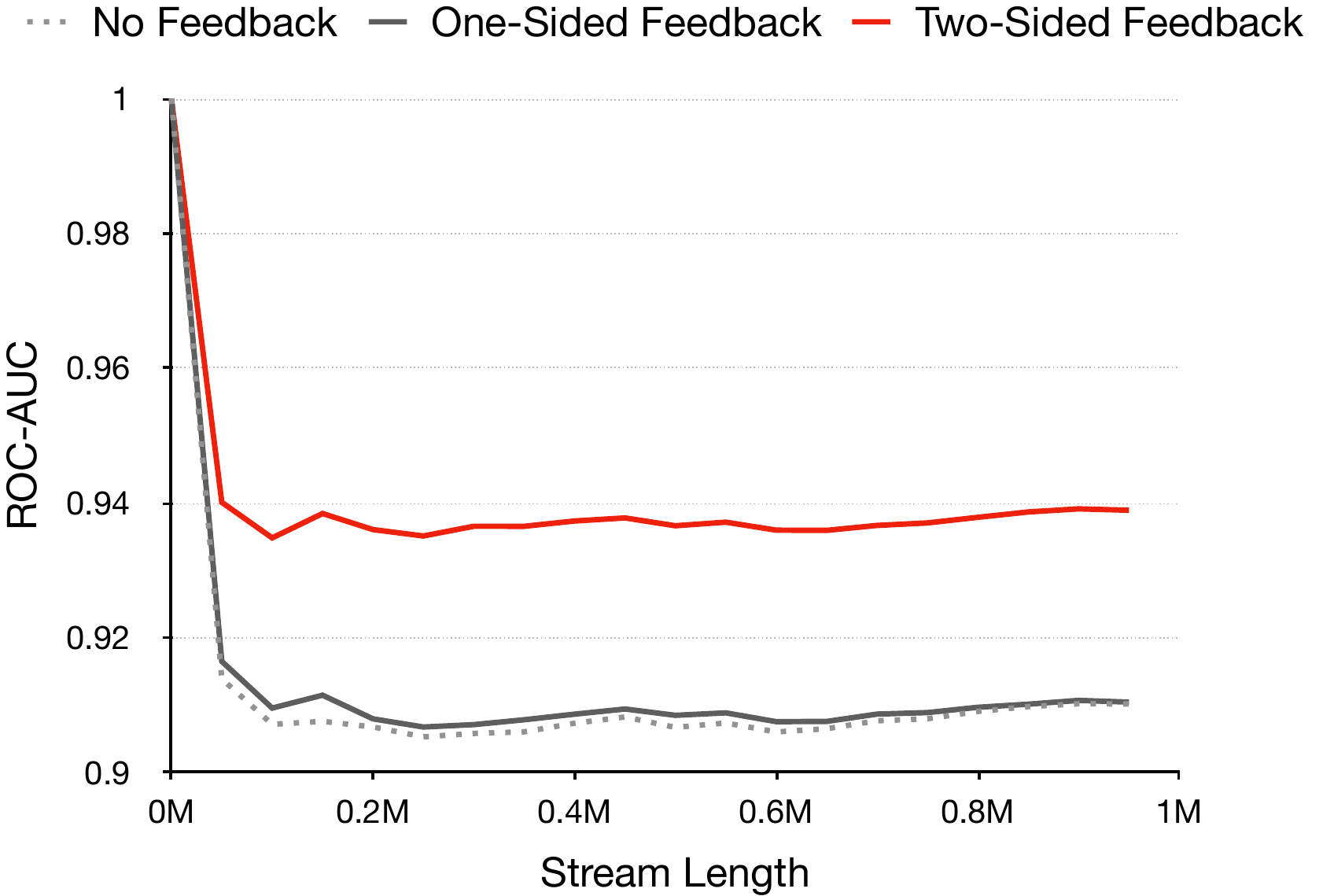}}
        \caption{\label{fig:imitate} Accuracy profile of Imitate.}
\end{figure}

Next, in Table \ref{tab:incorrect}, we study the effect of incorrectly estimated parameters. Unlike with one-sided (only anomalous) feedback, both {\textsc{ Imitate}} and {\textsc{ Opt}} have the ability to correct themselves with two-sided (both normal and anomalous) feedback. Note that as we move further away from the correct parameter ($\hat{q} = 0.02$ for {\textsc {Imitate}} and $L=50$ for {\textsc {Opt}}), the difference between one-sided and two-sided feedback becomes increasingly large.

\begin{table}[!htb]
\centering
\caption{Performance with one-sided and two-sided feedback using incorrect parameters. Stream size $1,000,000$, created with $p = 0.001$, $q = 0.02$. $2\%$ feedback (Averaged over $10$ runs).}
\label{tab:incorrect}
\begin{tabular}{@{}crrr@{}}
\toprule
Algorithm & {$\hat{q}$} & {One-sided} & {Two-sided} \\ \midrule
\multirow{4}{*}{\textsc{ Imitate}} & $0.02$  & $0.911$ & $0.939$ \\
& $0.002$ & $0.581$ & $0.917$ \\
& $0.0002$ & $0.157$ & $0.913$ \\
& $0.00002$ & $0.060$ & $0.912$ \\
\midrule
Algorithm & $L$ & {One-sided} & {Two-sided} \\ \midrule
\multirow{4}{*}{\textsc{ Opt}} & $50$ & $0.959$ & $0.963$ \\
& $500$ & $0.777$ & $0.953$ \\
& $5000$ & $0.111$ & $0.953$ \\
& $50000$ & $0.048$ & $0.953$ \\
\bottomrule
\end{tabular}
\end{table}

Finally, in Table \ref{tab:balanced}, we analyze the effect of balanced and unbalanced ratios of normal and anomalous samples. For $50\%$ normal observations i.e. $p=q=0.02$, accuracy of {\textsc{ Imitate}} one-sided is $0.545$, {\textsc{ Imitate}} two-sided is $0.598$,  {\textsc{ Opt}} one-sided is $0.630$, and {\textsc{ Opt}} two-sided is $0.651$. For $95\%$ normal observations i.e. $20p=q=0.02$, accuracy increases to $0.911$, $0.939$, $0.959$, and $0.963$ respectively. This reaffirms the fact that semi-supervision in the context of unbalanced classes provides a regimen of explorations.

\begin{table}[!htb]
\centering
\caption{Stream size $1,000,000$ with $2\%$ feedback (Averaged over $10$ runs).}
\label{tab:balanced}
\begin{tabular}{@{}rrrrrr@{}}
\toprule \\
& $p$ & $q$ & Algorithm & {One-sided} & {Two-sided} \\ \midrule
\multirow{2}{*} {Bal.} & 0.02 & 0.02 & {\textsc{ Imitate}}  & $0.545$ & $0.598$ \\
& 0.02 & 0.02 & {\textsc{ Opt}} & $0.630$ & $0.651$ \\
\midrule
\multirow{2}{*} {Unbal.} &  0.001 & 0.02 & {\textsc{ Imitate}}  & $0.911$ & $0.939$ \\
& 0.001 & 0.02 & {\textsc{ Opt}} & $0.959$ & $0.963$ \\
\bottomrule
\end{tabular}
\end{table}

The conceptual systems serve as an exemplar that (i) optimization needs to be considered carefully (ii) algorithms that work on synopsis and suboptimal at the outset need not have poor performance at the end of semi-supervision. This system model is abstracted to model a single edge and the feedback pertains to the same edge. For an extended object such as a graph, feedback would also correspond to many edges that likely were never anomalous, as a result, we expect the performance to slowly degrade.

\FloatBarrier

\section{Semi-Supervision}

\subsection{\midas}
\midas\ \cite{bhatia2020midas} detects anomalous edges from a stream of graph edges. It combines a chi-squared statistic with count-min sketches (CMS)~\cite{cormode2005improved} streaming data structures to get an anomaly score for each edge. \midas\ defines $s_{uv}$ as the total number of edges from node $u$ to $v$ up to the current time tick $t$, and $a_{uv}$ as the number of edges from node $u$ to $v$ only in the current time tick $t$ (excluding past time ticks). It then divides the edges into two classes: edges at the current time tick $t$ ($=a_{uv}$), and edges in past time ticks ($=s_{uv} - a_{uv}$), and computes the chi-squared statistic as $\left(a_{uv} - \frac{s_{uv}}{t}\right)^2 \frac{t^2}{s_{uv}(t-1)}$. \midas\ then uses two CMS data structures to maintain approximate counts $\hat{s}_{uv}$ and $\hat{a}_{uv}$ to estimate $s_{uv}$ and $a_{uv}$ respectively and defines the anomaly score for an edge as:

\begin{equation}
score(u,v,t) = \left(\hat{a}_{uv} - \frac{\hat{s}_{uv}}{t}\right)^2 \label{eqn:eq1sess} \frac{t^2}{\hat{s}_{uv}(t-1)}
\end{equation}

\subsection{Semi-Supervision on \midas}
\paragraph{\textbf{SESS}}

\begin{figure}[!htb]
        \center{\includegraphics[width=0.8\columnwidth]
        {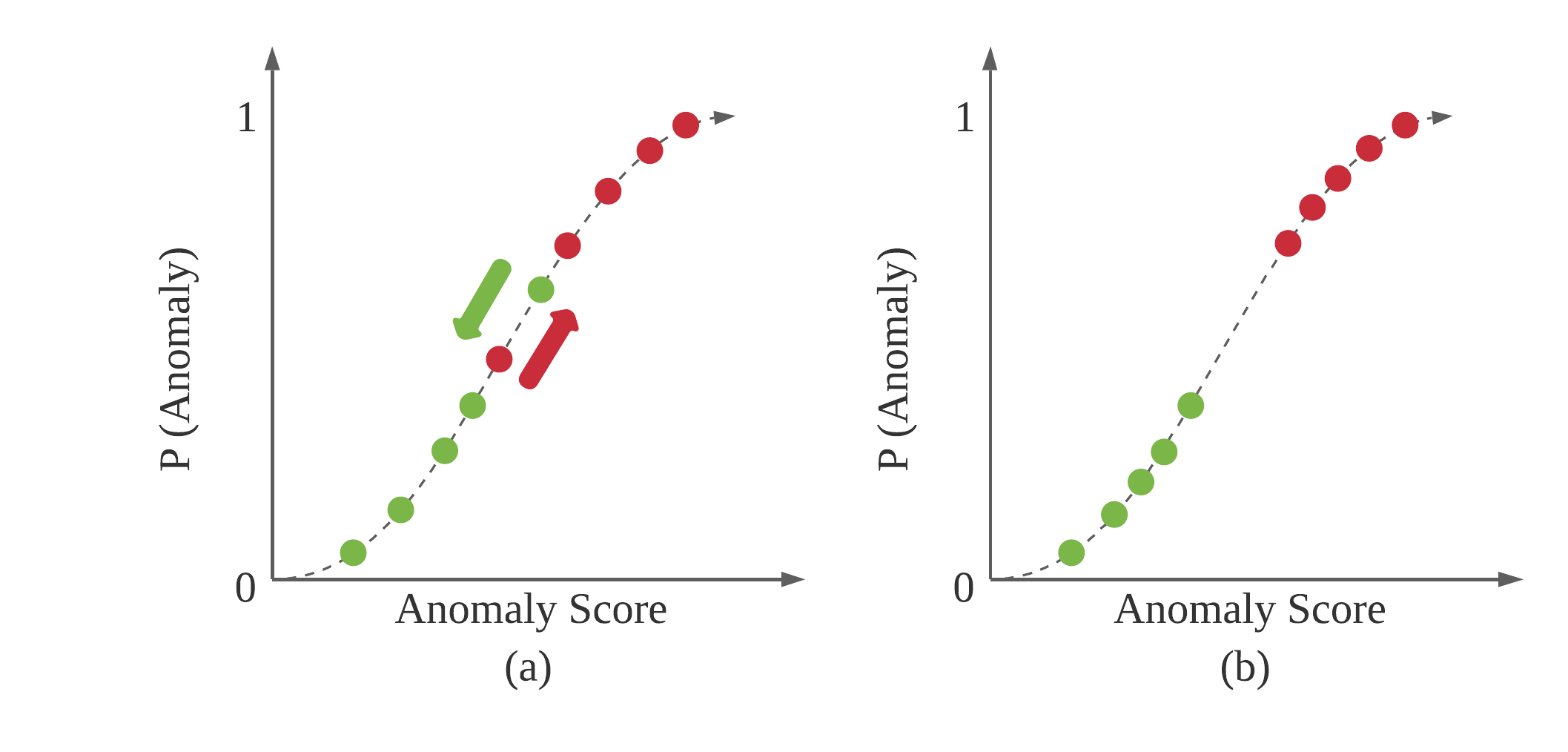}}
        \caption{\label{fig:sessentropy} `Sharpening effect' to increase the anomaly score of an anomalous edge and decrease that of a non-anomalous edge.}
\end{figure}

We incorporate semi-supervision by increasing the discriminative power of \midas. We create a {\em`sharpening effect'} in the scoring function such that we increase the anomaly score for an anomalous edge and decrease that of a non-anomalous edge as shown in Figure \ref{fig:sessentropy}. 

As seen in Equation \ref{eqn:eq1sess}, the anomaly score in \midas\ is proportional to the difference between $\hat{a}_{uv}$ and $\hat{s}_{uv}$. For an anomalous edge, we want to increase this difference, therefore we multiply $\hat{a}_{uv}$ by a factor of $\lambda$ ($>1$) and $\hat{s}_{uv}$ by a factor of $\mu$ $\in (0,1)$. For a non-anomalous edge, we want to reduce this difference, therefore we multiply $\hat{a}_{uv}$ by $\mu$ and $\hat{s}_{uv}$ by $\lambda$. Count-min sketches satisfy associative rules such that the order of updates is no longer relevant and the original guarantees hold as long as the update does not cause the sketch to become negative \cite{cormode2005improved}. We, therefore, multiply by a factor rather than subtracting the minimum across multiple hash functions, although subtractions will also hold as long as one ensures that the sketch counts always remain positive. \sess\ is summarized in Algorithm \ref{alg:sess}.

\begin{algorithm}
    \LinesNumberedHidden
	\caption{\sess:\ Semi-Supervision on \midas\ \label{alg:sess}}
	
	\KwIn{Stream of unlabeled and few labeled edges}
	\KwOut{Anomaly scores per edge}
	{\bfseries $\triangleright$ Initialize CMS data structures} \\
	Initialize CMS for total counts $s_{uv}, s_{u}, s_{v}$\\
	Initialize CMS for expected counts $a_{uv}, a_{u}, a_{v}$\\
	\While{\ new edge $e=(u,v,t)$ is received}{
	{\bfseries $\triangleright$ Calculate} \midas\ score for e\\
	\textbf{Update} CMS data structures\\
	\textbf{output} score\\ 
	{\bfseries $\triangleright$ Semi-Supervision}\\
	\If{label available}{
	\textsc{Update}($e$, $\lambda$, $\mu$, $label$, CMS for $s$ and $a$)
	}}
\end{algorithm}

\begin{algorithm}
        \LinesNumberedHidden
		\caption{\textsc{Update}}
		\label{alg:update}
		\KwIn{Edge $e$, $\lambda$, $\mu$, $label$, CMS for $s$ and $a$}
		\For{$i \gets 1$ to $r$}{
		$bucket = h_{i}(e)$   \tcp*[f]{i$^{th}$ hash function}\\
			\If{$label==0$}{ 
			    $\hat{s}[i,bucket]\ \asteq \lambda$\\
			    $\hat{a}[i,bucket]\ \asteq \mu$
            }
            \If{$label==1$}{
			    $\hat{s}[i,bucket]\ \asteq \mu$\\
			    $\hat{a}[i,bucket]\ \asteq \lambda$
            }
		}
\end{algorithm}

\FloatBarrier

\paragraph{\textbf{SESS-3D}}
\midas\ maintains different CMS data structures to keep track of current and total edge and node counts. We introduce a novel 3-Dimensional (3D) CMS data structure where $w$ buckets for each hash function of the original CMS data structure are now mapped to $\lfloor \sqrt w \rfloor * \lfloor \sqrt w \rfloor$ buckets. 

As shown in Figure \ref{fig:sess3d}, in \sess-3D, we hash source and destination nodes in separate dimensions as opposed to \sess\ where the source-destination pair are hashed together. The advantage of \sess-3D is that node feedback can directly be incorporated in addition to the usual edge feedback because of using separate buckets for hashing source and destination nodes. Moreover, this fits very well in the cache and results in a lower running time as discussed in Section \ref{sec:speedsess}.

\begin{figure}[!htb]
        \center{\includegraphics[width=0.8\columnwidth]
        {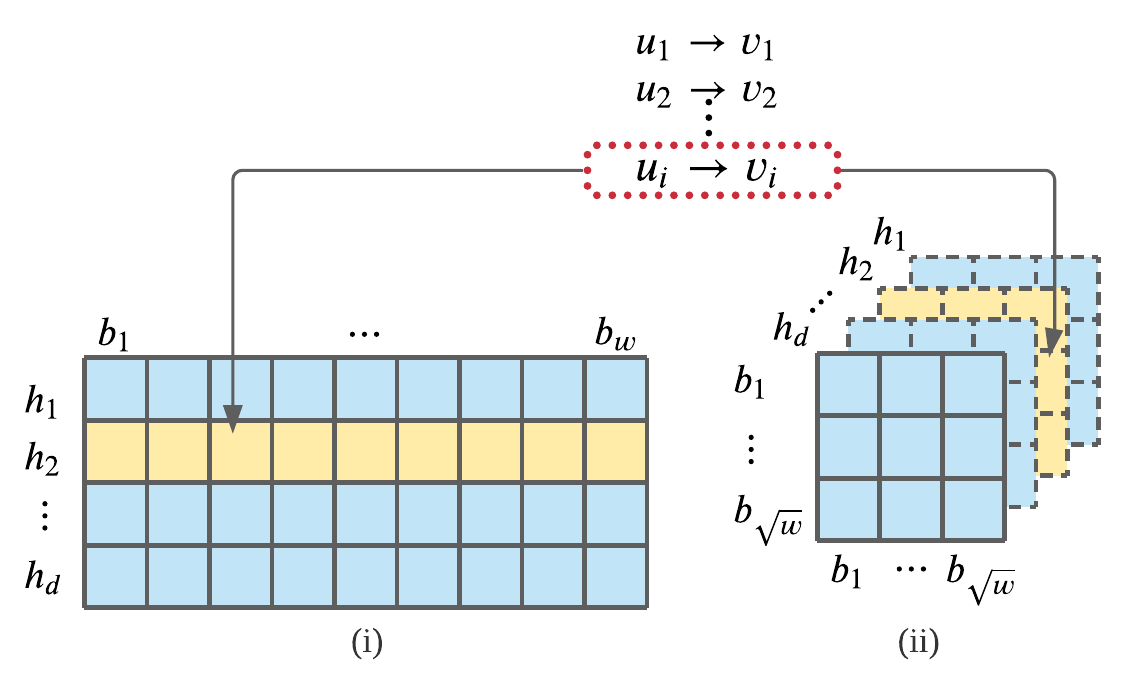}}
        \caption{\label{fig:sess3d} \sess-3D: $w$ buckets for each hash function of original CMS data structure are now mapped to $\lfloor \sqrt w \rfloor * \lfloor \sqrt w \rfloor$ buckets.}
\end{figure}

\paragraph{Time and Memory Complexity}
In terms of memory, both \sess\ and \sess-3D only need to make use of the original CMS data structures of \midas, which are proportional to $O(dw)$, where $d$ and $w$ are the number of hash functions and the number of buckets respectively; bounded by the data size. Thus, space complexity is $O(1)$. \midas\ either updates or queries the CMS, which takes $O(d)$ time per update step. For incorporating semi-supervision, the relevant steps in Algorithms \ref{alg:sess} and \ref{alg:update} run in constant time. Thus, the time complexity per update step is $O(1)$.

\FloatBarrier

\subsection{Experiments}

We now compare the performance of \sess\ and \sess-3D with \tlp\ (TLP) and vanilla \midas. We aim to answer the following questions:

\begin{enumerate}[label=\textbf{Q\arabic*.}]
\item {\bfseries Accuracy:} How accurately does \sess\ detect anomalies as compared to baselines, as evaluated using the ground truth labels?
\item {\bfseries Speed:} How does the time needed to process each input compare to the baseline approaches?
\end{enumerate}

\paragraph{\textbf{Datasets:}}
To evaluate a semi-supervised setting, we need labeled datasets to be able to sample and pass true labels as feedback to the algorithm. \emph{DARPA} \cite{lippmann1999results} is the only dataset containing ground truth used both by \midas\ or \spotlight. \emph{DARPA} has $4.5M$ communications over $1463$ hours. \emph{DARPA Unbalanced} is \emph{DARPA} but considering all Neptune attack type edges as non-anomalous. \emph{DARPA} has $60.1\%$ anomalies and \emph{DARPA Unbalanced} has $13\%$ anomalies of total edges.

\cite{ring2019survey} surveys more than $30$ intrusion detection datasets and recommends to use the newer \emph{CICIDS} \cite{sharafaldin2018toward,CICIDSDataset} datasets. \cite{prasad2019machine} further extracts and combines multiple CICIDS datasets to form \emph{Balanced DDoS} and \emph{Unbalanced DDoS} datasets containing $12.8M$ and $7.6M$ edges respectively. \emph{Balanced DDoS} has $50\%$ anomalies and \emph{Unbalanced DDoS} has $20\%$ anomalies of total edges.

\paragraph{\textbf{Baseline:}}
Note that even in the smaller \emph{DARPA} dataset there are $4.5$ $million$ edges. Label propagation algorithm \cite{zhulearning} in the Scikit-learn \cite{scikit-learn} library requires the entire graph to be in memory and therefore it runs out of memory on our datasets. We use state-of-the-art streaming label propagation \cite{wagner2018semi} as our baseline and define the similarity between adjacent edges to be higher ($score=10$) as compared to non-adjacent edges ($score=1$).

\paragraph{\textbf{Evaluation Metrics:}}
All methods output an anomaly score per edge (higher is more anomalous). We report the Area under the ROC curve (AUC) since it can be calculated using predicted scores. If MIDAS provided a fixed threshold, accuracy could have been reported since it is calculated on predicted classes. Recall that AUC lies in $[0,1]$ and a higher value is better. We measure the running time averaged over $21$ runs with $0.01\%$ random feedback (unless specified otherwise) and report the median values.

\paragraph{\textbf{Experimental Setup:}}
All experiments are carried out on a $2.8 GHz$ Intel Core $i7$ processor, $16 GB$ RAM, running OS $X$ $10.15.6$. We used an open-sourced implementation of \midas-R (better performing \midas\ variant), provided by the authors, following parameter settings as suggested in the original paper ($2$ hash functions, $2719$ buckets). We implement \sess\ in C\texttt{++}. We follow the same parameter settings as \midas\ ($2$ hash functions, $2719$ buckets).

$\lambda$ and $\mu$ are chosen as $2$ and $0.3$ respectively. Exact step sizes $\lambda$ and $\mu$ are dataset dependent but since \sess\ is based on multiplicative weights, we should ensure that computations using the step size remain bounded to maintain theoretical guarantees and give meaningful results. We did not find any significant difference in the accuracy on increasing $\lambda$ up to $2$ and reducing $\mu$ to $0.3$. Adding and subtracting the step size should also give similar results because multiplicative weights have corresponding additive versions, however, we omit this discussion in the interest of space.

\paragraph{\bfseries Accuracy:}
Table \ref{tab:aucsess} shows the AUC of TLP, \midas, \sess\ and \sess-3D on \emph{DARPA}, \emph{DARPA Unbalanced}, \emph{DDoS Balanced} and \emph{DDoS Unbalanced} datasets with $0.01\%$ feedback. By incorporating semi-supervision, \sess\ and \sess-3D achieve higher AUC as compared to \midas and TLP. Note that \midas\ was unable to perform well on unbalanced datasets: \emph{DARPA Unbalanced} and \emph{DDoS Unbalanced} where small feedback ($0.01\%$) in \sess\ was sufficient to improve the performance significantly.

\begin{table}[!htb]
\centering
\caption{AUC and Time with $0.01\%$ feedback.}
\label{tab:aucsess}
\begin{tabular}{@{}lrrrr@{}}
\toprule
Dataset
 & { TLP}
 & { \midas}
 & \textbf{\sess}
 & \textbf{\sess-3D}  \\ \midrule
 {\bfseries DARPA} & $0.764$ & $0.952$ & $0.969$ & $\textbf{0.977}$ \\
 
 & $\pm 0.022$ & $\pm 0$ & $\pm 0.001$ & $\pm 0.002$ \\
 & $\sim2000s$ & $1.6s$ & $1.4s$ & $0.6s$ 
  \vspace{0.2cm}
 \\
 
 {\bfseries DARPA} & $0.569$ & $0.613 $ & $0.865$ & $\textbf{0.885}$ \\
 {\bfseries Unbalanced} & $\pm 0.045$ & $\pm 0$ & $\pm 0.005$ & $\pm 0.006$ \\
 & $\sim2000s$ & $1.6s$ & $1.4s$ & $0.6s$ 
 \vspace{0.2cm}
 \\

 {\bfseries DDoS} & $0.651$ & $0.941 $ & $0.992$ &  $\textbf{0.998}$  \\
 {\bfseries Balanced} & $\pm 0.025$ & $\pm 0$ & $\pm 0.000$ & $\pm 0.000$ \\
 & $\sim12000s$ & $1.8s$ & $1.8s$ & $1.5s$
 \vspace{0.2cm}
 \\
 {\bfseries DDoS} & $0.931$ & $0.663 $ & $0.923$ &  $\textbf{0.990}$ \\
 {\bfseries Unbalanced} & $\pm 0.005$ & $\pm 0$ & $\pm 0.004$ & $\pm 0.000$ \\
 & $\sim7000s$  & $1.3s$ & $1.3s$ & $0.9s$
 \vspace{0.2cm}
 \\
\bottomrule
\end{tabular}
\end{table}

Sketches spread out the data well and perform better denoising as compared to TLP. It may be possible to have a more informative kernel when more information is available. However, without any other assumptions about the data, it is not clear what kernel to set and how to run Label Propagation on these datasets consisting of edges. Moreover, even though there is a potential to choose the right set of parameters as can be observed by $0.93$ for TLP in Table \ref{tab:aucsess}, running time will still be of the same order.

\paragraph{\bfseries Speed}
\label{sec:speedsess}
\sess\ and \sess-3D are at least three orders of magnitude faster ($<2s$ vs $\sim2000s$) compared to TLP on all datasets. It is worth noting in Table~\ref{tab:aucsess} that \sess\ and \sess-3D did not slow down compared to the original \midas\ algorithm. Also, \sess-3D by being cache-aware has a lower running time compared to \sess. Note that \sess\ and \sess-3D are scalable with increasing feedback since the time complexity is constant, whereas TLP requires time quadratic to the proportion of feedback.

\paragraph{\bfseries One-Sided Feedback:}
When we only provide anomalous one-sided feedback, AUC drops from $0.865$ to $0.704$ for \sess\ and from $0.885$ to $0.703$ for \sess-3D on \emph{DARPA Unbalanced} dataset. This shows that receiving randomized signals from both categories (normal and anomalous) is much more beneficial as compared to learning from one-sided feedback (anomalous labels).

\FloatBarrier

\paragraph{\bfseries Evaluating AUC in a streaming manner:}
\label{app:evaluatingauc}
Figure \ref{fig:streamingauc} plots the AUC for \midas, \sess$^{1}$, \sess-3D$^{1}$, \sess\ and \sess-3D on \emph{DARPA Unbalanced} dataset when evaluated over the stream with $0.01\%$ feedback. \sess$^{1}$ and \sess-3D$^{1}$ are \sess\ and \sess-3D with one-sided (only anomalous) feedback whereas \sess\ and \sess-3D receive two-sided feedback (both normal and anomalous feedback). Evaluation is performed after every $500K$ records. At the end of the stream, AUC for \midas, \sess\ and \sess-3D is $0.613$, $0.870$ and $0.885$ as also shown in Table~\ref{tab:aucsess}. Note that as the stream length increases, there is a continuous drop in AUC of \midas, whereas AUC for \sess\ and \sess-3D does not drop significantly. One-sided feedback: \sess$^{1}$ and \sess-3D$^{1}$ has better performance as compared to \midas\ but receiving signals from both categories (\sess\ and \sess-3D) achieves the highest AUC. Figure \ref{fig:feedbackauc} shows the influence of feedback on the AUC. AUC for \midas\ remains the same, whereas that of \sess\ and \sess-3D improves with increasing feedback. Note that these observations correspond well with the findings in Section \ref{sec:conceptual}.

\begin{figure}[!htb]
        \center{\includegraphics[width=0.7\columnwidth]
        {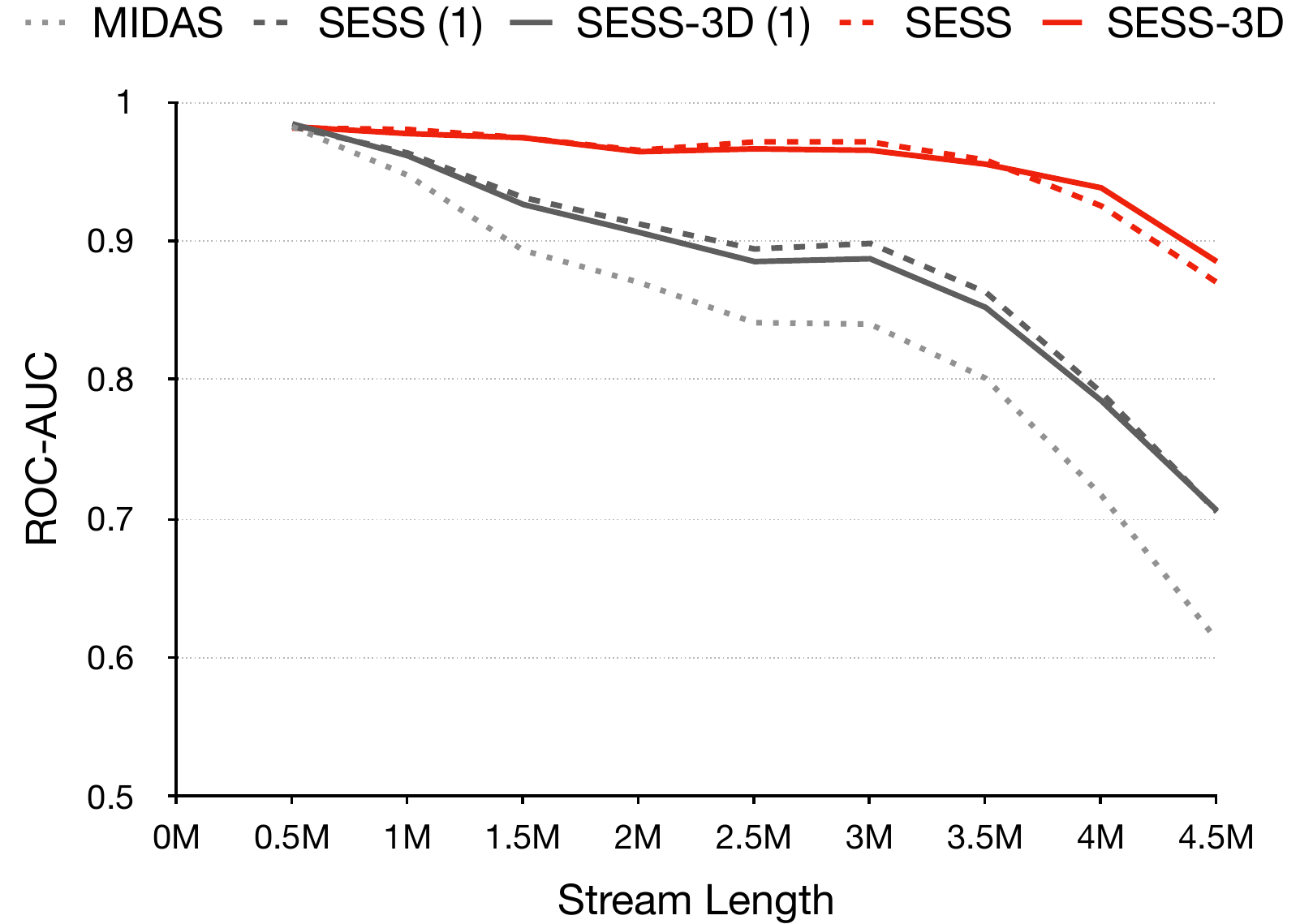}}
        \caption{\label{fig:streamingauc} AUC of \midas\ drops; \sess\ and \sess-3D are steady. Mean and standard deviations for $21$ runs are shown.}
\end{figure}

\begin{figure}[!htb]
        \center{\includegraphics[width=0.7\columnwidth]
        {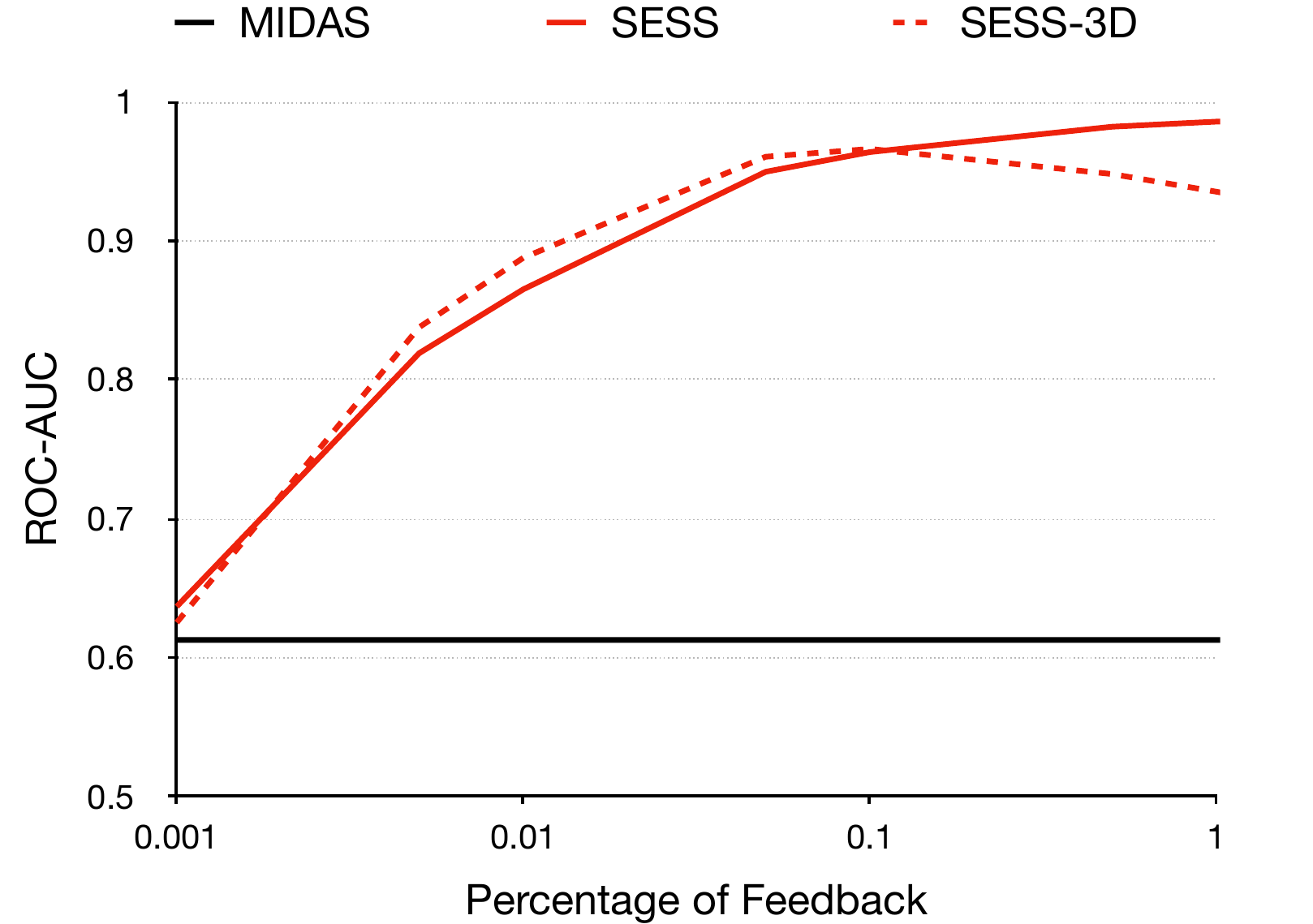}}
        \caption{\label{fig:feedbackauc} AUC of \sess\ and \sess-3D increases with more feedback. Mean and standard deviation for $21$ runs are shown.}
\end{figure}

\FloatBarrier

\section{Weakly Correlated Semi-Supervision}
In this section, we consider semi-supervision when feedback is weakly
correlated, taking \spotlight\ \cite{eswaran2018spotlight} as an
example. \spotlight\ detects anomalous graphs in a streaming manner,
however, we provide feedback on the anomalousness of individual edges.

At the outset, we note that two-sided feedback is infeasible in a
weakly correlated setting because we have no mechanism to argue that
the input is not anomalous -- the fact that a weakly correlated signal
is answering no is tangential evidence. Likewise, query-based
one-sided feedback is not defined for weakly correlated signals. One-sided feedback can be defined in this context and it
corresponds to providing the weakly correlated signal for a randomly
chosen set defined on that correlated signal. However, a sketch-based semi-supervised algorithm has an interesting capability in this
context. Based on the score it sees for the most recent input, and
any weakly correlated feedback, the algorithm can choose to update
or choose not to update itself. Thus, the algorithm does possess a
mechanism to amplify the feedback.

\subsection{\spotlight}
\spotlight\ uses randomized sketching to project graphs to points. It
initially chooses $K$ query subgraphs
$\{(\mathcal S'_k, \mathcal D'_k)\}_{k=1}^K$ by sampling each source
into each $\mathcal S'_k$ and each destination into each
$\mathcal D'_k$ with probabilities $p$ and $q$ respectively. The
sketch vector $\mathbf v(\mathcal G) \in \mathbb R^{K}$ is calculated
as
$\mathbf v_k(\mathcal G) = \sum_{s \in{\mathcal S'_k},d \in {\mathcal
    D'_k}} A_{sd}$. \spotlight\ then runs random cut forest (RCF)
\cite{guha2016robust} on the hashed space to calculate the anomaly
score for the sketch vector. Finally, the RCF is updated
with the embedding of the current graph to better reflect the trend and detect anomalies in future.

\subsection{Semi-Supervision on \spotlight}
For semi-supervision, we use the current graph embedding to update the
forest only when no edge is labeled as
anomalous. This is to prevent the graphs containing anomalous
edges from updating the random cut forest.

During the final step of
\spotlight, once the sketch vector is computed and an anomaly score is
calculated for that vector, the embedding is used to update the random
cut forest for better future predictions. Vanilla \spotlight\ however
fails to benefit from the predictions themselves and updates the
forest for every embedding instead of doing so only for normal
predictions. We, therefore, update the forest only when the anomaly
score is less than a particular threshold ($\beta$) in addition to the
semi-supervision. We choose $\beta=1$ as the threshold because RCF
scores are calibrated around a score of $1$ but for the purposes of measurement, one may also fix the best threshold by exploring the search space such that the AUC score
is maximum. Note that such a choice of the threshold would be fully
supervised and not semi-supervised. However, the variation provides an
interesting baseline for comparing the semi-supervised approach as well. Weakly Correlated Semi-Supervision on \spotlight\ is summarized in Algorithm \ref{alg:spotlight}.

\begin{algorithm}
    \LinesNumberedHidden
	\caption{Semi-Supervision on \spotlight\ \label{alg:spotlight}}
	\KwIn{Stream of unlabeled and few labeled edges aggregated into graphs}
	\KwOut{Anomaly scores per graph}
	{\bfseries $\triangleright$ Initialize Random Cut Forest $\mathcal F$}\\
	\While{new graph $\mathcal G$ is received}{
	$anom\_flag$ = 0\\
	\If{$\exists$ edge $e \in \mathcal G$ st. $e$ is labeled anomalous}
	{
	$anom\_flag$ = 1
	}
	{\bfseries $\triangleright$ Calculate} \spotlight\ score for graph $\mathcal G$ \\
	$\mathbf{v}$ = \textsc{Sketch}($\mathcal G$)\\
	score = \textsc{AnomalyScore}($\mathbf{v}$)\\
	{\bfseries $\triangleright$ Semi-Supervision}\\
	\If{$score < \beta$ and $anom\_flag \ne 1$}{
	\textsc{UpdateRCF}($\mathcal F$, $\mathbf{v}$)
	}
	}
\end{algorithm}

\FloatBarrier

\subsection{Experiments}

The authors of \spotlight\ use \emph{DARPA} dataset for evaluation, hence we use \emph{DARPA} and its variant \emph{DARPA Unbalanced} for comparison.
As described in the original paper, we used open-sourced implementations of RCF \cite{awsrando88:online} and Carter Wegman hashing \cite{carter1979universal}, and obtain a stream of graphs by aggregating edges in \emph{DARPA} and \emph{DARPA Unbalanced} occurring every $60$ minutes. Additionally, we show experiments for aggregations of $t=15$ and $30$ minutes. A graph is labeled as anomalous if it contains at least $\tau$ attack edges. AUC scores are averaged on five seeds and feedback is given on $1\%$ of the edges.

In Table \ref{tab:weaksemi}, we show AUC of \emph{Basic}, \emph{Semi-Supervised}, \emph{Semi-Supervised+$1$} and \emph{Fixed} \spotlight, for different $t$ and $\tau$ values, on \emph{DARPA} and \emph{DARPA Unbalanced} datasets. Basic refers to the original \spotlight\ without any semi-supervision. Semi-Supervised incorporates weakly correlated semi-supervision without an additional thresholding step. Semi-Supervised+$1$ refers to both semi-supervision as well as the thresholding step for $\beta = 1$. Fixed is similar to Semi-supervised+$1$ but the threshold is now fixed rather than being $1$ by searching across different thresholds to see which one works better. For \emph{DARPA} and \emph{DARPA Unbalanced} datasets, we find that $\beta = 0.6$ performs well.

\begin{table}[!htb]
\centering
\caption{Weakly Correlated Semi-Supervision on \spotlight.}
\label{tab:weaksemi}
\begin{tabular}{@{}rrrrrrr@{}}
\toprule
Dataset &
 {$t$} & {$\tau$} &
 { Basic}
 & {Semi}
 & {Semi + $1$} 
 & {Fixed} \\ \midrule
 
\multirow{4}{*}{DARPA} & $15$ & $25$ & $0.682$ & $0.831$ & $0.857$ & $0.859$ \\
& $30$ & $50$ & $0.682$ & $0.831$ & $0.857$ & $0.859$ \\
& $60$ & $50$ & $0.641$ & $0.790$ & $0.794$ & $0.795$ \\
& $60$ & $100$ & $0.719$ & $0.886$ & $0.869$ & $0.871$ \\ \midrule

& $15$ & $25$ & $0.658$ & $0.763$ & $0.838$ & $0.844$	\\
{DARPA} & $30$ & $50$ & $0.686$ & $0.813$ & $0.846$ & $0.859$	\\
{Unbalanced}& $60$ & $50$ & $0.622$ & $0.740$ & $0.778$ & $0.781$	\\
& $60$ & $100$ & $0.704$ & $0.838$ & $0.857$ & $0.862$	\\

\bottomrule
\end{tabular}
\end{table}

Note that feedback is on the edges but the predictions are on the anomalousness of graphs. We observe that both Semi-Supervised and Semi-Supervised+$1$ perform consistently better than Basic. Moreover, we observe that on \emph{DARPA Unbalanced},
thresholding shows substantial improvement; Semi-Supervised+$1$
performs similar to Fixed which is fully supervised to find the best threshold.

\FloatBarrier

\section{Conclusion}
This chapter explores semi-supervision via sketching for two anomaly detection algorithms over graphs where the input is provided as a stream of edges. A small number of labeled samples can provide significant benefits, even if the labels are weakly correlated with the objective and the feedback is forced to be one-sided. In contrast, it is not clear how to provide one-sided feedback to label propagation based methods. It is non-obvious how the feedback on edges can easily be propagated to other edges in a graph.

Note that none of the experiments discuss query-based feedback. In addition to posing difficulties in analysis, the notion of query-based one-sided feedback also poses significant challenges in defining a reasonable measurement strategy, especially when the input is a stream of edges and the anomalies are defined in the context of the overall graph. The resilience of \sess\ to wrong label propagation is also not discussed due to the lack of a reasonable baseline or evaluation measurement. However, in label propagation based algorithms, once the label has been propagated, there is no way to undo it. On the other hand, sketches have reversible properties because they are associative in nature.
\end{appendices}

\bookmarksetup{startatroot}
\printbibliography[heading=bibintoc]

\end{document}